\documentclass[letterpaper, 10 pt]{article}
\usepackage{Arxiv/style}
\usepackage{cancel}
\usepackage{blindtext}
\newcommand{\E}{\mathbb{E}}

\usepackage{amsmath,amssymb}
\usepackage{hyperref}
\usepackage{amsthm}
\usepackage{arydshln}
\usepackage{enumitem}
\usepackage{xcolor}
\usepackage{threeparttable}
\usepackage{tikz}
\usetikzlibrary{shapes, intersections}
\usetikzlibrary{shapes.arrows}
\usetikzlibrary{automata,positioning}
\usetikzlibrary[calc, math, decorations.pathreplacing]
\usepackage{pgfplots}
\pgfplotsset{compat=1.15}
\usepackage{adjustbox}
\usepackage{bm}
\usepackage{gincltex}
\usepackage{caption}
\usepackage{subcaption}
\usepgfplotslibrary{fillbetween}
\usepgfplotslibrary{groupplots}
\usetikzlibrary{plotmarks}
\usetikzlibrary{patterns}
\usetikzlibrary{shapes.geometric, arrows}
\tikzstyle{boxGreen} = [rectangle, rounded corners, minimum width=3cm, minimum height=1cm,text centered, draw=black, fill=green!10]
\tikzstyle{boxYellow} = [rectangle, rounded corners, minimum width=3cm, minimum height=1cm,text centered, draw=black, fill=yellow!20]
\tikzstyle{arrow} = [thick,->,>=stealth]
\tikzstyle{curve_arrow_left} = [thick,->,>=stealth, bend left]
\tikzstyle{curve_arrow_right} = [thick,->,>=stealth, bend right]
\pgfplotsset{
tick label style={font=\footnotesize},
label style={font=\footnotesize},
legend style={font=\footnotesize},
}

\newcommand{\xh}{\hat{x}}
\newcommand{\yh}{\hat{y}}
\newcommand{\fh}{\hat{f}}
\newcommand{\hh}{\hat{h}}
\newcommand{\tX}{\tilde{X}}
\newcommand{\tY}{\tilde{Y}}
\newcommand{\tZ}{\tilde{Z}}
\newcommand{\tx}{\tilde{x}}
\newcommand{\ty}{\tilde{y}}
\newcommand{\tC}{\tilde{C}}
\newcommand{\tO}{\tilde{O}}

\newcommand{\bt}{\tilde{b}}

\newcommand{\bc}{\bar{c}}
\newcommand{\cc}{\check{c}}
\newcommand{\cR}{\check{R}}
\newcommand{\brc}{\Breve{c}}
\newcommand{\Rk}[1]{R_k^{(#1)}}

\allowdisplaybreaks

\usepackage{multirow}
\RequirePackage{tikz}
\usepackage{pifont}
\usepackage{xcolor,colortbl}

\def\checkmark{\tikz\fill[scale=0.4](0,.35) -- (.25,0) -- (1,.7) -- (.25,.15) -- cycle;} 

\title{\LARGE \bf Tight Finite Time Bounds of Two-Time-Scale Linear Stochastic
Approximation with Markovian Noise}

\author{
{\normalsize Shaan Ul Haque}\footnote{H. Milton Stewart School of Industrial \& Systems Engineering, Georgia Institute of Technology, Atlanta, GA, 30332, USA, {\tt\small \{\href{mailto:shaque49@gatech.edu}{shaque49}, \href{mailto:siva.theja@gatech.edu}{siva.theja}\}@gatech.edu}} $\,$\and
{\normalsize Sajad Khodadadian}\footnote{Grado Department of Industrial and Systems Engineering, Virginia Polytechnic Institute and State University, Blacksburg, VA, 24061, USA, {\tt\small  \href{mailto:sajadk@vt.edu}{sajadk@vt.edu}} } \and
{\normalsize Siva Theja Maguluri}$^*$
}
\date{}

\begin{document}

\maketitle

\setlength{\abovedisplayskip}{5pt}
\setlength{\belowdisplayskip}{5pt}

\begin{abstract}
Stochastic approximation (SA) is an iterative algorithm for finding the fixed point of an operator using noisy samples and widely used in optimization and Reinforcement Learning (RL). The noise in RL exhibits a Markovian structure, and in some cases, such as gradient temporal difference (GTD) methods, SA is employed in a two-time-scale framework. This combination introduces significant theoretical challenges for analysis.

We derive an upper bound on the error for the iterations of linear two-time-scale SA with Markovian noise. We demonstrate that the mean squared error decreases as $trace (\Sigma^y)/k + o(1/k)$  where $k$ is the number of iterates, and $\Sigma^y$ is an appropriately defined covariance matrix. A key feature of our bounds is that the leading term, $\Sigma^y$, exactly matches with the covariance in the Central Limit Theorem (CLT) for the two-time-scale SA, and we call them tight finite-time bounds.
We illustrate their use in RL by establishing sample complexity for off-policy algorithms, TDC, GTD, and GTD2.

A special case of linear two-time-scale SA that is extensively studied is linear SA with Polyak-Ruppert averaging. We present tight finite time bounds corresponding to the covariance matrix of the CLT. Such bounds can be used to study TD-learning with Polyak-Ruppert averaging.
\end{abstract}

\section{Introduction}

Stochastic Approximation (SA) \cite{robbins1951stochastic} is an iterative algorithm for finding the fixed point of an operator using noisy samples. SA has a wide range of applications, including stochastic optimization \cite{jung2017fixed}, statistics \cite{hastie2009elements}, and Reinforcement Learning (RL) \cite{sutton2018reinforcement}. This versatility has motivated extensive research into its convergence properties, both asymptotically \cite{nevel1976stochastic, tsitsiklis1994asynchronous} and in finite time \cite{beck2012error, bhandari2018finite}.

In certain applications, SA operates in a two-time-scale manner \cite{borkar1997stochastic, doan2022nonlinear}. Specifically, a linear two-time-scale SA has the following update rule:
\begin{subequations}
\label{eq:two_time_scale_main}
\begin{align}
    y_{k+1}&= y_k+ \beta_k(b_1(O_{k})-A_{11}(O_{k})y_k-A_{12}(O_{k})x_k) \label{eq:two_time_scale_main_y}\\
    x_{k+1}&= x_k+\alpha_k(b_2 (O_{k})-A_{21}(O_{k})y_k-A_{22}(O_{k})x_k),\label{eq:two_time_scale_main_x}
\end{align}
\end{subequations}
where \(x_k\) and \(y_k\) are the two variables updated on separate time-scales determined by step sizes \(\alpha_k\) and \(\beta_k\). Furthermore, $A_{ij}(O_k), b_i(O_k), i,j=1,2$ are random matrices and vectors, and $O_k$ represents the randomness at the time step $k$. This two-time-scale structure appears in various algorithms such as TDC, GTD, and GTD2. While the asymptotic convergence of \eqref{eq:two_time_scale_main_y} and \eqref{eq:two_time_scale_main_x} has been studied extensively \cite{borkar2009stochastic, hu2024central}, including the characterization of asymptotic covariance \cite{konda2004convergence}, finite-time analysis remains less developed.  

A notable special case of the linear two-time-scale SA is linear SA with Polyak-Ruppert averaging \cite{polyak1990new}. In this setting, the variable \(x_k\) is updated as \(x_{k+1} = x_k + \alpha_k \big(A(O_k)x_k + b(O_k)\big)\), and \(y_k\) is defined as the running average of \(x_k\): \(y_{k+1} = \sum_{i=0}^k x_i/(k+1)\). It has been shown that SA with Polyak-Ruppert averaging achieves optimal asymptotic convergence rates \cite{polyak1992acceleration, li2021polyak, li2023statistical}. Moreover, its robustness to the choice of step size has been highlighted in \cite{nemirovski2009robust}, where \(\alpha_k\) can be chosen independently of problem-dependent constants while still ensuring optimal asymptotic performance.   

In this paper, we establish a tight finite time analysis of the linear two-time-scale SA with Markovian noise \eqref{eq:two_time_scale_main}. Our main contributions are summarized as follows:  
\begin{enumerate}  
    \item \textbf{Tight Finite-Time Bound:} We provide the first tight finite-time characterization of the the covariance matrix in two-time-scale linear SA with Markovian and multiplicative noise under minimal set of assumptions. Our results consist of a leading term which is asymptotically optimal and matches covariance in central limit theorem (CLT) established in \cite{hu2024central}, and a higher-order term. We bound the convergence rate of the higher-order terms, offering insights into optimal step-size selection. We also validate the minimality of our assumptions through experiments.
    \item \textbf{Single-Time-Scale vs Two-Time-Scale:} We study an alternative implementation of \eqref{eq:two_time_scale_main} where the updates are performed in a single-time-scale manner. We present conditions on the matrices under which single-time-scale implementation works, and we present the conditions which necessitates the use of two-time-scale. 
    \item \textbf{Polyak-Ruppert Averaging:} 
    Since our result is established under minimal assumptions, it enables us to study as a special case, Polyak-Ruppert averaging of linear SA under Markov noise. This setting is of independent interest and has been extensively studied. Recent work \cite{borkar2021ode} established a CLT and characterized the asymptotic covariance matrix. We present tight finite time bounds that match the covariance matrix in \cite{borkar2021ode}.
    \item \textbf{Applications to RL Algorithms:} Using our results, we analyze the convergence of TDC, GTD, and GTD2 algorithms, providing new insights into their performance.  
\end{enumerate}  

The remainder of this paper is structured as follows. Section \ref{sec:related_work} reviews related literature. Section \ref{sec:prob_form} formulates the problem of two-time-scale linear SA with Markovian noise and introduces our assumptions. In Section \ref{sec:main_result}, we present our main results, including discussions on step-size selection and comparisons between single and two-time-scale algorithms. This section also explores the convergence of linear SA with Polyak-Ruppert averaging and derives mean-square bounds for various RL algorithms. Section \ref{sec:pf_sketch} outlines the proof of our main results. Finally, Section \ref{sec:conc} concludes the paper and suggests future directions.

\begin{table*}[t]
\begin{center}

\caption{Summary of the results on convergence analysis of two-time-scale SA
}
\label{table:results2}
\renewcommand{\arraystretch}{1.4} 
\centering
\vskip 0.15in
\begin{small}
\begin{tabular}{ |c|c|c|c|c|c|c|c|}
\hline
\multirow{2}{4 em}{\centering Reference} & \multirow{2}{5 em}{\centering  Markovian Noise} & \multirow{2}{6 em}{\centering Multiplicative Noise} & \multirow{2}{7 em}{\centering Applicable beyond P-avg$^{[a]}$} & \multirow{2}{5 em}{\centering Tight Constant$^{[b] }$} &  \multirow{2}{8 em}{\centering Tight Convergence rate} & 
\multirow{2}{6 em}{\centering Convergence rate} \\
 &  &  &  &  &  & \\
 \hline
\cite{moulines2011non} & \xmark & \checkmark & \xmark & \checkmark & \checkmark & $\mathcal{O}(1/k)$\\
\hline
\cite{bach2014adaptivity} & \xmark & \checkmark & \xmark & \xmark & \checkmark & $\mathcal{O}(1/k)$\\
\hline
\cite{lakshminarayanan2017linear} & \xmark & \checkmark & \xmark & \checkmark & \checkmark & $\mathcal{O}(1/k)$\\
\hline
\cite{dalal2018finite} & \xmark & \checkmark & \checkmark & \xmark& \xmark & $\mathcal{O}(1/k^{2/3})$  \\
\hline
\cite{gupta2019finite}$^{[c]}$ & \checkmark & \checkmark & \checkmark & \xmark & \xmark & $\mathcal{O}(\log (k)/k^{2/3})$\\
\hline
\cite{doan2019linear} & \xmark & \xmark & \checkmark & \xmark & \xmark & $\mathcal{O}(1/k^{2/3})$ \\
\hline
\cite{dalal2020tale} &  \xmark & \checkmark & \checkmark & \xmark & \checkmark & $\mathcal{O}(\log(k)/k)$ \\
\hline
\cite{liu2020finite} & \xmark & \checkmark & \checkmark & \xmark & \checkmark & $\mathcal{O}(1/k)$\\
\hline
\cite{mou2020linear} & \xmark & \checkmark & \xmark & \xmark & \checkmark & $\mathcal{O}(1/k)$\\
\hline
\cite{kaledin2020finite} & \checkmark & \checkmark & \checkmark & \xmark & \checkmark & $\mathcal{O}(1/k)$\\
\hline
\cite{doan2021markov} & \checkmark & \checkmark & \checkmark & \xmark & \xmark & $\mathcal{O}(\log k/k^{2/3})$\\
\hline
\cite{mou2021optimal}  & \checkmark & \checkmark & \xmark & \checkmark & \checkmark & $\mathcal{O}(1/k)$\\
\hline
\cite{durmus2022finite}  & \checkmark & \checkmark & \xmark & \xmark & \checkmark & $\mathcal{O}(1/k)$\\
\hline
\hline
\rowcolor{green!20}
Our result   & \checkmark & \checkmark & \checkmark & \checkmark & \checkmark & $\mathcal{O}(1/k)$\\
\hline
\end{tabular}
\begin{tablenotes}
\small
     {\footnotesize [a]In this column we specify if the work only considers Polyak-Ruppert averaging as the special case of two-time-scale SA, or the result can be applied for a general two-time-scale algorithm. \newline
     [b]The convergence result in each work can be written as $\frac{D}{k^\nu}+o\left(\frac{1}{k^\nu}\right)$, where $\nu\in[0,1]$. In this column, we specify if the term $D$ in the convergence bound of the leading term is asymptotically tight.  \newline
     [c]In this paper, the author established a rate by assuming a constant step size. However, their proof can be easily modified to accommodate the time-varying step size.}
\end{tablenotes}

\end{small}
\end{center}
\end{table*}
\section{Related Work}\label{sec:related_work}

Since the introduction of SA by Robbins and Monro \cite{robbins1951stochastic}, extensive research has focused on its convergence properties \cite{benveniste2012, borkar2009stochastic, harold1997stochastic}. Many machine learning problems involve solving fixed-point equations, driving significant interest in the finite-time convergence analysis of single-time-scale SA algorithms \cite{chen2020finite, srikant2019finite, chen2023concentration, wainwright2019stochastic}. However, numerous applications, particularly in optimization and RL, necessitate two-time-scale SA approaches, prompting studies in both asymptotic and finite-time settings.

\textbf{Asymptotic Analysis:} A notable special case within two-time-scale SA involves averaging iterates from single-time-scale SA, known as Polyak-Ruppert averaging. This method is recognized for faster convergence and optimal asymptotic covariance, initially formalized by \cite{ruppert1988efficient,polyak1992acceleration} under independent and identically distributed (i.i.d.) noise conditions. Recent studies extended these results to Markovian noise scenarios \cite{borkar2021ode}. More broadly, convergence properties of general two-time-scale SA have been extensively analyzed \cite{borkar1997stochastic, borkar2009stochastic}. Specifically, asymptotic convergence rates and normality for linear SA under i.i.d. noise were established by \cite{konda2004convergence}, later generalized to non-linear cases by \cite{mokkadem2006convergence,han2024decoupled} and \cite{fort2015central} under both i.i.d. and Markovian noise, respectively.

\textbf{Finite-Time Analysis:} Increasing interest in two-time-scale SA has led to rigorous examination of its finite-time behavior. Studies such as those by  \cite{dalal2018finite}, \cite{doan2019linear}, and \cite{srikant2019finite} address linear SA under martingale, i.i.d., and Markovian noise, respectively, although these approaches yield suboptimal rates. Explicit analysis of Polyak-Ruppert averaging in finite-time settings appears in \cite{mou2021optimal, lauand2024revisiting} for linear cases and in \cite{moulines2011non, bach2013non, gadat2023optimal} for non-linear scenarios. Recently, \cite{kwon2024two} provided finite-time convergence results for linear two-time-scale SA with constant step sizes, highlighting geometric rates alongside non-vanishing bias and variance. \cite{doan2021markov} and \cite{chandak2025finite} studied general two-time-scale SA algorithms, yet their derived convergence rates lack tightness. Moreover,  \cite{shen2022single, doan2024fast, zeng2024fast, chandak20251} explored fast variants of non-linear SA, achieving optimal $\mathcal{O}(1/k)$ convergence rates. Although termed two-time-scale by the authors, according to our notation, the iterates studied in these papers are not considered ``two-time-scale''.
 
One of the closest works to ours is \cite{kaledin2020finite}. In this paper, the authors study the same setting as two-time-scale linear SA with Markovian noise. However, the convergence bounds in \cite{kaledin2020finite} are loose and have a linear dependence on the dimension of the variables. In contrast, in this paper, we develop a new approach to study the convergence behavior of the covariance matrix and achieve a tight bound. Furthermore, in our paper, we consider a more general set of assumptions on the step size compared to \cite{kaledin2020finite}. This helps us to study the convergence of the Polyak-Ruppert averaging, which was not possible in \cite{kaledin2020finite}. For a detailed comparison, we summarized the results in the literature together with our work in Table \ref{table:results2}.

\textbf{Reinforcement Learning:} In many settings, especially in RL, two-time-scale algorithms help overcome many difficulties, such as stability in off-policy TD-learning. GTD, GTD2 and TDC \cite{sutton2008convergent}, \cite{sutton2009fast}, \cite{sutton2018reinforcement}, \cite{szepesvari2022algorithms} are some of the most well-studied and widely used methods to stabilize algorithms with off-policy sampling. This success has led to growing attention on finite time behavior of linear two-time-scale SA in the context of RL. The work \cite{xu2019two} analyzes TDC under Markovian noise but the non-asymptotic rate is not optimal. In \cite{xu2021sample} the authors establish a mean-square bound only under a constant step size, which does not ensure convergence. Concentration bounds for GTD and TDC were studied in \cite{wang2017finite} and \cite{li2023sharp}, respectively. Furthermore, TDC with a non-linear function approximation was studied in \cite{wang2020finite} and \cite{wang2021finite} but their results could not match the optimal rate. \cite{raj2022faster} studied GTD algorithms but required bounded iterates, an assumption we do not impose.
 
\section{Problem Formulation}\label{sec:prob_form}
Consider the following set of linear equations which we aim to solve:
\begin{subequations}\label{eq:lin_main}
   \begin{align}
    A_{11}y+A_{12}x&=b_1\\
    A_{21}y+A_{22}x&=b_2.
\end{align} 
\end{subequations}
where $x\in \mathbb{R}^{d_x}$ and $y\in \mathbb{R}^{d_y}$. Here $A_{ij}, i,j\in\{1,2\}$ are constant matrices that satisfy the following assumption.
\begin{assumption}\label{ass:hurwitz_main}
    Define $\Delta=A_{11}-A_{12}A_{22}^{-1}A_{21}$. Then $-A_{22}$ and $-\Delta$ are Hurwitz, i.e., all their eigenvalues have negative real parts.
\end{assumption}

We note that using standard linear algebra, one can show that Assumption \ref{ass:hurwitz_main} on $A_{22}$ is weaker than the strong monotonicity assumption in prior work such as \cite[Eq. (5)]{mou2021optimal}, which studies the finite time convergence bound of two-time-scale SA.

Assumption \ref{ass:hurwitz_main} enables us to solve the set of linear equations \eqref{eq:lin_main} as follows. First, for a fixed value of $y$, the second equation has a unique solution $x^*(y) = A_{22}^{-1}(b_2-A_{21}y)$. Next, substituting $x^*(y)$ in the first equation, we can find $y^*=\Delta^{-1}(b_1-A_{12}A_{22}^{-1}b_2)$ and next $x^*=A_{22}^{-1}(b_2-A_{21}\Delta^{-1}(b_1-A_{12}A_{22}^{-1}b_2))$ as the unique solution of this linear set of equations. Given access to the exact value of the matrices $A_{ij}, i,j\in\{1,2\}$ and the vectors $b_{i}, i\in\{1,2\}$, the above steps can be used to evaluate the exact solution to the linear equations \eqref{eq:lin_main}. However, unfortunately, in practical settings, we only have access to an oracle which at each time step $k$, produces a noisy variant of these matrices in the form of $A_{ij}(O_k), i,j\in\{1,2\}$ and $b_{i}(O_k), i\in\{1,2\}$, where $O_k$ is the sample of the Markov chain $\{O_l\}_{l\geq0}$ at time $k$. We assume that this Markov chain satisfies the following assumption:
\begin{assumption}\label{ass:poisson_main}
$\{O_k\}_{k\geq 0}$ is sampled from a finite state, irreducible, and aperiodic Markov chain with state space $\mathcal{S}$, transition probability $P$ and unique stationary distribution $\mu$. Furthermore, the expectation of $A_{ij}(O_k), i,j\in\{1,2\}$ and $b_{i}(O_k), i\in\{1,2\}$ with respect to the stationary distribution $\mu$ is equal to $A_{ij}, i,j\in\{1,2\}$ and $b_{i}, i\in\{1,2\}$, respectively. 
\end{assumption}
The two-time-scale linear stochastic approximation is an iterative scheme for solving the set of linear equations \eqref{eq:lin_main}, using the noisy oracles. To ensure convergence of SA, we impose the following assumption on the step sizes:
\begin{assumption}\label{ass:step_size_main} We consider step sizes
    $\alpha_k=\alpha/(k+K_0)^{\xi}$ with $0.5<\xi<1$, and $\beta_k=\beta/(k+K_0)$, where $\alpha>0$ and $K_0\geq 1$ can be any constant and $\beta$ should be such that $-\left(\Delta-\beta^{-1}I/2\right)$ is Hurwitz.
\end{assumption}
Choices of step sizes in Assumption \ref{ass:step_size_main} can be justified as follows. Firstly, both $\alpha_k$ and $\beta_k$ converge to zero, which is necessary to ensure dampening of the updates of $x_k$ and $y_k$ to zero. Secondly, both of $\alpha_k$ and $\beta_k$ are non-summable, (i.e., $\sum_{k=1}^\infty \alpha_k=\sum_{k=1}^\infty \beta_k=\infty$.) Intuitively speaking, $\sum_{k=1}^\infty \alpha_k$ and $\sum_{k=1}^\infty \beta_k$ are proportional to the distance that can be traversed by the variables $x$ and $y$, respectively. Hence, in order to ensure that both the variables can explore the entire space, non-summability of the step sizes is essential.
Note that among the class of step sizes of the form $\beta_k=\beta/(k+K_0)^\nu$, $\nu=1$ is the maximum exponent that can satisfy this requirement. Thirdly, $\xi<1$ ensures a time-scale separation between the updates of the variables $x$ and $y$. In particular, $x_k$ is updated in a faster time-scale compared to $y_k$. Intuitively speaking, throughout the updates, $x_k$ ``observes'' $y_k$ as stationary, and Eq. \eqref{eq:two_time_scale_main_x} converges ``quickly'' to $x(y_k)\simeq A_{22}^{-1}(b_2-A_{21}y_k)$. Next, Eq. \eqref{eq:two_time_scale_main_y} uses $x(y_k)$ to further proceed with the updates. Moreover, in this Markovian noise setting, we need to have $0.5<\xi$, which means the faster time-scale Eq. \eqref{eq:two_time_scale_main_x} should not be ``too fast'' to avoid a long delay of $y_k$ compared to $x_k$. Finally, this assumption requires $\beta$ to be large enough so that $-\left(\Delta-\beta^{-1}I/2\right)$ is Hurwitz.

\section{Main Results}\label{sec:main_result}
Before proceeding with the result, we define $\tilde{b}_i(\cdot)=b_i(\cdot)-b_i+(A_{i1}-A_{i1}(\cdot))y^*+(A_{i2}-A_{i2}(\cdot))x^*$ for $i\in\{1,2\}$.
Notice that by definition, we have $\E_{O\sim\mu}[\bt_i(O)]=0$.
Furthermore, 
note that by Assumption \ref{ass:poisson_main}, as shown in \cite[Proposition 21.2.3]{douc2018markov} there exist $\hat{b}_i(\cdot)~i\in\{1,2\}$ functions which are solutions to the following Poisson equations,
    \begin{align*}
        &\hat{b}_i(o)=\bt_i(o)+\sum_{o'\in \mathcal{S}}P(o'|o)\hat{b}_i(o')~~\forall~o\in \mathcal{S},\\
        &\sum_{o\in \mathcal{S}}\mu(o)\hat{b}_i(o)=0.
    \end{align*}
Next, we introduce some definitions that will be essential in the presentation of the main theorem. 
\begin{definition}\label{def:var}
Define the following matrices:
\begin{align}
\Gamma^x=&\E_{O\sim \mu}[\hat{b}_2(O)\bt_2(O)^\top +\bt_2(O)\hat{b}_2(O)^\top-\bt_2(O)\bt_2(O)^\top]\nonumber\\ 
\Gamma^{xy}
=&\E_{O\sim \mu}[\hat{b}_2(O)\bt_1(O)^\top+ \bt_2(O) \hat{b}_1(O)^\top -\bt_2(O)\bt_1(O)^\top]\nonumber\\
\Gamma^y=&\E_{O\sim \mu}[\hat{b}_1(O)\bt_1(O)^\top + \bt_1(O) \hat{b}_1(O)^\top-\bt_1(O)\bt_1(O)^\top].\nonumber
\end{align}
\end{definition}
In the following proposition we show that $\Gamma^x,\Gamma^{xy}$, and $\Gamma^{y}$ can be expressed in terms of $\bt_i, i\in\{1,2\}$ only.
\begin{proposition}\label{prop:alternate_pe}
    Let $\{\Tilde{O}_k\}_{k\geq 0}$ denote a Markov chain with $\tO_0\sim \mu$. Then, we have the following:
\begin{align}
\Gamma^x=&\E{[\bt_2(\tilde{O}_0)\bt_2(\tilde{O}_0)^\top]}+\sum_{j=1}^\infty \E{[\bt_2(\tilde{O}_j)\bt_2(\tilde{O}_0)^\top + \bt_2(\tilde{O}_0)\bt_2(\tilde{O}_j)^\top]}\nonumber\\
\Gamma^{xy}
=&\E{[\bt_2(\tilde{O}_0) \bt_1(\tilde{O}_0)^\top]}+ \sum_{j=1}^\infty \E{[\bt_2(\tilde{O}_j) \bt_1(\tilde{O}_0)^\top + \bt_2(\tilde{O}_0)\bt_1(\tilde{O}_j)^\top]}\nonumber\\
\Gamma^y=&\E{[\bt_1(\tilde{O}_0) \bt_1(\tilde{O}_0)^\top]}+\sum_{j=1}^\infty \E{[\bt_1(\tilde{O}_j) \bt_1(\tilde{O}_0)^\top + \bt_1(\tilde{O}_0) \bt_1(\tilde{O}_j)^\top]}.\nonumber
\end{align}
\end{proposition}

The proof of Proposition \ref{prop:alternate_pe} can be found in Appendix \ref{sec:app_proof_main}. 

Next, in Theorem \ref{thm:Markovian_main} we state our main result. In this theorem, we study the convergence behavior of $y_k$ and $x_k$, where we state our result in terms of $\yh_k=y_k-y^*$ and $\xh_k=x_k - x^* + A_{22}^{-1}A_{21}(y_k-y^*)$. In this theorem, we establish the dependence of our upper bound with respect to $d=\max\{d_x,d_y\}$. 

\begin{theorem}\label{thm:Markovian_main}
Under Assumptions \ref{ass:hurwitz_main}, \ref{ass:poisson_main}, and \ref{ass:step_size_main}, for all $k\geq 0$ we have
\begin{align}
    \E[\yh_k\yh_k^\top] =& \beta_k \Sigma^y+\frac{1}{(k+K_0)^{1+(1-\varrho)\min(\xi-0.5, 1-\xi)}} C^y_k(\varrho, d)\label{eq:main_y_main}\\
    \E[\xh_k\yh_k^\top]=&\beta_k \Sigma^{xy}+\frac{1}{(k+K_0)^{\min(\xi +0.5, 2-\xi)}} C^{xy}_k(\varrho, d)\label{eq:main_xy_main}\\
    \E[\xh_k\xh_k^\top]=&\alpha_k \Sigma^x+\frac{1}{(k+K_0)^{\min(1.5\xi, 1)}} C^x_k(\varrho, d),\label{eq:main_x_main}
\end{align}
where $0<\varrho<1$ is an arbitrary constant, $\sup_k\max\{\|C_k^y(\varrho, d)\|,\|C_k^{xy}(\varrho, d)\|,\|C_k^x(\varrho, d)\|\}<c_0(\varrho,d)<\infty$ for some problem-dependent constant $c_0(\varrho, d)$\footnote{Throughout the paper, unless otherwise stated, $\|\cdot\|$ represents Euclidean 2-norm.}, and $\Sigma^y$, $\Sigma^{xy}={\Sigma^{yx}}^\top$ and $\Sigma^x$ are unique solutions to the following system of equations: 
\begin{subequations}\label{eq:lyap_eq_tts}    
\begin{align}
    &A_{22}\Sigma^x+ \Sigma^xA_{22}^\top=\Gamma^x\label{eq:sigma_x_def_main}\\
    &A_{12} \Sigma^x+\Sigma^{xy}A_{22}^\top =\Gamma^{xy}\label{eq:sigma_xy_def_main}\\
    &\left(\Delta-\frac{1}{2\beta}I\right)\Sigma^y +\Sigma^y \left(\Delta^\top-\frac{1}{2\beta}I\right)= \Gamma^y- A_{12}\Sigma^{xy}-\Sigma^{yx}A_{12}^\top\label{eq:sigma_y_def_main}.
\end{align}
\end{subequations}
Furthermore, the constant of the higher order term satisfies $c_0(\varrho,d)= \mathcal{O}(d^2)$. 
\end{theorem}
The proof of Theorem \ref{thm:Markovian_main} is provided in Appendix \ref{sec:app_proof_main}. Theorem \ref{thm:Markovian_main} shows that matrix $\E[\yh_k\yh_k^\top]$ can be written as a sum of two matrices $\beta_k\Sigma^y$ and $\frac{1}{(k+K_0)^{1+(1-\varrho)\min(\xi-0.5, 1-\xi)}} C^y_k(\varrho, d)$. The first term is the leading term, which dominates the behavior of $\E[\yh_k\yh_k^\top]$ asymptotically. In addition, since $\varrho<1$ and $0.5<\xi<1$, the second term behaves as a higher-order term. The parameter $\varrho$ determines the behavior of the higher-order term. As $\varrho$ gets closer to $0$, the convergence rate of the non-leading term approaches $\frac{1}{(k+ K_0)^{1+\min(\xi-0.5,1-\xi)}}$. However, $c_0(\varrho, d)$ might become arbitrarily large. In addition, the constant $c_0(\varrho, d)$ in Theorem \ref{thm:Markovian_main} depends on all the parameters of the problem, such as $P, \alpha,\beta$, and $A_{ij}, b_i, i\in\{i,j\}$, and the initial condition, i.e. $x_0$ and $y_0$. 

\textbf{Solution of Eq. \eqref{eq:lyap_eq_tts}:} To solve the set of Eqs. in \eqref{eq:sigma_x_def_main}-\eqref{eq:sigma_y_def_main}, we first obtain $\Sigma^x$ by solving the Lyapunov equation \eqref{eq:sigma_x_def_main}. Next, we solve for $\Sigma^{xy}$ using the linear equation \eqref{eq:sigma_xy_def_main}. Finally, we obtain $\Sigma^y$ by solving the Lyapunov equation \eqref{eq:sigma_y_def_main}. The following proposition whose proof can be found in Appendix \ref{sec:app_proof_main} shows that the right hand side of Eq. \eqref{eq:sigma_y_def_main} is a positive definite matrix, which verifies that the Lyapunov equation \eqref{eq:sigma_y_def_main} has a unique solution.
\begin{proposition}\label{prop:positive_def}
    Define the random vector $h_N=\frac{1}{N}\sum_{j=0}^{N-1} \tilde{b}_1(\tO_j)-A_{12}A_{22}^{-1}\tilde{b}_2(\tO_j)$. Then, we have
    \begin{align*}
        \Gamma^y- A_{12}\Sigma^{yx}-\Sigma^{xy}A_{12}^\top = \lim_{N\rightarrow\infty} \E[h_Nh_N^\top].
    \end{align*}
\end{proposition}

\textbf{Asymptotic optimality of Theorem \ref{thm:Markovian_main}:} The results in Theorem \ref{thm:Markovian_main} are asymptotically optimal. In particular, since the results in this theorem are in terms of equality, we have 
    \begin{align*}
        \lim_{k \rightarrow \infty} \frac{1}{\beta_k} \E[\yh_k \yh_k^\top]  &=  \Sigma^y,\\
        \lim_{k \rightarrow \infty} \frac{1}{\alpha_k} \E[\xh_k \xh_k^\top] & =  \Sigma^x.
    \end{align*}
In a work \cite{hu2024central} that appeared simultaneously as ours, central limit theorem for two-timescale SA with Markovian noise has been established. In this work, the authors show that $\yh_k/\sqrt{\beta_k}\xrightarrow{dist.}\mathcal{N}(0,\Sigma^y)$ and $\hat{x}_k/\sqrt{\alpha_k} \xrightarrow{dist}\mathcal{N}(0,\Sigma^x)$, which verifies the asymptotic optimality of our results. We also study the behavior of  $\E[\yh_k\xh_k^\top]$ and observe that $\E[\yh_k\xh_k^\top]$
has convergence with the rate $\beta_k$, and the asymptotic covariance of  $\E[\yh_k\xh_k^\top]/\beta_k$ is $\Sigma^{xy}$.

Given our result in Theorem \ref{thm:Markovian_main}, we can easily establish a convergence bound in terms of $\E[\|\yh_k\|^2]$. The following corollary states this result.

\begin{corollary}\label{cor:l2_bound}
For all $k\geq 0$, the iterations of two-time-scale linear SA \ref{eq:two_time_scale_main} satisfies
    \begin{align*}
        \E[\|\yh_k\|^2]\leq \beta_k \text{tr}(\Sigma^y)+\frac{c(d)}{(k+K_0)^{1+0.5\min(\xi-0.5, 1-\xi)}},
    \end{align*}
    where $c(d)=\mathcal{O}(d^3)$ is a problem-dependent constant. 
\end{corollary}

As a direct application of Theorem \ref{thm:Markovian_main}, we can establish the convergence bound of various RL algorithms such as TD-learning with Polyak-Ruppert averaging, TDC, GTD, and GTD2. In Sections \ref{sec:lin_pol} and \ref{sec:app_RL} we will study these algorithms. 

Several remarks are in order with respect to this result. 

\textbf{Dimension dependency of our result:} As discussed before, the leading term in the convergence result of Theorem \ref{thm:Markovian_main} is tight (including its dimension dependency), and the dimension dependency of the higher order term is $\mathcal{O}(d^2)$. Compared to the most related work to ours, \cite{kaledin2020finite} has $\mathcal{O}(d^5)$ and $\mathcal{O}(d^7)$ dimension dependency in their convergence bound of $\hat{y}_k$ and $\hat{x}_k$, respectively. Hence, our result significantly improves on the $d$-dependency compared to the prior work. For a complete analysis of the $d$-dependency of \cite{kaledin2020finite}, please look at Section \ref{sec:kaledin_d_dependency}.

\textbf{Higher order terms:} Theorem \ref{thm:Markovian_main} shows that $\max\{\|C_k^y(\varrho)\| , \|C_k^{xy}(\varrho)\| ,\|C_k^x(\varrho)\|\}$ is bounded with a problem-dependent constant for all $k\geq 0$. However, it might be that $\max\{\|C_k^y(\varrho)\|,\|C_k^{xy}(\varrho)\|,\|C_k^x(\varrho)\|\}$ is decreasing with respect to $k$. Studying the tightness of the bound on the higher order terms is a future research direction.

\textbf{Discussion on the Assumptions:} 
The result of Theorem \ref{thm:Markovian_main} is stated under Assumptions \ref{ass:hurwitz_main}, \ref{ass:poisson_main}, and \ref{ass:step_size_main}. Assumption \ref{ass:hurwitz_main} is standard in the asymptotic and finite time analysis of two-time-scale linear SA \cite{konda2004convergence, gupta2019finite, kaledin2020finite}. When dealing with Markovian noise, Assumption \ref{ass:poisson_main} is standard in the literature \cite{bhandari2018finite, khodadadian2022finite}. Finally, Assumption \ref{ass:step_size_main} is regarding the choice of step size, which will be elaborated further in Section \ref{sec:choice_of_step_size}.

\begin{remark}    
For general two-time-scale linear SA, when the matrix $\Delta$ is unknown, the algorithm can become sensitive to the choice of step size parameter $\beta$. A common approach to address this sensitivity is to employ iterate averaging alongside the updates \cite{mokkadem2006convergence}. However, implementing iterate averaging introduces a third time-scale, resulting in a more complex three-time-scale algorithm, which lies beyond the scope of this paper.
\end{remark}

\subsection{Choice of step size}\label{sec:choice_of_step_size}

In Assumption \ref{ass:step_size_main}, we impose several conditions on the step size parameters. Regarding the step size $\beta_k$, although we could select it as $\beta_k=\frac{\beta}{(k+K_0)^\nu}$ for any $\xi<\nu\leq 1$, we specifically choose the restrictive step size $\frac{\beta}{(k+K_0)}$. The rationale behind this choice is that the convergence of $\E[\hat{y}_k\hat{y}_k^\top]$ is inherently limited by the rate $\beta_k$. Hence, setting $\nu=1$ provides the optimal possible convergence rate for $\E[\hat{y}_k\hat{y}_k^\top]$. Additionally, we impose a restrictive condition $0.5<\xi$ in Assumption \ref{ass:step_size_main}. While it might appear as merely a technical condition of our proof, numerical simulations (see Figure \ref{fig:xi}) demonstrate that when the noise is Markovian and $\xi<0.5$, $\E[\hat{y}_k\hat{y}_k^\top]$ fails to exhibit the convergence behavior described in \eqref{eq:main_y_main}. Another essential condition is that $\beta$ must be sufficiently large to ensure that $-(\Delta - \frac{I}{2\beta})$ is Hurwitz. This necessity is further validated by the simulation results shown in Figure \ref{fig:beta}. More detailed simulation information is provided in Appendix \ref{appendix:sim}.

\begin{figure*}[t!]
     \centering
     \begin{subfigure}[b]{0.49\textwidth}
         \centering
         \includegraphics[width=1\textwidth]{Arxiv/images/sim4.tex}
         \caption{Effect of $\xi$}
         \label{fig:xi}
     \end{subfigure}
     \hfill
     \begin{subfigure}[b]{0.49\textwidth}
         \centering
         \includegraphics[width=1\textwidth]{Arxiv/images/Polyak.tex}
         \caption{Effect of $\beta$}
         \label{fig:beta}
     \end{subfigure}
        \caption{Convergence behaviour of $\mathcal{E}_k$ for various choices of $\xi$ and $\beta$, where $\mathcal{E}_k=\frac{\|\yh_k\yh_k^\top\|}{\beta_k}$. The bold lines show the mean behavior across 5 sample paths, while the shaded region is the standard deviation from the mean. Both plots show a transition from stability to divergence of $\mathcal{E}_k$ when $\xi$ or $\beta$ do not satisfy the assumption \ref{ass:step_size_main}.}
        \label{fig:parameter}
\end{figure*}

To verify which $\xi$ gives the best sample complexity, a lower bound must be established for the higher-order term, which is a potential future research direction.

\textbf{Optimal choice of step size in the slower time-scale:} To achieve the best rate for the higher-order terms in \eqref{eq:main_y_main}, we select $\xi$ to maximize $\min(\xi-0.5, 1-\xi)$, yielding an optimal value of $\xi=0.75$. Previous studies, such as \cite{moulines2011non,srikant2024rates}, suggest an optimal $\xi=2/3$. Specifically, \cite{moulines2011non} considers non-linear SA with martingale noise and Polyak-Ruppert averaging, and in their linear scenario, the optimal choice reduces further to $\xi=0.5$.
\footnote{In the linear setting, \cite[Theorem 3]{moulines2011non} simplifies to $\sqrt{\E[|y_n|^2]}\leq \frac{\sigma^2}{\sqrt{n}}+\mathcal{O}\left(\frac{1}{n^{1-\xi/2}} + \frac{1}{n^{(1+\xi)/2}}\right)$, leading to an optimal $\xi=0.5$.} It is important to note that these optimal choices for $\xi$ are derived from upper bound analyses rather than exact error minimization. Determining the definitive optimal step size via establishing lower bounds remains a promising direction for future research.

\textbf{Optimal choice of step size in the faster time-scale:} Our results facilitate choosing the optimal $\beta$ to achieve the fastest convergence of Algorithm \eqref{eq:two_time_scale_main}. Specifically, selecting $\beta$ to minimize $\|\beta\Sigma^y\|$, where $\Sigma^y$ solves Eq. \eqref{eq:sigma_y_def_main}, achieves the best asymptotic convergence for $\E[\hat{y}_k\hat{y}_k^\top]$. For instance, consider the special case where we assume $A_{21}(O_k)=0$, $b_1(O_k)=0$, $A_{11}(O_k)=I$ and $A_{12}(O_k)=-I$. In Appendix \ref{sec:best_step_size}, we show that $\beta=1$ achieves the best asymptotic covariance in the context of algorithm \eqref{eq:two_time_scale_main}, which corresponds to Polyak-Ruppert averaging.

\subsection{Single Time-Scale vs Two-Time-Scale}

In this section, we will discuss an alternative approach to find the solution of Eq. \eqref{eq:lin_main} given that at any time $k\geq0$ we have access to noisy oracles $A_{ij}(O_k)$ and $b_i(O_k)$, $i,j=1,2$. Consider constant $\kappa>0$, and  
\begin{align*}
    A_{\kappa}(O_k)= \begin{bmatrix}
        A_{11} (O_k) & A_{12}(O_k)\\
         \kappa A_{21}(O_k) & \kappa A_{22}(O_k)
    \end{bmatrix};~~~ b_\kappa(O_k)=\begin{bmatrix}
        b_1(O_k)\\
        \kappa b_2(O_k)
    \end{bmatrix}.
\end{align*}

Consider step size sequence of the form $\beta_k=\beta/(k+K_0)$ and denote $\mathrm{z}_k=[y_k, x_k]^\top$. Then, consider the following SA update rule
\begin{align}\label{eq:single_time_scale}
    \mathrm{z}_{k+1}= \mathrm{z}_k+\beta_k\left(b_\kappa(O_k)-A_\kappa(O_k)\mathrm{z}_k\right).
\end{align}
If $\kappa=\alpha/\beta$, the update rule \eqref{eq:single_time_scale} is equivalent to Eq. \eqref{eq:two_time_scale_main} with the the choice of step size such that $\alpha_k=\alpha\beta_k/\beta$. In addition, this SA is equivalent to single-time-scale linear SA  studied in \cite{srikant2019finite, chen2021lyapunov}. Denote $A_\kappa$ as the expectation of $A_\kappa(O)$ with respect to the stationary distribution. As shown in \cite{borkar2009stochastic,srikant2019finite}, assuming $-A_\kappa$ is Hurwitz, the SA \eqref{eq:single_time_scale} converges to $\mathrm{z}^*=[x^*, y^*]^\top$. 

\begin{remark}
Some of the prior works study the two-time-scale SA \eqref{eq:two_time_scale_main} under the framework of recursion \eqref{eq:single_time_scale}  \cite{shen2022single,doan2024fast, zeng2024fast}. Although these works refer to this algorithm as ``two-time-scale'', by the terminology of our work, \eqref{eq:single_time_scale} is a single-time-scale SA.
\end{remark}

We aim at answering the following two questions:
\begin{itemize}
    \item Consider the set of problems that can be solved by the two-time-scale SA \eqref{eq:two_time_scale_main}. How are they compared to the set of problems that can be solved by the single-time-scale SA \eqref{eq:single_time_scale}? This question is addressed in Section \ref{sec:comparison}.
    \item If our goal is to ensure the convergence of $(x_k,y_k)$ to $(x^*,y^*)$, which algorithm should we choose? This question is addressed in Section \ref{sec:gurantee}.
\end{itemize}

\subsubsection{Comparison of Set of Problems Solved by Single-Time-Scale vs Two-Time-Scale SA}\label{sec:comparison}
In this section, we show that Assumption \ref{ass:hurwitz_main} is a sufficient condition for the convergence of \eqref{eq:single_time_scale} with an appropriate choice of $\kappa$. However, the converse is not true. Fix a vector $b=[b_1,b_2]^\top$ and consider a set of linear equations of the form \eqref{eq:lin_main} with fixed vectors $b_1, b_2$ and matrices $A_{11},A_{12},A_{21},A_{22}$ such that $A=[A_{11},A_{12};A_{21},A_{22}]\in \mathcal{A}=\{A\in\mathbb{R}^{(d_x+d_y)\times(d_x+d_y)}\}$. Next, consider sets $\mathcal{B},\mathcal{C},\mathcal{D}$ defined as follows.
\begin{enumerate}
    \item $\mathcal{B}=\{A \in\mathcal{A}|\exists \kappa>0:-A_\kappa ~\text{is Hurwitz}\}:$ This is the set of linear problems that can be solved by a SA recursion of the form \eqref{eq:single_time_scale} with step sizes $\alpha_k=\alpha/(k+1)$ and $\beta_k=\beta/(k+1)$ for an appropriately chosen ratio $\alpha/\beta$. Note that this is a single-time-scale algorithm.
    \item $\mathcal{C}=\{A \in\mathcal{A}|-A ~\text{is Hurwitz}\}:$ This is the set of linear problems that can be solved by a SA recursion of the form \eqref{eq:single_time_scale} with step sizes $\alpha_k=\alpha/(k+1)$ and $\beta_k=\beta/(k+1)$ for any choice of $\alpha, \beta$ such that $\alpha=\beta$. This also corresponds to single-time-scale algorithm, albeit without any step-size tuning.
    \item $\mathcal{D}=\{A \in\mathcal{A}|-A_{22} ~\text{and}~-\Delta=-(A_{11}-A_{12}A_{22}^{-1}A_{21}) ~\text{are both Hurwitz}\}:$ This is the set of linear problems that can be solved by a SA recursion of the form \eqref{eq:two_time_scale_main} with step sizes $\alpha_k=\alpha/(k+1)^\xi$ and $\beta_k=\beta/(k+1)$ for any choice of $\alpha$ and $\beta$. This corresponds to the two-time-scale algorithm.
\end{enumerate}
The relation of the set of problems mentioned above is studied in Proposition \ref{prop:venn}.
\begin{proposition}\label{prop:venn}
    These sets of problems satisfy: $\mathcal{B}\subsetneq \mathcal{A}$, $\mathcal{C}\cup \mathcal{D}\subsetneq \mathcal{B}$, $\mathcal{C}\not\subset \mathcal{D}$, $\mathcal{D}\not\subset\mathcal{C}$, and $\mathcal{C}\cap \mathcal{D}\neq \varnothing$. 
\end{proposition}
Figure \ref{fig:Venn_diagram} shows the relationship stated in Proposition \ref{prop:venn}, and the proof of this proposition is stated in Appendix \ref{sec:app_proof_main}. According to Proposition \ref{prop:venn}, a bigger class of problems can be solved by single-time-scale SA \eqref{eq:single_time_scale} with an appropriate choice of $\alpha/\beta$. Nevertheless, as discussed in the following section, two-time-scale SA offers the advantage of guaranteed convergence for the problems within the set $\mathcal{D}$.

\begin{figure}
    \centering
    \begin{tikzpicture}
      \begin{scope}
        \draw[fill=gray!30, opacity=0.5, thick] (-3,-2.5) rectangle (3,2.5);
        \fill[red!50, opacity=0.7] (0,0) circle (2cm);
        \fill[blue!50, opacity=0.7] (-0.5,0) circle (1cm);
        \fill[green!50, opacity=0.7] (0.5,0) circle (1cm);
        \node at (-2.4,1.7) {$\mathcal{A}$};
        \node at (-0.8,1.5) {$\mathcal{B}$};
        \node at (-1,0.4) {$\mathcal{C}$};
        \node at (0.9,0.4) {$\mathcal{D}$};
      \end{scope}
    \end{tikzpicture}
    \caption{The relationship among $\mathcal{A}, \mathcal{B}, \mathcal{C}, \mathcal{D}$ as 4 sets of the linear equations of the form \eqref{eq:lin_main}.}
    \label{fig:Venn_diagram}
\end{figure}

\subsubsection{Guaranteed Convergence of Two-Time-Scale SA}\label{sec:gurantee}

It can be shown that under Assumption \ref{ass:hurwitz_main}, if the ratio $\alpha/\beta$ is chosen large enough, then the block matrix $A_{\alpha/\beta}$ becomes Hurwitz \cite[Theorem 6]{cothren2024onlinefeedbackoptimizationsingular}. In contrast, if $\alpha/\beta$ is not appropriately chosen, then the algorithm \eqref{eq:single_time_scale} may diverge. Figure \ref{fig:ssa_div} shows an example of this divergence behavior when the ratio $\alpha/\beta$ is such that the matrix $A$ is not Hurwitz. Details of the experiment are given in the Appendix \ref{appendix:sim}.

Next, we show that the two-time-scale algorithm \eqref{eq:two_time_scale_main} with $\xi<1$ can ensure convergence (not necessarily optimally) to $\mathrm{z}^*$. 

\begin{proposition}\label{prop:loose_convergence}
    Consider the iterates of $x_k$ and $y_k$ in \eqref{eq:two_time_scale_main} and the step sizes
    $\alpha_k=\alpha/(k+1)^{\xi}$ with $0.5<\xi<1$, and $\beta_k=\beta/(k+1)$. Suppose Assumptions \ref{ass:hurwitz_main} and \ref{ass:poisson_main} are satisfied. Then, $x_k\rightarrow x^*$ and $y_k\rightarrow y^*$ in the mean squared sense.
\end{proposition}
\begin{figure}[!ht]
    \centering
    \includegraphics[width=0.6\textwidth]{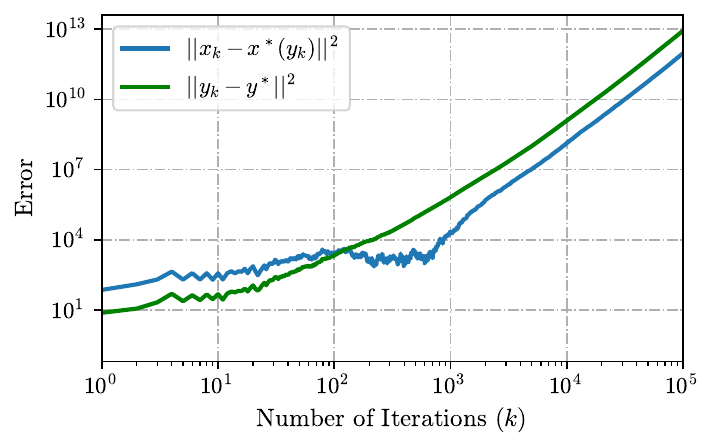}
    \caption{Divergence of two-time-scale linear SA when $\alpha_k=\alpha\beta_k/\beta$ and the ratio $\alpha/\beta$ is not carefully chosen.}
    \label{fig:ssa_div}
\end{figure}

Note that the assumption in the above proposition is a relaxation of Assumption \ref{ass:step_size_main}. In particular, Proposition \ref{prop:loose_convergence} shows that the condition in Assumption \ref{ass:step_size_main} on the choice of $\beta$ such that $-(\Delta-\beta^{-1}I/2)$ is Hurwitz  is only for optimal convergence, and it is not necessary if one is concerned with convergence alone. This proposition along with Figure \ref{fig:ssa_div} shows the significance of two-time-scale algorithms, compared to single-time-scale algorithms. Specifically, a two-time-scale algorithm has guaranteed convergence for any choice of ratio $\alpha/\beta$, while a wrong choice of ratio $\alpha/\beta$ might result in divergence for a single-time-scale algorithm.

Now consider the scenario in which the ratio $\alpha/\beta$ is carefully chosen such that $-A_{\kappa}$ is Hurwitz, and hence the single-time-scale algorithm \ref{eq:single_time_scale} has convergence. Next, we aim at achieving an optimal rate of convergence $\mathcal{O}(1/k)$ for this algorithm. To achieve this, it is again necessary to carefully choose $\beta$ such that $-(A_{\kappa}-\beta^{-1}I/2)$ is Hurwitz\footnote{This follows by considering update \eqref{eq:two_time_scale_main} with $A_{12}(O_k)=A_{21}(O_k)=A_{22}(O_k)=0$ and $b_2(O_k)=x_0=0$ and Figure \ref{fig:beta}.}. This condition requires $\beta$ to be large enough, which is similar to the requirements of Assumption \ref{ass:step_size_main}.

\subsection{Linear SA with Polyak-Ruppert averaging}\label{sec:lin_pol} 

An application of Theorem \ref{thm:Markovian_main} is to establish the convergence behavior of a Markovian linear SA with Polyak-Ruppert averaging. In particular, when we assume $A_{21}(O_k)=0$, $b_1(O_k)=0$, $A_{11}(O_k)=I$ and $A_{12}(O_k)=-I$, and consider $\beta=1$, the iterates in Eq. \eqref{eq:two_time_scale_main} effectively represent the following recursion
\begin{subequations} \label{eq:lin_with_polyak}    
\begin{align}
    x_{k+1}= & x_k+\alpha_k(b(O_k)-A(O_k)x_k)\label{eq:linearSA}\\
    y_{k+1}= & y_k+\frac{1}{k+1}(x_k-y_k)= \frac{\sum_{i=0}^kx_i}{k+1},\label{eq:Polyak}
\end{align}
\end{subequations}

where $\alpha_k=\alpha/(k+1)^\xi$. Theorem \ref{thm:polyak} specifies the convergence behavior of the Markovian linear SA with Polyak-Ruppert averaging.  

\begin{theorem}\label{thm:polyak}
    Consider the iterations in \ref{eq:lin_with_polyak}. Define $\E_{O\sim \mu}[A(O)]=A$, $\E_{O\sim \mu}[b(O)]=b$, and $x^*=A^{-1}b$. Assume the matrix $-A$ is Hurwitz, Assumption \ref{ass:poisson_main} is satisfied, and $0.5<\xi<1$. Then we have
    \begin{align*}
        \E[(y_k-x^*)(y_k-x^*)^\top]= \beta_kA^{-1} \Gamma^xA^{-\top}+\frac{1}{(k+1)^{1+0.5\min(\xi-0.5, 1-\xi)}} C^y_k,
    \end{align*}
    where $\Gamma^x= \E{[\bt(\tilde{O}_0)\bt(\tilde{O}_0)^\top]} +\sum_{j=1}^\infty \E{[\bt(\tilde{O}_j) \bt(\tilde{O}_0)^\top+\bt(\tilde{O}_0) \bt(\tilde{O}_j)^\top]}$ and $\|C_k^y\|<c_p$ for some problem-dependent constant $c_p$. Here $\bt(\cdot)=b(\cdot)-b+(A-A(\cdot))A^{-1}b$.
\end{theorem}
For proof, refer to Appendix \ref{sec:app_proof_main}.

\begin{remark}
    The leading term in the result of Theorem \ref{thm:polyak} matches the CLT covariance established in \cite[Theorem 5]{borkar2021ode}. This further verifies the optimality of our convergence bounds.
\end{remark}

\begin{remark}
    In a previous work, \cite{kaledin2020finite} studies the finite time convergence of two-time-scale linear SA with Markovian noise. However, due to the restrictive assumptions in this work (in particular \cite[Assumption A2]{kaledin2020finite}), their result cannot be used to study the convergence of the iterates \eqref{eq:linearSA} and \eqref{eq:Polyak}.   
\end{remark}

Note that the iterates in Eq. \eqref{eq:linearSA} are independent of $y_k$, and can be studied as a single-time-scale SA. The convergence behavior of Markovian linear SA \eqref{eq:linearSA} has been studied in prior work \cite{bhandari2018finite,srikant2019finite} in the mean-square sense. As shown in the prior work, a wide range of algorithms, such as TD$(n)$, TD$(\lambda)$ \cite{sutton1988learning} and Retrace \cite{munos2016safe}, can be categorized as iterations in Eq. \eqref{eq:linearSA}. In order to handle the complications arising due to the Markovian noise, the authors in \cite{bhandari2018finite} introduce a relatively different variant of the iterate in Eq. \eqref{eq:linearSA} with a projection step. However, in this algorithm, the projection radius has to be chosen in a problem-dependent manner, which is difficult to estimate in a general setting. Furthermore, their choice of step size depends on the unknown problem-dependent parameters. Later, the authors in \cite{srikant2019finite} studied the convergence of iterate \eqref{eq:linearSA} under constant step size. Reproducing the result in \cite{srikant2019finite} with a time-varying step size of the form $\alpha_k=\alpha/(k+1)$, we can show that $\E[\|x_k\|^2]\leq c\log(k)/k$. However, this analysis requires a problem-dependent choice of $\alpha$, which is difficult to characterize for an unknown problem. Furthermore, this bound is not optimal in terms of $c$, and is suboptimal up to the $\log(k)$ factor. It has been shown \cite{polyak1992acceleration} that the use of Polyak-Ruppert averaging \eqref{eq:Polyak} together with linear SA \eqref{eq:linearSA} will achieve the optimal convergence rate in a robust way, thus addressing the previously highlighted issues.

\cite{mou2021optimal} have studied the convergence of \eqref{eq:linearSA} along with the Polyak-Ruppert averaging step \eqref{eq:Polyak} in mean squared error sense. In this work, they show that linear Markovian SA with constant step size and Polyak-Ruppert averaging attains a $\mathcal{O}(1/k)$ rate of convergence for the leading term plus $\mathcal{O}(1/k^{4/3})$ for a higher-order term. The leading term in the convergence result of \cite{mou2021optimal} is a constant away from the optimal convergence possible. Furthermore, their setting is not robust, as the choice of their step size depends on unknown problem-dependent constants. In addition, they introduce a problem-dependent burn-in period that is not robust to the choice of the problem instance. Moreover, due to the dependence of the step size on the time horizon, their algorithm does not have asymptotic convergence.

Contemporaneous to this work, \cite{srikant2024rates} established a non-asymptotic central limit theorem result for the convergence of $y_k$ in \eqref{eq:lin_with_polyak}. In particular, \cite{srikant2024rates} bounds the Wasserstein-1 distance between the error $\sqrt{k}(y_k-x^*)$ and a Gaussian with convariance matrix $(\bar{A}^{-1}\Sigma_{\infty}\bar{A}^{-\top})^{1/2}$.
In contrast, we bound the difference between the covariance matrix of $\sqrt{k}(y_k-x^*)$ and the same matrix $(\bar{A}^{-1} \Sigma_{\infty}\bar{A}^{-\top})^{1/2}$. 

As opposed to the previous work, Theorem \ref{thm:polyak} characterizes a sharp finite time bound of $\E[\hat{y}_k\hat{y}_k^\top]$ for linear SA with Markovian noise and Polyak-Ruppert averaging. Our result does not require a problem-dependent choice of step size $\alpha$ or burn-in period, as in \cite{mou2021optimal}, nor do we assume a projection step, as in \cite{bkh}. This result is a direct application of Theorem \ref{thm:Markovian_main}. In particular, for the linear SA with Polyak-Ruppert averaging in the context of two-time-scale linear SA, it is easy to show that $\Delta=I$ and $\beta=1$. Hence $-\left(\Delta-\beta^{-1}I/2\right)$ is Hurwitz, satisfying Assumption \ref{ass:step_size_main}.

\subsection{Application in Reinforcement Learning}\label{sec:app_RL}

Consider a Markov Decision Process (MDP) defined by the tuple $(\mathcal{S}, \mathcal{A}, P, r, \gamma)$, where $\mathcal{S}$ is the finite state space, $\mathcal{A}$ is the finite action space, $P=[[P(s'|s,a)]]$ denotes the transition probability kernel, $r=[r(s,a)]$ is the reward function, and $\gamma\in(0,1)$ is the discount factor. We denote $\pi$ as a policy, representing a probabilistic mapping from states to actions. For each $s\in\mathcal{S}$, the value function is defined as $v^\pi(s)= \E[\sum_{k=0}^\infty \gamma^kr(S_k,A_k)|S_0=s, \pi]$ which measures the expected cumulative reward starting from state $s$ under policy $\pi$.

In many real-world applications, the state space $\mathcal{S}$ is extremely large. Consequently, function approximation methods are employed to approximate the value function using a lower-dimensional parameter vector $\theta^\pi$. In this work, we consider linear function approximation: $v^\pi(s) \approx \phi(s)^\top \theta^\pi$, where $\theta^\pi\in \mathbb{R}^d$ with $d\ll|\mathcal{S}|$, and $\phi(s)\in \mathbb{R}^d$ are features representing each state. The feature vectors collectively form the rows of a full-rank matrix $\Phi \in \mathbb{R}^{|\mathcal{S}| \times d}$.

Our focus is on the policy evaluation task, where given a fixed policy $\pi$, the goal is to estimate $\theta^\pi$ from samples. In some settings, one may interact directly with the environment to collect fresh samples. However, in many cases, only historical or off-policy data is available, as described next.

\subsubsection{Temporal Difference with Gradient Correction (TDC) and Gradient Temporal Difference Learning (GTD)}\label{sec:TDC}

In real-world applications, collecting online data can be costly, unethical, or impractical. Off-policy learning leverages historical data collected under a behavior policy different from the target policy. In this setting, a fixed behavior policy generates samples, and the objective is to evaluate the value function under the target policy $\pi$.

A well-known challenge in off-policy learning is that the mismatch between the behavior policy and the target policy can cause instability or divergence \cite{sutton2018reinforcement}. To address this, algorithms such as GTD \cite{sutton2008convergent}, TDC, and GTD2 \cite{sutton2009fast} have been proposed. We next describe these algorithms and their convergence properties.

Suppose we observe a sample path $\{S_k, A_k, S_{k+1}\}_{k\geq 0}$ generated by a fixed behavior policy $\pi_b$, inducing an ergodic Markov chain over $\mathcal{S}$ with stationary distribution $\mu_{\pi_b}$. Define the importance sampling ratio $\rho(s,a)=\pi(a|s)/\pi_b(a|s)$. Also, define the matrices and vectors $A_k=\rho(S_k,A_k)\phi(S_k)(\phi(S_k)-\gamma\phi(S_{k+1}))^\top$, \\$B_k =\gamma\rho(S_k,A_k)\phi(S_{k+1})\phi(S_k)^\top$, $C_k=\phi(S_k)\phi(S_k)^\top$ and $b_k=\rho(S_k,A_k) r(S_k, A_k)\phi(S_k)$.

We have the following update rules:

\begin{itemize}
    \item \textbf{GTD:}
    \begin{align*}
        \theta_{k+1}&= \theta_k+\beta_k(A_k^\top \omega_k)\\
        \omega_{k+1}&= \omega_k+\alpha_k(b_k-A_k\theta_k-\omega_k)
    \end{align*}
    \item \textbf{GTD2:}
    \begin{align*}
        \theta_{k+1}&= \theta_k+\beta_k(A_k^\top \omega_k)\\
        \omega_{k+1}&= \omega_k+\alpha_k(b_k-A_k\theta_k-C_k\omega_k)
    \end{align*}
    \item \textbf{TDC:}
    \begin{align*}
        \theta_{k+1}&= \theta_k+\beta_k(b_k-A_k\theta_k-B_k\omega_k)\\
        \omega_{k+1}&= \omega_k+\alpha_k(b_k-A_k\theta_k-C_k\omega_k).
    \end{align*}
\end{itemize}
We now characterize the convergence behavior of these algorithms. Denote the stationary expectation of the matrices as $A=\mathbb{E}_{\mu_{\pi_b}}[\rho(S,A)\phi(S)(\phi(S)-\gamma\phi(S'))^\top]$, $B=\gamma\mathbb{E}_{\mu_{\pi_b}}[\rho(S,A)\phi(S')\phi(S)^\top]$, $C=\mathbb{E}_{\mu_{\pi_b}}[\phi(S)\phi(S)^\top]$ and $b=\mathbb{E}_{\mu_{\pi_b}}[\rho(S,A)r(S, A)\phi(S)]$. We have the following theorem.
\begin{theorem}\label{thm:TDC}
    Let $\alpha_k=\frac{1}{(k+1)^{0.75}}$, $\beta_k=\frac{\beta}{k+1}$, and define $\theta^*=A^{-1}b$. We have
    \begin{enumerate}
        \item For the GTD algorithm, assume $-\left(A^{\top}A-\frac{\beta^{-1}}{2}I\right)$ is Hurwitz. Then we have
        \begin{align*}
        \E[\|\theta_k-\theta^*\|^2]=\frac{\sigma^2_{GTD}}{k+1}+\mathcal{O}\left(\frac{d^3}{k^{1.125}}\right).
        \end{align*}
        \item For the GTD2 algorithm, assume $-\left(A^\top C^{-1}A-\frac{\beta^{-1}}{2}I\right)$ is Hurwitz. Then we have
        \begin{align*}
        \E[\| \theta_k-\theta^*\|^2]=\frac{\sigma^2_{GTD2}}{k+1}+\mathcal{O}\left(\frac{d^3}{k^{1.125}}\right).
        \end{align*}
        \item For the TDC algorithm, assume $-\left(A-BC^{-1}A-\frac{\beta^{-1}}{2}I\right)$ is Hurwitz. Then we have
        \begin{align*}
        \E[\|\theta_k-\theta^*\|^2]=\frac{\sigma^2_{TDC}}{k+1}+\mathcal{O}\left(\frac{d^3}{k^{1.125}}\right).
        \end{align*}
    \end{enumerate}
\end{theorem} 
The exact forms of the constants in the leading and higher-order terms are detailed in Appendix \ref{sec:app_proof_main}.
\begin{remark}
Theorem \ref{thm:TDC} implies a sample complexity of $\sigma^2/\epsilon + \mathcal{O}(d^3/\epsilon^{8/9})$ for GTD, GTD2, and TDC. similar to Corollary \ref{cor:l2_bound}, we observe that the leading terms are tight constants while the higher-order terms scale as $\mathcal{O}(d^3)$. Additionally, simulations (see Figure \ref{fig:beta}) confirm that an appropriate choice of $\beta$ is crucial to achieving the optimal convergence rate, indicating that these algorithms may be sensitive to step size tuning.
\end{remark}

\section{Proof Sketch}\label{sec:pf_sketch}

In this section, we provide a sketch of the proof of Theorem \ref{thm:Markovian_main}.
Our proof has several ingredients to handle challenges due to two time-scale behavior, Markovian noise, vector-valued iterates, intertwined updates, and asymmetric matrices. In this section, we illustrate all the key ideas in our proof to overcome these challenges. We do this by first considering a simplified two-time-scale SA with scalar iterates and i.i.d. noise where one of the iterates does not depend on the other.

First, we consider the following simple SA.
\begin{subequations}
\label{eq:two_time_scale_pf_sk}
\begin{align}
    y_{k+1}=&y_k-\beta_k(y_k+x_k)+\beta_kv_k\label{eq:ty_mm}\\
    x_{k+1}=&(1-\alpha_k)x_k+\alpha_k u_k.\label{eq:tx_mm}
\end{align}
\end{subequations}
This recursion is a simplified variant of the general two-time-scale linear SA \eqref{eq:two_time_scale_main} in three aspects. First, $v_k$ and $u_k$ are assumed to be zero-mean i.i.d. noises, while the noise in \eqref{eq:two_time_scale_main} is assumed to be Markovian. Note that the zero mean noise results in $x^*=y^*=0$. Second, all parameters here are assumed to be scalars, while the parameters in \eqref{eq:two_time_scale_main} are assumed to be high-dimensional. Third, the update of $x_k$ in \eqref{eq:tx_mm} is independent of $y_k$. However, the updates of the variables in \eqref{eq:two_time_scale_main} are intertwined. In this subsection, we study this simplified recursion, and in the following subsections we show how this analysis can be extended to the study of \eqref{eq:two_time_scale_main}.

Consider the Lyapunov functions $X_k=\E[x_k^2]$, $Z_k=\E[x_k
y_k]$, and $Y_k=\E[y_k^2]$ and assume $U=\E[u_k^2]$, $W=\E[v_ku_k]$, and $V=\E[v_k^2]$. We can always find the numbers $C_k^x, C_k^{xy}$, and $C_k^{y}$ such that 
\begin{align}\label{eq:pf_sketch_k}
    X_k = \alpha_k U/2 + C_k^x\zeta_k^x, \quad Z_k = \beta_k (W-U/2) + C_k^{xy}\zeta_k^{xy} , \quad Y_k=\beta_k(2-\beta^{-1})^{-1}(V+2W-U) + C_k^{y}\zeta_k^{y},
\end{align}
where $\zeta_k^x=\frac{1}{(k+K_0)^{\min\{1.5\xi, 1\}}}, \zeta_k^{xy}=\frac{1}{(k+K_0)^{\min\{\xi+0.5, 2-\xi\}}}, \zeta_k^y=\frac{1}{(k+K_0)^{1+(1-\varrho)\min\{\xi-0.5, 1-\xi\}}}$. Our goal is to show that for the simple setting of the recursion \eqref{eq:two_time_scale_pf_sk}, we have 
\begin{align}\label{eq:C_k_bound_pf_sk}
\|C_k^x\|, \|C_k^{xy}\|, \|C_k^y\|\leq \bar{c} <\infty ,
\end{align}
for all $k\geq 0$. Later, we show how the analysis of the simplified two-time-scale linear SA can be generalized.

We show \eqref{eq:C_k_bound_pf_sk} by induction. First, we show that it holds for some $k\geq 0$, and then we prove that it holds for $k+1$. 

Calculating the square and the cross product of the two recursions in \eqref{eq:two_time_scale_pf_sk}, and taking expectation, we have
\begin{align}
    X_{k+1}=&(1-\alpha_k)^2X_k+\alpha_k^2 U,\label{eq:tX_k}\\
    Z_{k+1}= &(1-\alpha_k)(1-\beta_k) Z_k + \beta_k\alpha_kW-\beta_k(1-\alpha_k)X_k \label{eq:tZ_k}\\
    Y_{k+1} =& (1-\beta_k)^2Y_k + \beta_k^2X_k+\beta_k^2V+ 2\beta_k(1-\beta_k) Z_k. \label{eq:tY_k}
\end{align}
Replacing $X_k, Z_k,$ and $Y_k$ with the values in \eqref{eq:pf_sketch_k} and using the upper bound \eqref{eq:C_k_bound_pf_sk}, we can show that $X_{k+1}, Z_{k+1}, Y_{k+1}$ can be written in the form of \eqref{eq:pf_sketch_k} with $\|C_{k+1}^x\|, \|C_{k+1}^{xy}\|, \|C_{k+1}^y\|\leq \bar{c} <\infty$

Notice that here we show that $Z_k$ behaves like $\mathcal{O}(\beta_k)$. This is indeed necessary to achieve the optimal rate $\mathcal{O}(\beta_k)$ for the convergence of $Y_k$. For a more detailed discussion of the convergence of $Z_k$, see Appendix \ref{sec:tZ_k_rate}. Alternatively, one could aim to study this recursion in a single step and analyze the recursion of a single Lyapunov function consisting of $X_k, Z_k,$ and $Y_k$. Although this approach has been considered before in the literature \cite{doan2021markov}, it is not clear how to achieve a tight convergence bound using a single Lyapunov function. In particular, \cite{doan2021markov} considers $\E\left[\|y_k\|^2+\beta_k\|x_k\|^2/\alpha_k\right]$ as the Lyapunov function, and studies its convergence bound. However, to handle the cross-term, the author uses the Cauchy-Schwarz inequality, which results in a loose inequality and a suboptimal convergence rate. Establishing a tight convergence bound using a single Lyapunov function is left as an open question for future research direction.

This forms the skeleton of our proof, and in Sections \ref{sec:pf_2}, \ref{sec:pf_3}, \ref{sec:pf_4}, and \ref{sec:pf_5}, we show how to relate the general two-time-scale recursion \eqref{eq:two_time_scale_main} to the simplified recursion in \eqref{eq:two_time_scale_pf_sk} by handling Markovian noise, vector-valued iterates, and interdependence between the iterates.

\subsection{Handling the Markovian noise}\label{sec:pf_2}

In the previous section, $v_k$ and $u_k$ were assumed to be zero-mean i.i.d., and the expected value of the cross term between noise and iterate was zero. However, in the Markovian noise setting, this is no longer true.

There are two approaches in the literature to handle Markovian noise in SA. The authors in \cite{bhandari2018finite} and \cite{srikant2019finite} used the geometric mixing property of the Markov chain to handle Markov noise. A classical approach to handle Markovian noise is based on the Poisson equation for Markov chains \cite{douc2018markov}, which converts Markovian noise to martingale noise along with other manageable terms. For ease of exposition, in this section, we present the use of the Poisson equation in a single time-scale setting as in \eqref{eq:tx_mm}. This machinery can be extended similarly to the general two-time-scale setting. Furthermore, we consider scalar iterations, since generalizing to the vector case is straightforward. Let $\{O_k\}_{k\geq 0}$ be a Markov chain that satisfies Assumption \ref{ass:poisson_main}. Let $a(O_k)$ and $b(O_k)$ be functions of the Markov chain with $\E_{O\sim \mu}[a(O)]=a>0$ and $\E_{O\sim \mu}[b(O)]=b$. Without loss of generality, we assume $b=0$, which implies that $x^*=0$. Now consider the following iteration,
\begin{align}
    x_{k+1}=&x_k-\alpha_k(a(O_k)x_k+b(O_k)).\label{eq:tx_mn}\\
    =&(1-a\alpha_k)x_k-\alpha_ku(x_k, O_k).\label{eq:tx_rec}
\end{align}
where $u(x_k, O_k)=(a(O_k)-a)x_k+b(O_k)$. Squaring both sides and taking expectation, we get,
\begin{align}\label{eq:recursion_X_k}
    X_{k+1}=\underbrace{(1-a\alpha_k)^2X_k}_{T_1}+\underbrace{\alpha_k^2\E[u^2(x_k, O_k)]}_{T_2}-2\alpha_k(1-a\alpha_k)\underbrace{\E[x_ku(x_k, O_k)]}_{T_3}
\end{align}
$T_1$ is similar to the first term in \eqref{eq:tX_k}. $T_2$ consists of two terms as $T_{21}=\E[b^2(O_k)]$ and $ T_{22}= \E[(a(O_k)-a)^2x_k^2 + 2(a(O_k)-a)b(O_k)x_k]$.  $T_{21}$ is the same as the second term in \eqref{eq:tX_k}, and for $T_{22}$ we use the induction assumption \ref{eq:pf_sketch_k}. The term $T_3$ was not present in \eqref{eq:tX_k} because it is equal to zero for the i.i.d. noise, but that is not the case for the Markovian noise. Thus, to obtain a handle for $T_3$, we use the framework of the Poisson equation.

For a given $x$, the set of equations, 
\begin{align}\label{poisson_eq}
    \hat{u}(x, o)=u(x, o)+\sum_{o'\in S}P(o'|o)\hat{u}(x, o'), \forall o\in S
\end{align}
are denoted as Poisson equation, and the function $\hat{u}(x, \cdot)$ that solves the Poisson equation is unique up to an additive factor.  We seek a unique solution and therefore impose the constraint $\sum_{o\in S}\mu(o)\hat{u}(x, o)=0$. Note that $\hat{u}(x,o)$ is Lipschitz with respect to $x$. For more details, refer to Lemma \ref{lem:possion_sol} in Appendix \ref{appendix:tech_lemmas}. The Poisson equation is the same as the Bellman equation for the average-reward Markov process (with rewards $u(x, \cdot)$), and its solution is the corresponding differential value function \cite{howard1960dynamic}. 

Substituting $u(\tx_k, O_k)$ in the cross-term in \eqref{eq:recursion_X_k}, we get, 
\begin{align*}
    \E[x_ku(x_k, O_k)]&=\E\left[x_k\left(\hat{u}(x_k, O_k)-\sum_{o\in S}P(o|O_k)\hat{u}(x_k, o)\right)\right]\\
    &=\E\left[x_k\left(\hat{u}(x_k, O_k)-\sum_{o\in S}P(o|O_{k-1})\hat{u}(x_k, o)+\sum_{o\in S}P(o|O_{k-1})\hat{u}(x_k, o)-\sum_{o\in S}P(o|O_k)\hat{u}(x_k, o)\right)\right].
\end{align*}
Define a sigma field $\mathcal{F}_k=\sigma(\{x_i, O_i\}_{0\leq i\leq k})$. Note that $\hat{u}(x_k, O_k)-\sum_{o\in S}P(o|O_{k-1})\hat{u}(x_k, o)$ is a martingale difference with respect to $\mathcal{F}_{k-1}$, which implies $\E\left[x_k(\hat{u}(x_k, O_k)-\sum_{o\in S}P(o|O_{k-1})\hat{u}(x_k, o)|\mathcal{F}_{k-1}\right]=0$. Thus, we have:
\begin{align*}
    \E[x_ku(x_k, O_k)]
    =& \E\left[ x_k\left( \sum_{o\in S}P(o|O_{k-1})\hat{u}(x_k, o)-\sum_{o \in S}P(o| O_k)\hat{u}(x_k, o)\right)\right]\\
    =&\E\left[ x_k\sum_{ o\in S}P(o|O_{k-1}) \hat{u}(x_k, o)\right]-\E\left[ x_{k+1}\sum_{o\in S}P(o|O_k)\hat{u}(x_{k+1}, o)\right]\\
    &+\E\left[(x_{k+1}-x_k)\sum_{o\in S}P(o|O_k)\hat{u}(x_k, o)\right]+\E\left[x_{k+1}\sum_{o\in S}P(o|O_k)(\hat{u}(x_{k+1}, o)-\hat{u}(x_k, o))\right]\\
    =&\underbrace{\E \left[x_k \sum_{o\in S}P(o| O_{k-1})\hat{u}(x_k, o)\right]-\E\left[x_{k+1}\sum_{o\in S}P(o|O_k)\hat{u}(x_{k+1}, o)\right]}_{T_{31}}\\
    & \underbrace{- \alpha_k \E\left[u(x_k, O_k)\sum_{o\in S}P(o|O_k)\hat{u}(x_k, o)\right]}_{T_{32}}\\
    &  \underbrace{-\alpha_k a \E\left[x_k\sum_{o\in S}P(o|O_k)\hat{u}(x_k, o)\right]}_{T_{33}}+\underbrace{\E\left[x_{k+1}\sum_{o\in S}P(o|O_k)(\hat{u}(x_{k+1}, o)-\hat{u}(x_k, o)))\right]}_{T_{34}}
\end{align*}

The term $T_{31}$ is of the telescopic form $d_k-d_{k+1}$. In order to incorporate this term in the one step recursion, we introduce a new variable $X_k' = X_k + 2\alpha_k d_k$, and we establish a recursion on the new variable $X_k'$. In this recursion, the telescopic $d_k-d_{k+1}$ term will be absorbed in $X_{k+1}'$ and $X_k'$ (up to some higher order terms). In general, the absorption of $d_k$ to $X_k$ and the introduction of the new variable $X_k'$ are how we handle the Markovian noise. Furthermore, the terms $T_{32}$, $T_{33}$, and $T_{34}$ also appear in the recursion of $X_k'$.  For $T_{32}$ we use Lemma \ref{lem:possion_sol} to substitute $\hat{u}(\cdot, \cdot)$ explicitly in terms of $u(\cdot, \cdot)$. After some algebraic manipulations, it can be shown that $T_{32}$ corresponds to the infinite sum in the expression for $\Gamma_x$ in Lemma \ref{prop:alternate_pe}. In $T_{33}$ we use the induction hypothesis \eqref{eq:pf_sketch_k} and show that this term is of higher order. Analyzing the final term $T_{34}$ efficiently is more subtle and will be discussed in the following.

\subsubsection{\texorpdfstring{Absolute upper bound to handle $T_{34}$}{Absolute upper bound to handle T34}}\label{sec:absolute_upper_bound}
First, in Lemma \ref{lem:boundedness} we establish an absolute constant upper bound on the mean square error of the iterates of the two-time-scale SA. Next, to upper bound $T_{34}$, we use the Lipschitz property of $\hat{u}(\cdot, \cdot)$ to show that $T_{34}= \mathcal{O}( \alpha_k\E[x_{k+1}(x_k+b(O_k))]) = \mathcal{O}(\alpha_k\E[x_k^2])+\mathcal{O}(\alpha_k^2 \E[x_k])$. For the first term we use the induction hypothesis, while for the second term we use the absolute upper bound in Lemma \ref{lem:boundedness}. Besides this, the recursion established in the proof of Lemma \ref{lem:boundedness} helps us in the proof of Proposition \ref{prop:loose_convergence}.

For the general setting of two-time-scale linear SA, a similar procedure is performed for $Z_k$ and $Y_k$, where we establish a recursion similar to \eqref{eq:recursion_X_k}. These recursions will consist of a leading term with infinite sums in the expression for $\Gamma_z$ and $\Gamma_y$, a telescopic term, and some higher-order terms. Then we introduce two new variables $Z_k'$ and $Y_k'$, and we show that the telescopic terms turn to some higher-order terms in the recursion of these new variables.

\subsection{Extension to high dimensional vectors}\label{sec:pf_3}
The second difference of the recursion in \eqref{eq:two_time_scale_pf_sk}  compared to the original two-time-scale recursion is in the scalar versus vector variables. To accommodate the vector variables, we take the expectation of the outer product of the variables as Lyapunov functions. For example, for the cross term, we take $Z_k=\E[x_ky_k^\top]$, and we establish Eq. \eqref{eq:tZ_k} in terms of matrices. At first glance, it might be tempting to use the inner product as a Lyapunov function. However, to establish a recursion for the inner product, we need to employ the Cauchy-Schwartz inequliaty for the cross-term, which does not achieve a tight convergence bound.  In particular, the outer product results in a recursion of the form $Z_{k+1}=(I-A_{22}\alpha_k)Z_k(I-\Delta^\top\beta_k) + \mathcal{O}(\alpha_k\beta_k)$. However, an attempt to establish a recursion for the inner product results in $\E[x_{k+1}^\top y_{k+1}]= x_k^\top (I-\alpha_kA_{22})^\top(I-\beta_k\Delta)y_k+\mathcal{O}(\alpha_k\beta_k)$. Unfortunately, this relation cannot be translated into a one-step recursion, since there does not exist any matrix property that relates $x^TAy$ to $x^Ty$.

We would like to point out that in the special case of SA with Polyak-Ruppert averaging, as considered in \cite{moulines2011non}, inner product can be used to establish a tight convergence bound. However, in the general two-time-scale SA, the special structure of the Polyak-Ruppert averaging does not exist, and it appears that the use of the outer-product for establishing a tight convergence bound is necessary.
\subsubsection{Dealing with Asymmetric Matrices}\label{sec:pf_5}
In the most general setting of two-time-scale linear SA, the vector-valued parameters are multiplied by (potentially asymmetric) matrices. To deal with asymetry, we use the Lyapunov equation. To observe this, assume  the vector valued variant of the recursion \eqref{eq:tx_mm} as $x_{k+1}=(I-\alpha_kA)x_k+\alpha_ku_k$, where the matrix $-A$ is assumed to be Hurwitz (not necessarily symmetric). The matrix $X_k=\E[x_kx_k^\top]$ satisfies the following recursion: $X_{k+1}=(I-\alpha_kA)X_k(I-\alpha_kA)^\top + \alpha_k^2U$. Then we can show that $X_k=\alpha_k\Sigma + o(\alpha_k)$, where $\Sigma$ satisfies the Lyapunov function $A\Sigma+\Sigma A^\top = -U$. By extending this approach, we can study the general two-time-scale SA with asymmetric matrices.

\subsection{Handling intertwined relation between variables}\label{sec:pf_4}

The third difference is the independence of the recursion of $x_k$ from $y_k$ in \eqref{eq:two_time_scale_pf_sk}, while we observe that in \eqref{eq:two_time_scale_main} these variables are intertwined. It is well known that SA algorithms can be studied as discretizations of ordinary differential equations (ODEs) whereas two-time-scale SA algorithms are discretizations of two ODEs \cite{borkar1997stochastic,borkar2009stochastic} of the form, 
\begin{subequations}\label{eq:ttode}
    \begin{align}
         \dot{y}=A_{11}y+ A_{12}x\label{eq:ttode_y}\\
        \varepsilon\dot{x}= A_{21}y+A_{22}x,\label{eq:ttode_x}
    \end{align}
\end{subequations}
where $x=x(t)$ and $y=y(t)$ are functions of continuous time  $t$. Here, the parameter $\varepsilon$ can be used to model different time-scales in \eqref{eq:ttode_y} and \eqref{eq:ttode_x}. When $\varepsilon$ is small, \eqref{eq:ttode_x} operators on a faster time-scale than \eqref{eq:ttode_x}, and as $\varepsilon$ goes to zero, $x$ converges to its equilibrium instantly. In the context of two-time-scale SA \eqref{eq:two_time_scale_main}, $\varepsilon$ can be thought of as the ratio of two time scales, i.e., ratio of step-sizes $\beta_k/\alpha_k$. 

In order to study the convergence of \eqref{eq:ttode}, \cite{kokotovic84} have shown that there exists a linear transformation $\tilde{x}=x+M_{\varepsilon}y$ such that the system \eqref{eq:ttode} transforms into block-triangular form:
\begin{align*}
    \dot{y}&=(A_{11}+BM_{\varepsilon})y+A_{12}\tilde{x}\\
    \varepsilon\dot{\tilde{x}}&=(A_{22}+\varepsilon M_{\varepsilon} A_{12})\tilde{x},
\end{align*}
where, $M_{\varepsilon}$ is the solution of the Ricatti equation $A_{22}M_{\varepsilon}-\varepsilon M_{\varepsilon} A_{11}+\varepsilon M_{\varepsilon}A_{12}M_{\varepsilon}-A_{21}=0$. This equation helps us to disentangle the variables in \eqref{eq:ttode}.  From singular perturbation theory \cite{kokotovic84}, it is known that $M_\varepsilon \to A_{22}^{-1} A_{21}$ as $\varepsilon\to 0$.

A slight modification of a similar logic can be applied to disentangle the variables of the two-time-scale SA \eqref{eq:two_time_scale_main}. Since the two-time-scale SA \eqref{eq:two_time_scale_main} uses time-varying step sizes, this corresponds to having a time-varying $\varepsilon$ parameter in the ODE. Therefore, to disentangle the variables in \eqref{eq:two_time_scale_main}, \cite{konda2004convergence} proposed a time-varying bijective linear transformation $M_k$ that is inspired by the Ricatti equation
\begin{align}\label{eq:lin_trans}
\begin{bmatrix}
x_k\\y_k
\end{bmatrix}
  \xleftrightarrow{}
  \begin{bmatrix}
\tx_k=x_k+M_ky_k\\
\ty_k=y_k    
\end{bmatrix}.
\end{align}
In Lemma \ref{lem:L_k_bound} it is shown that $M_k$ can be written as $M_k=L_k+A_{22}^{-1}A_{21}$ where the matrices $L_k$  are deterministic and are recursively defined in
Eq. \eqref{eq:L_k22}. Furthermore, it can be shown that $L_k\to 0$ as $k\to\infty$. Therefore, $M_\varepsilon$ and $M_k$ have similar asymptotic converging points. To handle the intertwined updates \eqref{eq:two_time_scale_main}, in our analysis we use the linear transformation \eqref{eq:lin_trans} to disentangle the variables. Once the convergence bounds of the disentangled variables $\tx_x$ and $\ty_k$ are achieved, they are translated back to the intertwined variables using the transformation \eqref{eq:lin_trans}.

\section{Conclusion and Future Directions}\label{sec:conc}

In this work, we analyzed linear two-time-scale stochastic approximation (SA) under Markovian noise and established tight finite-time convergence bounds for the covariance of the iterates. Our results characterize the dependence of the mean squared error on key hyperparameters, particularly the step sizes, under a natural set of assumptions. We further demonstrated—both theoretically and empirically—that these assumptions are minimal for the convergence guarantees to hold. In addition, our analysis provides principled guidance for choosing step sizes to optimize performance.

A notable application of our results is to Polyak-Ruppert averaging, where we showed that it achieves the optimal convergence rate in a robust manner, even under Markovian noise. We also applied our framework to key reinforcement learning algorithms—TDC, GTD, and GTD2—establishing the convergence bound of $\sigma^2/k + \mathcal{O}(d^3)o(1/k)$, where $\sigma^2$ is the covariance of the CLT of the corresponding algorithm.

This work opens several promising directions for future research. First, while tight convergence bounds for non-linear operators under Polyak-Ruppert averaging are known in the i.i.d. setting \cite{moulines2011non}, extending such results to general non-linear operators under Markovian noise remains an important challenge. This could lead to new insights into the sample complexity of algorithms such as Watkins' $Q$-learning \cite{watkins1989learning} and Zap $Q$-learning \cite{devraj2017zap} with averaging. Another direction is to further reduce the dimension dependence in the higher-order terms through refined step-size selection. Identifying step-size schemes that minimize dimensional dependencies while preserving tight bounds is a valuable avenue for both theory and practice.

\bibliographystyle{alpha}
\bibliography{refs}
\pagebreak
\appendix
\appendixpage
\section{Convergence analysis of the cross term in the proof sketch }\label{sec:tZ_k_rate}

In this section, we explain the significant role that $\tZ_k$ plays in determining the convergence rate of the iterates. In addition, the convergence behavior of the cross-term $\tZ_k$ will also be discussed.

\subsection{Importance of the cross term}\label{sec:rec_sol}

First, we emphasize that it is critical to establish a tight bound on the convergence of the cross term. Let $a_k=o(1)$ and $b_k=o(1)$ and consider a recursion of the form
\begin{align*}
    V_{k+1}=(1-a_k)V_k+a_kb_k.
\end{align*}
If the sequence $\{a_k\}$ goes to zero at a sufficiently slow rate, then we can show that $V_k\leq \mathcal{O}(b_k)$. 

Next, as shown in \eqref{eq:tY_k}, we have $\tY_{k+1} = (1-\beta_k)\tY_k + \beta_k^2V+ 2\beta_k \E[\tx_k\ty_k] + o(\beta_k^2)$. Hence, the convergence rate of $\tY_k$ is $\mathcal{O}(\beta_k+\E[\tx_k\ty_k])$. As a result, to achieve $\mathcal{O}(\beta_k)$ convergence rate for $\tY_k$, it is essential to show that $\E[\tx_k\ty_k] = \mathcal{O}(\beta_k)$.

\subsection{Studying a special case}

Consider two random variables $(x_k,y_k)$ that are updated as follows
\begin{align*}
\begin{cases}
    x_{k+1}&=x_k+\alpha_k(-x_k+w_k)=(1-\alpha_k)x_k+\alpha_kw_k\\
    y_{k+1}&=y_k+\frac{1}{k+1}(x_k-y_k)=(1-\frac{1}{k+1})y_k+\frac{1}{k+1}x_k=\frac{1}{k+1}\sum_{i=0}^kx_i.
\end{cases}
\end{align*}
Here we assume $w_k$ to be an i.i.d. noise with zero mean  and variance $\E[w_k^2]=\sigma^2$. Observe that since the value of $x_{k+1}$ depends only on $x_k$ and $w_k$, $\{x_i\}_{i\geq 0}$ is a (time-varying) continuous state space Markov chain. However, in the special case of constant step size, $\{x_i\}_{i\geq 0}$ is a time-homogeneous Markov chain.

Since $\{x_i\}_{i\geq 0}$ is a Markov chain, $y_k$ can be viewed as averaging of the Markov random variables. In this section, our goal is to study the variance of $y_k$. Unlike the i.i.d. case where variance of average just depends on variance of each term, in a Markovian setting, the cross-covariance between the random variables also shows up in the variance of the average. Mathematically,
\begin{align*}
    \E[y_{k+1}^2]=\frac{1}{(k+1)^2}\sum_{i=0}^k\E[x_i^2]+\frac{2}{(k+1)^2}\sum_{i=0}^k\sum_{j=i+1}^k\underbrace{\E[x_ix_j]}_{\neq 0}.
\end{align*}
This shows that in the Markovian SA establishing the optimal convergence of the iterates requires a precise analysis of the cross term.

Next, we take an indirect approach to obtain the variance of $y_k$. Rewriting $\E[y_k^2]$ in a recursive manner, we have:
\begin{align}\label{eq:main_rec}
    \nonumber\E[y_{k+1}^2]&=\left(1-\frac{1}{k+1}\right)^2\E[y_k^2]+\frac{1}{(k+1)^2}\E[x_k^2]+\frac{2}{k+1}\left(1-\frac{1}{k+1}\right)\E[y_kx_k]\\
    &\approx \left(1-\frac{2}{k+1}\right)\E[y_k^2]+\frac{1}{(k+1)^2}\E[x_k^2]+\frac{2}{k+1}\E[y_kx_k],
\end{align}
where in the last line we assume $k$ large enough so that $\frac{1}{k+1}<< 1$. Rewriting the cross term, we have
\begin{align}\label{eq:cross_term}
    \E[y_{k+1}x_{k+1}]=\frac{1}{k+1}\sum_{i=0}^k\E[x_ix_{k+1}].
\end{align}
For each $i<k$, we have $x_{k+1}=(1-\alpha_k)x_k+\alpha_kw_k = (1-\alpha_k)(1-\alpha_{k-1})x_{k-1}+\alpha_kw_k+(1-\alpha_k)\alpha_{k-1}w_{k-1}=\dots = (\prod_{j=i}^k(1-\alpha_j))x_i + \sum_{j=i}^k\alpha_jw_j\Pi_{l=j+1}^k(1-\alpha_l)$. Inserting it in \eqref{eq:cross_term} we get
\begin{align}
    \E[y_{k+1}x_{k+1}]&=\frac{1}{k+1}\sum_{i=0}^k\E\left[x_i^2\left(\prod_{j=i}^k(1-\alpha_j)\right)\right]\nonumber\\
    &=\frac{1}{k+1}\sum_{i=0}^k\left(\prod_{j=i}^k(1-\alpha_j)\right)\E[x_i^2],\label{eq:Ex^2}
\end{align}
where the term corresponding to the noise is zero in expectation as we assumed $w_k$ is i.i.d zero mean. Solving the recursion on $x_k$, it is easy to see that $\E[x_i^2]\approx\frac{\sigma^2}{2}\alpha_i$. Replacing this in \eqref{eq:Ex^2} we get:
\begin{align*}
    \E[y_{k+1}x_{k+1}]&\approx\frac{1}{k+1}\frac{\sigma^2}{2}\sum_{i=0}^k\left(\prod_{j=i}^k(1-\alpha_j)\right)\alpha_i.
\end{align*}

Next, we show how \eqref{eq:main_rec} can be analyzed under different step sizes.
\begin{itemize}
    \item Let $\alpha_i=\alpha<1$. We have
\begin{align*}
    \E[y_{k+1}x_{k+1}]&\approx\frac{\alpha}{k+1}\frac{\sigma^2}{2}\underbrace{\sum_{i=0}^k(1-\alpha)^{k-i+1}}_{\text{geometric sum}}.
\end{align*}
Replacing this in \eqref{eq:main_rec}, we get:
\begin{align*}
    \E[y_{k+1}^2]&\approx (1-\frac{2}{k+1})\E[y_k^2]+\underbrace{\frac{\alpha}{(k+1)^2}\frac{\sigma^2}{2}}_{\text{variance term}}+\underbrace{\frac{\sigma^2\alpha}{(k+1)^2}\sum_{i=0}^k(1-\alpha)^{k-i+1}.}_{\text{cross-covariance term}}
\end{align*}
After solving the recursion for large enough $k$ we get
\begin{align}
   \Rightarrow \E[y_k^2]&\approx \frac{\alpha\sigma^2/2+\alpha\sigma^2\sum_{i=0}^\infty(1-\alpha)^{i+1}}{k}.\label{eq:var_of_y_final}
\end{align}
The geometric sum in \eqref{eq:var_of_y_final} corresponds to the infinite sum of cross-covariance terms in the expression for $\Gamma^y$ in Proposition \ref{prop:alternate_pe}. 

In addition, for function $f(\cdot)$ and a Markov chain $\{X^t\}_{t\geq0}$, \cite[Lemma 3]{mou2024heavy} establishes asymptotic variance of $\frac{f(X^1)+f(X^2)+\dots + f(X^m)}{m}$ as $m$ goes to infinity. At first look, one might expect that this asymptotic variance depends only on the variance of $f(\tilde{X})$, where $\tilde{X}$ follows the stationary distribution of the Markov chain. However, as shown in \cite[Lemma 3]{mou2024heavy}, this asymptotic variance has two terms, one corresponding to the variance of $f(\tilde{X})$ and the other corresponding to the auto covariance of $\{f(X^i)\}_{i\geq0}$. These two terms correspond to $\alpha\sigma^2/2$ and $\alpha\sigma^2\sum_{i=0}^\infty(1-\alpha)^{i+1}$ in \eqref{eq:var_of_y_final}, respectively. 

\item Let $\alpha_i=\frac{\alpha}{(i+1)^\xi}, 0<\xi<1$. We have:
\begin{align}
    \E[y_{k+1}x_{k+1}]&\approx\frac{1}{k+1}\frac{\sigma^2}{2}\sum_{i=0}^k\left(\prod_{j=i}^k(1-\alpha_j)\right)\alpha_i\nonumber\\
    &=\frac{1}{k+1}\frac{\sigma^2}{2}\left(1-\prod_{j=0}^k(1-\alpha_j)\right),\label{eq:inner_prod_cons_step_size}
\end{align}
where in the last equality we used the fact that $\sum_{i=0}^k\left(\prod_{j=i}^k(1-\alpha_j)\right)\alpha_i+\prod_{j=0}^k(1-\alpha_j)=1$. Replacing \eqref{eq:inner_prod_cons_step_size} in \eqref{eq:main_rec}, we get:
\begin{align*}
    \E[y_{k+1}^2]&\approx (1-\frac{2}{k+1})\E[y_k^2]+\frac{\alpha\sigma^2}{2(k+1)^2}+\frac{\sigma^2}{(k+1)^2}\bigg(1-\underbrace{\prod_{j=0}^k(1-\alpha_j)\bigg)}_{=O(e^{-k^{1-\xi}})}\\
    &\approx (1-\frac{2}{k+1})\E[y_k^2]+\frac{\sigma^2}{(k+1)^2}+\mathcal{O}\left(\frac{1}{(k+1)^2}\right).
\end{align*}
Solving the recursion gives us $\E[y_{k}^2]\approx \frac{\sigma^2}{k}$.
\end{itemize}

\section{Notation and Assumptions}
\textbf{Note:} Throughout the proof, any $c_{\cdot}$ (such as $c$ or $c_2$), indicates a problem-dependent constant. Furthermore, unless otherwise stated, $\|\cdot\|$ denotes the Euclidean 2-norm. Also, $\|\cdot\|_Q$ and $\langle\cdot,\cdot\rangle_Q$ denote the $Q$ weighted norm and inner product, i.e. $\langle x,y\rangle_Q = x^\top Q y$ and  $\|x\|_Q=\sqrt{\langle x,x\rangle_Q}$. 

We consider the following two-time-scale linear stochastic approximation with multiplicative noise:
\begin{equation}
    \begin{aligned}\label{eq:two_time_scale}
        y_{k+1}&=y_k+ \beta_k(b_1(O_{k})-A_{11}(O_{k})y_k-A_{12}(O_{k})x_k)\\
        x_{k+1}&= x_k+\alpha_k(b_2 (O_{k})-A_{21}(O_{k})y_k-A_{22}(O_{k})x_k),
    \end{aligned}
\end{equation}
\begin{align}
    x_{k+1}-x^*=x_k-x^*+\alpha_k(b-A_{22}(x_k-x^*)+b_2(O_k)-b+(A_{22}(O_k)-A_{22})x_k)
\end{align}
Without loss of generality, throughout the proof we assume $b_1=0$ and $b_2=0$. Note that this can be done simply by centering the variables as $x_k\rightarrow x_k-x^*$ and $y_k\rightarrow y_k-y^*$.

\begin{definition}
    Denote $\{\Tilde{O}_k\}_{k\geq 0}$ as a Markov chain with the starting distribution as the stationary distribution of $\{O_k\}_{k\geq0}$. 
    \begin{align}
    \Gamma_{11}=\E[b_1(\Tilde{O}_k)b_1(\Tilde{O}_k)^\top];~~~\Gamma_{21}^\top=\Gamma_{12}=\E[b_1(\Tilde{O}_k)b_2(\Tilde{O}_k)^\top];~~~\Gamma_{22}=\E[b_2(\Tilde{O}_k)b_2(\Tilde{O}_k)^\top];~~~
    \end{align}
\end{definition}
\begin{definition}
    Define $\mathbb{E}_O[f(\cdot)]=\sum_{\cdot\in S} P(\cdot|O)f(\cdot)$
\end{definition}
By Assumption \ref{ass:poisson_main}, and \cite[Theorem 22.1.8]{douc2018markov}, we know that there exist $\rho\in (0,1)$ which satisfies \\$\max_{o}d_{TV}(P^k(\cdot|o)||\mu(\cdot)))\leq  \rho^k$, where $d_{TV}(p(\cdot)||q(\cdot))=\frac{1}{2}\int|p(x)-q(x)|dx$. Furthermore, we define the mixing time of the Markov chain $\{O_k\}_{k\geq 0}$ with the transition probability $P(\cdot|\cdot)$ as $\tau_{mix} = \min_n \{n: \max_o d_{TV}(P^n(\cdot|o)||\mu(\cdot))\leq1/4\}$.
\begin{definition}\label{def_noise}
    Let 
    \begin{align*}
        f_1(O, x, y)&=b_1(O)-(A_{11}(O)-A_{11})y-(A_{12}(O)-A_{12})x\\
        f_2(O, x, y)&=b_2(O)-(A_{21}(O)-A_{21})y-(A_{22}(O)-A_{22})x
    \end{align*}
\end{definition}
Throughout the proof, for the ease of notation we will denote $f_1(O_k ,x_k,y_k)\equiv v_k$ and $f_2(O_k ,x_k,y_k)\equiv w_k$.
\begin{remark}\label{rem:pois_eq_geo_mix}By Assumption \ref{ass:poisson_main}, 
     there exist functions $\hat{f}_i,~i\in\{1,2\}$ that are solutions to the following Poisson equations, i.e. \cite[Proposition 21.2.3]{douc2018markov}
\begin{align}\label{eq:poisson_eq_f_i}
    \hat{f}_i(o, x, y)&=f_i(o, x, y)+\sum_{o'\in S} P(o'|o)\hat{f}_i(o', x, y). 
\end{align}
Furthermore, the assumption \ref{ass:poisson_main} shows that the Markov chain $\{O_k\}_{k\geq 0}$ has a geometric mixing time. 
\end{remark}

Before stating the lemmas, we present the following definitions which will be used within the proof of the lemmas.

    Throughout the proof of Theorem \ref{thm:Markovian_main}, we define the matrix $Q_{\Delta,\beta}$ and $q_{\Delta,\beta}$ according to Definition \ref{def:Q_delta}.  
    \begin{definition}\label{def:Q_delta}Define $Q_{\Delta,\beta}$ as the solution to the following Lyapunov equation:
\begin{align}
    \left(\Delta-\frac{\beta^{-1}}{2}I\right)^\top Q_{\Delta,\beta}+Q_{\Delta,\beta}\left(\Delta-\frac{\beta^{-1}}{2}I\right) = I. \label{eq:Q_delta}
\end{align}
Furthermore, we denote $q_{\Delta,\beta} = \frac{\beta\|Q_{\Delta,\beta}\|^{-1}}{4+\beta\|Q_{\Delta,\beta}\|^{-1}}$.
Note that due to the Assumption \ref{ass:hurwitz_main}, Eq. \eqref{eq:Q_delta} always has a unique positive-definite solution. 
\end{definition}
In the proof of Theorem \ref{thm:Markovian_main} we take $\varrho$ such that $q_{\Delta,\beta} = 1-\varrho$. Although in our proof we use this special case of $\varrho$, the extension of our result to the general $\varrho$ is straightforward. 

\begin{definition}\label{def_list}
    Define 
    \begin{align*}
        X_k&= \E[x_kx_k^\top];~~Z_k= \E[x_ky_k^\top];~~Y_k=\E[y_ky_k^\top];\\
        V_k&=\E[\|x_k\|_{Q_{22}}^2];~~W_k=\E[\|y_k\|_{Q_{\Delta}}^2~~U_k=V_k+W_k;\\
        \xh_k&=x_k+A_{22}^{-1}A_{21}y_k;~~\tx_k=L_ky_k+\xh_k;~~\yh_k=\ty_k=y_k;\tag{where $L_k$ is defined in Eq. \eqref{eq:L_k22}}\\
        \tX_k&=\E[\tx_k\tx_{k}^\top];~~\tZ_k=\E[\tx_k\ty_k^\top];~~\tY_k=\E[\ty_k\ty_{k}^\top];\\
        d_k^{xv}= \E\bigg[\Big(\E_{O_{k-1}}&\fh_1(\cdot,x_k,y_k)\Big)\tx_k^\top\bigg];~~d_k^{xw}= \E\bigg[\left(\E_{O_{k-1}}\fh_2(\cdot,x_k,y_k)\right)\tx_k^\top\bigg];~~d_k^x= d_k^{xw}+\frac{\beta_k}{\alpha_k}(L_{k+1}+A_{22}^{-1}A_{21})d_k^{xv};\\
        d_k^{yv}=\E\bigg[\Big(\E_{O_{k-1}}&\fh_1(\cdot,x_k,y_k)\Big)\ty_k^\top\bigg];~~d_k^{yw}=\E\bigg[\left(\E_{O_{k-1}}\fh_2(\cdot,x_k,y_k)\right)\ty_k^\top\bigg];~~d_k^y=d_k^{yw}+\frac{\beta_k}{\alpha_k}(L_{k+1}+A_{22}^{-1}A_{21})d_k^{yv};\\
        \tX'_k&=\tX_k+\alpha_k (d_k^x+{d_k^x}^\top);~~\tZ_k'= \tZ_k+\alpha_kd_k^{y}+\beta_k{d_k^{xv}}^\top;~~\tY_k' = \tY_k + \beta_k (d_k^{yv}+{d_k^{yv}}^\top);\\
        \zeta_k^x&=\frac{1}{(k+K_0)^{\min\{1.5\xi, 1\}}};~~\zeta_k^{xy}=\frac{1}{(k+K_0)^{\min\{\xi+0.5, 2-\xi\}}};~~\zeta_k^y=\frac{1}{(k+K_0)^{1+q_{\Delta, \beta}\min\{\xi-0.5, 1-\xi\}}};\\
         u_k&=w_k+\frac{\beta_k}{\alpha_k}(L_{k+1}+A_{22}^{-1}A_{21})v_k\\
         F^{(i,j)}(O',O,x,y)&=\left( \fh_i(O',x,y)\right)(f_j(O,x,y))^\top ~~ \text{for}~~ i,j\in\{1,2\};\\
        I&=A_{22}^\top Q_{22}+Q_{22}A_{22}\tag{$Q_{22}$ is the unique solution to this equation};\\
        I&=\Delta^\top Q_{\Delta}+Q_{\Delta}\Delta;\tag{$Q_{\Delta}$ is the unique solution to this equation}\\
        a_{22}&=\frac{1}{2\|Q_{22}\|};~~\delta=\frac{1}{2\|Q_{\Delta}\|};\\
        C_i(O)&=\sum_{k=0}^\infty \E{[b_i(O_k)|O_o=O]};\\
        C_{ij}(O)&=\left(\sum_{k=0}^\infty \E{[A_{ij}(O_k)-A_{ij}|O_0=O]}\right);\\
        C_{22}^k&=\frac{\beta_k}{\alpha_k}\left(L_{k+1}+A_{22}^{-1}A_{21}\right)A_{12};\\
        k_C&=\min\bigg\{k: \frac{\alpha}{(k+K_0)^\xi}\leq \frac{1}{2\|Q_{22}\|\|A\|_{Q_{22}}^2}, \frac{\beta}{k+K_0}\leq \frac{1}{2\|Q_{\delta}\|\|\Delta\|^2_{Q_{\Delta}}},\\
        & \frac{8\alpha\max\left\{b_{max}\sqrt{\gamma_{max}(Q_{22})} , \frac{\check{h}_3}{2}\right\} }{(1-\rho)(k+K_0)^\xi}\leq 0.3,\frac{8\beta \max\left\{b_{max}\sqrt{\gamma_{max}(Q_{22})} , \frac{\check{h}_4}{2}\right\}}{(1-\rho)(k+K_0)}\leq 0.3,\\
        &\frac{2\|A_{12}\|^2_{Q_{\Delta}}\gamma_{max}(Q_{\Delta})}{\gamma_{min}(Q_{22})\delta}\frac{\beta}{k+K_0}+\frac{\brc_3\beta^2}{\alpha(k+K_0)^{2-\xi}}\leq \frac{a_{22}\alpha}{4(k+K_0)^\xi}, \frac{\brc_3\beta}{\alpha(k+K_0)^\xi}\leq \frac{\delta}{4}\bigg\};
        \tag{where $\{\check{h}_i\}_{i}$ and $\brc_3$ are defined in Lemma \ref{lem:f_hat_bound} and Eq. \eqref{c_3def}, respectively}\\
        k_L&=\min\bigg\{k: \beta_k\leq \frac{\sqrt{\gamma_{\min}(Q_{22})}}{2(\|\Delta\|_{Q_{22}}+\|A_{12}\|_{Q_{22}})\sqrt{\gamma_{\max}(Q_{22})}},\\
        &~~~~~~\frac{\beta_k}{\alpha_k}\leq \frac{a_{22}/2}{(\|A_{22}^{-1}A_{21}\|_{Q_{22}}+1)(\|\Delta\|_{Q_{22}}+\|A_{12}\|_{Q_{22}})}~\forall k\geq k_C\bigg\};\\
        k_1&=\min\left\{k:\frac{a_{22}}{2}\geq \frac{1-\xi}{\alpha(k+K_0)^{1-\xi}}~\forall k\geq k_L\right\};\\
        d&=\max\{d_x,d_y\};\\
        b_{max}& = \max_{j\in\{1,2\}}\max_{o'\in \mathcal{S}} \left|b_{j}^{(i)}(o') \right|, \text{where } b_{j}^{(i)}(o') \text{is the } i'th \text{element of the vector }b_{j}(o');\\
        \kappa_{Q_{22}}& =\frac{\sqrt{\gamma_{max}(Q_{22})}}{\sqrt{\gamma_{min}(Q_{22})}};~~\kappa_{Q_{\Delta}} =\frac{\sqrt{\gamma_{max}(Q_{\Delta})}}{\sqrt{\gamma_{min}(Q_{\Delta})}};~~\kappa_{Q_{\Delta, \beta}}=\frac{\sqrt{\gamma_{max}(Q_{\Delta, \beta})}}{\sqrt{\gamma_{min}(Q_{\Delta, \beta})}};\\
        A_{max}&=\max_{o\in \mathcal{S}}\left\{\max_{i,j\in \{1,2\}}\{\|A_{ij}(o)\|\}\right\};~~\varrho_{x} =\kappa_{Q_{22}}+\|A_{22}^{-1}A_{21}\|;~~\varrho_y= \|\Delta\|+\|A_{12}\|\kappa_{Q_{22}}.
    \end{align*}
\end{definition}

In this paper, our aim is to establish the dependency of the second order term in terms of the dimension of the variables $x_k$ and $y_k$. For doing so, we will keep track of all the constants in the paper which we assume to be independent of the dimension. Specifically, we will assume that matrix operator norms and eigenvalues of various matrices do not scale with the dimension. For example, the following constants are assumed to be dimension independent: $a_{22}, \|A_{22}^{-1} A_{21}\|_{Q_{22}},\|\Delta\|_{Q_{22}}, \|A_{12}\|_{Q_{22}}, \sqrt{\gamma_{\min}(Q_{22})}, \sqrt{\gamma_{\max}(Q_{22})}$, etc. Also, note that the mixing constant $\rho$ may also contribute to the dimensional dependence, but we do not study that here.

Before starting the proof of the main results, we will state some properties of the matrix $\Sigma^x, \Sigma^{xy}$ and $\Sigma^y$ given in Eqs. \eqref{eq:sigma_x_def_main}-\eqref{eq:sigma_y_def_main} which will be used extensively. Firstly, observe that $\|b_i(o)\|\leq b_{max}\sqrt{d},~i\in\{1,2\}$. Let $U_1J_1U_1^{-1}=A_{22}$ be the Jordan canonical decomposition of $A_{22}$. The Lyapunov equation is given by:
\begin{align*}
    A_{22}\Sigma^x+\Sigma^xA_{22}^\top=\Gamma^x
\end{align*}
where $\Gamma^x=\E{[\bt_2(\tilde{O}_0)\bt_2(\tilde{O}_0)^\top]}+\sum_{j=1}^\infty \E{[\bt_2(\tilde{O}_j)\bt_2(\tilde{O}_0)^\top + \bt_2(\tilde{O}_0)\bt_2(\tilde{O}_j)^\top]}$. Then, we have
\begin{align*}
    \|\Gamma^x\|&\leq \|\E{[\bt_2(\tilde{O}_0)\bt_2(\tilde{O}_0)^\top]}\|+\|\sum_{j=1}^\infty \E{[\bt_2(\tilde{O}_j)\bt_2(\tilde{O}_0)^\top + \bt_2(\tilde{O}_0)\bt_2(\tilde{O}_j)^\top]\|}\\
    &\leq \|\E{[\bt_2(\tilde{O}_0)\bt_2(\tilde{O}_0)^\top]}\|+2\|\sum_{j=1}^\infty \E{[\bt_2(\tilde{O}_j)\bt_2(\tilde{O}_0)^\top\|}\\
    &\leq b_{max}^2d+\frac{8\tau_{mix}}{3}b_{max}^2d\tag{Lemma \ref{lem:mix_time_sum}}\\
    &\leq 4b_{max}^2d\tau_{mix}\tag{$\tau_{mix}\geq 1$}
\end{align*}
Define the following:
\begin{align*}
    \sigma^x= 4b^2_{max}\|U_1\|\|U_1^{-1}\|\sum_{n,n'=0}^{m_{A_{22}}}  {n+n'\choose k}\frac{1}{(-2r_{A_{22}})^{n+n'+1}},
\end{align*}
where $m_{A_{22}}$ is the largest algebraic multiplicity of the matrix $A_{22}$ and $r_{A_{22}}=\max_i\mathfrak{Re}[\lambda_i]$, where $\lambda_i$ is the $i$-th eigen value. Then, using Lemma \ref{lem:Lyap_bound}, we have $\|\Sigma^x\|\leq \sigma^xd\tau_{mix}$.

\begin{align*}
    A_{12}\Sigma^x+\Sigma^{xy}A_{22}^\top=\Gamma^{xy}
\end{align*}
where $\Gamma^{xy}=\E{[\bt_2(\tilde{O}_0) \bt_1(\tilde{O}_0)^\top]}+ \sum_{j=1}^\infty \E{[\bt_2(\tilde{O}_j) \bt_1(\tilde{O}_0)^\top + \bt_2(\tilde{O}_0)\bt_1(\tilde{O}_j)^\top]}$. Similar to bounding to $\Gamma^x$, we get $\|\Gamma^{xy}\|\leq 4b_{max}^2d\tau_{mix}$. Define $\sigma^{xy}=\|A_{22}^{-1}\|\left(4b_{max}^2+\|A_{12}\|\sigma^x\right)$. Thus, we have
\begin{align*}
    \|\Sigma^{xy}\|\leq \sigma^{xy}d\tau_{mix}.
\end{align*}

Finally, note that $\Sigma^y$ is the solution to the following Lyapunov equation:
\begin{align*}
    \left(\Delta-\frac{\beta^{-1}I}{2}\right)\Sigma^y+\Sigma^y\left(\Delta^\top-\frac{\beta^{-1}I}{2}\right)=\Gamma^y-A_{12}\Sigma^{yx}-\Sigma^{xy}A_{12}^\top.
\end{align*}
Similar to bounding to $\Gamma^x$, we get $\|\Gamma^{y}\|\leq 4b_{max}^2d\tau_{mix}$. From the previous bounds we can bound the norm of the r.h.s as follows:
\begin{align*}
    \|\Gamma^y-A_{12}\Sigma^{yx}-\Sigma^{xy}A_{12}^\top\|&\leq \|\Gamma^y\|+2\|A_{12}\|\|\Sigma^{yx}\|\\
    &\leq b_{max}^2d+\frac{8\tau_{mix}}{3}b_{max}^2d+2\|A_{12}\|\|A_{22}^{-1}\|\left(4b_{max}^2+\sigma^x\right)d\tau_{mix}\\
    &\leq \left(4b_{max}^2+2\|A_{12}\|\sigma^{xy}\right)d\tau_{mix}
\end{align*}

Assume $U_2J_2U_2^{-1}=\Delta-\frac{\beta^{-1}I}{2}$ to be Jordan canonical decomposition of $\Delta-\frac{\beta^{-1}I}{2}$.
\begin{align*}
    \sigma^y=\left(4b_{max}^2+2\|A_{12}\|\sigma^{xy}\right)\|U_2\|\|U_2^{-1}\|\sum_{n,n'=0}^{m_{\Delta, \beta}}  {n+n'\choose k}\frac{1}{(-2r_{\Delta, \beta})^{n+n'+1}},
\end{align*}
where $m_{\Delta, \beta}$ is the largest algebraic multiplicity of the matrix $A$ and $r_{\Delta, \beta}=\max_i\mathfrak{Re}[\lambda_i]$, where $\lambda_i$ is the $i$-th eigen value. Then, using Lemma \ref{lem:Lyap_bound}, we have $\|\Sigma^y\|\leq \sigma^yd\tau_{mix}$.

Before we start the proof, we give a schematic road map of the proof of Theorem \ref{thm:Markovian_main} in Figure \ref{fig:proof_sketch}. Recall that the proof of our main lemma \ref{lem:x_xy_y_prime} that pillars our theorem is based upon induction argument. Thus, we have divided the auxiliary lemmas into two groups: Induction dependent lemmas that are proved using the induction hypothesis and Induction independent lemmas that proved using only the problem structure and assumptions.
\begin{figure}%
    \begin{tikzpicture}[shorten >=1pt,node distance=4cm,on grid,auto]
    \node (main_thm) [text width=8.5cm, boxGreen] {Theorem \ref{thm:Markovian_main}: Use the linear transformation to show that the coupled iterates satisfy the following\begin{align*}
        \E[\xh_{k}\xh_{k}^\top]&=\alpha_{k}\Sigma^x +  C^x_{k}\zeta_{k}^x\\
        \E[\xh_{k}y_{k}^\top]&=\beta_{k}\Sigma^{xy} + C^{xy}_{k}\zeta_{k}^{xy}\\
        \E[y_{k}y_{k}^\top]&=\beta_{k}\Sigma^y + C^{y}_{k}\zeta_{k}^y
    \end{align*}
    where $\max\{\|C^x_{k}\|, \|C^z_{k}\|, \|C^y_{k}\|\}\leq c_0(\varrho, d)$ for all $k\geq 0$.};

    \node (tech_lemD1)  [text width=10cm,boxGreen,  below= 4cm of main_thm]  {Lemma \ref{lem:x_xy_y}: Use Lemma \ref{lem:x_xy_y_prime} to show that the decoupled iterates satisfy the following
    \begin{align*}
        \left[\tX_{k}, \tZ_{k}, \tY_{k}\right]=\left[\alpha_{k}(\Sigma^x +  o(1)), \beta_{k}(\Sigma^{xy} + o(1)), \beta_{k}(\Sigma^y + o(1))\right]
    \end{align*}};

    \node (tech_lemD2)  [text width=12cm, boxGreen, below =4cm of tech_lemD1]  {Lemma \ref{lem:x_xy_y_prime}: Induction step showing the following     
    \begin{align*}
        \left[\tX'_{k+1}, \tZ'_{k+1},\tY'_{k+1}\right]& =\left[\alpha_{k+1}(\Sigma^x +  o(1)),\beta_{k+1}( \Sigma^{xy} + o(1)), \beta_{k+1}( \Sigma^y + o(1))\right]
    \end{align*}};
    
    \node (aux_lem_ind) [text width=6.5cm, boxGreen, below left=6cm and 5cm   of tech_lemD2] {Section \ref{sec:Aux_lem_ind}: \\
    \textbf{Induction independent lemmas}\\
        
        \begin{tikzpicture}[shorten >=1pt,on grid,auto]
        \node (aux_lem_ind1) [text width=5.5cm, boxYellow]  {Lemma \ref{lem:L_k_bound}: Properties of the deterministic iterate $L_k$};
        \node (aux_lem_ind2) [text width=5.5cm, boxYellow, below= 1.5cm of aux_lem_ind1]  {Lemma \ref{lem:f_hat_bound}: Norm of functions of noise have almost sure linear growth w.r.t. to iterates.};
        \node (aux_lem_ind3) [text width=5.5cm, boxYellow, below= 1.5cm of aux_lem_ind2]  {Lemma \ref{lem:crdue_upper_bnd_norm}: One step recursive relation of iterates};
        \node (aux_lem_ind4) [text width=6cm, boxYellow, below= 1.5cm of aux_lem_ind3]  {Lemma \ref{lem:boundedness}:  Uniform constant upper bound for the mean square error};
        \node (aux_lem_ind5) [text width=5.5cm, boxYellow, below= 1.5cm of aux_lem_ind4]  {Lemmas \ref{lem:tel_term_bound} and \ref{lem:noise_crude_bound}: Loose bounds on norm of noise terms};
        \draw [arrow] (aux_lem_ind1) -- (aux_lem_ind2);
        \draw [arrow] (aux_lem_ind2) -- (aux_lem_ind3);
        \draw [arrow] (aux_lem_ind3) -- (aux_lem_ind4);
        \draw [arrow] (aux_lem_ind4) -- (aux_lem_ind5);
        \end{tikzpicture}};
        
    \node (aux_lem_dep) [text width=6.5cm, boxGreen, below right=6cm and 5cm of tech_lemD2] { Section \ref{sec:Aux_lem_dep}:\\ \textbf{Induction dependent lemmas}
    \begin{tikzpicture}[shorten >=1pt,on grid,auto]
        \node (ind_assump) [text width=6cm, boxYellow]  {Assume 
    $
    \begin{cases}
        \tX'_k=\alpha_k(\Sigma^x +  o(1)),\\
        \tZ'_k=\beta_k(\Sigma^{xy} + o(1)),\\
        \tY'_k=\beta_k(\Sigma^y + o(1))
    \end{cases}
    $\\
    };
    \node (imp) [text width=6cm, boxYellow, below=2.5cm of ind_assump ]  {\vspace{-4mm}\begin{itemize}
        \item  Lemma  \ref{lem:go_from_xp_to_x}: Some implications of the assumption from induction.
        \item  Lemma \ref{lem:F_conv}: Markovian correlation of the noise across time.
    \end{itemize}};
    \node (mark_noise) [text width=6cm, boxYellow, below=2.5cm of imp ]  {\vspace{-4mm}\begin{itemize}
        \item  Lemma \ref{lem:noise_bound}: Noise variance under equilibrium.
        \item  Lemma \ref{lem:noise_bound2}: Expected behavior of the cross term.
    \end{itemize}};
    \draw [arrow] (ind_assump) -- (imp);
    \draw [arrow] (imp) -- (mark_noise);
    \end{tikzpicture}

    };
    
        \draw [arrow] (aux_lem_ind) -- node[midway, above] {Section \ref{sec:pf_3}} (aux_lem_dep);
        \draw [arrow] (aux_lem_ind) -- node[midway, left] {Section \ref{sec:absolute_upper_bound}} (tech_lemD2);
        \draw [arrow] ([xshift=-5cm, yshift=0pt]aux_lem_ind) |- (main_thm);
        \draw [arrow] ([xshift=-3.6cm, yshift=0pt] aux_lem_ind) |- (tech_lemD1);
        \draw [arrow] (tech_lemD2) -- node[midway, right] {Section \ref{sec:pf_2}} (tech_lemD1);
        \draw [arrow] (tech_lemD1) -- node[midway, right] {Section \ref{sec:pf_4}} (main_thm);
        \draw [arrow] (aux_lem_dep) -- node[midway, right] {Sections \ref{sec:pf_2} and \ref{sec:pf_5}} (tech_lemD2);

        \end{tikzpicture}
    \caption{Road map of the proof of the paper}
    \label{fig:proof_sketch}
\end{figure}

\section{Proofs of the results in the main paper}\label{sec:app_proof_main}

\begin{proof}[Proof of Proposition \ref{prop:alternate_pe}]
    
We will prove the lemma only for $\Gamma^x$. The other terms follow in a similar way. From Lemma \ref{lem:possion_sol}, taking $A_1$ and $A_2$ to be all zero matrices we have that:
\begin{align*}
    \hat{b}_i(O)=\sum_{k=0}^\infty\E[b_i(O_k)|O_0=O]
\end{align*}
Replacing the above solution in Definition \ref{def:var} we have:
\begin{align*}
    \Gamma^x=\E_{O\sim \mu}\left[\left(\sum_{j=0}^\infty\E[b_2(O_j)|O_0=O]\right)b_2(O)^\top+b_2(O)\left(\sum_{j=0}^\infty\E[b_2(O_j)|O_0=O]^\top\right)-b_2(O)b_2(O)^\top\right]
\end{align*}
Since $\{\tO_j\}_{j\geq 0}$ comes from Markov chain whose starting distribution is $\mu$, we have:
\begin{align*}
    \Gamma^x&=\E\left[\left(\sum_{j=0}^\infty\E[b_2(\tO_j)|\tO_0]\right)b_2(\tO_0)^\top+b_2(\tO_0)\left(\sum_{j=0}^\infty\E[b_2(\tO_j)|\tO_0]^\top\right)-b_2(\tO_0)b_2(\tO_0)^\top\right]\\
    &=\E\left[\sum_{j=0}^\infty\E[b_2(\tO_j)|\tO_0]b_2(\tO_0)^\top\right]+\E\left[\sum_{j=0}^\infty b_2(\tO_0)\E[b_2(\tO_j)|\tO_0]^\top\right]-\E[b_2(\tO_0)b_2(\tO_0)^\top]\\
    &=\E\left[\sum_{j=0}^\infty\E[b_2(\tO_j)b_2(\tO_0)^\top|\tO_0]\right]+\E\left[\sum_{j=0}^\infty\E[b_2(\tO_0)b_2(\tO_j)|\tO_0]^\top\right]-\E[b_2(\tO_0)b_2(\tO_0)^\top]\\
    &=\sum_{j=0}^\infty\E[\E[b_2(\tO_j)b_2(\tO_0)^\top|\tO_0]]+\sum_{j=0}^\infty\E[\E[b_2(\tO_0)b_2(\tO_j)|\tO_0]^\top]-\E[b_2(\tO_0)b_2(\tO_0)^\top]\tag{Fubini-Tonelli Theorem}\\
    &=\sum_{j=0}^\infty\E[b_2(\tO_j)b_2(\tO_0)^\top]+\sum_{j=0}^\infty\E[b_2(\tO_0)b_2(\tO_j)]^\top-\E[b_2(\tO_0)b_2(\tO_0)^\top]\tag{Tower property}\\
    &=\E[b_2(\tO_0)b_2(\tO_0)^\top]+\sum_{j=1}^\infty\E[b_2(\tO_j)b_2(\tO_0)^\top+b_2(\tO_0)b_2(\tO_j)^\top]
\end{align*}
\end{proof}

\begin{proof}[Proof of Theorem \ref{thm:Markovian_main}]
We can write recursion \eqref{eq:two_time_scale} as
\begin{align*}
    y_{k+1}&=y_k-\beta_k(A_{11}y_k+A_{12}x_k)+\beta_k\left(b_1(O_{k})-(A_{11}(O_{k})-A_{11})y_k-(A_{12}(O_{k})-A_{12})x_k\right)\\
    &=y_k-\beta_k(A_{11}y_k+A_{12}x_k)+\beta_kf_1(O_k,x_k,y_k),
\end{align*}
and
\begin{align*}
    x_{k+1}&= x_k-\alpha_k(A_{21}y_k+A_{22}x_k)+\alpha_k\left(b_2 (O_{k})-(A_{21}(O_{k})-A_{21})y_k-(A_{22}(O_{k})-A_{22})x_k\right)\\
    &= x_k-\alpha_k(A_{21}y_k+A_{22}x_k)+\alpha_kf_{2}(O_k,x_k,y_k).
\end{align*}

We first construct the auxiliary iterates of $\ty_k$ and $\tx_k$ as follows:
    \begin{align}
        \ty_k&=y_k\label{eq:y_tran_2}\\
        \tx_k&=L_ky_k+x_k+A_{22}^{-1}A_{21}y_k,\label{eq:x_tran_2}
    \end{align}
    where 
    \begin{align}
        L_k&=0, ~~~ 0\leq k< k_L\label{eq:L_k12}\\
        L_{k+1}&=(L_k-\alpha_kA_{22}L_k+\beta_kA_{22}^{-1}A_{21}B_{11}^k)(I-\beta_kB_{11}^k)^{-1},~~~\forall k\geq k_L, \label{eq:L_k22}\\
        B_{11}^k&=\Delta-A_{12}L_k\nonumber\\
        B_{21}^k&=\frac{L_k-L_{k+1}}{\alpha_k}+\frac{\beta_k}{\alpha_k}\left(L_{k+1}+A_{22}^{-1}A_{21}\right)B_{11}^k-A_{22}L_k\nonumber\\
        B_{22}^k&=\frac{\beta_k}{\alpha_k}\left(L_{k+1}+A_{22}^{-1}A_{21}\right)A_{12}+A_{22}=C_{22}^k+A_{22}\nonumber,
    \end{align}
    where we denote $C_{22}^k=\frac{\beta_k}{\alpha_k}\left(L_{k+1}+A_{22}^{-1}A_{21}\right)A_{12}$.
    The existence of $k_L$ is guaranteed due to Lemma \ref{lem:L_k_abs_bound}, the fact that $\Delta$ and $A_{12}$ are finite, and Assumptions \ref{ass:step_size_main} on the step size. In addition, this choice of $k_L$ results in $I \succ \beta_kB_{11}^k $ for all $ k \geq k_L $.

    Then we have the following update for the new variables
    \begin{align}
        \ty_{k+1}=&\ty_k-\beta_k(B_{11}^k\ty_k+A_{12}\tx_k)+\beta_kv_k\label{eq:y_t_update_2}\\
        \tx_{k+1}=&\tx_k-\alpha_k(B_{21}^k\ty_k+B_{22}^k\tx_k)+\alpha_k w_k+\beta_k(L_{k+1}+A_{22}^{-1}A_{21})v_k,\label{eq:x_t_update_2}
    \end{align}
where recall $v_k=f_1(O_k,x_k,y_k)$ and $w_k = f_2(O_k,x_k,y_k)$.
\begin{itemize}
    \item  Since we assumed $b_1=b_2=0$, we have $\ty_k=y_k=\hat{y}$. By Lemma \ref{lem:x_xy_y}, we get . 
    \begin{align*}
        \E[\hat{y}_k\hat{y}_k^\top]=\beta_k\Sigma^y+C_k^y\zeta_k^y
    \end{align*}
    where $C_k^y=\tC_k^y$ and $\|C^k_y\|\leq c^*d^2=c^{(y)}d^2$.
\item By Lemma \ref{lem:x_xy_y}, we have
\begin{align*}
    \beta_k \Sigma^{xy} + \tC_k^{xy}\zeta_k^{xy}=\E[\tx_k\ty_{k}^\top]= \E[(L_ky_k+\xh_k)y_{k}^\top] 
\end{align*}
\begin{align*}
    \implies \E[\xh_ky_{k}^\top] = \beta_k \Sigma^{xy} + \tC_k^{xy}\zeta_k^{xy} - L_k\E[y_ky_k^\top]. 
\end{align*}
Define $C_k^{xy}$ such that $\tC_k^{xy}\zeta_k^{xy} - L_k\E[y_ky_k^\top] = C_k^{xy}\zeta_k^{xy}$. Then, we have 
\begin{align*}
\|C_k^{xy}\|=\bigg\|\tC_k^{xy} -\frac{1}{\zeta_k^{xy}} L_k\E[y_ky_k^\top]\bigg\|
&\leq \|\tC_k^{xy}\| +\frac{1}{\zeta_k^{xy}}\|L_k\|\|\E[y_ky_k^\top]\| \\ 
&\leq c^*d^2+c_1^L\frac{\beta_k}{\zeta_k^{xy}\alpha_k}\left(\sigma^y\tau_{mix}d\beta_k+c^*d^2\zeta_k^y\right)\tag{Lemma \ref{lem:x_xy_y} and \ref{lem:L_k_abs_bound}}\\
&\leq c^*d^2+c_1^L\frac{\beta}{\alpha}\left(\sigma^y\tau_{mix}d\beta+c^*d^2\right)=c^{(z)}d^2.
\end{align*} 
where $c^{(z)}= c^*+c_1^L\frac{\beta}{\alpha}\left(\sigma^y\tau_{mix}\beta+c^*\right)$.

\item Again by Lemma \ref{lem:x_xy_y}, we have
\begin{align*}
    \E[(L_ky_k+\xh_k)(L_ky_k+\xh_k)^\top]=\alpha_k \Sigma^x + \tilde{C}^x_k\zeta_k^x
    \end{align*}
    \begin{align*}
    \implies \E[\xh_k\xh_k^\top] = \alpha_k \Sigma^x + \tilde{C}^x_k\zeta_k^x- L_k\E[y_ky_k^\top]L_k^\top -L_k\E[y_k \xh_k^\top] - \E[\xh_k y_k^\top]L_k^\top.
\end{align*}

Define $C^x_k$ such that
$C^x_k\zeta_k^x=\tilde{C}^x_k\zeta_k^x-L_k\E[y_ky_k^\top]L_k^\top -L_k\E[y_k \xh_k^\top] - \E[\xh_k y_k^\top]L_k^\top$. Then, we have
\begin{align*}
    \|C_k^x\|=\left\|\tilde{C}^x_k-\frac{1}{\zeta_k^x}\left( L_k\E[y_ky_k^\top]L_k^\top +L_k\E[y_k \xh_k^\top]+\E[\xh_k y_k^\top]L_k^\top\right) \right\|&\leq \|\tilde{C}^x_k\|+\frac{1}{\zeta_k^x}\|L_k\|^2\|\E[y_ky_k^\top]\|\\
    &~~+\frac{2}{\zeta_k^x}\|L_k\|\|\E[y_k \xh_k^\top]\|.
\end{align*}
Using Lemma \ref{lem:L_k_bound}, we can bound $\|L_k\|\leq \kappa_{Q_{22}}$. For the other terms, we use the previous parts to get, 
\begin{align*}
    \|C_k^x\|&\leq c^*d^2+\frac{\kappa_{Q_{22}}^2}{\zeta_k^x}\left(\sigma^y\tau_{mix}d\beta_k+c^*d^2\zeta_k^y\right)+\frac{2\kappa_{Q_{22}}}{\zeta_k^x}\left(\sigma^{xy}\tau_{mix}d\beta_k+c^{(z)}d^2\zeta_k^{xy}\right)\\
    &\leq c^*d^2+\kappa_{Q_{22}}^2\left(\sigma^y\tau_{mix}d\beta+c^*d^2\right)+2\kappa_{Q_{22}}\left(\sigma^{xy}\tau_{mix}d\beta+c^{(z)}d^2\right)\tag{$\beta_k\leq \beta\zeta_k^x$}\\
    &=c^{(x)}d^2,
\end{align*}
where $c^{(x)}=c^*+\kappa_{Q_{22}}^2\left(\sigma^y\tau_{mix}\beta+c^*\right)+2\kappa_{Q_{22}}\left(\sigma^{xy}\tau_{mix}\beta+c^{(z)}\right)$.
\end{itemize}

\end{proof}

\begin{proof}[Proof of Proposition \ref{prop:positive_def}]
    The covariance of $h_N$ is given as
    \begin{align*}
        \E[h_Nh_N^\top] &= \frac{1}{N} \E\left[\sum_{k,k'=0}^{N-1} \tilde{b}_1(\tilde{O}_k)\tilde{b}_1(\tilde{O}_{k'})^\top\right] + A_{12}A_{22}^{-1}\frac{1}{N}\E\left[\sum_{k,k'=0}^{N-1}\tilde{b}_2(\tilde{O}_k)\tilde{b}_2(\tilde{O}_{k'})^\top \right]A_{22}^{-\top}A_{12}\\
        &~~-\frac{1}{N}\E\left[\sum_{k,k'=0}^{N-1}\tilde{b}_1(\tilde{O}_k)\tilde{b}_2(\tilde{O}_{k'})^\top\right] A_{22}^{-\top}A_{12}-A_{12}A_{22}^{-1}\frac{1}{N}\E\left[\sum_{k,k'=0}^{N-1}\tilde{b}_2(\tilde{O}_k)\tilde{b}_1(\tilde{O}_{k'})^\top\right]
    \end{align*}
    Let the first term be denoted as $T_1$. In what follows, we will only analyze $T_1$ and show that 
    \begin{align*}
        \lim_{N\to \infty}\frac{1}{N} \E\left[\sum_{k,k'=0}^{N-1} \tilde{b}_1(\tilde{O}_k)\tilde{b}_1(\tilde{O}_{k'})^\top\right]=\Gamma^y.
    \end{align*}
    Convergence for other terms can be shown by following the exact steps, hence omitted for brevity. Expanding $T_1$, we get
        \begin{align*}    
            T_1&=\frac{1}{N} \left(\sum_{k=0}^{N-1}\E\left[\tilde{b}_1(\tilde{O}_k)\tilde{b}_1(\tilde{O}_{k})^\top \right]+\sum_{k=0}^{N-1}\sum_{k'=k+1}^{N-1}\E\left[\tilde{b}_1(\tilde{O}_k)\tilde{b}_1(\tilde{O}_{k'})^\top \right]+\sum_{k'=0}^{N-1}\sum_{k=k'+1}^{N-1}\E\left[\tilde{b}_1(\tilde{O}_{k})\tilde{b}_1(\tilde{O}_{k'})^\top \right]\right).
        \end{align*}
        Recall that $\{\tilde{O}_k\}$ is a stationary process. Hence, 
        \begin{align*}
            \E\left[\tilde{b}_1(\tilde{O}_k)\tilde{b}_1(\tilde{O}_{k})^\top \right]=\E\left[\tilde{b}_1(\tilde{O}_0)\tilde{b}_1(\tilde{O}_{0})^\top \right]~~\forall k\geq 0.
        \end{align*}
        
        Next, to simplify the second term in $T_1$, let $j\in \{1,\dots, N\}$ and $k'-k=j$. Then, we note that there are exactly $N-j$ pairs $(k',k)$ such that $k'-k=j$ and $0\leq k<k'\leq N-1$. A similar argument holds for the third term, with the only difference that the indices $k$ and $k'$ are swapped. Combining this observation with the strong Markov property, we can rewrite the expression for $T_1$ as
        \begin{align*}
            T_1&=\E\left[\tilde{b}_1(\tilde{O}_0)\tilde{b}_1(\tilde{O}_{0})^\top \right]+\frac{1}{N} \left(\sum_{j=0}^{N-1}(N-j)\E\left[\tilde{b}_1(\tilde{O}_0)\tilde{b}_1(\tilde{O}_{j})^\top \right]+\sum_{j=0}^{N-1}(N-j)\E\left[\tilde{b}_1(\tilde{O}_{j})\tilde{b}_1(\tilde{O}_{0})^\top \right]\right).
        \end{align*}
        To show the convergence of second and third term, we use the mixing property of the Markov chain. Recall that $\{\tilde{O}_k\}$ is sampled from a finite state ergodic Markov chain, hence it mixes exponentially fast \cite{levin2017markov}, that is, for all $o\in\mathcal{S} $, we have $d_{TV}(P^k(\cdot|o)||\mu(\cdot))\leq\rho^k$ for some $\rho\in[0,1)$. Thus, we have
        \begin{align*}
            \Bigg\|\frac{1}{N} \sum_{j=0}^{N-1}(N-j)\E\left[\tilde{b}_1(\tilde{O}_0)\tilde{b}_1(\tilde{O}_{j})^\top \right]-&\sum_{j=0}^{N-1}\E\left[\tilde{b}_1(\tilde{O}_0)\tilde{b}_1(\tilde{O}_{j})^\top \right]\Bigg\|\\
            &=\frac{1}{N}\left\|\sum_{j=0}^{N-1}j\E\left[\tilde{b}_1(\tilde{O}_0)\tilde{b}_1(\tilde{O}_{j})^\top \right]\right\|\\
            &\leq \frac{1}{N}\sum_{j=0}^{N-1} j\max_{o} \left\|\E \left[\tilde{b}_1(\tO_k)|\tO_0=o\right]\right\|\left\| \tilde{b}_1(o) \right\|\\
            &\leq  \frac{1}{N}\sum_{j=0}^{N-1} j\max_{o} \left\|\sum_{o'\in \mathcal{S}}P^k(o'|o) \tilde{b}_1(o')\right\|\left\| \tilde{b}_1(o) \right\|\\
            &=  \frac{1}{N}\sum_{j=0}^{N-1} j\max_{o} \left\|\sum_{o'\in \mathcal{S}}(P^k(o'|o)-\mu(o')) \tilde{b}_1(o')\right\|\left\| \tilde{b}_1(o) \right\|\\
            &\leq \frac{b_{max}^2d}{N}\sum_{j=0}^{N-1} j \max_{o'}d_{TV}(P^k(\cdot|o')||\mu(\cdot)) \\
            &\leq \frac{b_{max}^2d}{N}\sum_{j=0}^{N-1} j \rho^k\leq \frac{b_{max}^2d\rho}{N(1-\rho)^2}\xrightarrow[N\uparrow \infty]{}0.
        \end{align*}
        Combining the above relations, we have
        \begin{align*}
            \lim_{N\to \infty}T_1=\E\left[\tilde{b}_1(\tilde{O}_0)\tilde{b}_1(\tilde{O}_{0})^\top \right]+\sum_{j=0}^{\infty}\E\left[\tilde{b}_1(\tilde{O}_0)\tilde{b}_1(\tilde{O}_{j})^\top \right]+\sum_{j=0}^{\infty}\E\left[\tilde{b}_1(\tilde{O}_j)\tilde{b}_1(\tilde{O}_{0})^\top \right].
        \end{align*}
        Using a similar analysis for $\tilde{b}_2(\tilde{O}_k)$ and the cross terms, we obtain the asymptotic covariance of $h_N$ as 
        \begin{align*}
            \lim_{N\to \infty}\E[h_Nh_N^\top] &= \Gamma^y + A_{12}A_{22}^{-1}\Gamma^xA_{22}^{-\top}A_{12}-\Gamma^{yx} A_{22}^{-\top}A_{12}-A_{12}A_{22}^{-1}\Gamma^{xy}.
        \end{align*}
        Now, we are only left to show that the r.h.s of the above equation can be equivalently written as $\Gamma^y- A_{12}\Sigma^{yx}-\Sigma^{xy}A_{12}^\top$. To see this, we first solve for $\Sigma^{xy}$ from Eq. \eqref{eq:sigma_xy_def_main} to get
        \begin{align*}
            \Sigma^{xy}=(\Gamma^{xy}-A_{12}\Sigma^{x})A_{22}^{-\top}.
        \end{align*}
        Substituting the above expression in the r.h.s. of Eq. \eqref{eq:sigma_y_def_main}, we get
        \begin{align*}
            \Gamma^y- A_{12}\Sigma^{yx}-\Sigma^{xy}A_{12}^\top&=\Gamma^y- A_{12}A_{22}^{-1}(\Gamma^{xy}-A_{12}\Sigma^{x})-(\Gamma^{yx}-\Sigma^{x}A_{12}^{\top})A_{22}^{-\top}A_{12}^\top\\
            &=\Gamma^y- A_{12}A_{22}^{-1}\Gamma^{xy}+A_{12}A_{22}^{-1}\Sigma^{x}A_{12}^{\top}-\Gamma^{yx}A_{22}^{-\top}A_{12}^\top+A_{12}\Sigma^{x}A_{22}^{-\top}A_{12}^\top\\
            &=\Gamma^y+A_{12}A_{22}^{-1}(\Sigma^{x}A_{22}^\top+A_{22}\Sigma^{x})A_{22}^{-\top}A_{12}^\top- A_{12}A_{22}^{-1}\Gamma^{xy}-\Gamma^{yx}A_{22}^{-\top}A_{12}^\top\\
            &=\Gamma^y+A_{12}A_{22}^{-1}\Gamma^xA_{22}^{-\top}A_{12}^\top- A_{12}A_{22}^{-1}\Gamma^{xy}-\Gamma^{yx}A_{22}^{-\top}A_{12}^\top\tag{Eq. \eqref{eq:sigma_x_def_main}}
        \end{align*}
\end{proof}

\begin{proof}[Proof for Corollary \ref{cor:l2_bound}]
    The claim follows by taking trace on both sides of Eq. \ref{eq:main_y_main} and using $trace(C_k^y(0.5))\leq d\|C_k^y(0.5)\|\leq dc_0(0.5)$. Note that since Theorem \ref{thm:Markovian_main} holds for any $\varrho\in(0,1)$, we choose $\varrho=0.5$.
\end{proof}

\begin{proof}[Proof of Proposition \ref{prop:loose_convergence}] 
Since in this proposition we are only concerned with convergence,
throughout this proof we replace all the constants with $c$.

From \eqref{eq:V_k_recursion} and \eqref{eq:W_k_recursion} in the proof of Lemma \ref{lem:boundedness}, we have
    \begin{align*}
        V_{k+1} 
        \leq & (1-\frac{a_{22}\alpha_k}{2})V_k+ \alpha_k^2 c (1+V_k+W_k) +\frac{c\beta_k^2}{\alpha_k}(1+V_k+W_k) + \alpha_k (\bar{d}_{k}^x-\bar{d}_{k+1}^x),\\
        W_{k+1}\leq & (1-\frac{\delta\beta_k}{2})W_k+ \alpha_k\beta_kc (1+V_k+W_k) +\beta_k \frac{2\|A_{12}\|^2_{Q_{\Delta}}\gamma_{max}(Q_{\Delta})}{\gamma_{min}(Q_{22})\delta}V_k+ \beta_k(\bar{d}_k^y- \bar{d}_{k+1}^y)
    \end{align*}
    Let $\omega_k=\frac{8\|A_{12}\|^2_{Q_{\Delta}}\gamma_{max}(Q_{\Delta})\beta_k}{\gamma_{min}(Q_{22})\delta a_{22}\alpha_k}$. Define $V_k'=\omega_kV_k$.  Then, rewriting both the recursions in terms of $V_k'$ we get:
    \begin{align*}
        V'_{k+1} 
        \leq & (1-\frac{a_{22}\alpha_k}{2})V'_k+ \alpha_k^2 c (\omega_k+V'_k+\omega_kW_k) +\frac{c\beta_k^2}{\alpha_k}(\omega_k+V'_k+\omega_kW_k) + c\beta_k (\bar{d}_{k}^x-\bar{d}_{k+1}^x)\\
        &+c\omega_k\frac{1}{k}V'_k+c\omega_k\frac{1}{k}(\alpha_k^2+\beta_k)(1+W_k),\tag{by \eqref{eq:d^x_k_upper_bnd}}\\
        W_{k+1}\leq & (1-\frac{\delta\beta_k}{2})W_k+ \frac{\alpha_k\beta_kc}{\omega_k} V_k'+\alpha_k\beta_kc(1+W_k) +\frac{a_{22}\alpha_k}{4}V'_k+ \beta_k(\bar{d}_k^y- \bar{d}_{k+1}^y)
    \end{align*}
    Adding  the recursions, we get:
    \begin{align*}
        V'_{k+1}+W_{k+1} 
        \leq & (1-\frac{a_{22}\alpha_k}{4})V'_k+ \alpha_k^2 c (\omega_k+V'_k+\omega_kW_k) +\frac{c\beta_k^2}{\alpha_k}(\omega_k+V'_k+\omega_kW_k) + c\beta_k (\bar{d}_{k}^x-\bar{d}_{k+1}^x)\\
        &+c\omega_k\frac{1}{k}V'_k+c\omega_k\frac{1}{k}(\alpha_k^2+\beta_k)(1+W_k)+ (1-\frac{\delta\beta_k}{2})W_k+ \frac{\alpha_k\beta_kc}{\omega_k} V_k'+\alpha_k\beta_kc(1+W_k) \\&+\beta_k(\bar{d}_k^y- \bar{d}_{k+1}^y)\\
        V'_{k+1}+W_{k+1} 
        \leq & (1-\frac{a_{22}\alpha_k}{8})V'_k+ \alpha_k^2 c\omega_k +\frac{c\beta_k^2}{\alpha_k}\omega_k + c\beta_k (\bar{d}_{k}^x-\bar{d}_{k+1}^x)+c\omega_k\frac{1}{k}(\alpha_k^2+\beta_k)\\
        &+ (1-\frac{\delta\beta_k}{4})W_k+\alpha_k\beta_kc+\beta_k(\bar{d}_k^y- \bar{d}_{k+1}^y)\tag{for large enough $k$}\\
        V'_{k+1}+W_{k+1} 
        \leq & (1-\frac{a_{22}\alpha_k}{8})V'_k+ \beta_k (\hat{d}_{k}-\hat{d}_{k+1})+ (1-\frac{\delta\beta_k}{4})W_k+o(\beta_k)\tag{$\hat{d}_{k}=c\bar{d}_{k}^x+\bar{d}_{k}^y$}\\
        V'_{k+1}+W_{k+1} 
        \leq & (1-\frac{\delta\beta_k}{4})(V'_k+W_k)+ \beta_k (\hat{d}_{k}-\hat{d}_{k+1})+o(\beta_k)\tag{for large enough $k$}
    \end{align*}
    Let $\bar{K}$ be the minimum $k$ at the which the above recursion holds. Then opening the recursion from $\bar{K}$ to $k$ and using the telescopic structure leads to the following:
    \begin{align*}
        V'_k+W_k \leq (V'_{\bar{K}}+W_{\bar{K}})\prod_{i=\bar{K}}^k(1-\frac{\delta\beta_i}{4})+\beta_{\bar{K}}\hat{d}_{\bar{K}}\prod_{l=\bar{K}+1}^k(1-\frac{\delta\beta_l}{4})+\beta_{k}|\hat{d}_{k+1}|+\sum_{j=\bar{K}}^k(\beta_j^2|\hat{d}_j|+o(\beta_k))\prod_{l=j+1}^k(1-\frac{\delta\beta_l}{4}).
    \end{align*}
    Notice that for all $j\geq0$, $\hat{d}_{j}$ is upper bounded by a constant due to \eqref{eq:d^x_k_upper_bnd}, \eqref{eq:d^y_k_upper_bnd} and Lemma \ref{lem:boundedness}. Thus, using the observation in \ref{sec:rec_sol}, we obtain $V_k'+W_k=o(1)$. This shows that $y_k\to y^*$ in mean square sense. To further show that $x_k\to x^*$, we replace $W_k$ with $o(1)$ in \eqref{eq:V_k_recursion} and expand from $\bar{K}$ to $k$ to get:
    \begin{align*}
        V_k \leq V_{\bar{K}}\prod_{i=\bar{K}}^k(1-\frac{a_{22}\alpha_k}{8})+\alpha_{\bar{K}}\bar{d}^x_{\bar{K}}\prod_{l=\bar{K}+1}^k(1-\frac{a_{22}\alpha_k}{8})+\alpha_{k}|\bar{d}^x_{k+1}|+\sum_{j=\bar{K}}^k(\alpha_j^2|\bar{d}^x_j|+o(\alpha_k))\prod_{l=j+1}^k(1-\frac{a_{22}\alpha_k}{8}).
    \end{align*}
    The claim follows.
    
\end{proof}

\begin{proof}[Proof of Theorem \ref{thm:polyak}]
    In the setting of Polyak-Ruppert averaging, the parameters reduce to the following:
    \begin{align*}
        A_{21}(O_k)=0;~b_1(O_k)=0;~A_{11}(O_k)=I;~A_{12}(O_k)=-I:~\beta=1
    \end{align*}
    This results in $\Delta=I$. Let $\bt(\cdot)=b(\cdot)-b+(A-A(\cdot))A^{-1}b$. Then, we have:
    \begin{align*}
        \Gamma^x=\E{[\bt(\tilde{O}_0)\bt(\tilde{O}_0)^\top]}
+\sum_{j=1}^\infty \E{[\bt(\tilde{O}_j) \bt(\tilde{O}_0)^\top+\bt(\tilde{O}_0) \bt(\tilde{O}_j)^\top]}
    \end{align*}
    Note that it is possible to find the explicit expression of $\Sigma^y$ in the case of Polyak-Ruppert averaging. To show this we have the following three systems of equations:
    \begin{align*}
    &A\Sigma^x+\Sigma^x A^\top=\Gamma^x\\
    &-\Sigma^x+\Sigma^{xy}A^\top=0\Rightarrow \Sigma^{xy}=A^{-\top}\Sigma^x\\
    &\Sigma^y-\Sigma^{yx}-\Sigma^{xy}=0\Rightarrow\Sigma^y=\Sigma^{yx}+\Sigma^{xy}
\end{align*}
Using second equation in the last one we get:
\begin{align*}
    \Sigma^y=\Sigma^xA^{-1}+A^{-\top}\Sigma^x
\end{align*}
Left multiplying $A^{-1}$ and right multiplying $A^{-\top}$ of the first equation we get:
\begin{align*}
    \Sigma^xA^{-\top}+A^{-1}\Sigma^x=A^{-1}\Gamma^xA^{-\top}
\end{align*}
which from the previous equation is equal to $\Sigma^y$. Finally, using Theorem \ref{thm:Markovian_main} and replacing $1-\varrho=0.5$ defined in \ref{def:Q_delta}, we get the result.
\end{proof}

\begin{proof}[Proof for Theorem \ref{thm:TDC}]
Denote the tuple $O_k=\{s_k, a_k, s_{k+1}\}$ and consider the Markov chain $\{O_l\}_{l\geq 0}$. Here $\hat{P}(O_{k+1}|O_k)=\pi_b(a_{k+1}|s_{k+1})P(s_{k+2}|s_{k+1},a_{k+1})$ and the stationary distribution is given by $\mu(s, a, s')=\mu_b(s, a)P(s'|s,a)$. Since we assume that the behavior policy induces an ergodic Markov chain, we have that $\{O_k\}_{k\geq 0}$ satisfies Assumption \ref{ass:poisson_main}. We will denote $\{\tO_k\}_{k\geq 0}$ as the Markov chain where $\{(s_0, a_0)\sim \mu_b\}$. Assumption \ref{ass:step_size_main} is also satisfied, since $\xi=0.75\in (0.5, 1)$, and $\beta$ is chosen appropriately. Thus, all that is left to verify is that the appropriate matrices in the three settings are Hurwitz. Recall that we defined $A=\mathbb{E}_{\mu_{\pi_b}}[\rho(s,a)\phi(s)(\phi(s)-\gamma\phi(s'))^\top]$, $B=\gamma\mathbb{E}_{\mu_{\pi_b}}[\rho(s,a)\phi(s')\phi(s)^\top]$, $C=\mathbb{E}_{\mu_{\pi_b}}[\phi(s)\phi(s)^\top]$ and $b=\mathbb{E}_{\mu_{\pi_b}}[\rho(s,a)r(s, a)\phi(s)]$. We verify the Hurwitz property and characterize the variance in the dominant term for each setting as follows:
\begin{itemize}
    \item \textbf{GTD}: Clearly, $A_{22}(O_k)=I$ for all $k\geq 0$ in this case which implies $-A_{22}=-I$ is Hurwitz. Furthermore, $A_{11}=0$, thus $-\Delta=-A^\top A$, which is a positive definite matrix and is therefore Hurwitz. 

    Next, note that $b_1(O_k)=0$ for all $k\geq 0$ and $b_2(O_k)=b_k$. Let $(\theta^*, \omega^*)$ denote the fixed point. Then, we define the following:
    \begin{align*}
        \bt_2(O_k)=b_k-b+(A-A_k)\theta^*.
    \end{align*}
    The above gives us the following asymptotic covariance matrices:
    \begin{align*}
        \Gamma^\omega=\E{[\bt_2(\tilde{O}_0)\bt_2(\tilde{O}_0)^\top]}\nonumber+\sum_{j=1}^\infty \E{[\bt_2(\tilde{O}_j)\bt_2(\tilde{O}_0)^\top + \bt_2(\tilde{O}_0)\bt_2(\tilde{O}_j)^\top]};~~
        \Gamma^{\omega\theta}
        =0;~~\Gamma^\theta=0.
    \end{align*}
    Then, using Theorem \ref{thm:Markovian_main} we get:
    \begin{align*}
        \E[(\theta_k-\theta^*)(\theta_k-\theta^*)^\top]=&\beta_k \Sigma^\theta+\frac{1}{k^{1+(1-\varrho)\min(\xi-0.5, 1-\xi)}} C^\theta_k(\varrho)\\
        \E[(\theta_k-\theta^*)(\omega_k-\omega^*)^\top]=&\beta_k \Sigma^{\omega\theta}+\frac{1}{k^{\min(\xi+0.5, 2-\xi)}} C^{\omega\theta}_k(\varrho)\\
        \E[(\omega_k-\omega^*)(\omega_k-\omega^*)^\top]=&\alpha_k \Sigma^\omega+\frac{1}{k^{\min(1.5\xi, 1)}} C^\omega_k(\varrho)
    \end{align*}
    where $0<\varrho<1$ is an arbitrary constant, $\sup_k\max\{\|C_k^\omega(\varrho)\|,\|C_k^{\theta\omega}(\varrho)\|,\|C_k^\theta(\varrho)\|\}<c_0(\varrho)<\infty$ for some problem dependent constant $c_0(\varrho)$, and $\Sigma^\theta$, $\Sigma^{\omega\theta}={\Sigma^{\theta\omega}}^\top$ and $\Sigma^\omega$ are unique solutions to the following system of equations: 
    \begin{align*}
        &\Sigma^\omega=\frac{1}{2}\Gamma^\omega\tag{$A_{22}=I$}\\
        &A\Sigma^\omega+\Sigma^{\omega\theta}=0\\
        &\left(A^TA-\frac{1}{2\beta}I\right)\Sigma^\theta+\Sigma^\theta\left(A^TA-\frac{1}{2\beta}I\right)=-A^\top\Sigma^{\theta\omega}-\Sigma^{\omega\theta}A.
    \end{align*}
    The variance $\sigma^2_{GTD}$ is obtained by taking the trace of $\Sigma^{\theta}$.

    \item \textbf{GTD2}: In this setting $A_{22}(O_k)=C_k$. Thus, we have $-A_{22}=-C$. Since $\Phi$ is a full-rank matrix, we have that $-C$ is a Hurwitz matrix as $-x^\top Cx=-\mathbb{E}_{\mu_b}[x^\top\phi(s)\phi(s)^\top x]<0$, in particular it is negative definite. However, similar to GTD, $A_{11}=0$ in this case, so $-\Delta=-A^\top C^{-1}A$ which is negative definite and thus Hurwitz.

    Furthermore, we again have $b_1(O_k)=0$ for all $k\geq 0$ and $b_2(O_k)=b_k$. However, the definition of $b_2(\cdot)$ will change as $A_{22}(O_k)\neq A$ unlike the previous setting.
    \begin{align*}
        \bt_2(O_k)=b_k-b+(A-A_k)\theta^*+(C-C_k)\omega^*.
    \end{align*}
    Thus, we have the following asymptotic covariance matrices:
    \begin{align*}
        \Gamma^\omega=\E{[\bt_2(\tilde{O}_0)\bt_2(\tilde{O}_0)^\top]}\nonumber+\sum_{j=1}^\infty \E{[\bt_2(\tilde{O}_j)\bt_2(\tilde{O}_0)^\top + \bt_2(\tilde{O}_0)\bt_2(\tilde{O}_j)^\top]};~~
        \Gamma^{\omega\theta}
        =0;~~\Gamma^\theta=0.
    \end{align*}
    Then, again using Theorem \ref{thm:Markovian_main} we get:
    \begin{align*}
        \E[(\theta_k-\theta^*)(\theta_k-\theta^*)^\top]=&\beta_k \Sigma^\theta+\frac{1}{k^{1+(1-\varrho)\min(\xi-0.5, 1-\xi)}} C^\theta_k(\varrho)\\
        \E[(\theta_k-\theta^*)(\omega_k-\omega^*)^\top]=&\beta_k \Sigma^{\omega\theta}+\frac{1}{k^{\min(\xi+0.5, 2-\xi)}} C^{\omega\theta}_k(\varrho)\\
        \E[(\omega_k-\omega^*)(\omega_k-\omega^*)^\top]=&\alpha_k \Sigma^\omega+\frac{1}{k^{\min(1.5\xi, 1)}} C^\omega_k(\varrho)
    \end{align*}
    where $0<\varrho<1$ is an arbitrary constant, $\sup_k\max\{\|C_k^\omega(\varrho)\|,\|C_k^{\theta\omega}(\varrho)\|,\|C_k^\theta(\varrho)\|\}<c_0(\varrho)<\infty$ for some problem dependent constant $c_0(\varrho)$, and $\Sigma^\theta$, $\Sigma^{\omega\theta}={\Sigma^{\theta\omega}}^\top$ and $\Sigma^\omega$ are unique solutions to the following system of equations: 
    \begin{align*}
        &C\Sigma^\omega+\Sigma^\omega C^\top=\Gamma^\omega\tag{$C^\top=C$}\\
        &A\Sigma^\omega+\Sigma^{\omega\theta}C=0\\
        &\left(A^TC^{-1}A-\frac{1}{2\beta}I\right)\Sigma^\theta+\Sigma^\theta\left(A^TC^{-1}A-\frac{1}{2\beta}I\right)=-A^\top\Sigma^{\theta\omega}-\Sigma^{\omega\theta}A.
    \end{align*}
    The variance $\sigma^2_{GTD2}$ is obtained by taking the trace of $\Sigma^{\theta}$.
    \item \textbf{TDC}: Note that $A=C-B^\top=C^\top-B^\top$. Thus we have,
    \begin{align*}
        A-BC^{-1}A&=(C-B)C^{-1}A\\
        &=A^\top C^{-1}A
    \end{align*}
    Since $x^\top A^\top C^{-1}Ax>0$, $-(A-BC^{-1}A)$ is Hurwitz. Let $(\theta^*, \omega^*)$ denote the fixed point. Note that unlike previous cases, $b_1(O_k)\neq 0$. Thus, we define the following:
    \begin{align*}
        \bt_1(O_k)=b_k-b+(A-A_k)\theta^*+(B-B_k)\omega^*\\
        \bt_2(O_k)=b_k-b+(A-A_k)\theta^*+(C-C_k)\omega^*
    \end{align*}
    Then, we have the following asymptotic covariance matrices:
    \begin{align*}
        \Gamma^\omega=&\E{[\bt_2(\tilde{O}_0)\bt_2(\tilde{O}_0)^\top]}\nonumber+\sum_{j=1}^\infty \E{[\bt_2(\tilde{O}_j)\bt_2(\tilde{O}_0)^\top + \bt_2(\tilde{O}_0)\bt_2(\tilde{O}_j)^\top]}\nonumber\\ 
        \Gamma^{\omega\theta}
        =&\E{[\bt_2(\tilde{O}_0) \bt_1(\tilde{O}_0)^\top]}+ \sum_{j=1}^\infty \E{[\bt_2(\tilde{O}_j) \bt_1(\tilde{O}_0)^\top + \bt_2(\tilde{O}_0)\bt_1(\tilde{O}_j)^\top]}\nonumber\\
        \Gamma^\theta=&\E{[\bt_1(\tilde{O}_0) \bt_1(\tilde{O}_0)^\top]}+\sum_{j=1}^\infty \E{[\bt_1(\tilde{O}_j) \bt_1(\tilde{O}_0)^\top + \bt_1(\tilde{O}_0) \bt_1(\tilde{O}_j)^\top]}.\nonumber
    \end{align*}
    Then, employing Theorem \ref{thm:Markovian_main} we get:
    \begin{align*}
        \E[(\theta_k-\theta^*)(\theta_k-\theta^*)^\top]=&\beta_k \Sigma^\theta+\frac{1}{k^{1+(1-\varrho)\min(\xi-0.5, 1-\xi)}} C^\theta_k(\varrho)\\
        \E[(\theta_k-\theta^*)(\omega_k-\omega^*)^\top]=&\beta_k \Sigma^{\omega\theta}+\frac{1}{k^{\min(\xi+0.5, 2-\xi)}} C^{\omega\theta}_k(\varrho)\\
        \E[(\omega_k-\omega^*)(\omega_k-\omega^*)^\top]=&\alpha_k \Sigma^\omega+\frac{1}{k^{\min(1.5\xi, 1)}} C^\omega_k(\varrho)
    \end{align*}
    where $0<\varrho<1$ is an arbitrary constant, $\sup_k\max\{\|C_k^\omega(\varrho)\|,\|C_k^{\theta\omega}(\varrho)\|,\|C_k^\theta(\varrho)\|\}<c_0(\varrho)<\infty$ for some problem dependent constant $c_0(\varrho)$, and $\Sigma^\theta$, $\Sigma^{\omega\theta}={\Sigma^{\theta\omega}}^\top$ and $\Sigma^\omega$ are unique solutions to the following system of equations:
    \begin{align*}
        &C\Sigma^\omega+\Sigma^\omega C^\top=\Gamma^\omega\\
        &B\Sigma^\omega+\Sigma^{\omega\theta}C^\top=\Gamma^{\omega\theta}\\
        &\left(A^TC^{-1}A-\frac{1}{2\beta}I\right)\Sigma^\theta+\Sigma^\theta\left(A^TC^{-1}A-\frac{1}{2\beta}I\right)=\Gamma^\theta-B\Sigma^{\theta\omega}-\Sigma^{\omega\theta}B^\top
    \end{align*}
    The variance $\sigma^2_{TDC}$ is obtained by taking the trace of $\Sigma^{\theta}$.
    \end{itemize}
\end{proof}

\begin{proof}[Proof of Proposition \ref{prop:venn}]

\begin{enumerate}
    \item $\mathcal{B}\subsetneq \mathcal{A}$: By definition, it is clear that $\mathcal{B}\subseteq\mathcal{A}$. Next, consider the following matrix
    \begin{align*}
    A&=\begin{bmatrix}
        A_{11}=-4 & A_{12}=-2\\
        A_{21}=-1 & A_{22}=-3
    \end{bmatrix}.
    \end{align*}
    Here we have $A\in\mathcal{A}$. Furthermore, there does not exist any $\kappa>0$ such that $-A_\kappa$ is Hurwitz, which means that $A\notin \mathcal{B}$. This can be easily seen by observing that sum of the eigenvalues is equal to trace of the matrix and the $tr(-A_\kappa)=3\kappa+4>0$. Thus, the $-A_{\kappa}$ cannot be Hurwitz for any $\kappa>0$.

    \item $\mathcal{C}\cup \mathcal{D}\subsetneq \mathcal{B}$:
    Firstly, by definition, it is easy to see that $\mathcal{C}\subset\mathcal{B}$. Secondly, by \cite[Theorem 6]{cothren2024onlinefeedbackoptimizationsingular}, we have $\mathcal{D}\subset\mathcal{B}$. Next, we show that $\mathcal{B}\backslash (\mathcal{C}\cup \mathcal{D})\neq \varnothing$. Consider the following matrix:
    \begin{align*}
    A&=\begin{bmatrix}
        A_{11}=2 & A_{12}=-4\\
        A_{21}=3 & A_{22}=-5
    \end{bmatrix}.
    \end{align*}
    Since $tr(-A)=3>0$, $-A$ is not Hurwitz, and hence $A\notin \mathcal{C}$. In addition, $-A_{22}=5>0$, which means that $A\notin \mathcal{D}$.
    Furthermore, we have
    \begin{align*}
    A_{0.2}&=\begin{bmatrix}
        2 & -4\\
        0.6 & -1
    \end{bmatrix}.
    \end{align*}
    Then, the eigenvalues of $-A_{0.2}$ are $-0.5\pm i\sqrt{15}/10$. Hence, $-A_{0.2}$ is Hurwitz, and $A\in \mathcal{B}$.

    \item $\mathcal{C}\backslash\mathcal{D}\neq\varnothing$: Consider the following matrix
    \begin{align*}
    A&=\begin{bmatrix}
        A_{11}=3& A_{12}=4\\
        A_{21}=-1& A_{22}=-1
        \end{bmatrix}
    \end{align*}
    Both the eigenvalues of $-A$ are $=-1$, which shows that $A\subset \mathcal{C}$. However, $-A_{22}=1>0$, which means that $A\notin \mathcal{D}$.
    \item $\mathcal{D}\backslash\mathcal{C}\neq\varnothing$:
    Consider the following matrix
    \begin{align*}
    A&=\begin{bmatrix}
        A_{11}=-5 & A_{12}=3\\
        A_{21}=-4 & A_{22}=2
    \end{bmatrix}.
    \end{align*}
    Then, $-A_{22}=-2<0$ and $-\Delta = -(-5-3\times 0.5\times -4)=-1<0$. Hence, $A\in\mathcal{D}$. However, $tr(-A)=3>0$, which means that $A$ is not Hurwitz, and hence $A\notin \mathcal{C}$.
    \item $\mathcal{C}\cap\mathcal{D}\neq\varnothing$: Consider the following matrix
    \begin{align*}
    A&=\begin{bmatrix}
        A_{11}=4 & A_{12}=2\\
        A_{21}=1 & A_{22}=3
    \end{bmatrix}
    \end{align*}
The eigenvalues of $-A$ are $-2$ and $-5$. Hence, $A\in \mathcal{C}$. In addition, $-A_{22}=-3<0$ and $-\Delta=-(4-2\cdot\frac{1}{3}\cdot1)=-\frac{10}{3}<0$. Hence, $A\in \mathcal{D}$.
\end{enumerate}
\end{proof}

\section{Lemmas}
\subsection{Technical lemmas}\label{appendix:tech_lemmas}
\begin{lemma}\label{lem:x_xy_y}
    Suppose that Assumptions \ref{ass:hurwitz_main}, \ref{ass:poisson_main}, and \ref{ass:step_size_main} are satisfied.
    For the iterations of $\tx_k$ and $\ty_k$ in \eqref{eq:y_t_update_2} and \eqref{eq:x_t_update_2} we have 
    \begin{align}
        \tX_k &= \alpha_k \Sigma^x + \tilde{C}^x_k\zeta_k^x\label{eq:lem2_1}\\
        \tZ_k &= \beta_k \Sigma^{xy} + \tilde{C}_k^{xy}\zeta_k^{xy}\label{eq:lem2_2}\\
        \tY_k &= \beta_k \Sigma^{y} + \tC_k^y\zeta_k^y,\label{eq:lem2_3}
    \end{align}
    where $\Sigma^x$, $\Sigma^{xy}$ and $\Sigma^y$ are defined in \eqref{eq:sigma_x_def_main}, \eqref{eq:sigma_xy_def_main}, and \eqref{eq:sigma_y_def_main}, and $\sup_k\max\{\|\tilde{C}^x_k\|,\|\tilde{C}^{xy}_k\|,\|\tilde{C}^y_k\|\}\leq c^*d^2<\infty$ for some problem dependent constant $c^*$.
\end{lemma}

\begin{lemma}\label{lem:x_xy_y_prime}
    Suppose that Assumptions \ref{ass:hurwitz_main}, \ref{ass:poisson_main}, and \ref{ass:step_size_main} are satisfied. For $k\geq 0$, the iterations of $\tX_k'$, $\tZ_k'$, and $\tY_k'$ satisfy \begin{align} \tX_k' &= \alpha_k \Sigma^x + \tilde{C}'^x_k\zeta_k^x\label{eq:lem2_10}\\
        \tZ_k' &= \beta_k \Sigma^{xy} + \tilde{C}'^{xy}_k\zeta_k^{xy}\label{eq:lem2_20}\\
        \tY_k' &= \beta_k \Sigma^{y} + \tC'^y_k\zeta_k^y,\label{eq:lem2_30}
    \end{align}
    where $\Sigma^x$, $\Sigma^{xy}$ and $\Sigma^y$ are defined in \eqref{eq:sigma_x_def_main}, \eqref{eq:sigma_xy_def_main}, and \eqref{eq:sigma_y_def_main}, and $\sup_k\max\{\|\tilde{C}'^x_k\|_{Q_{22}},\|\tilde{C}'^{xy}_k\|_{Q_{22}},\|\tilde{C}'^y_k\|_{Q_{\Delta, \beta}}, 1\}\leq \bar{c}d^2<\infty$ for some problem-dependent constant $\bar{c}$.
\end{lemma}
\subsection{Proof of technical lemmas}

\begin{proof}[Proof of Lemma \ref{lem:x_xy_y}]
We first focus on $\tX_k'$. Recall that $\tX_k' = \tX_k+\alpha_k (d_k^x+{d_k^x}^\top)$. Using Lemma \ref{lem:tel_term_bound}, we have
\begin{align*}
    \|\tX_k\|&\leq \|\tX_k'\|+2\alpha_k\|d_k^x\|\\
    &\leq \|\tX_k'\|+\frac{4\alpha_k\sqrt{3d}}{1-\rho}\brc_{f}\left(1+\frac{\beta}{\alpha}\varrho_x\right)\sqrt{\E[\|\tx_k\|^2]}\\
    &\leq \|\tX_k'\|+\frac{4d\alpha_k\sqrt{3\brc}}{1-\rho}\brc_{f}\left(1+\frac{\beta}{\alpha}\varrho_x\right)\tag{Lemma \ref{lem:boundedness}}\\
    &\leq \alpha_kd\left(\sigma^x\tau_{mix}+\frac{4\sqrt{3\brc}}{1-\rho}\brc_{f}\left(1+\frac{\beta}{\alpha}\varrho_x\right)\right)+\kappa_{Q_{22}}\bc d^2\zeta_k^x.\tag{Lemma \ref{lem:x_xy_y_prime}}
\end{align*}
Since $\E[\|\tx_k\|^2]\leq d\|\tX_k\|$, we have
\begin{align*}
    \sqrt{\E[\|\tx_k\|^2]}\leq d\sqrt{\alpha_k\left(\sigma^x\tau_{mix}+\frac{4\sqrt{3\brc}}{1-\rho}\brc_{f}\left(1+\frac{\beta}{\alpha}\varrho_x\right)\right)}+d^{1.5}\sqrt{\kappa_{Q_{22}}\bc \zeta_k^x}.
\end{align*}
Define $\tC_k^x=\tC_k'^x-\frac{\alpha_k}{\zeta_k^x}(d_k^x+d_k^{x\top})$. Using Lemma \ref{lem:tel_term_bound} and the above bound on $\E[\|\tx_k\|^2]$, we get
\begin{align*}
    \|\tC_k^x\|&\leq \|\tC_k'^x\|+\frac{2\alpha_k}{\zeta_k^x}\|d_k^x\|\\
    &\leq \sqrt{\kappa_{Q_{22}}}\bc d^2+\frac{\alpha_k}{\zeta_k^x}\frac{4\sqrt{3}}{1-\rho}\brc_{f}\left(1+\frac{\beta}{\alpha}\varrho_x\right)\left(d^{1.5}\sqrt{\alpha_k\left(\sigma^x\tau_{mix}+\frac{4\sqrt{3\brc}}{1-\rho}\brc_{f}\left(1+\frac{\beta}{\alpha}\varrho_x\right)\right)}+d^2\sqrt{\kappa_{Q_{22}}\bc \zeta_k^x}\right).
\end{align*}
Recall that by definition $\alpha_k^{1.5}\leq \alpha^{1.5}\zeta_k^x$. Thus, we get
\begin{align*}
    \|\tC_k^x\|&\leq \sqrt{\kappa_{Q_{22}}}\bc d^2+\frac{4\sqrt{3}d^2}{1-\rho}\brc_{f}\left(1+\frac{\beta}{\alpha}\varrho_x\right)\left(\alpha^{1.5}\sqrt{\left(\sigma^x\tau_{mix}+\frac{4\sqrt{3\brc}}{1-\rho}\brc_{f}\left(1+\frac{\beta}{\alpha}\varrho_x\right)\right)}+\alpha\sqrt{\kappa_{Q_{22}}\bc}\right)=c^{*(x)}d^2.
\end{align*}

Next, recall that $\tY_k'=\tY_k+\beta_k(d_k^{yv}+d_k^{yv\top})$. By following the exact set of arguments as before, one can show that
\begin{align*}
    \sqrt{\E[\|\ty_k\|^2]}\leq d\sqrt{\beta_k\left(\sigma^y\tau_{mix}+\frac{4\sqrt{3\brc}}{1-\rho}\brc_{f}\right)}+d^{1.5}\sqrt{\kappa_{Q_{\Delta, \beta}}\bc \zeta_k^y}.
\end{align*}
Define $\tC_k^y=\tC_k'^y-\frac{\beta_k}{\zeta_k^y}(d_k^{yv}+d_k^{yv\top})$. Using Lemma \ref{lem:tel_term_bound} and the above bound on $\E[\|\ty_k\|^2]$, we get
\begin{align*}
    \|\tC_k^y\|&\leq \|\tC_k'^y\|+\frac{2\beta_k}{\zeta_k^y}\|d_k^{yv}\|\\
    &\leq \sqrt{\kappa_{Q_{\Delta, \beta}}}\bc d^2+\frac{\beta_k}{\zeta_k^y}\frac{4\sqrt{3}}{1-\rho}\brc_{f}\left(d^{1.5}\sqrt{\beta_k\left(\sigma^y\tau_{mix}+\frac{4\sqrt{3\brc}}{1-\rho}\brc_{f}\right)}+d^2\sqrt{\kappa_{Q_{\Delta, \beta}}\bc \zeta_k^y}\right).
\end{align*}
Again by definition $\beta_k^{1.5}\leq \beta^{1.5}\zeta_k^y$. Thus, we get
\begin{align*}
    \|\tC_k^y\|&\leq \sqrt{\kappa_{Q_{\Delta, \beta}}}\bc d^2+\frac{4\sqrt{3}d^2}{1-\rho}\brc_{f}\left(\beta^{1.5}\sqrt{\left(\sigma^y\tau_{mix}+\frac{4\sqrt{3\brc}}{1-\rho}\brc_{f}\right)}+\beta\sqrt{\kappa_{Q_{\Delta, \beta}}\bc}\right)=c^{*(y)}d^2.
\end{align*}

Finally, from the definition of $\tZ_k'$ we have
    \begin{align*}
        \tZ_k=&\tZ_k'-(\alpha_kd_k^{y}+\beta_kd_k^{xv\top})\\
        =&\beta_k \Sigma^{xy} + \tilde{C}'^{xy}_k\zeta_k^{xy}-(\alpha_kd_k^{y}+\beta_kd_k^{xv\top}).\tag{by Lemma \ref{lem:x_xy_y_prime}}
    \end{align*}
    Define $\tC_k^{xy}=\tilde{C}'^{xy}_k-\frac{\alpha_kd_k^{y}+\beta_kd_k^{xv\top}}{\zeta_k^{xy}}$. Hence, 
    \begin{align*}
        \|\tC_k^{xy}\|\leq &\sqrt{\kappa_{Q_{22}}} \bc d^2+\frac{\alpha_k\|d_k^{y}\|+\beta_{k}\|d_k^{xv}\|}{\zeta_k^{xy}}\\
        \leq&   \sqrt{\kappa_{Q_{22}}}\bc d^2+\frac{2\sqrt{3d}}{1-\rho}\brc_{f}\left(\frac{\alpha_k}{\zeta^{xy}_k}\left(1+\frac{\beta}{\alpha}\varrho_x\right)\sqrt{\E[\|\ty_k\|^2]}+\frac{\beta_k}{\zeta^{xy}_k}\sqrt{\E[\|\tx_k\|^2]}\right)
        \tag{Lemma \ref{lem:tel_term_bound}}.
    \end{align*}
    Recall from the previous parts, we have
    \begin{align*}
        \sqrt{\E[\|\tx_k\|^2]}&\leq d\sqrt{\alpha_k\sigma^x\tau_{mix}}+d^{1.5}\sqrt{c^{*(x)}\zeta_k^x} \\
        \sqrt{\E[\|\ty_k\|^2]}&\leq d\sqrt{\beta_k\sigma^y\tau_{mix}}+d^{1.5}\sqrt{c^{*(y)}\zeta_k^y} .
    \end{align*}
    Plugging the above relation in the bound for $\|\tC_k^{xy}\|$, we get
    \begin{align*}
        \|\tC_k^{xy}\|\leq&   \sqrt{\kappa_{Q_{22}}}\bc d^2+\frac{2\sqrt{3}d^2}{1-\rho}\brc_{f}\left(\frac{\alpha_k}{\zeta^{xy}_k}\left(1+\frac{\beta}{\alpha}\varrho_x\right)\left(\sqrt{\beta_k\sigma^y\tau_{mix}}+\sqrt{c^{*(y)}\zeta_k^y}\right)+\frac{\beta_k}{\zeta^{xy}_k}\left(\sqrt{\alpha_k\sigma^x\tau_{mix}}+\sqrt{c^{*(x)}\zeta_k^x}\right)\right)\\
        \leq &\sqrt{\kappa_{Q_{22}}}\bc d^2+\frac{2\sqrt{3}d^2}{1-\rho}\brc_{f}\left(\left(\alpha+\beta\varrho_x\right)\left(\sqrt{\beta\sigma^y\tau_{mix}}+\sqrt{c^{*(y)}}\right)+ \beta\left(\sqrt{\alpha\sigma^x\tau_{mix}}+\sqrt{c^{*(x)}}\right)\right)=c^{*(z)}d^2.
    \end{align*}
    Thus, we have $\sup_k\max\{\|\tilde{C}^x_k\|,\|\tilde{C}^{xy}_k\|,\|\tilde{C}^y_k\|\}\leq c^*d^2$, where $c^*=\max\{c^{*(x)}, c^{*(z)}, c^{*(y)}\}$.
\end{proof}

\begin{proof}[Proof of Lemma \ref{lem:x_xy_y_prime}]

For consistency, throughout the proof $R_k^{(\cdot)}$ represents remainder or higher order terms. Furthermore, note that by equivalence of norms $\|\cdot\|\leq \kappa_{Q_{22}}\|\cdot\|_{Q_{22}}$ and $\|\cdot\|\leq \kappa_{Q_{\Delta}}\|\cdot\|_{Q_{\Delta}}$ which will be used extensively without explicitly mentioning.

We prove this lemma by induction. Assume that at time $k$, we have the following decomposition of the terms. 
\begin{align}
    \tX_k' &= \alpha_k \Sigma^x + \tilde{C}'^x_k\zeta_k^x\label{eq:X_ass}\\
    \tZ_k' &= \beta_k \Sigma^{xy} + \tilde{C}'^{xy}_k\zeta_k^{xy}\label{eq:XY_ass}\\
    \tY_k' &= \beta_k \Sigma^{y} + \tC'^y_k\zeta_k^y,\label{eq:Y_ass}
\end{align}
where $\max\{\|\tilde{C}'^x_k\|_{Q_{22}},\|\tilde{C}'^{xy}_k\|_{Q_{22}},\|\tilde{C}'^y_k\|_{Q_{\Delta, \beta}}\}= \hbar_k$. Note that $\hbar_k$ depends on $k$. 

The goal of this proof is to show that there exists a problem dependent constant $k_0$ such that for $k\geq k_0$, we have
\begin{align*}
    \max\{\|\tilde{C}'^{y}_{k+1}\|_{Q_{\Delta, \beta}}, \|\tilde{C}'^{xy}_{k+1}\|_{Q_{22}}, \|\tilde{C}'^{x}_{k+1}\|_{Q_{22}}\} \leq \max\left\{\hbar_k,\hat{c}\right\},
\end{align*}
where $\hat{c}$ is a problem dependent constant, independent of $\hbar_k$ or $k$. We show that this constant $k_0$ is given as the maximum of six problem-dependent constants $k_1, \bar{k}_1,\bar{k}_2,\bar{k}_3,\bar{k}_4,\bar{k}_5$. The constant $k_1$ was defined in the proof of Lemma \ref{lem:L_k_bound}, and the rest of the constants are defined in Eq. \eqref{eq:k_1}-\eqref{eq:k_5} in the proof. Finding the closed form expressions of these constants will lead to the proof being extremely messy. Hence, we will only highlight the conditions that they must satisfy. It is worth noting that if $K_0$ in the step-size is chosen large enough, then $k_0$ can be set to zero. Having this, we define
\begin{align*}
    \bar{c}=\max\left\{\max_{1\leq k\leq k_0}\max\{\|\tilde{C}'^{y}_{k}\|_{Q_{\Delta, \beta}},\|\tilde{C}'^{xy}_{k}\|_{Q_{22}},\|\tilde{C}'^{x}_{k}\|_{Q_{22}}\},\hat{c}\right\}.
\end{align*}
for a problem-dependent constant $\bar{c}$. Then by induction, we have that $\max\{\|\tilde{C}'^{y}_{k}\|_{Q_{\Delta, \beta}},\|\tilde{C}'^{xy}_{k}\|_{Q_{22}},\|\tilde{C}'^{x}_{k}\|_{Q_{22}}\}\leq \bar{c}$ for all $k\geq 0.$. 

\begin{enumerate}
    \item  For $k\geq k_1$, by the definition of $L_k$ in \eqref{eq:L_k22}, we have $B_{21}^k=0$. We have
\begin{align*}
    \tX_{k+1}'=&\E[\tx_{k+1}\tx_{k+1}^\top] + \alpha_{k+1} (d_{k+1}^x+{d_{k+1}^x}^\top)\\
    =& \E [((I-\alpha_k B_{22}^k)\tx_k+\alpha_k u_k)((I-\alpha_k B_{22}^k)\tx_k+\alpha_k u_k)^\top]+\alpha_{k+1} (d_{k+1}^x+{d_{k+1}^x}^\top)\\
    =& \E [((I-\alpha_k A_{22}-\alpha_k C_{22}^k)\tx_k+\alpha_k u_k)((I-\alpha_k A_{22}-\alpha_k C_{22}^k)\tx_k+\alpha_k u_k)^\top]+ \alpha_{k+1} (d_{k+1}^x+{d_{k+1}^x}^\top)\\
     = &\E [\tx_{k}\tx_{k}^\top -\alpha_k A_{22}\tx_{k}\tx_{k}^\top - \alpha_k\tx_{k}\tx_{k}^\top A_{22}^\top+\alpha_k^2A_{22}\tx_k\tx_k^\top A_{22}^\top\\
    &-\alpha_k(I-\alpha_k A_{22}-\alpha_k C_{22}^k)\tx_{k}\tx_{k}^\top (C_{22}^k)^\top-\alpha_kC_{22}^k\tx_{k}\tx_{k}^\top(I-\alpha_k A_{22})^\top \\
    &+\alpha_k^2u_ku_k^\top + \alpha_k(I-\alpha_kA_{22}-\alpha_kC_{22}^k)\tx_ku_k^\top+\alpha_ku_k\tx_k^\top(I-\alpha_kA_{22}-\alpha_kC_{22}^k)^\top]
    + \alpha_{k+1} (d_{k+1}^x+{d_{k+1}^x}^\top)\\
    = &\tX_k'-\alpha_kA_{22}\tX_k'-\alpha_k\tX_k'A_{22}^\top \\
    &\underbrace{-\alpha_k(I-\alpha_k A_{22}-\alpha_k C_{22}^k)\tX_k' (C_{22}^k)^\top-\alpha_kC_{22}^k\tX_k'(I-\alpha_k A_{22})^\top+\alpha_k^2A_{22}\tX_k'A_{22}^\top}_{T_1}\\
    &+\underbrace{\alpha_k^2\E[u_ku_k^\top]}_{T_2} + \underbrace{\alpha_k\left((I-\alpha_kA_{22}-\alpha_kC_{22}^k)\E[\tx_ku_k^\top]+\E[u_k\tx_k^\top](I-\alpha_kA_{22}-\alpha_kC_{22}^k)^\top\right)}_{T_3}
    \\
    &+ \alpha_{k+1} (d_{k+1}^x+{d_{k+1}^x}^\top)-\alpha_k (d_k^x+{d_k^x}^\top) \\
    &+ \underbrace{\alpha_k^2A_{22} (d_k^x+{d_k^x}^\top)+\alpha_k^2 (d_k^x+{d_k^x}^\top)A_{22}^\top-\alpha_k^3A_{22} (d_k^x+{d_k^x}^\top)A_{22}^\top}_{T_4}\\
    &+\underbrace{\alpha_k^2(I-\alpha_k A_{22}-\alpha_k C_{22}^k) (d_k^x+{d_k^x}^\top) (C_{22}^k)^\top+\alpha_k^2C_{22}^k (d_k^x+{d_k^x}^\top)(I-\alpha_k A_{22})^\top}_{T_5}
\end{align*}

\begin{itemize}
\item For $T_1$, from Definition \ref{def_list} and Lemma \ref{lem:L_k_bound}, we have $\|C_{22}^k\|\leq \varrho_xA_{max}\frac{\beta_k}{\alpha_k}$. By the assumption of induction, we have $\|\tX_k'\|\leq \|\Sigma^x\| \alpha_k+ \kappa_{Q_{22}} \hbar_k\zeta_k^x$. Furtheremore, note that $\|I-\alpha_k A_{22}-\alpha_k C_{22}^k\|\leq 1+\alpha A_{max}+\beta\varrho_xA_{max}$. 
In addition, by Lemma \ref{lem:Lyap_bound} and \ref{lem:mix_time_sum} we have
\begin{align*}
    \|\Sigma^x\| \leq \tau_{mix}d\sigma^x.
\end{align*}

Hence, we have:
\begin{align*}
    \alpha_k\|(I-\alpha_k A_{22}-\alpha_k C_{22}^k)\tX_k' {C_{22}^k}^\top\|&\leq\beta_k  (1+\alpha A_{max}+\beta\varrho_xA_{max})(\alpha_k\tau_{mix}d \sigma^x + \kappa_{Q_{22}}\hbar_k\zeta_k^x) \varrho_xA_{max}\\
    &=\bc_1d \alpha_k\beta_k+\bc_2\hbar_k\beta_k\zeta_k^x,
\end{align*}
where $\bc_1=(1+\alpha A_{max}+\beta \varrho_xA_{max})\tau_{mix}\sigma^x\varrho_xA_{max}$ and $\bc_2=(1+\alpha A_{max}+\beta \varrho_xA_{max})\kappa_{Q_{22}}\varrho_xA_{max}$
\begin{align*}
    \alpha_k\|-C_{22}^k\tX_k'(I-\alpha_k A_{22})^\top\|& \leq\beta_k\varrho_xA_{max}(\alpha_k\tau_{mix}d \sigma^x + \kappa_{Q_{22}}\hbar_k\zeta_k^x)(1+\alpha A_{max})\\
    &=\bc_3d\alpha_k\beta_k+\bc_4\hbar_k\beta_k\zeta_k^x
\end{align*}
where $\bc_3=\varrho_xA_{max}(1+\alpha A_{max})\tau_{mix}\sigma^x$ and $\bc_4=\varrho_xA_{max}(1+\alpha A_{max})\kappa_{Q_{22}}$. In addition,
\begin{align*}
    \alpha_k^2 \|A_{22}\tX_k' A_{22}^\top\|&\leq A_{max}^2 (\tau_{mix}\sigma^xd \alpha_k^3+ \kappa_{Q_{22}}\hbar_k\zeta_k^x\alpha_k^2)\\
    &\leq A_{max}^2 \frac{\alpha^2}{\beta}(\tau_{mix}\sigma^xd \alpha_k\beta_k+ \kappa_{Q_{22}}\hbar_k\zeta_k^x\beta_k)\tag{$\xi>0.5\implies \alpha_k^2\leq \frac{\alpha^2}{\beta}\beta_k$}
\end{align*}
Combining all the bounds together, we get
\begin{align*}
    \Rightarrow \|T_1\|&\leq \bc_{5}d\beta_k\alpha_k+\bc_6\hbar_k\beta_k\zeta_k^x,
\end{align*}
where $\bc_5=\bar{c}_1 + \bar{c}_3 + \frac{\alpha^2}{\beta}A_{max}^2\tau_{mix}\sigma^x$ and $\bc_6 =\bar{c}_2 + \bar{c}_4 + \frac{\alpha^2}{\beta}A_{max}^2\kappa_{Q_{22}}$.

\item For $T_2$, using Lemma \ref{lem:noise_bound}, we have 
\begin{align*}
    T_2= \alpha_k^2 \Gamma_{22} +\alpha_k^2 \cR^u_k
\end{align*}
where $\|\cR^u_k\|\leq \left(1+\frac{\beta}{\alpha}\varrho_x\right)^2\left(\cc_1d^2 \sqrt{\alpha_k} + \cc_2d\hbar_k \sqrt{\zeta_k^x}\right)+\frac{\beta_k}{\alpha_k}\varrho_x
        \left(\|\Gamma_{21}\| + \frac{\beta}{\alpha}\varrho_x\|\Gamma_{11}\|\right)$.
\item For $T_3$, we first study $\E[u_k\tx_k^\top]$. We have $\E[u_k\tx_k^\top]=\E[w_k\tx_k^\top]+\frac{\beta_k}{\alpha_k}(L_{k+1}+A_{22}^{-1}A_{21})\E[v_k\tx_k^\top]$. By Lemma \ref{lem:noise_bound2} we have
\begin{align*}
    \E[u_k\tx_k^\top] =& \alpha_k\sum_{j=1}^\infty \E{[b_2(\tilde{O}_j)b_2(\tilde{O}_0)^\top]}+d_{k}^{xw}-d_{k+1}^{xw}+G^{(2,2)}_k\\
    &+\frac{\beta_k}{\alpha_k}(L_{k+1}+A_{22}^{-1}A_{21}) \left[\alpha_k\sum_{j=1}^\infty \E{[b_1(\tilde{O}_j)b_2(\tilde{O}_0)^\top]}+d_{k}^{xv} - d_{k+1}^{xv}+G^{(1,2)}_k\right]\\
    =&\alpha_k\sum_{j=1}^\infty \E{[b_2(\tilde{O}_j)b_2(\tilde{O}_0)^\top]} + d_k^{xw}-d_{k+1}^{xw} + \frac{\beta_k}{\alpha_k}(L_{k+1}+A_{22}^{-1}A_{21})\left(d_k^{xv}-d_{k+1}^{xv}\right) + \Rk{1},
\end{align*}
where $\|\Rk{1}\|\leq  g_3d^2(1+\frac{\beta}{\alpha}\varrho_x)(\alpha_k^{1.5}+\beta_k)+ \frac{b_{max}^2d\varrho_x}{1-\rho}\beta_k +\hbar_kg_4d\alpha_k\sqrt{\zeta_k^x} (1+\frac{\beta}{\alpha}\varrho_x) $. Recall that $d_k^x= d_k^{xw}+\frac{\beta_k}{\alpha_k}(L_{k+1}+A_{22}^{-1}A_{21})d_k^{xv}$. Thus, we can rewrite the above expression as
\begin{align*}
    \E[u_k\tx_k^\top] =&\alpha_k\sum_{j=1}^\infty \E{[b_2(\tilde{O}_j)b_2(\tilde{O}_0)^\top]} + d_k^{x}-d_{k+1}^{x} \\+&\left(\frac{\beta_{k+1}}{\alpha_{k+1}}(L_{k+2}+A_{22}^{-1}A_{21})-\frac{\beta_k}{\alpha_k}(L_{k+1}+A_{22}^{-1}A_{21})\right)d_{k+1}^{xv} + \Rk{1}\\
    =&\alpha_k\sum_{j=1}^\infty \E{[b_2(\tilde{O}_j)b_2(\tilde{O}_0)^\top]} + d_k^{x}-d_{k+1}^{x} + \Rk{2},
\end{align*}
where $\|\Rk{2}\|\leq \|\Rk{1}\|+\left\|\left(\frac{\beta_{k+1}}{\alpha_{k+1}}(L_{k+2}+A_{22}^{-1}A_{21})-\frac{\beta_k}{\alpha_k}(L_{k+1}+A_{22}^{-1}A_{21})\right)d_{k+1}^{xv}\right\|$. Observe we have 
\begin{align*}
    \left(\frac{\beta_{k+1}}{\alpha_{k+1}}(L_{k+2}+A_{22}^{-1}A_{21})-\frac{\beta_k}{\alpha_k}(L_{k+1}+A_{22}^{-1}A_{21})\right)d_{k+1}^{xv}=&\Bigg(\frac{\beta_{k+1}}{\alpha_{k+1}}(L_{k+2}-L_{k+1})\\
    +&\left(\frac{\beta_{k+1}}{\alpha_{k+1}}-\frac{\beta_k}{\alpha_k}\right)(L_{k+1}+A_{22}^{-1}A_{21})\Bigg)d_{k+1}^{xv}.
\end{align*}
Furthermore, we have
\begin{align*}
    \|d^{xv}_{k+1}\|\leq &\frac{2\sqrt{3d}}{1-\rho}\brc_{f}\sqrt{\E[\|\tx_{k+1}\|^2]}\tag{Lemma \ref{lem:tel_term_bound}}\\
    \leq &\frac{2\sqrt{3\brc}}{1-\rho}\brc_{f} d\tag{Lemma \ref{lem:boundedness}}.
\end{align*}
By Lemma, \ref{lem:L_k_bound} we have
\begin{align*}
    \frac{\beta_{k+1}}{\alpha_{k+1}}\|L_{k+2}-L_{k+1}\| \leq c^L_2\beta_{k+1}.
\end{align*}
And by Lemma, \ref{lem:step_size_gap} we have
\begin{align*}
    \left(\frac{\beta_{k+1}}{\alpha_{k+1}}-\frac{\beta_k}{\alpha_k}\right)(L_{k+1}+A_{22}^{-1}A_{21}) \leq \frac{\varrho_x(1-\xi)}{\alpha}\beta_k.
\end{align*}

Therefore, we get 
\begin{align*}
    \left\|\left(\frac{\beta_{k+1}}{\alpha_{k+1}}(L_{k+2}+A_{22}^{-1}A_{21})-\frac{\beta_k}{\alpha_k}(L_{k+1}+A_{22}^{-1}A_{21})\right)d_{k+1}^{xv}\right\|\leq \frac{2\sqrt{3\brc}}{1-\rho}\brc_{f} \left(2c^L_2+\frac{\varrho_x(1-\xi)}{\alpha}\right)d\beta_k\tag{$\beta_{k+1}\leq 2\beta_k$}. 
\end{align*}
Hence, we have 
\begin{align*}
    \|\Rk{2}\|&\leq g_3d^2\left(1+\frac{\beta}{\alpha}\varrho_x\right)(\alpha_k^{1.5}+\beta_k)+ \left(\frac{b_{max}^2\varrho_x}{1-\rho}+\frac{2\sqrt{3\brc}}{1-\rho}\brc_{f} \left(2c^L_2+\frac{\varrho_x(1-\xi)}{\alpha}\right)\right)d\beta_k\\
    &~~~+\hbar_kg_4d\alpha_k\sqrt{\zeta_k^x} \left(1+\frac{\beta}{\alpha}\varrho_x\right).
\end{align*}

Therefore,
\begin{align*}
    T_3=&\alpha_k(d_k^{x}+d_k^{x\top}-d_{k+1}^{x}-d_{k+1}^{x\top}) +\alpha_k^2\left[\sum_{j=1}^\infty \E{[b_2(\tilde{O}_j)b_2(\tilde{O}_0)^\top+b_2(\tilde{O}_0)b_2(\tilde{O}_j)^\top]}\right] +\Rk{3},
\end{align*}
where $\Rk{3} = - \alpha_k^2\left((A_{22}+C_{22}^k)\E[\tx_ku_k^\top] + \E[u_k\tx_k^\top](A_{22}+C_{22}^k)^\top\right)+\alpha_k\Rk{2}$. 
Hence,
\begin{align*}
    \|\Rk{3}\|&\leq \alpha_k\|\Rk{2}\| + 2A_{max}\left(1+\frac{\beta}{\alpha}\varrho_x\right)\alpha_k^2\|\E[\tx_ku_k^\top]\|.
\end{align*}
To bound the second term, we proceed as follows:
\begin{align*}
    \|\E[\tx_ku_k^\top]\|&\leq \sqrt{\E[\|\tx_k\|^2]}\sqrt{\E[\|u_k\|^2]}\tag{Cauchy-Schwarz}\\
    &\leq \sqrt{\E[\|\tx_k\|^2]}\sqrt{6d\left(1+\frac{\beta^2}{\alpha^2}\varrho_x^2\right)\left(b_{\max}^2+4A_{max}^2\brc\right)}\tag{Lemma \ref{lem:noise_crude_bound}}\\
    &\leq \sqrt{\alpha_k\underbar{c}_1d^2+\hbar_k d \kappa_{Q_{22}}\zeta_k^x}\sqrt{6d\left(1+\frac{\beta^2}{\alpha^2}\varrho_x^2\right)\left(b_{\max}^2+4A_{max}^2\brc\right)}\tag{Lemma \ref{lem:norm_bound} and Lemma \ref{lem:go_from_xp_to_x}}
\end{align*}

Combining both the bounds together, we get
\begin{align*}
    \|\Rk{3}\|&\leq \bc_7d^2(\alpha_k^{2.5}+\alpha_k\beta_k)+\bc_8\hbar_k d\alpha_k^2\sqrt{\zeta_k^x}.
\end{align*}
where 
\begin{align*}
    \bc_7&=\max\Bigg\{\left(1+\frac{\beta}{\alpha}\varrho_x\right)\left(g_3+2A_{max}\sqrt{6\underbar{c}_1\left(1+\frac{\beta^2}{\alpha^2}\varrho_x^2\right)\left(b_{\max}^2+4A_{max}^2\brc\right)}\right)\, \\
    &~~~\left(\frac{b_{max}^2\varrho_x}{1-\rho}+\frac{2\sqrt{3\brc}}{1-\rho}\brc_{f} \left(2c^L_2+\frac{\varrho_x(1-\xi)}{\alpha}\right)\right)\Bigg\},\\
    \bc_8&=\left(1+\frac{\beta}{\alpha}\varrho_x\right)\left(g_4+2A_{max}\sqrt{6\kappa_{Q_{22}}\left(1+\frac{\beta^2}{\alpha^2}\varrho_x^2\right)\left(b_{\max}^2+4A_{max}^2\brc\right)}\right).
\end{align*}

\item For $T_4$, we have
\begin{align*}
    \|T_4\|&\leq \alpha_k^2A_{max}\|d_{k}^x\|(2+\alpha A_{max})\\
    &\leq \alpha_k^2A_{max}(2+\alpha A_{max})\frac{2\sqrt{3d}}{1-\rho}\brc_{f}\left(1+\frac{\beta}{\alpha}\varrho_x\right)\sqrt{\E[\|\tx_k\|^2]}\tag{Lemma \ref{lem:tel_term_bound}}\\
    &\leq  \alpha_k^2A_{max}(2+\alpha A_{max})\frac{2\sqrt{3d}}{1-\rho}\brc_{f}\left(1+\frac{\beta}{\alpha}\varrho_x\right)\sqrt{\alpha_k\underbar{c}_1d^2+\hbar_k d \kappa_{Q_{22}}\zeta_k^x}\tag{Lemma \ref{lem:norm_bound} and Lemma \ref{lem:go_from_xp_to_x}}\\
    &\leq \bc_9d^2\alpha_k^{2.5}+\bc_{10}\hbar_k d\alpha_k^2\sqrt{\zeta_k^x}
\end{align*}
where 
\begin{align*}
    \bc_9&=A_{max}(2+\alpha A_{max})\frac{2\sqrt{3\underbar{c}_1}}{1-\rho}\brc_{f}\left(1+\frac{\beta}{\alpha}\varrho_x\right)\\
    \bc_{10}&=A_{max}(2+\alpha A_{max})\frac{2\sqrt{3\kappa_{Q_{22}}}}{1-\rho}\brc_{f}\left(1+\frac{\beta}{\alpha}\varrho_x\right).
\end{align*}

\item For $T_5$, we have
\begin{align*}
    \|T_5\|&\leq 4\alpha_k\beta_k(1+\alpha A_{max}+\beta\varrho_xA_{max})\|d_k^x\|\varrho_xA_{max}\\
    &\leq 4\alpha_k\beta_k(1+\alpha A_{max}+\beta\varrho_xA_{max})\varrho_xA_{max}\frac{2\sqrt{3d}}{1-\rho}\brc_{f}\left(1+\frac{\beta}{\alpha}\varrho_x\right)\sqrt{\E[\|\tx_k\|^2]}\tag{Lemma \ref{lem:noise_crude_bound}}\\
    &\leq \bc_{11}d\alpha_k\beta_k, \tag{Lemma \ref{lem:boundedness}}
\end{align*}
where $\bc_{11}=\frac{8\sqrt{3\brc}}{1-\rho}(1+\alpha A_{max}+\beta\varrho_xA_{max})\varrho_xA_{max}\brc_{f}\left(1+\frac{\beta}{\alpha}\varrho_x\right)$.
\end{itemize}

Hence, we have the following recursion 
\begin{align*}
    \tX_{k+1}' = &\tX_k'-\alpha_kA_{22}\tX_k'-\alpha_k\tX_k'A_{22}^\top+\alpha_k^2\Gamma^{x}+(\alpha_{k+1}-\alpha_k)(d_{k+1}^x+{d_{k+1}^x} ^\top)+\Rk{4}
\end{align*}
where $\|\Rk{4}\|\leq \bc_{12}d^2(\alpha_k^{2.5}+\alpha_k\beta_k)+\bc_{13}\hbar_k d(\beta_k\zeta_k^x+\alpha_k^2\sqrt{\zeta_k^x})$. Here 
\begin{align*}
    \bc_{12}&=\bc_{5}+\left(1+\frac{\beta}{\alpha}\varrho_x\right)^2\cc_1+\varrho_x
        \left(\|\Gamma_{21}\| + \frac{\beta}{\alpha}\varrho_x\|\Gamma_{11}\|\right)+\bc_7+\bc_9+\bc_{11},\\
    \bc_{13}&=\bc_6+\left(1+\frac{\beta}{\alpha}\varrho_x\right)^2\cc_2+\bc_8+\bc_{10}.
\end{align*}
Furthermore, to bound $(\alpha_{k+1}-\alpha_k)(d_{k+1}^x+{d_{k+1}^x} ^\top)$, we proceed as follows:
\begin{align*}
    \|(\alpha_{k+1}-\alpha_k)(d_{k+1}^x+{d_{k+1}^x} ^\top)\|\leq& 2|\alpha_{k+1}-\alpha_k|\|d_{k+1}^x\|\\
    \leq & \frac{2\xi}{\beta}\alpha_k\beta_k\|d_{k+1}^x\|\tag{Lemma \ref{lem:step_size_gap} and Assumption \ref{ass:step_size_main}}\\
    \leq&\frac{2\xi}{\beta}\alpha_k\beta_k\frac{2\sqrt{3d}}{1-\rho}\brc_{f}\left(1+\frac{\beta}{\alpha}\varrho_x\right)\sqrt{\E[\|\tx_{k+1}\|^2]}\tag{Lemma \ref{lem:noise_crude_bound}}\\
    \leq&\frac{\xi d}{\beta}\frac{4\sqrt{3}}{1-\rho}\alpha_k\beta_k\brc_{f}\left(1+\frac{\beta}{\alpha}\varrho_x\right)\brc\tag{Lemma \ref{lem:boundedness}}
\end{align*}

Hence,
\begin{align*}
    \tX_{k+1}' = &\tX_k'-\alpha_kA_{22}\tX_k'-\alpha_k\tX_k'A_{22}^\top+\alpha_k^2\Gamma^{x}+\Rk{5},
\end{align*}
where $\|\Rk{5}\|\leq \left(\bc_{12}+\frac{\xi}{\beta}\frac{4\sqrt{3}}{1-\rho}\brc_{f}\left(1+\frac{\beta}{\alpha}\varrho_x\right)\brc\right)d^2(\alpha_k^{2.5}+\alpha_k\beta_k)+\bc_{13}\hbar_k d(\beta_k\zeta_k^x+\alpha_k^2\sqrt{\zeta_k^x})$. 

By definition of $\tilde{C}'^x_{k}$ we have
\begin{align*}
    \tX_{k+1}' = &\alpha_k \Sigma^x + \tilde{C}'^x_k\zeta_k^x-\alpha_kA_{22}(\alpha_k \Sigma^x + \tilde{C}'^x_k\zeta_k^x)-\alpha_k(\alpha_k \Sigma^x + \tilde{C}'^x_k\zeta_k^x)A_{22}^\top+\alpha_k^2\Gamma^{x}+\Rk{5}\\
    =&\alpha_{k+1} \Sigma^x +(\alpha_{k}-\alpha_{k+1}) \Sigma^x +\tilde{C}'^x_k\zeta_k^x-\alpha_kA_{22}  \tilde{C}'^x_k\zeta_k^x-\alpha_k \tilde{C}'^x_k\zeta_k^xA_{22}^\top+\Rk{5}.\tag{Eq. \eqref{eq:sigma_x_def_main}}
\end{align*}

Define $\tilde{C}'^x_{k+1}\zeta_{k+1}^x = (\alpha_{k}-\alpha_{k+1}) \Sigma^x+\tilde{C}'^x_k\zeta_k^x-\alpha_kA_{22}  \tilde{C}'^x_k\zeta_k^x-\alpha_k \tilde{C}'^x_k\zeta_k^xA_{22}^\top+R_k^5$. We have
\begin{align*}
    \|\tilde{C}'^x_{k+1}\|_{Q_{22}} \leq \underbrace{\frac{|\alpha_k-\alpha_{k+1}|}{\zeta_{k+1}^x}\|\Sigma^x\|_{Q_{22}}}_{T_{6}} +\underbrace{\frac{\zeta_{k}^x}{\zeta_{k+1}^x} \left\| \tilde{C}'^x_k-\alpha_kA_{22}  \tilde{C}'^x_k-\alpha_k \tilde{C}'^x_kA_{22}^\top\right\|_{Q_{22}}}_{T_{7}}+\frac{1}{\zeta_{k+1}^x}\|R_k^5\|_{Q_{22}}.
\end{align*}
For $T_6$, we have 
\begin{align*}
    T_6\leq & \kappa_{Q_{22}}\tau_{mix}d\sigma^x\frac{\xi\alpha_k\beta_k}{\beta\zeta_{k}^x}\tag{Lemma \ref{lem:step_size_gap} and Assumption \ref{ass:step_size_main}}\\
    =& \frac{\kappa_{Q_{22}}\tau_{mix}d\sigma^x\xi\alpha}{(k+K_0)^{1+\xi-\min(1.5\xi,1)}}\\
    \leq &\kappa_{Q_{22}}\tau_{mix}d\sigma^x\xi\alpha_k\tag{$1-\min(1.5\xi,1)\geq 0$}.
\end{align*}
For $T_7$, we have
\begin{align*}
    T_7 =& \left\| \tilde{C}'^x_k-\alpha_kA_{22}  \tilde{C}'^x_k-\alpha_k \tilde{C}'^x_kA_{22}^\top\right\|_{Q_{22}} \nonumber\\
    &+\left\| \tilde{C}'^x_k-\alpha_kA_{22}  \tilde{C}'^x_k-\alpha_k \tilde{C}'^x_kA_{22}^\top\right\|_{Q_{22}}\left(\frac{\zeta_{k}^x}{\zeta_{k+1}^x}-1\right).
\end{align*}

But we have $\tilde{C}'^x_k-\alpha_kA_{22}  \tilde{C}'^x_k-\alpha_k \tilde{C}'^x_kA_{22}^\top = (I-\alpha_k A_{22})\tilde{C}'^x_k(I-\alpha_k A_{22})^\top-\alpha_k^2A_{22}\tilde{C}'^x_kA_{22}^\top$. Hence, 
\begin{align}
    \left\| \tilde{C}'^x_k-\alpha_kA_{22}  \tilde{C}'^x_k-\alpha_k \tilde{C}'^x_kA_{22}^\top\right\|_{Q_{22}}\leq& \|I-\alpha_k A_{22}\|_{Q_{22}}^2\|\tilde{C}'^x_k\|_{Q_{22}} +\alpha_k^2\|A_{22}\|_{Q_{22}}^2\|\tilde{C}'^x_k\|_{Q_{22}}\nonumber\\
    \leq& (1-\alpha_k a_{22})\|\tilde{C}'^x_k\|_{Q_{22}} +\kappa_{Q_{22}}^2A_{max}^2\alpha_k^2\|\tilde{C}'^x_k\|_{Q_{22}}\tag{Lemma \ref{lem:contraction_prop}}
\end{align}
Note that for the last inequality we assume that $k\geq k_C$. Combining the bounds together and using Lemma \ref{lem:step_size_gap} for the second term, we have
\begin{align*}
    T_7&\leq (1-\alpha_k a_{22})\|\tilde{C}'^x_k\|_{Q_{22}} ++\kappa_{Q_{22}}^2A_{max}^2\alpha_k^2\|\tilde{C}'^x_k\|_{Q_{22}} \\
    &~~+ \frac{2}{k+K_0}\left((1-\alpha_k a_{22})\|\tilde{C}'^x_k\|_{Q_{22}} +\kappa_{Q_{22}}^2A_{max}^2\alpha_k^2\|\tilde{C}'^x_k\|_{Q_{22}}\right)\\
    &\leq (1-\alpha_k a_{22})\|\tilde{C}'^x_k\|_{Q_{22}} +\kappa_{Q_{22}}^2A_{max}^2\alpha_k^2\|\tilde{C}'^x_k\|_{Q_{22}} + \frac{2(1+\kappa_{Q_{22}}^2A_{max}^2\alpha^2)}{k+K_0}\|\tilde{C}'^x_k\|_{Q_{22}} \\
    &\leq (1-\alpha_k a_{22})\|\tilde{C}'^x_k\|_{Q_{22}} + \frac{2+3\kappa_{Q_{22}}^2A_{max}^2\alpha^2}{k+K_0}\|\tilde{C}'^x_k\|_{Q_{22}}\tag{$\xi>0.5$}.
\end{align*}
Thus, we have
\begin{align*}
    \|\tilde{C}'^x_{k+1}\|_{Q_{22}} &\leq (1-\alpha_k a_{22})\|\tilde{C}'^x_k\|_{Q_{22}} + \kappa_{Q_{22}}\tau_{mix}d\sigma^x\xi\alpha_k + \frac{2+3\kappa_{Q_{22}}^2A_{max}^2\alpha^2}{k+K_0}\|\tilde{C}'^x_k\|_{Q_{22}}\\
    &~~~+\frac{1}{\zeta_{k+1}^x}\left(\left(\bc_{12}+\frac{\xi}{\beta}\frac{4\sqrt{3}}{1-\rho}\brc_{f}\left(1+\frac{\beta}{\alpha}\varrho_x\right)\brc\right)d^2(\alpha_k^{2.5}+\alpha_k\beta_k)+\bc_{13}\hbar_k d(\beta_k\zeta_k^x+\alpha_k^2\sqrt{\zeta_k^x})\right).
\end{align*}
Observe that $\zeta_k^x=\frac{1}{(k+K_0)^{\min\{1.5\xi, 1\}}}$. Thus, $\frac{\alpha_k^{2.5}+\alpha_k\beta_k}{\zeta_{k+1}^x}\leq 2(\alpha^{1.5}+\beta)\alpha_k$. Thus,
\begin{align*}
    \|\tilde{C}'^x_{k+1}\|_{Q_{22}} &\leq (1-\alpha_k a_{22})\|\tilde{C}'^x_k\|_{Q_{22}} + \kappa_{Q_{22}}\tau_{mix}d\sigma^x\xi\alpha_k + \frac{2+3\kappa_{Q_{22}}^2A_{max}^2\alpha^2}{k+K_0}\|\tilde{C}'^x_k\|_{Q_{22}}\\
    &~~~+2\left(\bc_{12}+\frac{\xi}{\beta}\frac{4\sqrt{3}}{1-\rho}\brc_{f}\left(1+\frac{\beta}{\alpha}\varrho_x\right)\brc\right)d^2(\alpha^{1.5}+\beta)\alpha_k+2\bc_{13}\hbar_k d\left(\beta_k+\frac{\alpha_k^2}{\sqrt{\zeta_k^x}}\right)\\
    &\leq  (1-\alpha_k a_{22})\hbar_k + \left(\kappa_{Q_{22}}\tau_{mix}d\sigma^x\xi+2\left(\bc_{12}+\frac{\xi}{\beta}\frac{4\sqrt{3}}{1-\rho}\brc_{f}\left(1+\frac{\beta}{\alpha}\varrho_x\right)\brc\right)d^2(\alpha^{1.5}+\beta)\right)\alpha_k \\
    &~~~+ \frac{2+3\kappa_{Q_{22}}^2A_{max}^2\alpha^2}{k+K_0}\hbar_k+2\bc_{13}\hbar_k d\left(\beta_k+\frac{\alpha_k^2}{\sqrt{\zeta_k^x}}\right).\tag{$\|\tilde{C}'^x_k\|_{Q_{22}}\leq \hbar_k$}
\end{align*}
Let $\bar{k}_1$ be a large enough constant such that 
\begin{align}\label{eq:k_1}
    \frac{\alpha_k a_{22}}{2}\geq \frac{2+3\kappa_{Q_{22}}^2A_{max}^2\alpha^2}{k+K_0}+2\bc_{13}d\left(\beta_k+\frac{\alpha_k^2}{\sqrt{\zeta_k^x}}\right)~~\forall k\geq \bar{k}_1
\end{align}
Furthermore, define $\bc^{(x)}=\kappa_{Q_{22}}\tau_{mix}\sigma^x\xi+2\left(\bc_{12}+\frac{\xi}{\beta}\frac{4\sqrt{3}}{1-\rho}\brc_{f}\left(1+\frac{\beta}{\alpha}\varrho_x\right)\brc\right)(\alpha^{1.5}+\beta)$. Then, for $k\geq \max\{k_1, \bar{k}_1\}$
\begin{align*}
    \|\tilde{C}'^x_{k+1}\|_{Q_{22}} &\leq\left(1-\frac{\alpha_k a_{22}}{2}\right)\hbar_k + \bc^{(x)}d^2\alpha_k.
\end{align*}

Hence, we have $\|\tilde{C}'^x_{k+1}\|_{Q_{22}}\leq \max\left\{\hbar_k,\frac{2\bc^{(x)}d^2}{a_{22}}\right\}$.

\item For $\tZ_{k}'$, we proceed as follows:
\begin{align*}
    \tZ_{k+1}'=& \E[\tx_{k+1}\ty_{k+1}^\top]+ \alpha_{k+1}d_{k+1}^{y}+\beta_{k+1}{d_{k+1}^{xv}}^\top\\
    =& \E [((I-\alpha_k B_{22}^k)\tx_k+\alpha_k u_k)((I-\beta_kB_{11}^k)\ty_k-\beta_kA_{12}\tx_k+\beta_kv_k)^\top]+\alpha_{k+1}d_{k+1}^{y}+\beta_{k+1}{d_{k+1}^{xv}}^\top\\
    =& \E [((I-\alpha_k A_{22}-\alpha_k C_{22}^k)\tx_k+\alpha_k u_k)((I-\beta_k(\Delta-A_{12}L_k))\ty_k-\beta_kA_{12}\tx_k+\beta_kv_k)^\top]\\
    &+\alpha_{k+1}d_{k+1}^{y}+\beta_{k+1}{d_{k+1}^{xv}}^\top\\
    =& \E[(I-\alpha_k A_{22}-\alpha_k C_{22}^k)\tx_k\ty_k^\top(I-\beta_k(\Delta-A_{12}L_k))^\top-\beta_k(I-\alpha_k A_{22}-\alpha_k C_{22}^k)\tx_k\tx_k^\top A_{12}^\top\\
    &+\beta_k(I-\alpha_k A_{22}-\alpha_k C_{22}^k)\tx_kv_k^\top\\
    &+\alpha_ku_k\ty_k^\top(I-\beta_k(\Delta-A_{12}L_k))^\top-\alpha_k\beta_ku_k\tx_k^\top A_{12}^\top+\alpha_k\beta_ku_kv_k^\top]+\alpha_{k+1}d_{k+1}^{y}+\beta_{k+1}{d_{k+1}^{xv}}^\top\\
     = &\E [\tx_{k}\ty_{k}^\top -\alpha_k A_{22}\tx_{k}\ty_{k}^\top -\beta_k\tx_k\tx_k^\top A_{12}^\top\\
    &-\beta_k(I-\alpha_kA_{22}-\alpha_kC_{22}^k)\tx_k\ty_k^\top(\Delta-A_{12}L_k)^\top-\alpha_k C_{22}^k\tx_{k}\ty_{k}^\top+\alpha_k\beta_k(A_{22}^\top+C_{22}^k)\tx_{k}\tx_{k}^\top A_{12}^\top +\beta_k\tx_kv_k^\top\\
    &+\alpha_ku_k\ty_k^\top+ \alpha_k\beta_ku_kv_k^\top-\alpha_k\beta_k(A_{22}+C_{22}^k)\tx_kv_k^\top-\alpha_k\beta_ku_k\ty_k^\top(\Delta-A_{12}L_k)^\top-\alpha_k\beta_ku_k\tx_k^\top A_{12}^\top]\\
    &+\alpha_{k+1} d_{k+1}^{y}+\beta_{k+1}{d_{k+1}^{xv}}^\top\\
    = &\tZ_k'-\alpha_kA_{22}\tZ_k'-\beta_k\tX_k'A_{12}^\top \\
    &\underbrace{-\beta_k(I-\alpha_kA_{22}-\alpha_kC_{22}^k)\tZ_k'(\Delta-A_{12}L_k)^\top-\alpha_k C_{22}^k{(\tZ_k')}^\top+\alpha_k\beta_k(A_{22}^\top+C_{22}^k)\tX_k' A_{12}^\top}_{T_8} \\
    &+\underbrace{\beta_k\E[\tx_kv_k^\top]+\alpha_k\E[u_k\ty_k^\top]+\alpha_k\beta_k\E[u_kv_k^\top]}_{T_9}\\
    &\underbrace{-\alpha_k\beta_k(A_{22}+C_{22}^k)\E[\tx_kv_k^\top]-\alpha_k\beta_k\E[u_k\ty_k^\top](\Delta-A_{12}L_k)^\top-\alpha_k\beta_k\E[u_k\tx_k^\top] A_{12}^\top}_{T_{10}}\\
    & \left. \begin{array}{r}
         + \beta_k(I-\alpha_kA_{22}-\alpha_kC_{22}^k)(\alpha_kd_k^{y}+\beta_k{d_k^{xv}}^\top)(\Delta-A_{12}L_k)^\top+\alpha_k C_{22}^k(\alpha_k{d_k^{y}}^\top+\beta_k{d_k^{xv}})
         \\ -\alpha_k^2\beta_k(A_{22}^\top+C_{22}^k)(d_k^x+{d_k^x}^\top) A_{12}^\top
         \end{array} \right\rbrace T_{11} \\
&\underbrace{+\alpha_kA_{22}(\alpha_{k} d_{k}^{y}+\beta_{k}{d_{k}^{xv}}^\top)+\beta_k\alpha_k(d_k^x+{d_k^x}^\top)A_{12}^\top}_{T_{12}}
    +\alpha_{k+1} d_{k+1}^{y}+\beta_{k+1}{d_{k+1}^{xv}}^\top-\alpha_{k} d_{k}^{y}-\beta_{k}{d_{k}^{xv}}^\top\\
\end{align*}

\begin{itemize}
    \item For $T_8$, we have:
    \begin{align*}
        T_8=\underbrace{-\beta_k(I-\alpha_kA_{22}-\alpha_kC_{22}^k)\tZ_k'(\Delta-A_{12}L_k)^\top}_{T_{8,1}}\underbrace{-\alpha_k C_{22}^k(\tZ_k')^\top}_{T_{8,2}}+\underbrace{\alpha_k\beta_k(A_{22}^\top+C_{22}^k)\tX_k 'A_{12}^\top}_{T_{8,3}}.
    \end{align*}
    By assumptions on induction, we get:
    \begin{align*}
        \|T_{8,1}\|\leq& \beta_k\|(I-\alpha_kA_{22}-\alpha_kC_{22}^k)\|\|\tZ_k'\|\|(\Delta-A_{12}L_k)\|\\
        \leq &\beta_k(1+\alpha A_{max} + \beta \varrho_x A_{max}) \varrho_y (\beta_k\sigma^{xy}d\tau_{mix} + \kappa_{Q_{22}}\hbar_k \zeta_k^{xy} )\tag{Eq. \eqref{eq:XY_ass}, Lemmas \ref{lem:Lyap_bound} and \ref{lem:mix_time_sum}}
    \end{align*}
    Recall $\|C_{22}^k\|\leq \varrho_xA_{max}\frac{\beta_k}{\alpha_k}$ from Definition \ref{def_list} and Lemma \ref{lem:L_k_bound}, we have:
    \begin{align*}
        \|T_{8,2}\|\leq \varrho_x A_{max}\beta_k(\beta_k\sigma^{xy}d\tau_{mix} + \kappa_{Q_{22}}\hbar_k \zeta_k^{xy} ) .
    \end{align*}
    In addition, we have:
    \begin{align*}
        \|T_{8,3}\|\leq& \alpha_k\beta_k\|(A_{22}^\top+C_{22}^k)\|\|\tX_k'\| \|A_{12}\|\\
        \leq & \alpha_k\beta_k A_{max}^2\left(1 + \varrho_x\frac{\beta}{\alpha}\right)(\sigma^xd\tau_{mix} \alpha_k + \hbar_k\kappa_{Q_{22}}\zeta_{k}^{x})
    \end{align*}
    Combining all the bounds, we have:
    \begin{align*}
        \|T_8\|\leq \bc_{14} d \beta_k^2+\bc_{15}\hbar_k\beta_k\zeta_{k}^{xy},
    \end{align*}
    where $\bc_{14} = (1+\alpha A_{max} + \beta \varrho_x A_{max}) \varrho_y \sigma^{xy}\tau_{mix} + \varrho_x A_{max}\sigma^{xy}\tau_{mix} + \alpha^2A_{max}^2\left(1 + \varrho_x\frac{\beta}{\alpha}\right)\sigma^x\tau_{mix}/\beta$ and $\bc_{15} = (1+\alpha A_{max} + \beta \varrho_x A_{max}) \varrho_y \kappa_{Q_{22}} + \varrho_x A_{max} \kappa_{Q_{22}} + \alpha A_{max}^2\left(1 + \varrho_x\frac{\beta}{\alpha}\right)\kappa_{Q_{22}}$. Here we used the fact that $\alpha_k\zeta_k^x\leq\alpha\zeta_k^{xy}$.
    
    \item For $T_9$, we have:
    \begin{align*}
        T_9= \underbrace{\beta_k\E[\tx_kv_k^\top]}_{T_{9,1}}+\underbrace{\alpha_k\E[u_k\ty_k^\top]}_{T_{9,2}}+\underbrace{\alpha_k\beta_k\E[u_kv_k^\top]}_{T_{9,3}}.
    \end{align*}
    For $T_{9,1}$, by Lemma \ref{lem:noise_bound2} we have
    \begin{align*}
        T_{9,1} = \alpha_k\beta_k\sum_{j=1}^\infty \E{[b_2(\tilde{O}_0)b_1(\tilde{O}_j)^\top]}+\beta_k(d_{k}^{xv} - d_{k+1}^{xv})^\top+\beta_k{G_k^{(1,2)}}^\top.
    \end{align*}

    For $T_{9,2}$, we have
    \begin{align*}
        T_{9,2} =& \alpha_k\E\left[\left(w_k+\frac{\beta_k}{\alpha_k}(L_{k+1}+A_{22}^{-1}A_{21})v_k\right) \ty_k^\top\right] \\
        =& \alpha_k\E\left[w_k \ty_k^\top\right] + \beta_k(L_{k+1}+A_{22}^{-1}A_{21})\E\left[v_k \ty_k^\top\right] \\
        =& \alpha_k\left(\beta_k\sum_{j=1}^\infty \E{[b_2(\tilde{O}_j)b_1(\tilde{O}_0)^\top]}+d_{k}^{yw}-d_{k+1}^{yw}+G^{(2,1)}_k\right) \\
        &+ \beta_k(L_{k+1}+A_{22}^{-1}A_{21})\left(\beta_k\sum_{j=1}^\infty \E{[b_1(\tilde{O}_j)b_1(\tilde{O}_0)^\top]} + d_{k}^{yv}-d_{k+1}^{yv}+G^{(1,1)}_k\right) \\
        =& \alpha_k\beta_k\sum_{j=1}^\infty \E{[b_2(\tilde{O}_j)b_1(\tilde{O}_0)^\top]}+\alpha_k\left(d_{k}^{yw}-d_{k+1}^{yw}\right)+\beta_k(L_{k+1}+A_{22}^{-1}A_{21})(d_{k}^{yv}-d_{k+1}^{yv})+\Rk{6},
    \end{align*}
    where 
    \begin{align*}
        \|\Rk{6}\| \leq & \alpha_k \|G^{(2,1)}_k\| +  \beta_k\|L_{k+1}+A_{22}^{-1}A_{21}\|\left(\beta_k\left\|\sum_{j=1}^\infty \E{[b_1(\tilde{O}_j)b_1(\tilde{O}_0)^\top]}\right\|+\|G_k^{(1,1)}\|\right)\\
        \leq & \alpha_k \left(g_1d^2\alpha_k\sqrt{\beta_k}+g_2d\hbar_k\alpha_k\sqrt{\zeta_k^y}\right) + \beta_k\varrho_x \left(\beta_k \frac{b_{max}^2d}{1-\rho} + g_1d^2\alpha_k\sqrt{\beta_k}+g_2d\hbar_k\alpha_k\sqrt{\zeta_k^y}\right)\\
        = & \bc_{16}d^2 (\alpha_k^{2}\sqrt{\beta_k} +\beta_k^2) +\bc_{17}d \hbar_k \alpha_k^2\sqrt{\zeta_k^y},
    \end{align*}
    where $\bc_{16} = g_1\left(1+\varrho_x\frac{\beta}{\alpha}\right) + \frac{\varrho_x b_{max}^2}{1-\rho}$ and $\bc_{17} = g_2\left(1+\varrho_x\frac{\beta}{\alpha}\right)$.

    For the final term, we have
    \begin{align*}
        T_{9,3} =&\alpha_k \beta_k\E\left[\left(w_k+\frac{\beta_k}{\alpha_k}(L_{k+1}+ A_{22}^{-1}A_{21})v_k\right)v_k^\top\right]\\
        =&\alpha_k \beta_k \E\left[w_kv_k^\top\right] + \beta_k^2(L_{k+1}+A_{22}^{-1}A_{21})\E[v_kv_k^\top]\\
        =&\alpha_k \beta_k(\Gamma_{21}+ \cR^{(2,1)}_k) + \beta_k^2 (L_{k+1}+ A_{22}^{-1} A_{21})\E[v_kv_k^\top],\tag{by Lemma \ref{lem:noise_bound}}
    \end{align*}
    where $\|\cR^{(2,1)}_k\|\leq \cc_1d^2 \sqrt{\alpha_k} + \cc_2d\hbar_k \sqrt{\zeta_k^x}$. We simply bound the second term using Lemma \ref{lem:noise_crude_bound}, to get $\|\beta_k^2 (L_{k+1}+ A_{22}^{-1} A_{21})\E[v_kv_k^\top]\|\leq 3d\beta_k^2\varrho_x\left(b_{\max}^2+4A_{max}^2\brc\right)$. Therefore,
    \begin{align*}
        T_{9,3} = \alpha_k\beta_k\Gamma_{21} + \Rk{7},
    \end{align*}
    where $\|\Rk{7}\| \leq \cc_1d^2 \beta_k\alpha_k^{1.5}+ 3d \beta_k^2\varrho_x \left(b_{\max}^2+ 4A_{max}^2 \brc\right)+ \cc_2d\hbar_k\alpha_k\beta_k \sqrt{\zeta_k^x}$. In total, for $T_9$, we have
    \begin{align*}
        T_9 = \alpha_k\beta_k \Gamma^{xy} +\beta_k(d_{k}^{xv} - d_{k+1}^{xv})^\top+\alpha_k\left(d_{k}^{yw}-d_{k+1}^{yw}\right)+\beta_k(L_{k+1}+A_{22}^{-1}A_{21})(d_{k}^{yv}-d_{k+1}^{yv})+ \Rk{8},
    \end{align*}
    where $\|\Rk{8}\|\leq \bc_{18}d^2 (\alpha_k^{2}\sqrt{\beta_k} +\beta_k^2) +\bc_{19}d \hbar_k \alpha_k^2\sqrt{\zeta_k^y}$. Here 
    \begin{align*}
        \bc_{18}&=\left(g_3+\cc_1\right)\sqrt{\frac{\beta}{\alpha}}+\bc_{16}+3\varrho_x \left(b_{\max}^2+ 4A_{max}^2 \brc\right),\\
        \bc_{19}&=\left(g_4+\cc_2\right)\frac{\beta}{\alpha}+\bc_{17}.
    \end{align*}
    Now rewriting $T_9$ as the following, we have
    \begin{align*}
        T_9 = &\alpha_k\beta_k \Gamma^{xy} +\beta_kd_{k}^{xv\top}-\beta_{k+1}d_{k+1}^{xv\top}+\alpha_kd_{k}^{yw} - \alpha_{k+1}d_{k+1}^{yw}\\
        &+\beta_k(L_{k+1}+A_{22}^{-1}A_{21})d_{k}^{yv}-\beta_{k+1}(L_{k+2}+A_{22}^{-1}A_{21})d_{k+1}^{yv}+ \Rk{9},\\
        = &\alpha_k\beta_k \Gamma^{xy} +\beta_kd_{k}^{xv\top}-\beta_{k+1}d_{k+1}^{xv\top}+\alpha_kd_{k}^{y} - \alpha_{k+1}d_{k+1}^{y}+\Rk{9}
    \end{align*}
    where $\Rk{9} = \Rk{8}+(\beta_{k+1}-\beta_k)d_{k+1}^{xv\top}+(\alpha_{k+1}-\alpha_{k})d_{k+1}^{yw}+(\beta_{k+1}
    (L_{k+2}+A_{22}^{-1}A_{21})-\beta_{k}(L_{k+1}+A_{22}^{-1}A_{21}))d_{k+1}^{yv}$. Using Lemmas \ref{lem:step_size_gap}, \ref{lem:tel_term_bound} and \ref{lem:boundedness}, we bound the second term as follows:
    \begin{align*}
        |\beta_{k+1}-\beta_k|\|d_{k+1}^{xv\top}\|&\leq \beta_k^2\frac{2\sqrt{3}\xi d}{\beta(1-\rho)}\brc_{f}\sqrt{\brc}.
    \end{align*}
    For the third term, we again use Lemmas \ref{lem:step_size_gap} and \ref{lem:tel_term_bound} to get
    \begin{align*}
        |\alpha_{k+1}-\alpha_k|\|d_{k+1}^{yw\top}\|&\leq \alpha_k\beta_k\frac{2\sqrt{3d}\xi }{\beta(1-\rho)}\brc_{f}\sqrt{\E[\|\ty_k\|^2]}\\
        &\leq \left(\alpha_k\beta_k\frac{2\sqrt{3d}\xi }{\beta(1-\rho)}\brc_{f}\right)\sqrt{\beta_k \underbar{c}_2d^2+\hbar_k d\kappa_{Q_{\Delta, \beta}}\zeta_k^y}\\
        &\leq \left(\alpha_k\beta_k\frac{2\sqrt{3}\xi }{\beta(1-\rho)}\brc_{f}\right)\left(\sqrt{\beta_k \underbar{c}_2}d^{1.5}+\hbar_k d\sqrt{\kappa_{Q_{\Delta, \beta}}\zeta_k^y}\right)\\
        &\leq\alpha_k^2\sqrt{\beta_k}\frac{2\sqrt{3\underbar{c}_2}\xi d^{1.5}}{\alpha(1-\rho)}\brc_{f} +\hbar_k d\alpha_k^2\sqrt{\zeta_k^y}\frac{2\sqrt{3\kappa_{Q_{\Delta, \beta}}}\xi }{\alpha(1-\rho)}\brc_{f}.
    \end{align*}
    For the last term, we proceed in a similar manner:
    \begin{align*}
        \|(\beta_{k+1}
    (L_{k+2}+A_{22}^{-1}A_{21})-\beta_{k}(L_{k+1}+A_{22}^{-1}A_{21}))d_{k+1}^{yv}\|&\leq \Bigg(\beta_{k+1}\|L_{k+2}-L_{k+1}\|\\
    &~+|\beta_{k+1}-\beta_k|\|L_{k+1}+A_{22}^{-1}A_{21}\|\Bigg)\|d_{k+1}^{yv}\|
    \end{align*}
    We bound these two terms as follows:
    \begin{align*}
        \beta_{k+1}\|L_{k+2}-L_{k+1}\|\|d_{k+1}^{yv}\|&\leq \alpha_{k+1}\beta_{k+1}\frac{2\sqrt{3d}\xi c_2^L}{\beta(1-\rho)}\brc_{f}\sqrt{\E[\|\ty_k\|^2]}\tag{Lemmas \ref{lem:L_k_bound} and \ref{lem:tel_term_bound}}\\
        &\leq \alpha_k^2\sqrt{\beta_k}\frac{8\sqrt{3\underbar{c}_2}\xi c_2^Ld^{1.5}}{\alpha(1-\rho)}\brc_{f} +\hbar_k d\alpha_k^2\sqrt{\zeta_k^y}\frac{8\sqrt{3\kappa_{Q_{\Delta, \beta}}}\xi c_2^L}{\alpha(1-\rho)}\brc_{f}\tag{$\alpha_{k+1}\leq 2\alpha_k, \beta_{k+1}\leq 2\beta_k$}\\
        |\beta_{k+1}-\beta_k|\|L_{k+1}+A_{22}^{-1}A_{21}\|\|d_{k+1}^{yv}\|&\leq \beta_k^2\frac{2\sqrt{3}\xi\varrho_x d}{\beta(1-\rho)}\brc_{f}\sqrt{\brc}.
    \end{align*}
    Thus, we have
    \begin{align*}
        \|(\beta_{k+1}
    (L_{k+2}+A_{22}^{-1}A_{21})&-\beta_{k}(L_{k+1}+A_{22}^{-1}A_{21}))d_{k+1}^{yv}\|\\
    &\leq \left(\alpha_k^2\sqrt{\beta_k}+\beta_k^2\right)\left(\frac{8\sqrt{3\underbar{c}_2}\xi c_2^Ld^{1.5}}{\alpha(1-\rho)}\brc_{f}+\frac{2\sqrt{3}\xi\varrho_x d}{\beta(1-\rho)}\brc_{f}\sqrt{\brc}\right) +\hbar_kd\alpha_k^2\sqrt{\zeta_k^y}\frac{8\sqrt{3\kappa_{Q_{\Delta, \beta}}}\xi c_2^L}{\alpha(1-\rho)}\brc_{f}.
    \end{align*}

    Combining the previous bounds, we get
    \begin{align*}
        \|R_k^9\|\leq& \bc_{20}d^2 (\alpha_k^{2}\sqrt{\beta_k} +\beta_k^2) +\bc_{21}d \hbar_k \alpha_k^2\sqrt{\zeta_k^y}
    \end{align*}
    where $\bc_{20}=\bc_{18}+\frac{2\sqrt{3}\xi }{\beta(1-\rho)}\brc_{f}\sqrt{\brc}(1+\varrho_x)+\frac{2\sqrt{3\underbar{c}_2}\xi}{\alpha(1-\rho)}\brc_{f}(1+4c_2^L)$ and $\bc_{21}=\bc_{19}+\frac{2\sqrt{3\kappa_{Q_{\Delta, \beta}}}\xi }{\alpha(1-\rho)}\brc_{f}(1+4c_2^L)$.
\item For $T_{10}$, we have:
\begin{align*}
    \|T_{10}\| \leq & \alpha_k\beta_k \left(\sqrt{\E[\|v_k\|^2]}\sqrt{\E[\|\tx_k\|^2]}A_{max}\left(1+\varrho_x\frac{\beta}{\alpha}\right)+\sqrt{\E[\|u_k\|^2]}\sqrt{\E[\|\ty_k\|^2]}\varrho_y+\sqrt{\E[\|u_k\|^2]}\sqrt{\E[\|\tx_k\|^2]}A_{max}\right)\tag{Cauchy-Schwarz inequality}\\
    \leq & \alpha_k\beta_k\sqrt{3d}\sqrt{\left(b_{\max}^2+4A_{max}^2\brc\right)}\left(A_{max}\Bigg(1+\varrho_x\frac{\beta}{\alpha}+\sqrt{2\left(1+\frac{\beta^2}{\alpha^2}\varrho_x^2\right)}\right)\sqrt{\E[\|\tx_k\|^2]}\\
    &~+\varrho_y\sqrt{2\left(1+\frac{\beta^2}{\alpha^2}\varrho_x^2\right)}\sqrt{\E[\|\ty_k\|^2]}\Bigg)\tag{Lemma \ref{lem:noise_crude_bound}}\\
    \leq &\alpha_k\beta_k\sqrt{3d}\sqrt{\left(b_{\max}^2+4A_{max}^2\brc\right)}\left(A_{max}\Bigg(1+\varrho_x\frac{\beta}{\alpha}+\sqrt{2\left(1+\frac{\beta^2}{\alpha^2}\varrho_x^2\right)}\right)\sqrt{\alpha_k\underbar{c}_1d^2 +  \hbar_k\kappa_{Q_{22}}d\zeta_k^x}\\
    &~+\varrho_y\sqrt{2\left(1+\frac{\beta^2}{\alpha^2}\varrho_x^2\right)}\sqrt{\beta_k\underbar{c}_2d^2+\hbar_k \kappa_{Q_{\Delta, \beta}}d\zeta_k^y}\Bigg)\tag{Lemma \ref{lem:go_from_xp_to_x}}\\
    \leq & \bc_{22}d^2 (\alpha_k^{2}\sqrt{\beta_k} +\beta_k^2) +\bc_{23}d \hbar_k \alpha_k^2\sqrt{\zeta_k^y},
\end{align*}
where 
\begin{align*}
    \bc_{22}&=\sqrt{3b_{\max}^2+12A_{max}^2\brc}\left(A_{max}\sqrt{\frac{\underbar{c}_1\beta}{\alpha}}\left(1+\varrho_x\frac{\beta}{\alpha}+\sqrt{2\left(1+\frac{\beta^2}{\alpha^2}\varrho_x^2\right)}\right)+\frac{\beta\varrho_y}{\alpha}\sqrt{2\underbar{c}_2\left(1+\frac{\beta^2}{\alpha^2}\varrho_x^2\right)}\right),\\
    \bc_{23}&=\sqrt{3b_{\max}^2+12A_{max}^2\brc}\left(\frac{A_{max}\sqrt{\kappa_{Q_{22}}}\beta}{\alpha}\left(1+\varrho_x\frac{\beta}{\alpha}+\sqrt{2\left(1+\frac{\beta^2}{\alpha^2}\varrho_x^2\right)}\right)+\frac{\beta\varrho_y}{\alpha}\sqrt{2\kappa_{Q_{\Delta, \beta}}\left(1+\frac{\beta^2}{\alpha^2}\varrho_x^2\right)}\right).
\end{align*}

\item For $T_{11}$, we first provide a bound on $\alpha_kd_k^{y}+\beta_k{d_k^{xv}}^\top$.
\begin{align*}
    \|\alpha_kd_k^{y}+\beta_k{d_k^{xv}}^\top\|&\leq \alpha_k\|d_k^{y}\|+\beta_k\|{d_k^{xv}}^\top\|\\
    &\leq \frac{2\sqrt{3d}}{1-\rho}\brc_{f}\left(\alpha_k\left(1+\frac{\beta}{\alpha}\varrho_x\right)\sqrt{\E[\|\ty_k\|^2]}+\beta_k\sqrt{\E[\|\tx_k\|^2]}\right)\tag{Lemma \ref{lem:tel_term_bound}}\\
    &\leq \frac{2\sqrt{3d}}{1-\rho}\brc_{f}\left(\alpha_k\left(1+\frac{\beta}{\alpha}\varrho_x\right)\sqrt{\beta_k\underbar{c}_2d^2+\hbar_k \kappa_{Q_{\Delta, \beta}}d\zeta_k^y}+\beta_k\sqrt{\alpha_k\underbar{c}_1d^2 +  \hbar_k\kappa_{Q_{22}}d\zeta_k^x}\right)\tag{Lemma \ref{lem:go_from_xp_to_x}}\\
    &\leq \bc_{24}d^2\alpha_k\sqrt{\beta_k} +\bc_{25} d\hbar_k \alpha_k\sqrt{\zeta_k^y}
\end{align*}
where 
\begin{align*}
    \bc_{24}&=\frac{2\sqrt{3}}{1-\rho}\brc_{f}\left(\sqrt{\underbar{c}_2}\left(1+\frac{\beta}{\alpha}\varrho_x\right)+\sqrt{\frac{\beta\underbar{c}_1}{\alpha}}\right),\\
    \bc_{25}&=\frac{2\sqrt{3}}{1-\rho}\brc_{f}\left(\sqrt{\kappa_{Q_{\Delta, \beta}}}\left(1+\frac{\beta}{\alpha}\varrho_x\right)+\sqrt{\kappa_{Q_{22}}}\frac{\beta}{\alpha}\right).
\end{align*}
Using the above, we get
\begin{align*}
     \|T_{11}\| \leq &  \beta_k\left((1+\alpha A_{max}+\beta \varrho_xA_{max})\varrho_y+\varrho_x\right)\|\alpha_k{d_k^{y}}^\top+\beta_k{d_k^{xv}}\|+2\alpha_k^2\beta_kA^2_{max}\left(1+\varrho_x\frac{\beta}{\alpha}\right)\|d_k^x\|\\
     \leq & \beta_k\left((1+\alpha A_{max}+\beta \varrho_xA_{max})\varrho_y+\varrho_x\right)\left(\bc_{24}d^2\alpha_k\sqrt{\beta_k} +\bc_{25} d\hbar_k \alpha_k\sqrt{\zeta_k^y}\right)\\
     &~+2\alpha_k^2\beta_kA^2_{max}\left(1+\varrho_x\frac{\beta}{\alpha}\right)\|d_k^x\|\\
     \leq & \beta_k\left((1+\alpha A_{max}+\beta \varrho_xA_{max})\varrho_y+\varrho_x\right)\left(\bc_{24}d^2\alpha_k\sqrt{\beta_k} +\bc_{25} d\hbar_k \alpha_k\sqrt{\zeta_k^y}\right)\\
     &~+2\alpha_k^2\beta_kA^2_{max}\left(1+\varrho_x\frac{\beta}{\alpha}\right)\frac{2\sqrt{3d}}{1-\rho}\brc_{f}\left(1+\frac{\beta}{\alpha}\varrho_x\right)\sqrt{\brc d}\tag{Lemmas \ref{lem:go_from_xp_to_x} and \ref{lem:boundedness}}\\
     \leq & \bc_{26}d^2\alpha_k^2\sqrt{\beta_k} +\bc_{27} d\hbar_k \alpha_k^2\sqrt{\zeta_k^y},
\end{align*}
where 
\begin{align*}
    \bc_{26}&=\left((1+\alpha A_{max}+\beta \varrho_xA_{max})\varrho_y+\varrho_x\right)\frac{\beta\bc_{24}}{\alpha}+\frac{4\sqrt{3\brc}}{1-\rho}\sqrt{\beta}A^2_{max}\left(1+\varrho_x\frac{\beta}{\alpha}\right)\brc_{f}\left(1+\frac{\beta}{\alpha}\varrho_x\right),\\
    \bc_{27}&=\left((1+\alpha A_{max}+\beta \varrho_xA_{max})\varrho_y+\varrho_x\right)\frac{\beta\bc_{25}}{\alpha}.
\end{align*}

\item Similar to $T_{11}$, for $T_{12}$ we have:
\begin{align*}
    \|T_{12}\| \leq & \alpha_kA_{max}\|\alpha_{k} d_{k}^{y}+\beta_{k}{d_{k}^{xv}}^\top\|+2\beta_k\alpha_kA_{max}\|d_k^x\|\\
    \leq& \alpha_kA_{max}\left(\bc_{24}d^2\alpha_k\sqrt{\beta_k} +\bc_{25} d\hbar_k \alpha_k\sqrt{\zeta_k^y}\right)+2\alpha_k\beta_kA_{max}\|d_k^x\|\\
    \leq &\alpha_kA_{max}\left(\bc_{24}d^2\alpha_k\sqrt{\beta_k} +\bc_{25} d\hbar_k \alpha_k\sqrt{\zeta_k^y}\right)+2\alpha_k\beta_kA_{max}\frac{2\sqrt{3d}}{1-\rho}\brc_{f}\left(1+\frac{\beta}{\alpha}\varrho_x\right)\sqrt{\E[\|\tx_k\|^2]}.\tag{Lemma \ref{lem:tel_term_bound}}\\
    \leq &\alpha_kA_{max}\left(\bc_{24}d^2\alpha_k\sqrt{\beta_k} +\bc_{25} d\hbar_k \alpha_k\sqrt{\zeta_k^y}\right)+2\alpha_k\beta_kA_{max}\frac{2\sqrt{3d}}{1-\rho}\brc_{f}\left(1+\frac{\beta}{\alpha}\varrho_x\right)\sqrt{\alpha_k\underbar{c}_1d^2 +  \hbar_
    kd\kappa_{Q_{22}}\zeta_k^x}.\tag{Lemma \ref{lem:go_from_xp_to_x}}\\
    \leq & \bc_{28}d^2\alpha_k^2\sqrt{\beta_k} +\bc_{29} d\hbar_k \alpha_k^2\sqrt{\zeta_k^y},
\end{align*}
where 
\begin{align*}
    \bc_{28}&=A_{max}\bc_{24}+A_{max}\frac{4\sqrt{3}}{1-\rho}\brc_{f}\left(1+\frac{\beta}{\alpha}\varrho_x\right)\sqrt{\frac{\beta\underbar{c}_1}{\alpha}},\\
    \bc_{29}&=A_{max}\bc_{26}+A_{max}\frac{4\sqrt{3}}{1-\rho}\brc_{f}\left(1+\frac{\beta}{\alpha}\varrho_x\right)\sqrt{\kappa_{Q_{22}}}\frac{\beta}{\alpha}.
\end{align*}
\end{itemize}

Combining everything, we have
\begin{align*}
    \tZ_{k+1}'= &\tZ_k'-\alpha_kA_{22}\tZ_k'-\beta_k\tX_k'A_{12}^\top +\alpha_k\beta_k \Gamma^{xy}+\Rk{10}
\end{align*}
where $\Rk{10}= T_8+\Rk{9}+T_{10}+T_{11}+T_{12}$ and $\|\Rk{10}\|\leq\bc_{30}d^2\left(\alpha_k^2\sqrt{\beta_k}+\beta_k^2\right)+\bc_{31}d\hbar_k\left(\beta_k\zeta_k^{xy}+\alpha_k^2\sqrt{\zeta_k^y}\right)$. Here 
\begin{align*}
    \bc_{30}&=\bc_{14}+\bc_{20}+\bc_{22}+\bc_{26}+\bc_{28},\\
    \bc_{31}&=\bc_{15}+\bc_{21}+\bc_{23}+\bc_{27}+\bc_{29}.
\end{align*}

Next, by induction on \eqref{eq:XY_ass}, we have
\begin{align*}
    \tZ_{k+1}' =& \beta_{k+1} \Sigma^{xy} + (\beta_{k}-\beta_{k+1}) \Sigma^{xy}+\tilde{C}'^{xy}_k\zeta_k^{xy}-\alpha_k A_{22}(\beta_k \Sigma^{xy} + \tilde{C}'^{xy}_k\zeta_k^{xy}) - \beta_k (\alpha_k \Sigma^x + \tilde{C}'^x_k\zeta_k^x)A_{12}^\top \\
    +& \alpha_k\beta_k \Gamma^{xy}+R_k^{10}\\
    =& \beta_{k+1} \Sigma^{xy} + (\beta_{k}-\beta_{k+1}) \Sigma^{xy}+\tilde{C}'^{xy}_k\zeta_k^{xy}-\alpha_k A_{22} \tilde{C}'^{xy}_k\zeta_k^{xy} - \beta_k  \tilde{C}'^x_k\zeta_k^xA_{12}^\top +\Rk{10}.\tag{by Eq. \eqref{eq:sigma_xy_def_main}}
\end{align*}
Define $\tilde{C}'^{xy}_{k+1}$ such that $\tilde{C}'^{xy}_{k+1}\zeta_{k+1}^{xy} = (\beta_{k}-\beta_{k+1}) \Sigma^{xy}+\tilde{C}'^{xy}_k\zeta_k^{xy}-\alpha_k A_{22} \tilde{C}'^{xy}_k\zeta_k^{xy} - \beta_k  \tilde{C}'^x_k\zeta_k^xA_{12}^\top +\Rk{10}$. We have
\begin{align*}
    \|\tilde{C}'^{xy}_{k+1}\|_{Q_{22}} \leq \underbrace{\frac{|\beta_k-\beta_{k+1}|}{\zeta_{k+1}^{xy}}\|\Sigma^{xy}\|_{Q_{22}}}_{T_{13}} +\underbrace{\frac{\zeta_{k}^{xy}}{\zeta_{k+1}^{xy}} \left\| \tilde{C}'^{xy}_k-\alpha_kA_{22}  \tilde{C}'^{xy}_k\right\|_{Q_{22}}}_{T_{14}}+\underbrace{\beta_k\frac{\zeta_k^x}{\zeta_{k+1}^{xy}}\|\tC_k'^x\|\|A_{12}\|}_{T_{15}}+\frac{1}{\zeta_{k+1}^{xy}}\|\Rk{10}\|_{Q_{22}}.
\end{align*}
For $T_{13}$, using Lemma \ref{lem:step_size_gap}, we have 
\begin{align*}
    T_{13}&\leq \frac{\beta_k^2}{\beta\zeta_{k+1}^{xy}}\|\Sigma^{xy}\|_{Q_{22}}\tag{Lemma \ref{lem:step_size_gap} and Assumption \ref{ass:step_size_main}}\\
    &\leq \kappa_{Q_{22}}d\tau_{mix}\sigma^{xy}\alpha_k\frac{\beta_k^2}{\beta\alpha_k\zeta_{k+1}^{xy}}\\
    &\leq \kappa_{Q_{22}}d\tau_{mix}\sigma^{xy}\alpha_k\frac{\beta}{\alpha}\tag{$2-\xi-\min(\xi+0.5, 2-\xi)\geq 0$}.
\end{align*}
For $T_{14}$, we have 
\begin{align*}
    T_{14} =& \left\| (I-\alpha_k A_{22})\tilde{C}'^{xy}_k\right\|_{Q_{22}}+ \left(\frac{\zeta_{k}^{xy}}{\zeta_{k+1}^{xy}}-1\right) \left\| (I-\alpha_k A_{22})\tilde{C}'^{xy}_k\right\|_{Q_{22}}\\
    \leq &\left\| I-\alpha_k A_{22}\right\|_{Q_{22}}\left\|\tilde{C}'^{xy}_k\right\|_{Q_{22}}+ \left(\frac{\zeta_{k}^{xy}}{\zeta_{k+1}^{xy}}-1\right) \left\| I-\alpha_k A_{22}\right\|_{Q_{22}}\left\|\tilde{C}'^{xy}_k\right\|_{Q_{22}}\tag{Matrix norm property}\\
    \leq &\left(1-\frac{\alpha_ka_{22}}{2}\right)\hbar_k+ \frac{\min\{\xi+0.5, 2-\xi\}}{k+K_0} \left(1-\frac{\alpha_ka_{22}}{2}\right)\hbar_k\tag{Lemma \ref{lem:step_size_gap} and $k>k_C$}\\
    \leq &\left(1-\frac{\alpha_ka_{22}}{2}\right)\hbar_k+ \frac{2}{\beta}\beta_k\hbar_k.
\end{align*}
For $T_{15}$, we have $T_{15}\leq 4A_{max}\beta_k\hbar_k$.
Combining all the bounds with the bound on $\Rk{10}$, we have
\begin{align*}
    \|\tilde{C}'^{xy}_{k+1}\|_{Q_{22}}&\leq \left(1-\frac{\alpha_ka_{22}}{2}\right)\hbar_k + \left(\beta_k\left(\frac{2}{\beta}+4A_{max}\right)+\frac{ \bc_{31}d(\alpha_k^2\sqrt{\zeta_{k}^{y}}+\beta_k\zeta_k^{xy})}{\zeta_{k+1}^{xy}}\right)\hbar_k \\
    &+\kappa_{Q_{22}}d\tau_{mix}\sigma^{xy}\alpha_k\frac{\beta}{\alpha}+ \frac{\bc_{30}d^2
    (\alpha_k^2\sqrt{\beta_k}+\beta_k^2) }{\zeta_{k+1}^{xy}}.
\end{align*}
Note that $\zeta_{k}^{xy}$ is of the same order as $\alpha_k^2\sqrt{\beta_k}, +\beta_k^2$, i.e., $\zeta_{k}^{xy}=\Theta\left(\alpha_k^2\sqrt{\beta_k}+\beta_k^2\right)$. Thus, we have

\begin{align*}
    \|\tilde{C}'^{xy}_{k+1}\|_{Q_{22}}&\leq \left(1-\frac{\alpha_ka_{22}}{2}\right)\hbar_k + \left(\beta_k\left(\frac{2}{\beta}+4A_{max}+4d\bc_{31}\right)+4\alpha\sqrt{\beta}\bc_{31}d\alpha_k\sqrt{\frac{\zeta_k^y}{\beta_k}}\right)\hbar_k \\
    &+ \alpha_k\left(\kappa_{Q_{22}}d\tau_{mix}\sigma^{xy}\frac{\beta}{\alpha}+\bc_{30}d^2\left(\alpha\sqrt{\beta}+\frac{\beta^2}{\alpha}\right)\right).
\end{align*}
Note that $\sqrt{\frac{\zeta_k^y}{\beta_k}}=o(1)$. Thus, there exists a large enough constant $\bar{k}_2$ such that 
\begin{align}\label{eq:k_2}
    \frac{\alpha_ka_{22}}{4}\geq \beta_k\left(\frac{2}{\beta}+4A_{max}+4d\bc_{31}\right)+4\alpha\sqrt{\beta}\bc_{31}d\alpha_k\sqrt{\frac{\zeta_k^y}{\beta_k}}~~\forall k\geq \bar{k}_2
\end{align}
Thus for all $k\geq \max\{k_1, \bar{k}_1, \bar{k}_2\}$, we get
\begin{align*}
    \|\tilde{C}'^{xy}_{k+1}\|_{Q_{22}}&\leq \left(1-\frac{\alpha_ka_{22}}{4}\right)\hbar_k  + \bc^{(z)}\alpha_k,
\end{align*}
where $\bc^{(z)}=\kappa_{Q_{22}}d\tau_{mix}\sigma^{xy}\frac{\beta}{\alpha}+\bc_{30}d^2\left(\alpha\sqrt{\beta}+\frac{\beta^2}{\alpha}\right)$. Hence, we have $\|\tilde{C}'^{xy}_{k+1}\|_{Q_{22}}\leq \max\left\{\hbar_k,\frac{4\bc^{(z)}d^2}{a_{22}}\right\}$.

    \item Finally, we have:
    \begin{align*}
        \ty_{k+1}=&\ty_k-\beta_k(B_{11}^k\ty_k+A_{12}\tx_k)+\beta_kv_k\\
        &=(I-\beta_kB_{11}^k)\ty_k-\beta_kA_{12}\tx_k+\beta_kv_k
    \end{align*}
    Then we have the following recursion:
    \begin{align*}
        \tY_{k+1}'&=(I-\beta_kB_{11}^k)\tY_k(I-\beta_kB_{11}^k)^\top-\beta_k(I-\beta_kB_{11}^k)\tZ_k^\top A_{12}^\top+\beta_k(I-\beta_kB_{11}^k)\E[\ty_kv_k^\top]\\&-\beta_kA_{12}\tZ_k(I-\beta_kB_{11}^k)^\top+\beta_k^2A_{12}\tX_kA_{12}^\top-\beta_k^2A_{12}\E[\tx_kv_k^\top]\\&+\beta_k\E[v_k\ty_k^\top](I-\beta_kB_{11}^k)^\top-\beta_k^2\E[v_k\tx_k^\top]A_{12}^\top+\beta_k^2\E[v_kv_k^\top]\\
        &+\beta_{k+1}(d_{k+1}^{yv}+{d_{k+1}^{yv}}^\top)\\
        &=\tY_k'-\beta_k\Delta\tY_k'-\beta_k\tY_k'\Delta^\top-\beta_k(\tZ_k')^\top A_{12}^\top-\beta_kA_{12}\tZ_k'\\
        &+\underbrace{\beta_kA_{12}L_k\tY_k'+\beta_k\tY_k'L_k^\top A_{12}^\top+\beta_k^2B_{11}^k\tY_k'B_{11}^{k\top}+\beta_k^2B_{11}^k\tZ_k^\top A_{12}^\top+\beta_k^2A_{12}\tZ_k B_{11}^k}_{T_{16}}\\
        &+\underbrace{\beta_k\E[\ty_kv_k^\top]+\beta_k\E[v_k\ty_k^\top]+\beta_k^2\E[v_kv_k^\top]}_{T_{17}}\\
        &+\underbrace{\beta_k^2A_{12}\tX_kA_{12}^\top-\beta_k^2A_{12}\E[\tx_kv_k^\top]-\beta_k^2\E[v_k\tx_k^\top]A_{12}^\top-\beta_k^2B_{11}^k\E[\ty_kv_k^\top]-\beta_k^2\E[v_k\ty_k^\top](B_{11}^k)^\top}_{T_{18}}\\
        &+\beta_{k+1}(d_{k+1}^{yv}+{d_{k+1}^{yv}}^\top)
        -\beta_{k}(d_{k}^{yv}+{d_{k}^{yv}}^\top)\\&+\underbrace{\beta_k^2\Delta(d_{k}^{yv}+{d_{k}^{yv}}^\top)+\beta_k^2(d_{k}^{yv}+{d_{k}^{yv}}^\top)\Delta^\top+\beta_k (\alpha_kd_k^{yw}+\beta_k{d_k^{xv}}^\top)A_{12}^\top+\beta_kA_{12}(\alpha_kd_k^{yw}+\beta_k{d_k^{xv}}^\top)}_{T_{19}}\\
        &\underbrace{-\beta_k^2A_{12}L_k(d_k^{yv}+{d_k^{yv}}^\top)-\beta_k^2(d_k^{yv}+{d_k^{yv}}^\top)L_k^\top A_{12}^\top-\beta_k^3 B_{11}^k(d_k^{yv}+{d_k^{yv}}^\top)B_{11}^{k\top}}_{T_{20}}        \end{align*}

        \begin{itemize}
            \item For $T_{16}$, we have
            \begin{align*}
                T_{16}=\underbrace{\beta_kA_{12}L_k\tY_k'+\beta_k\tY_k'L_k^\top A_{12}^\top}_{T_{16,1}}+\underbrace{\beta_k^2B_{11}^k\tY_k'B_{11}^{k\top}}_{T_{16,2}}+\underbrace{\beta_k^2B_{11}^k\tZ_k^\top A_{12}^\top+\beta_k^2A_{12}\tZ_k B_{11}^k}_{T_{16,3}}
            \end{align*}
            \begin{align*}
                \|T_{16,1}\|&\leq \frac{2\beta_k^2}{\alpha_k}A_{max}c_{L}^1(\beta_k\sigma^{y}d\tau_{mix} + \kappa_{Q_{\Delta, \beta}}\hbar_k \zeta_k^{y} )\\
                \|T_{16,2}\|&\leq \beta_k^2\varrho_y^2(\beta_k\sigma^{y}d\tau_{mix} + \kappa_{Q_{\Delta, \beta}}\hbar_k \zeta_k^{y})\\
                \|T_{16,3}\|&\leq 2\beta_k^2A_{max}\varrho_y(\beta_k\sigma^{xy}d\tau_{mix} + \kappa_{Q_{22}}\hbar_k \zeta_k^{xy} ).
            \end{align*}
            Combining the bounds, we get
            \begin{align*}
                \|T_{16}\|\leq \bc_{32}d\frac{\beta_k^3}{\alpha_k}+\bc_{33}\hbar_k\frac{\beta_k^2}{\alpha_k}\zeta_k^y.
            \end{align*}
            where 
            \begin{align*}
                \bc_{32}&=2A_{max}c_L^1\sigma^y\tau_{mix}+\alpha\varrho_y^2\sigma^yd\tau_{mix}+2A_{max}\varrho_y\sigma^{xy}\tau_{mix}\\
                \bc_{33}&=2A_{max}c_L^1\kappa_{Q_{\Delta, \beta}}+\alpha\varrho_y^2\kappa_{Q_{\Delta, \beta}}+2A_{max}\varrho_y\kappa_{Q_{22}}.
            \end{align*}
            \item For $T_{17}$, using Lemmas \ref{lem:noise_bound} and \ref{lem:noise_bound2} we have
            \begin{align*}
                T_{17}=& \beta_k \Bigg(\beta_k\sum_{j=1}^\infty \E{[b_1(\tilde{O}_0)b_1(\tilde{O}_j)^\top]} + \left(d_{k}^{yv}-d_{k+1}^{yv}\right)^\top+\left(G^{(1,1)}_k\right)^\top \\
                +& \beta_k\sum_{j=1}^\infty \E{[b_1(\tilde{O}_j)b_1(\tilde{O}_0)^\top]} + d_{k}^{yv}-d_{k+1}^{yv}+G^{(1,1)}_k\Bigg) + \beta_k^2\left(\Gamma_{11}+\cR^{(1,1)}_k\right)\\
                = & \beta_k^2\Gamma^y+\beta_k\left(d_{k}^{yv}-d_{k+1}^{yv}\right)^\top +\beta_k\left(d_{k}^{yv}-d_{k+1}^{yv}\right)+ \Rk{11},
            \end{align*}
            where $\|\Rk{11}\|\leq \left(\cc_1\sqrt{\frac{\beta}{\alpha}}+2g_1\right)d^2\alpha_k\beta_k^{1.5} + \left(\cc_2 \frac{\beta}{\alpha}+2g_2\right)d\hbar_k\alpha_k\beta_k\sqrt{\zeta_k^y}$.
            \item For $T_{18}$, we have
            \begin{align*}
                T_{18}=\underbrace{\beta_k^2A_{12}\tX_kA_{12}^\top}_{T_{18,1}}\underbrace{-\beta_k^2A_{12}\E[\tx_kv_k^\top]-\beta_k^2\E[v_k\tx_k^\top]A_{12}^\top}_{T_{18,2}}\underbrace{-\beta_k^2B_{11}^k\E[\ty_kv_k^\top]-\beta_k^2\E[v_k\ty_k^\top](B_{11}^k)^\top}_{T_{18,3}}
            \end{align*}
            \begin{align*}
                \|T_{18,1}\|&\leq \beta_k^2A_{max}^2\|\tX_k\|\leq \beta_k^2A_{max}^2(\alpha_k\underbar{c}_1d +  \hbar_k\kappa_{Q_{22}}\zeta_k^x)\tag{Lemma \ref{lem:go_from_xp_to_x}}\\
                \|T_{18,2}\|&\leq 2\beta_k^2A_{max}\sqrt{\E[\|\tx_k\|^2]}\sqrt{\E[\|v_k\|^2]}\tag{Cauchy-Schwarz}\\
                &\leq 2\beta_k^2A_{max}\sqrt{\alpha_k\underbar{c}_1d^2 +  \hbar_kd\kappa_{Q_{22}}\zeta_k^x}\sqrt{3d\left(b_{\max}^2+4A_{max}^2\brc\right)}\tag{Lemma \ref{lem:go_from_xp_to_x} and \ref{lem:noise_crude_bound}}\\
                &\leq 2\beta_k^2A_{max}\sqrt{3\left(b_{\max}^2+4A_{max}^2\brc\right)}\left(\sqrt{\alpha_k\underbar{c}_1}d^{1.5} +  \hbar_kd\sqrt{\kappa_{Q_{22}}\zeta_k^x}\right)\\
                \|T_{18,3}\|&\leq 2\beta_k^2\varrho_y\sqrt{\E[\|\ty_k\|^2]}\sqrt{\E[\|v_k\|^2]}\tag{Cauchy-Schwarz}\\
                &\leq 2\beta_k^2\varrho_y\sqrt{\beta_k\underbar{c}_2d^2 +  \hbar_kd\kappa_{Q_{\Delta, \beta}}\zeta_k^y}\sqrt{3d\left(b_{\max}^2+4A_{max}^2\brc\right)}\tag{Lemma \ref{lem:go_from_xp_to_x} and \ref{lem:noise_crude_bound}}\\
                &\leq 2\beta_k^2\varrho_y\sqrt{3\left(b_{\max}^2+4A_{max}^2\brc\right)}\left(\sqrt{\frac{\beta}{\alpha}}\sqrt{\alpha_k\underbar{c}_2}d^{1.5} +  \hbar_kd\sqrt{\kappa_{Q_{\Delta, \beta}}\zeta_k^x}\right).
            \end{align*}
            Combining the bounds, we get
            \begin{align*}
                \|T_{18}\|\leq \bc_{34}d^{1.5}\beta_k^2\sqrt{\alpha_k}+\bc_{35}d\hbar_k\beta_k^2\sqrt{\zeta_k^x}
            \end{align*}
            where 
            \begin{align*}
                \bc_{34}&=A_{max}^2\sqrt{\alpha}\underbar{c}_1+2A_{max}\sqrt{3\underbar{c}_1\left(b_{\max}^2+4A_{max}^2\brc\right)}+2\varrho_y\sqrt{\frac{3\beta\underbar{c}_2}{\alpha}\left(b_{\max}^2+4A_{max}^2\brc\right)}\\
                \bc_{35}&=A_{max}^2\sqrt{\alpha}\kappa_{Q_{22}}+2A_{max}\sqrt{3\kappa_{Q_{22}}\left(b_{\max}^2+4A_{max}^2\brc\right)}+2\varrho_y\sqrt{\frac{3\beta\kappa_{Q_{\Delta, \beta}}}{\alpha}\left(b_{\max}^2+4A_{max}^2\brc\right)}.
            \end{align*}

            \item For $T_{19}$, we have
            \begin{align*}
                T_{19}=\underbrace{\beta_k^2\Delta(d_{k}^{yv}+{d_{k}^{yv}}^\top)+\beta_k^2(d_{k}^{yv}+{d_{k}^{yv}}^\top)\Delta^\top}_{T_{19,1}}+\underbrace{\beta_k (\alpha_kd_k^{yw}+\beta_k{d_k^{xv}}^\top)A_{12}^\top+\beta_kA_{12}(\alpha_kd_k^{yw}+\beta_k{d_k^{xv}}^\top)}_{T_{19,2}}
            \end{align*}
            \begin{align*}
                \|T_{19,1}\|&\leq 4\beta_k^2\|\Delta\|\|d_k^{yv}\|\\
                &\leq \beta_k^2\|\Delta\|\frac{8\sqrt{3d}}{1-\rho}\brc_{f}\sqrt{\E[\|\ty_k\|^2]}\tag{Lemma \ref{lem:tel_term_bound}}\\
                &\leq \beta_k^2\|\Delta\|\frac{8\sqrt{3d}}{1-\rho}\brc_{f}\sqrt{\beta_k \underbar{c}_2d^2+\hbar_kd \kappa_{Q_{\Delta, \beta}}\zeta_k^y}\tag{Lemma \ref{lem:go_from_xp_to_x}}\\
                &\leq \beta_k^2\|\Delta\|\frac{8\sqrt{3}}{1-\rho}\brc_{f}\left(\sqrt{\beta_k \underbar{c}_2}d^{1.5}+\hbar_kd \sqrt{\kappa_{Q_{\Delta, \beta}}\zeta_k^y}\right)\\
                \|T_{19,2}\|&\leq 2\beta_kA_{max}\left(\alpha_k\|d_k^{yw}\|+\beta_k\|d_k^{xv}\|\right)\\
                &\leq \beta_kA_{max}\frac{4\sqrt{3d}}{1-\rho}\brc_{f}\left(\alpha_k\sqrt{\E[\|\ty_k\|^2]}+\beta_k\sqrt{\E[\|\tx_k\|^2]}\right)\tag{Lemma \ref{lem:tel_term_bound}}\\
                &\leq \beta_kA_{max}\frac{4\sqrt{3d}}{1-\rho}\brc_{f}\left(\alpha_k\sqrt{\beta_k \underbar{c}_2d^2+\hbar_kd \kappa_{Q_{\Delta, \beta}}\zeta_k^y}+\beta_k\sqrt{\alpha_k \underbar{c}_1d^2+\hbar_kd \kappa_{Q_{22}}\zeta_k^x}\right)\tag{Lemma \ref{lem:go_from_xp_to_x}}\\
                &\leq \beta_kA_{max}\frac{4\sqrt{3}}{1-\rho}\brc_{f}\left(\left(\alpha_k\sqrt{\beta_k}\left(\sqrt{\underbar{c}_2}+\sqrt{\frac{\beta}{\alpha}\underbar{c}_1}\right) d^{1.5}\right)+\hbar_kd \alpha_k\sqrt{\zeta_k^y}\left(\sqrt{\kappa_{Q_{\Delta, \beta}}}+\frac{\beta}{\alpha}\sqrt{\kappa_{Q_{22}}}\right)\right).\\
            \end{align*}
            Combining the bounds, we get
            \begin{align*}
                \|T_{19}\|\leq \bc_{36}d^{1.5}\alpha_k\beta_k^{1.5}+\bc_{37}d\hbar_k\beta_k\alpha_k\sqrt{\zeta_k^y}
            \end{align*}
             where 
            \begin{align*}
                \bc_{36}&=\frac{4\sqrt{3}}{1-\rho}\brc_{f}\left(\frac{2\beta}{\alpha}\|\Delta\|\sqrt{\underbar{c}_2}+A_{max}\left(\sqrt{\underbar{c}_2}+\sqrt{\frac{\beta}{\alpha}\underbar{c}_1}\right)\right)\\
                \bc_{37}&=\frac{4\sqrt{3}}{1-\rho}\brc_{f}\left(\frac{2\beta}{\alpha}\|\Delta\|\sqrt{\kappa_{Q_{\Delta, \beta}}}+A_{max}\left(\sqrt{\kappa_{Q_{\Delta, \beta}}}+\frac{\beta}{\alpha}\sqrt{\kappa_{Q_{22}}}\right)\right).
            \end{align*}
            
            \item For $T_{20}$, we have
            \begin{align*}
                T_{20}=\underbrace{-\beta_k^2A_{12}L_k(d_k^{yv}+{d_k^{yv}}^\top)-\beta_k^2(d_k^{yv}+{d_k^{yv}}^\top)L_k^\top A_{12}^\top}_{T_{20,1}}\underbrace{-\beta_k^3 B_{11}^k(d_k^{yv}+{d_k^{yv}}^\top)B_{11}^{k\top}}_{T_{20,2}}
            \end{align*}
            \begin{align*}
                \|T_{20,1}\|&\leq 4A_{max}c_1^L\frac{\beta_k^3}{\alpha_k}\|d_k^{yv}\|\tag{Lemma \ref{lem:L_k_bound}}\\
                &\leq A_{max}c_1^L\frac{\beta_k^3}{\alpha_k}\frac{8\sqrt{3d}}{1-\rho}\brc_{f}\sqrt{\E[\|\ty_k\|^2]}\tag{Lemma \ref{lem:tel_term_bound}}\\
                &\leq A_{max}c_1^L\brc_{f}\brc d\frac{\beta_k^3}{\alpha_k}\frac{8\sqrt{3}}{1-\rho}\tag{Lemma \ref{lem:boundedness}}\\
                \|T_{20,2}\|&\leq \beta_k^3\varrho_y^2\brc_{f}\brc d\frac{4\sqrt{3}}{1-\rho}.
            \end{align*}
            Combining the bounds, we get
            \begin{align*}
                \|T_{20}\|\leq \bc_{38}d\frac{\beta_k^3}{\alpha_k}.
            \end{align*}
        \end{itemize}
        where $\bc_{38}=\frac{4\sqrt{3}}{1-\rho}\brc_{f}\brc \left(2A_{max}c_1^L+\alpha\varrho_y^2\right)$. Combining the bounds for all the terms, we get
        \begin{align*}
            \tY_{k+1}' &=\tY_k'-\beta_k\Delta\tY_k'-\beta_k\tY_k'\Delta^\top-\beta_k\tZ_k'A_{12}^\top-\beta_kA_{12}(\tZ_k')^\top +\beta_k^2\Gamma^y+(\beta_{k+1}-\beta_k)(d_{k+1}^{yv}+d_{k+1}^{yv \top})+ \Rk{12},
        \end{align*}
        where $\|\Rk{12}\|\leq \bc_{39}d^2\left(\frac{\beta_k^3}{\alpha_k}+\alpha_k\beta_k^{1.5}\right)+\bc_{40}d\hbar_k\left(\frac{\beta_k^2}{\alpha_k}\zeta_k^y+\beta_k\alpha_k\sqrt{\zeta_k^y}\right)$ and 
        \begin{align*}
            \bc_{39}&= \bc_{32}+\bc_{38}+\cc_1\sqrt{\frac{\beta}{\alpha}}+2g_1+\sqrt{\frac{\beta}{\alpha}}\bc_{34}+\bc_{36}\\
            \bc_{40}&=\bc_{33}+\cc_2 \frac{\beta}{\alpha}+2g_2+\frac{\beta}{\alpha}\bc_{35}+\bc_{37}.
        \end{align*}

        Using Lemma \ref{lem:step_size_gap} and \ref{lem:tel_term_bound}, we have 
        \begin{align*}
            \|(\beta_{k+1}-\beta_k)(d_{k+1}^{yv}+d_{k+1}^{yv \top})\|&\leq 2\beta\beta_k^2\|d_{k+1}^{yv}\|\\
            &\leq \beta_k^2\frac{4\beta\sqrt{3d}}{1-\rho}\brc_{f}\sqrt{\E[\|\ty_{k+1}\|^2]}\tag{Lemma \ref{lem:tel_term_bound}}\\
            &\leq \beta_k^2\frac{4\beta\sqrt{3d}}{1-\rho}\brc_{f}\sqrt{\underbar{c}_7d^2 \beta_k + \underbar{c}_8d\hbar_k \zeta_k^y}\\
            &\leq \beta_k^2\frac{4\beta\sqrt{3}}{1-\rho}\brc_{f}\left(\sqrt{\underbar{c}_7\beta_k}d^{1.5}+d\hbar_k\sqrt{\underbar{c}_8\zeta_k^y}\right).
        \end{align*}
        
        Hence, 
        \begin{align*}
            \tY_{k+1}' &=\tY_k'-\beta_k\Delta\tY_k'-\beta_k\tY_k'\Delta^\top-\beta_k(\tZ_k')^{\top}A_{12}^\top-\beta_kA_{12}\tZ_k'+\beta_k^2\Gamma^y + \Rk{13},
        \end{align*}
        where $\|\Rk{13}\|\leq d^2\left(\bc_{39}+\frac{4\beta^2\sqrt{3\underbar{c}_7}}{\alpha(1-\rho)}\brc_{f}\right)\left(\frac{\beta_k^3}{\alpha_k}+\alpha_k\beta_k^{1.5}\right)+d\hbar_k\left(\bc_{40}+\frac{4\beta^2\sqrt{3\underbar{c}_8}}{\alpha(1-\rho)}\brc_{f}\right)\left(\frac{\beta_k^2}{\alpha_k}\zeta_k^y+\beta_k\alpha_k\sqrt{\zeta_k^y}\right)$. 

        Substituting \eqref{eq:Y_ass} we get
        \begin{align*}
            \tY_{k+1}' =& \beta_{k+1}\Sigma^{y}+ (\beta_k-\beta_{k+1}) \Sigma^{y} + \tC'^y_k\zeta_k^y - \beta_k\Delta (\beta_k \Sigma^{y} + \tC'^y_k\zeta_k^y) - \beta_k(\beta_k \Sigma^{y} + \tC'^y_k\zeta_k^y)\Delta^\top \\
            &- \beta_k(\beta_k \Sigma^{xy} + \tilde{C}'^{xy}_k\zeta_k^{xy})^\top A_{12}^\top-\beta_k A_{12}(\beta_k \Sigma^{xy} + \tilde{C}'^{xy}_k\zeta_k^{xy})+\beta_k^2\Gamma^y + \Rk{13}\\
            =&\beta_{k+1}\Sigma^{y}+ \frac{\beta_k^2}{\beta} \Sigma^{y} + \tC'^y_k\zeta_k^y - \beta_k\Delta (\beta_k \Sigma^{y} + \tC'^y_k\zeta_k^y) - \beta_k(\beta_k \Sigma^{y} + \tC'^y_k\zeta_k^y)\Delta^\top\tag{Assumption \ref{ass:step_size_main}} \\
            &- \beta_k(\beta_k \Sigma^{xy} + \tilde{C}'^{xy}_k\zeta_k^{xy})^\top A_{12}^\top-\beta_k A_{12}(\beta_k \Sigma^{xy} + \tilde{C}'^{xy}_k\zeta_k^{xy})+\beta_k^2\Gamma^y + \Rk{14}\\
            =&\beta_{k+1}\Sigma^{y}+  \tC'^y_k\zeta_k^y - \beta_k\Delta ( \tC'^y_k\zeta_k^y) - \beta_k( \tC'^y_k\zeta_k^y)\Delta^\top - \beta_k( \tilde{C}'^{xy}_k\zeta_k^{xy})A_{12}^\top-\beta_k A_{12}( \tilde{C}'^{xy}_k\zeta_k^{xy})^\top + \Rk{14}\tag{Eq. \eqref{eq:sigma_y_def_main}}
        \end{align*}
        where $\Rk{14}=\Rk{13}+\left(\beta_k-\beta_{k+1}-\frac{\beta_k^2}{\beta}\right)\Sigma^y$. Note that 
        \begin{align*}
            \beta_k-\beta_{k+1}-\frac{\beta_k^2}{\beta}&=\frac{\beta}{(k+K_0)(k+K_0+1)}-\frac{\beta}{(k+K_0)^2}\\
            &=\frac{\beta}{(k+K_0)^2(k+K_0+1)}\\
            &\leq \frac{2\beta_k^3}{\beta^2}.
        \end{align*}
        Using the above relation, we get
        \begin{align*}
            \|\Rk{14}\|\leq \bc_{41}d^2\left(\frac{\beta_k^3}{\alpha_k}+\alpha_k\beta_k^{1.5}\right)+\bc_{42}d\hbar_k\left(\frac{\beta_k^2}{\alpha_k}\zeta_k^y+\beta_k\alpha_k\sqrt{\zeta_k^y}\right).
        \end{align*}
        where 
        \begin{align*}
            \bc_{41}&=\bc_{39}+\frac{4\beta^2\sqrt{3\underbar{c}_7}}{\alpha(1-\rho)}\brc_{f}+\frac{2\alpha\sigma^y\tau_{mix}}{\beta^2}\\
            \bc_{42}&=\bc_{40}+\frac{4\beta^2\sqrt{3\underbar{c}_8}}{\alpha(1-\rho)}\brc_{f}.
        \end{align*}
        
        Define $\tilde{C}'^{y}_{k+1}$ such that $\tilde{C}'^{y}_{k+1}\zeta_{k+1}^{y} = \tC'^y_k\zeta_k^y - \beta_k\Delta ( \tC'^y_k\zeta_k^y) - \beta_k( \tC'^y_k\zeta_k^y)\Delta^\top - \beta_k( \tilde{C}'^{xy}_k\zeta_k^{xy})A_{12}^\top-\beta_k A_{12}( \tilde{C}'^{xy}_k\zeta_k^{xy})^\top + \Rk{14}$. We have
    \begin{align*}
        \|\tilde{C}'^{y}_{k+1}\|_{Q_{\Delta,\beta}} \leq &\frac{\zeta_k^y}{\zeta_{k+1}^y}\|(I-\beta_k\Delta)\tC'^y_k(I-\beta_k\Delta)^\top\|_{Q_{\Delta,\beta}} +\frac{\beta_k^2\zeta_k^y}{\zeta_{k+1}^y}\|\Delta \tC'^y_k\Delta^\top\|_{Q_{\Delta,\beta}} \\
        &+ \frac{\beta_k }{\zeta_{k+1}^y}\|( \tilde{C}'^{xy}_k\zeta_k^{xy})A_{12}^\top+ A_{12}( \tilde{C}'^{xy}_k\zeta_k^{xy})^\top\|_{Q_{\Delta,\beta}} +\frac{1 }{\zeta_{k+1}^y}\| \Rk{14}\|_{Q_{\Delta,\beta}} \\
        \leq &\frac{\zeta_k^y}{\zeta_{k+1}^y} \|(I-\beta_k\Delta)\tC'^y_k(I-\beta_k\Delta)^\top\|_{Q_{\Delta,\beta}}+\frac{\|\Delta\|^2\kappa_{Q_{\Delta, \beta}}^2\hbar_k\beta_k^2\zeta_k^y}{\zeta_{k+1}^y} + \frac{2A_{max}\kappa_{Q_{\Delta, \beta}}\hbar_k\beta_k\zeta_{k}^{xy}}{\zeta_{k+1}^y}\\
        &+\frac{\bc_{41}d^2\left(\frac{\beta_k^3}{\alpha_k}+\alpha_k\beta_k^{1.5}\right)+\bc_{42}d\hbar_k\left(\frac{\beta_k^2}{\alpha_k}\zeta_k^y+\beta_k\alpha_k\sqrt{\zeta_k^y}\right)}{\zeta_{k+1}^y}\\
        \leq &\frac{\zeta_k^y}{\zeta_{k+1}^y} \|(I-\beta_k\Delta)\tC'^y_k(I-\beta_k\Delta)^\top\|_{Q_{\Delta,\beta}} \\
        &+\frac{2A_{max}\kappa_{Q_{\Delta, \beta}}\hbar_k\beta_k\zeta_{k}^{xy}+\bc_{42}d\hbar_k\left(\frac{\beta_k^2\zeta_k^y}{\alpha_k}+\beta_k\alpha_k\sqrt{\zeta_k^y}\right)+\|\Delta\|^2\kappa_{Q_{\Delta, \beta}}^2\hbar_k\beta_k^2\zeta_k^y}{\zeta_{k+1}^y} \\
        &+ \frac{\bc_{41}d^2\left(\frac{\beta_k^3}{\alpha_k}+\alpha_k\beta_k^{1.5}\right)}{\zeta_{k+1}^y}\\
        \leq & \underbrace{\frac{\zeta_k^y}{\zeta_{k+1}^y}\|(I-\beta_k\Delta)\tC'^y_k(I-\beta_k\Delta)^\top\|_{Q_{\Delta,\beta}}}_{T_{21}} 
        + \frac{\bc_{43}\hbar_k\beta_k\zeta_k^{xy}}{\zeta_{k+1}^y} +\frac{\bc_{41}d^2\left(\frac{\beta_k^3}{\alpha_k}+\alpha_k\beta_k^{1.5}\right)}{\zeta_{k+1}^y}.
    \end{align*}
    
    where for the last inequality we used $\frac{\beta_k^2\zeta_k^y}{\alpha_k}+\beta_k\alpha_k\sqrt{\zeta_k^y}\leq \left(\frac{\beta}{\alpha}+\alpha\right) \beta_k\zeta_{k}^{xy}$, $\beta_k^2\zeta_k^y\leq \beta\beta_k\zeta_k^{xy}$, and 
    \begin{align*}
        \bc_{43}=2A_{max}\kappa_{Q_{\Delta, \beta}}+\bc_{42}\left(\frac{\beta}{\alpha}+\alpha\right)+\|\Delta\|^2\beta\kappa_{Q_{\Delta, \beta}}^2.
    \end{align*}
    Next we aim at analyzing $T_{21}$. First, note that $T_{21}\leq \frac{\zeta_k^y}{\zeta_{k+1}^y}\|I-\beta_k\Delta\|_{Q_\Delta,\beta}^2\|\tC'^y_k\|_{Q_{\Delta,\beta}}\leq \frac{\zeta_k^y}{\zeta_{k+1}^y}\|I-\beta_k\Delta\|_{Q_\Delta,\beta}^2\hbar_k$. Recall that $Q_{\Delta,\beta}$ is the solution to the following Lyapunov equation:
\begin{align*}
    \left(\Delta-\frac{\beta^{-1}}{2}I\right)^\top Q_{\Delta,\beta}+ Q_{\Delta,\beta}\left(\Delta-\frac{\beta^{-1}}{2}I\right)&=I\\
    \Rightarrow \Delta^\top Q_{\Delta,\beta}+Q_{\Delta,\beta}\Delta&=I+\beta^{-1}Q_{\Delta,\beta}.
\end{align*}

Hence,
\begin{align*}
        \|I-\beta_k \Delta\|_{Q_{\Delta,\beta}}^2&=\max_{\|x\|_{Q_{\Delta,\beta}}= 1}x^\top(I-\beta_k \Delta)^\top Q_{\Delta,\beta}(I-\beta_k \Delta)x\\
        &=\max_{\|x\|_{Q_{\Delta,\beta}}= 1} \left(x^\top Q_{\Delta,\beta}x-\beta_k x^\top(\Delta^\top Q_{\Delta,\beta}+Q_{\Delta,\beta}\Delta)x+\beta_k^2x^\top \Delta^\top Q_{\Delta,\beta}\Delta x\right)\\
        &\leq 1-\beta_k\min_{\|x\|_{Q_{\Delta,\beta}}=1}\|x\|^2-\beta_k\beta^{-1}+\beta_k^2 \max_{\|x\|_{Q_{\Delta,\beta}}=1}\|\Delta x\|_{Q_{\Delta,\beta}}^2\\
        &\leq 1-\beta_k\|Q_{\Delta,\beta}\|^{-1}-\beta_k\beta^{-1}+\beta_k^2\|\Delta\|_{Q_{\Delta,\beta}}^2.
    \end{align*}
    Let $\bar{k}_3$ to be such that 
    \begin{align}\label{eq:k_3}
        -\beta_k\|Q_{\Delta,\beta}\|^{-1}+\beta_k^2\|\Delta\|_{Q_{\Delta,\beta}}^2\leq -\frac{3\beta_k\|Q_{\Delta,\beta}\|^{-1}}{4}~~\forall k\geq \bar{k}_3
    \end{align}
    Then, for $k\geq \max\{k_1, \bar{k}_1, \bar{k}_2, \bar{k}_3\}$ we have
    \begin{align*}
        \|I-\beta_k \Delta\|_{Q_{\Delta,\beta}}^2 &\leq 1-\frac{3\beta_k\|Q_{\Delta,\beta}\|^{-1}}{4}-\beta_k\beta^{-1}.
    \end{align*}
    In the inequality above, by choosing a larger $\bar{k}_3$, instead of $-\frac{3 \beta_k \|Q_{\Delta, \beta}\|^{-1}}{4}$, we could get a tighter bound such as $-\frac{5\beta_k\|Q_{\Delta,\beta}\|^{-1}}{6}$. This is the reason why $c_0(\varrho)$ in Theorem \ref{thm:Markovian_main} might be arbitrarily large as $\varrho$ goes to zero. Hence, we have
    \begin{align*}
        T_{21} &\leq \frac{\zeta^y_k}{\zeta_{k+1}^y}\left(1-\left(\frac{3\|Q_{\Delta,\beta}\|^{-1}}{4}+\beta^{-1}\right)\beta_k\right) \hbar_k\\
        &\leq \left(1-\left(\frac{3\|Q_{\Delta,\beta}\|^{-1}}{4}+\beta^{-1}\right)\beta_k\right) \hbar_k+\frac{\zeta_k^y-\zeta_{k+1}^y}{\zeta_{k+1}^y}\left(1-\left(\frac{3\|Q_{\Delta,\beta}\|^{-1}}{4}+\beta^{-1}\right)\beta_k\right) \hbar_k.
    \end{align*}
    Furthermore, we have 
    \begin{align*}
        \frac{\zeta_k^y-\zeta_{k+1}^y}{\zeta_{k+1}^y} &= \frac{\zeta_k^y-\zeta_{k+1}^y}{\zeta_{k}^y}\frac{\zeta_k^y}{\zeta_{k+1}^y}\\
        &\leq \frac{1+q_{\Delta,\beta}\min(\xi-0.5, 1-\xi)}{k+K_0}\left(1+\frac{1}{k+K_0}\right)^{1+q_{\Delta,\beta}\min(\xi-0.5, 1-\xi)}\tag{Lemma \ref{lem:step_size_gap}}\\
        &\leq \beta_k\beta^{-1}\left(1+q_{\Delta,\beta}\min(\xi-0.5, 1-\xi)\right) \left(1+\frac{\|Q_{\Delta,\beta}\|^{-1}\beta}{4}\right),
    \end{align*}
    where in the last inequality we assumed $\bar{k}_4$ is such that 
    \begin{align}\label{eq:k_4}
        \left(1+\frac{1}{k+K_0}\right)^{1+q_{\Delta,\beta}\min(\xi-0.5, 1-\xi)}\leq \left(1+\frac{\|Q_{\Delta,\beta}\|^{-1}\beta}{4}\right)~~\forall k\geq \bar{k}_4
    \end{align}
    Hence, for $k\geq \max\{k_1, \bar{k}_1, \bar{k}_2, \bar{k}_3, \bar{k}_4\}$, we have
    \begin{align*}
        T_{21}&\leq \left(1-\left(\frac{3\|Q_{\Delta,\beta}\|^{-1}}{4}+\beta^{-1}\right)\beta_k\right)\left(1+\beta_k\beta^{-1}\left(1+q_{\Delta,\beta}\min(\xi-0.5, 1-\xi)\right) \left(1+\frac{\|Q_{\Delta,\beta}\|^{-1}\beta}{4}\right)\right) \hbar_k\\
        &\leq \left(1-\left(\frac{3\|Q_{\Delta,\beta}\|^{-1}}{4}+\beta^{-1}\right)\beta_k+\beta_k\beta^{-1}\left(1+q_{\Delta,\beta}\min(\xi-0.5, 1-\xi)\right) \left(1+\frac{\|Q_{\Delta,\beta}\|^{-1}\beta}{4}\right)\right) \hbar_k\\
        &= \left(1-\frac{3\beta_k\|Q_{\Delta,\beta}\|^{-1}}{4}+\beta_k\beta^{-1}q_{\Delta,\beta}\min(\xi-0.5, 1-\xi)\left(1+\frac{\|Q_{\Delta,\beta}\|^{-1}\beta}{4}\right)\right) \hbar_k\\
        &=\left(1-\frac{3\beta_k\|Q_{\Delta,\beta}\|^{-1}}{4}+\frac{\beta_k\|Q_{\Delta,\beta}\|^{-1}\min(\xi-0.5, 1-\xi)}{4}\right) \hbar_k\tag{$q_{\Delta,\beta} = \beta\|Q_{\Delta,\beta}\|^{-1}/\left(4+\beta\|Q_{\Delta,\beta}\|^{-1}\right)$}\\
        &\leq \left(1-\frac{3\beta_k\|Q_{\Delta,\beta}\|^{-1}}{4}+\frac{\beta_k\|Q_{\Delta,\beta}\|^{-1}}{16}\right) \hbar_k\tag{$\max_{0.5<\xi<1} \min\{0.5-\xi, 1-\xi\}=1/4$}\\
        &=\left(1-\frac{11\beta_k\|Q_{\Delta,\beta}\|^{-1}}{16}\right) \hbar_k.
    \end{align*}
    
    Combining the bounds, we get
        
    \begin{align*}
        T_{21}+&\frac{\bc_{43}\hbar_k\beta_k\zeta_k^{xy}}{\zeta_{k+1}^y}\leq \left(1-\frac{11\beta_k\|Q_{\Delta,\beta}\|^{-1}}{16}\right) \hbar_k+\frac{\bc_{43}\hbar_k\beta_k\zeta_k^{xy}}{\zeta_{k+1}^y}.
    \end{align*}
    Finally, we choose $\bar{k}_5$ large enough such that 
    \begin{align}\label{eq:k_5}
        \frac{\bc_{43}\beta_k\zeta_k^{xy}}{\zeta_{k+1}^y}\leq \frac{3\beta_k\|Q_{\Delta,\beta}\|^{-1}}{16}~~\forall k\geq \bar{k}_5
    \end{align}
    This can always be done since $\zeta_k^{xy}=o(\zeta_k^{y})$. Thus, for all $k\geq \max\{k_1, \bar{k}_1, \bar{k}_2, \bar{k}_3, \bar{k}_4, \bar{k}_5\}$, we get
    \begin{align*}
        \|\tilde{C}'^{y}_{k+1}\|_{Q_{\Delta,\beta}}\leq \left(1-\frac{\beta_k\|Q_{\Delta,\beta}\|^{-1}}{2}\right) \hbar_k+\frac{\bc_{41}d^2\left(\frac{\beta_k^3}{\alpha_k}+\alpha_k\beta_k^{1.5}\right)}{\zeta_{k+1}^y}.
    \end{align*}
    Note that $\frac{\left(\frac{\beta_k^3}{\alpha_k}+\alpha_k\beta_k^{1.5}\right)}{\zeta_{k+1}^y}\leq 4\beta_k\left(\frac{\beta^2}{\alpha}+\alpha\sqrt{\beta}\right)$. Denote $\bc^{(y)}=4\bc_{41}d^2\left(\frac{\beta^2}{\alpha}+\alpha\sqrt{\beta}\right)$. This implies

    \begin{align*}
    \|\tilde{C}'^{y}_{k+1}\|_{Q_{\Delta,\beta}}
    &\leq \Bigg(1-\beta_k\frac{\|Q_{\Delta,\beta}\|^{-1}}{2}\Bigg)\hbar_k+\bc^{(y)}\beta_k.
    \end{align*}
    Hence, we have $\|\tilde{C}'^{y}_{k+1}\|_{Q_{\Delta,\beta}}\leq \max\left\{\hbar_k,\frac{2\bc^{(y)}d^2}{\|Q_{\Delta,\beta}\|^{-1}}\right\}$. 
\end{enumerate}

Combining the above bounds, we have 
\begin{align}
    \max\{\|\tilde{C}'^{x}_{k+1}\|_{Q_{22}},\|\tilde{C}'^{xy}_{k+1}\|_{Q_{22}},\|\tilde{C}'^{y}_{k+1}\|_{Q_{\Delta,\beta}}\}\leq \max\left\{\hbar_k,\frac{2\bc^{(x)}d^2}{a_{22}},\frac{4\bc^{(z)}d^2}{a_{22}}, \frac{2\bc^{(y)}d^2}{\|Q_{\Delta,\beta}\|^{-1}}\right\}. \label{eq:C_k_iter}
\end{align}

Define $k_0=\max\{k_1, \bar{k}_1, \bar{k}_2, \bar{k}_3, \bar{k}_4, \bar{k}_5\}$, which is a finite problem dependent number, and
\begin{align*}
    \bar{c}d^2=\max\left\{\max_{0\leq k\leq k_0}\max\{\|\tilde{C}'^{y}_{k}\|_{Q_{\Delta, \beta}},\|\tilde{C}'^{xy}_{k}\|_{Q_{22}},\|\tilde{C}'^{x}_{k}\|_{Q_{22}}\},\frac{2\bc^{(x)}d^2}{a_{22}},\frac{4\bc^{(z)}d^2}{a_{22}}, \frac{2\bc^{(y)}d^2}{\|Q_{\Delta,\beta}\|^{-1}}\right\}.
\end{align*}
Note that here $\bar{c}$ is a bounded, problem dependent constant. To find an absolute bound on $\max\{\|\tilde{C}'^{y}_{k}\|_{Q_{\Delta, \beta}},\|\tilde{C}'^{xy}_{k}\|_{Q_{22}},\|\tilde{C}'^{x}_{k}\|_{Q_{22}}\}$ for $0\leq k\leq k_0$, we use Lemma \ref{lem:boundedness} as follows. Note that we have $\tilde{C}'^{y}_{k}\zeta_k^y=\tY_k'-\beta_k\Sigma^y$. Thus,
\begin{align*}
    \|\tilde{C}'^{y}_{k}\|_{Q_{\Delta, \beta}}\zeta_k^y&\leq \|\tY_k'\|_{Q_{\Delta, \beta}}+\beta_k\|\Sigma^y\|_{Q_{\Delta, \beta}}\\
    &\leq \kappa_{Q_{\Delta, \beta}}\|\tY_k'\|+\beta_k\|\Sigma^y\|_{Q_{\Delta, \beta}} \tag{Norm equivalence}\\
    &\leq \kappa_{Q_{\Delta, \beta}}\left(\|\tY_k\|+2\beta_k\|d_k^{yv}\|\right)+\beta_k\|\Sigma^y\|_{Q_{\Delta, \beta}}\\
    &\leq \kappa_{Q_{\Delta, \beta}}\left(\E[\|\ty_k\|^2]+2\beta_k\left(\frac{2\sqrt{3d}}{1-\rho}\brc_{f}\sqrt{\E[\|\ty_k\|^2]}\right)\right)+\beta_k\|\Sigma^y\|_{Q_{\Delta, \beta}}\tag{Lemma \ref{lem:tel_term_bound}}\\
    &\leq \kappa_{Q_{\Delta, \beta}}\left(\brc d+\frac{4d\sqrt{3\brc}}{1-\rho}\brc_{f}\right)+\beta_k\|\Sigma^y\|_{Q_{\Delta, \beta}}.\tag{Lemma \ref{lem:boundedness}}
\end{align*}
Note that $\beta_k/\zeta_k^y$ is an increasing function. Thus, using the above bound, we get
\begin{align*}
    \|\tilde{C}'^{y}_{k}\|_{Q_{\Delta, \beta}}\leq \frac{\kappa_{Q_{\Delta, \beta}}d}{\zeta_{k_0}^y}\left(\brc +\frac{4\sqrt{3\brc}}{1-\rho}\brc_{f}\right)+\frac{\beta_{k_0}}{\zeta_{k_0}^y}\|\Sigma^y\|_{Q_{\Delta, \beta}}~~ 0\leq k \leq k_0
\end{align*}
Using similar steps for $\tilde{C}'^{x}_{k}$, we get
\begin{align*}
    \|\tilde{C}'^{x}_{k}\|_{Q_{22}}\leq \frac{\kappa_{Q_{22}}d}{\zeta_{k_0}^x}\left(\brc +\frac{4\sqrt{3\brc}}{1-\rho}\left(1+\frac{\beta}{\alpha}\varrho_x\right)\brc_{f}\right)+\frac{\alpha_{k_0}}{\zeta_{k_0}^x}\|\Sigma^x\|_{Q_{22}}~~ 0\leq k \leq k_0
\end{align*}
Finally, for the cross term $\tilde{C}'^{xy}_{k}$, we have
\begin{align*}
    \|\tilde{C}'^{xy}_{k}\|_{Q_{22}}\zeta_k^{xy}&\leq \|\tZ_k'\|_{Q_{22}}+\beta_k\|\Sigma^{xy}\|_{Q_{22}}\\
    &\leq \kappa_{Q_{22}}\|\tZ_k'\|+\beta_k\|\Sigma^{xy}\|_{Q_{22}}\tag{Norm Equivalence}\\
    &\leq \kappa_{Q_{22}}\left(\|\tZ_k\|+\alpha_k\|d_k^{y}\|+\beta_k\|d_k^{xv}\|\right)+\beta_k\|\Sigma^{xy}\|_{Q_{22}}\\
    &\leq \kappa_{Q_{22}}\Bigg(\frac{1}{2}\left(\E[\|\tx_k\|^2]+\E[\|\ty_k\|^2]\right)+\alpha_k\frac{2\sqrt{3d}}{1-\rho}\brc_{f}\left(1+\frac{\beta}{\alpha}\varrho_x\right)\sqrt{\E[\|\ty_k\|^2]}\\
    &~~+\beta_k\frac{2\sqrt{3d}}{1-\rho}\brc_{f}\sqrt{\E[\|\tx_k\|^2]}\Bigg)+\beta_k\|\Sigma^{xy}\|_{Q_{22}}\tag{Young's inequality and Lemma \ref{lem:tel_term_bound}}\\
    &\leq \kappa_{Q_{22}}\Bigg(\frac{\brc d}{2}+\alpha_k\frac{2d\sqrt{3\brc}}{1-\rho}\brc_{f}\left(1+\frac{\beta}{\alpha}\varrho_x\right)+\beta_k\frac{2d\sqrt{3\brc}}{1-\rho}\brc_{f}\Bigg)+\beta_k\|\Sigma^{xy}\|_{Q_{22}}\tag{Lemma \ref{lem:boundedness}}
\end{align*}

Again note that $\beta_k/\zeta_k^{xy}$ and $\alpha_k/\zeta_k^{xy}$ are increasing functions of $k$. Thus, we finally get
\begin{align*}
    \|\tilde{C}'^{xy}_{k}\|_{Q_{22}}\leq \frac{\kappa_{Q_{22}}d}{\zeta_{k_0}^{xy}}\Bigg(\frac{\brc }{2}+\frac{2\sqrt{3\brc}}{1-\rho}\brc_{f}\left(\alpha_{k_0}\left(1+\frac{\beta}{\alpha}\varrho_x\right)+\beta_{k_0}\right)\Bigg)+\frac{\beta_{k_0}}{\zeta_{k_0}^{xy}}\|\Sigma^{xy}\|_{Q_{22}}~~ 0\leq k \leq k_0
\end{align*}

Then by the definition, $\max\{\|\tilde{C}'^{y}_{k_0}\|_{Q_{\Delta, \beta}},\|\tilde{C}'^{xy}_{k_0}\|_{Q_{22}},\|\tilde{C}'^{x}_{k_0}\|_{Q_{22}}\}\leq \bar{c}d^2$. 
Now suppose at time $k\geq k_0$, we have \newline $\max\{\|\tilde{C}'^{y}_{k}\|_{Q_{\Delta,, \beta}},\|\tilde{C}'^{xy}_{k}\|_{Q_{22}},\|\tilde{C}'^{x}_{k}\|_{Q_{22}}\}=\hbar_k\leq \bar{c}d^2$. 
Then, by \eqref{eq:C_k_iter}, we have

\begin{align*}
    \max\{\|\tilde{C}'^{x}_{k+1}\|_{Q_{22}},\|\tilde{C}'^{xy}_{k+1}\|_{Q_{22}},\|\tilde{C}'^{y}_{k+1}\|_{Q_{\Delta,\beta}}\}&\leq \max\left\{\hbar_k, \frac{2\bc^{(x)}d^2}{a_{22}},\frac{4\bc^{(z)}d^2}{a_{22}}, \frac{2\bc^{(y)}d^2}{\|Q_{\Delta,\beta}\|^{-1}}\right\}\\
    &\leq \max\left\{\bar{c}d^2, \frac{2\bc^{(x)}d^2}{a_{22}},\frac{4\bc^{(z)}d^2}{a_{22}}, \frac{2\bc^{(y)}d^2}{\|Q_{\Delta,\beta}\|^{-1}}\right\}\leq  \bar{c}d^2. 
\end{align*}
Hence, by induction, $\max\{\|\tilde{C}'^{x}_{k}\|_{Q_{22}},,\|\tilde{C}'^{xy}_{k}\|_{Q_{22}},\|\tilde{C}'^{y}_{k}\|_{Q_{\Delta,\beta}}\}\leq \bar{c}d^2$  for all $k\geq 0$.

\end{proof}

\subsection{Auxiliary lemmas}\label{sec:Aux_lem}
Since we employ induction to prove our main lemmas, there are two categories of auxiliary lemmas that enable us to achieve this. The first category consists of lemmas that are true irrespective of the hypothesis considered true in the induction, while the second category consists of lemmas that are a consequence of the hypothesis in the induction. For better exposition, we divide this section into these two categories.

\subsubsection{Induction independent lemmas}\label{sec:Aux_lem_ind}
\begin{lemma}\label{lem:L_k_bound}
    Consider the recursion of the matrix $L_k$ in \eqref{eq:L_k12} and \eqref{eq:L_k22}. Then $\forall k\geq 0$, we have 
    \begin{align*}
        \|L_{k}\|\leq& \kappa_{Q_{22}}
    \end{align*}
    Furthermore, define $c_L=\max\{\frac{2c_{D}}{a_{22}}, \left(\|L_{k_1-1}\|_{Q_{22}}+c_{D}\beta_{k_1-1}\right)\frac{\alpha_{k_1}}{\beta_{k_1}}\}$. Then $\forall k\geq k_1$, we have
    \begin{align*}
        \|L_k\|\leq& c^L_1\frac{\beta_k}{\alpha_k},\\
        \|L_{k+1}-L_k\|\leq&c^L_2\alpha_k.
    \end{align*}
    where $c^L_1=c_L\kappa_{Q_{22}}$ and $c^L_2=2\max\{\|A_{22}\|_{Q_{22}}, c_D\}\kappa_{Q_{22}}$.
\end{lemma}

\begin{lemma}\label{lem:f_hat_bound}
    Consider $f_i(o,x_k,y_k)$ and $\hat{f}_i(o,x_k,y_k)$ as the solution of \eqref{eq:poisson_eq_f_i} for $i=1,2$. We have the following
    \begin{enumerate}
        \item $\|\hat{f}_i(o,x_k,y_k)\| \leq \frac{2}{1-\rho}\left(b_{max} \sqrt{d}+ A_{max}(\left\| y_k\right\| + \left\| x_k\right\|)\right)\leq \frac{2}{1-\rho}\left[b_{max} \sqrt{d}+ \check{h}_1\left(\left\| \xh_k\right\| +  \|\yh_k\|\right) \right]$
        \item $\|\fh_i(o,x_{k+1},y_{k+1})-\fh_i(o,x_k,y_k)\|\leq \check{h}_2( \|x_{k+1}-x_k\| + \|y_{k+1}-y_k\|)$
        \item
        \begin{enumerate}[label=(\roman*)]
            \item $\|\hat{f}_i(o,x_k,y_k)\|_{Q_{22}}\leq  \frac{2}{1-\rho}\left(b_{max}\sqrt{\gamma_{max}(Q_{22})} \sqrt{d}+ \frac{\check{h}_3}{2}\left(\left\| \xh_k\right\|_{Q_{22}} +  \|\yh_k\|_{Q_{\Delta}}\right) \right)$
            \item $\|\hat{f}_i(o,x_k,y_k)\|_{Q_{\Delta}} \leq  \frac{2}{1-\rho}\left(b_{max}\sqrt{\gamma_{max}(Q_{\Delta})} \sqrt{d}+ \frac{\check{h}_4}{2}\left(\left\| \xh_k\right\|_{Q_{22}} +  \|\yh_k\|_{Q_{\Delta}}\right) \right)$
        \end{enumerate}
        \item $\|f_i(o,x_k,y_k) \| \leq b_{max}\sqrt{d} + 2A_{max}(\|x_k\|+\|y_k\|)$.
        
        \item
        \begin{enumerate}[label=(\roman*)]
            \item $\|f_i(o,x_k,y_k) \|_{Q_{22}}\leq \sqrt{\gamma_{max}(Q_{22})}\|f_i(o,x_k,y_k) \| \leq \sqrt{\gamma_{max}(Q_{22})}( b_{max}\sqrt{d} + 2A_{max}(\|x_k\|+\|y_k\|))  \leq \sqrt{\gamma_{max}(Q_{22})} b_{max}\sqrt{d} +\check{h}_3 (\|\xh_k\|_{Q_{22}} +\|\yh_k\|_{Q_{\Delta}})$
            \item $\|f_i(o,x_k,y_k) \|_{Q_{\Delta}} \leq \sqrt{\gamma_{max}(Q_{\Delta})} \|f_i(o,x_k,y_k) \| \leq \sqrt{\gamma_{max}(Q_{\Delta})}( b_{max} \sqrt{d} + 2A_{max}(\|x_k\|+\|y_k\|))  \leq \sqrt{\gamma_{max}(Q_{\Delta})} b_{max}\sqrt{d} +\check{h}_4 (\|\xh_k\|_{Q_{22}} +\|\yh_k\|_{Q_{\Delta}})$
        \end{enumerate}
        \item $\|A_{22}^{-1}A_{21}\left(-(\Delta\yh_k +A_{12}\xh_k)+f_1(O_k,x_k,y_k)\right)\|_{Q_{22}}\leq \check{h}_6\sqrt{d} +\check{h}_5 (\|\xh_k\|_{Q_{22}} +\|\yh_k\|_{Q_{\Delta}})$
    \end{enumerate}
    where $\check{h}_1 = A_{max}(1+\|A_{22}^{-1}A_{21}\|)$, 
    $\check{h}_2=\frac{2}{1-\rho}A_{max}$, $\check{h}_3=2\check{h}_1 \max\left\{\kappa_{Q_{22}}, \sqrt{\frac{\gamma_{max}(Q_{22})}{\gamma_{min}(Q_{\Delta})}}\right\}$, $\check{h}_4=2\check{h}_1 \max\left\{\kappa_{Q_{\Delta}}, \sqrt{\frac{\gamma_{max}(Q_{\Delta})}{\gamma_{min}(Q_{22})}}\right\}$, $\check{h}_5=\|A_{22}^{-1}A_{21}\|_{Q_{22}}\left(\check{h}_3+\max\left\{\frac{\sqrt{\gamma_{max} (Q_{22})}}{\sqrt{\gamma_{min}(Q_{\Delta})}}\|\Delta\|_{Q_{22}}, \|A_{12}\|_{Q_{22}}\right\}\right)$ and $\check{h}_6=\|A_{22}^{-1}A_{21}\|_{Q_{22}}\sqrt{\gamma_{max}(Q_{22})} b_{max}$.
\end{lemma}

\begin{lemma}\label{lem:crdue_upper_bnd_norm} Consider the update of the variables in \eqref{eq:two_time_scale}. Then, we have
\begin{enumerate}
    \item $\|\xh_{k+1}\|^2_{Q_{22}} \leq (1 +\alpha_k\hh_1^{xx}) \|\xh_k\|^2_{Q_{22}} +\hh_2^{xx} \alpha_k(d+\|\yh_k\|_{Q_{\Delta}}^2)$.
    \item $\|\yh_{k+1}\|_{Q_{\Delta}}^2 \leq (1+ \beta_k \hh_1^{yy})\|\yh_k \|_{Q_{\Delta}}^2 +\hh_2^{yy}\beta_k(d +  \|\xh_k\|_{Q_{22}}^2)$.
    \item $U_{k+1}\leq (1+\alpha_k(\hat{h}_1+\hh_2)) U_k + \alpha_k \hat{h}_2d$.
    \item $\|x_{k+1}-x_k\| \leq \alpha_k\hh_3(\sqrt{d}+\|\xh_k\|_{Q_{22}} +\|\yh_k\|_{Q_{\Delta}})$.
    \item $|y_{k+1}-y_k\| \leq \beta_k \hh_4(\sqrt{d}+\|\xh_k\|_{Q_{22}} +\|\yh_k\|_{Q_{\Delta}})$.
\end{enumerate}
for some problem dependent constants. The exact expression for the constants are given in the proof of this lemma.

\end{lemma}

\begin{lemma}\label{lem:crude_Uk_bnd}
For all $k\geq 0$, we have
\begin{align*}
        U_{k}\leq& U_0\exp \left(\frac{(\hat{h}_1 +\hh_2)\alpha}{K_0^\xi}+\frac{(\hat{h}_1 +\hh_2)\alpha}{(1-\xi)}\left[(k+K_0)^{1-\xi}-K_0^{1-\xi}\right]\right) \\
        &+\alpha\hat{h}_2 d\left(\frac{1}{K_0^\xi}+\frac{1}{(\hat{h}_1 +\hh_2)\alpha}\right)\exp\left(\frac{(\hat{h}_1 +\hh_2)\alpha}{(1-\xi)}\left((k+K_0)^{1-\xi}-K_0^{1-\xi}\right)\right)
    \end{align*}
    for some problem dependent constants. The exact expression for the constants are given in the proof of this lemma.     
\end{lemma}

\begin{lemma}\label{lem:boundedness}
    Suppose that Assumptions \ref{ass:hurwitz_main}, \ref{ass:poisson_main} and \ref{ass:step_size_main} are satisfied. Then, there exists a constant $\brc$ such that 
    \begin{align*}
        \E[\|x_k\|^2]+\E[\|y_k\|^2]&\leq \brc d\\
        \E[\|\tx_k\|^2]+\E[\|\ty_k\|^2]&\leq \brc d.
    \end{align*}
\end{lemma}

\begin{lemma}\label{lem:tel_term_bound}
Suppose that Assumptions \ref{ass:hurwitz_main}, \ref{ass:poisson_main} and \ref{ass:step_size_main} are satisfied. Denote $\brc_{f}=\sqrt{b^2_{max}+ A^2_{max}\brc}$. Then, the following relation holds
    \begin{enumerate}
        \item $\left\|\E\bigg[\left(\E_{O_{k-1}}\fh_i(\cdot,x_k,y_k)\right)\tx_k^\top\bigg]\right\|\leq \frac{2\sqrt{3d}}{1-\rho}\brc_{f}\sqrt{\E[\|\tx_k\|^2]}$.
        \item $\left\|\E\bigg[\left(\E_{O_{k-1}}\fh_i(\cdot,x_k,y_k)\right)\ty_k^\top\bigg]\right\|\leq \frac{2\sqrt{3d}}{1-\rho}\brc_{f}\sqrt{\E[\|\ty_k\|^2]}$.
        \item $\|d_k^x\|\leq \frac{2\sqrt{3d}}{1-\rho}\brc_{f}\left(1+\frac{\beta}{\alpha}\varrho_x\right)\sqrt{\E[\|\tx_k\|^2]}$.
        \item $\|d_k^y\|\leq \frac{2\sqrt{3d}}{1-\rho}\brc_{f}\left(1+\frac{\beta}{\alpha}\varrho_x\right)\sqrt{\E[\|\ty_k\|^2]}$.
    \end{enumerate}
\end{lemma}

\begin{lemma}\label{lem:noise_crude_bound}
Suppose that Assumptions \ref{ass:hurwitz_main}, \ref{ass:poisson_main} and \ref{ass:step_size_main} are satisfied. Then, the following relation holds
\begin{enumerate}
    \item $\E[\|v_k\|^2]\leq 3d\left(b_{\max}^2+4A_{max}^2\brc\right)$.
    \item $\E[\|w_k\|^2]\leq 3d\left(b_{\max}^2+4A_{max}^2\brc\right)$.
    \item $\E[\|u_k\|^2]\leq 6d\left(1+\frac{\beta^2}{\alpha^2}\varrho_x^2\right)\left(b_{\max}^2+4A_{max}^2\brc\right)$.
\end{enumerate}
\end{lemma}

\subsubsection{Proof of the induction independent lemmas}

\begin{proof}[Proof of Lemma \ref{lem:L_k_bound}]
From Lemma \ref{lem:L_k_abs_bound}, we have that for $k\geq k_L$,
\begin{align*}
    L_{k+1}&=((I-\alpha_k A_{22})L_k+\beta_k A_{22}^{-1}A_{21}B^k_{11})(I-\beta_k B^k_{11})^{-1}\\
    &=(I-\alpha_k A_{22})L_k+\beta_k D(L_k)
\end{align*}
where $D(L_k)=(A_{22}^{-1}A_{21}+(I-\alpha_k A_{22})L_k)B^k_{11}(I-\beta_k B^k_{11})^{-1}$. Note that because of the choice of $k_L$, we have $\|L_k\|_{Q_{22}}\leq 1~~\forall k\geq 0$, which implies $\|D(L_k)\|_{Q_{22}}\leq c_D$. We will prove the lemma by induction. $\|L_{k_1}\|_{Q_{22}}\leq \frac{c_L\beta_{k_1}}{\alpha_{k_1}}$ by construction. Assume that $\|L_k\|_{Q_{22}}\leq \frac{c_L\beta_k}{\alpha_k}$ for some $k\geq k_1$. Then for $k+1$ we have:
\begin{align*}
    \frac{c_L\beta_{k+1}}{\alpha_{k+1}}-\|L_{k+1}\|_{Q_{22}}&\geq \frac{c_L\beta_{k+1}}{\alpha_{k+1}}-(1-\alpha_k a_{22})\|L_{k}\|_{Q_{22}} - c_D\beta_k\\
    &\geq \frac{c_L\beta_{k+1}}{\alpha_{k+1}}-(1-\alpha_k a_{22})\frac{c_L\beta_k}{\alpha_k} - c_{D}\beta_k\\
    &=\frac{c_L\beta_{k+1}}{\alpha_{k+1}}-\frac{c_L\beta_k}{\alpha_k}+c_La_{22}\beta_k- c_D\beta_k\\
    &= c_L\beta_k\left(\frac{\beta_{k+1}}{\beta_k\alpha_{k}}-\frac{1}{\alpha_k}+a_{22}- \frac{c_{D}}{c_L}\right)\\
    &=c_L\beta_k\left(a_{22}- \frac{c_{D}}{c_L}-\frac{1}{\alpha_k}(1-\frac{\alpha_k\beta_{k+1}}{\alpha_{k+1}\beta_k})\right)\\
    &\geq c_L\beta_k\left(\frac{a_{22}}{2}-\frac{1}{\alpha_k}(1-\frac{\alpha_k\beta_{k+1}}{\alpha_{k+1}\beta_k})\right)
\end{align*}
Substituting the values for $\beta_k$ and $\alpha_k$, we have:
\begin{align*}
    \frac{\alpha_k\beta_{k+1}}{\alpha_{k+1}\beta_k}=\big(\frac{k+K_0}{k+K_0+1}\big)^{1-\xi}=\left(1+\frac{1}{k+K_0}\right)^{\xi-1}\geq \exp{\frac{\xi-1}{k+K_0}}\geq 1-\frac{1-\xi}{k+K_0}\\
\end{align*}
Using this, we get:
\begin{align*}
    \frac{1}{\alpha_k}(1-\frac{\alpha_k\beta_{k+1}}{\alpha_{k+1}\beta_k})=\frac{(k+K_0)^{\xi}}{\alpha}\big(1-\big(\frac{k+K_0}{k+K_0+1}\big)^{1-\xi}\big)\leq \frac{1-\xi}{\alpha(k+K_0)^{1-\xi}}
\end{align*}
Note that $k_1$ is large enough that $\frac{1-\xi}{\alpha(k+K_0)^{1-\xi}}\leq \frac{a_{22}}{2}$. Thus, we get,
\begin{align*}
    \|L_{k+1}\|_{Q_{22}}\leq \frac{c_L\beta_{k+1}}{\alpha_{k+1}}
\end{align*}
By norm equivalence we get,
\begin{align*}
    \Rightarrow\|L_k\|\leq \frac{c^L_1\beta_{k}}{\alpha_{k}}
\end{align*}
where $c^L_1=c_L\kappa_{Q_{22}}$.
For the second part we have,
\begin{align*}
    \|L_{k+1}-L_k\|_{Q_{22}}&=\|-\alpha_kA_{22}L_{k} + \beta_kD_k(L_k)\|_{Q_{22}}\leq \alpha_k\|A_{22}\|_{Q_{22}}+c_D\beta_k\leq 2\max\{\|A_{22}\|_{Q_{22}}, c_D\}\alpha_k\\
    &\Rightarrow \|L_{k+1}-L_k\|\leq c^L_2\alpha_k
\end{align*}
where $c^L_2=2\max\{\|A_{22}\|_{Q_{22}}, c_D\}\kappa_{Q_{22}}$.
\end{proof}

\begin{proof}[Proof of Lemma \ref{lem:f_hat_bound}]
\begin{enumerate}
    \item By Lemma \ref{lem:possion_sol}, for $\hat{f}_i(o,x_k,y_k)$, we have
\begin{align*}
        &\|\fh_i(o,x_k,y_k)\| \\
        &= \left\|\sum_{l=0}^\infty \E{[b_i(O_l)|O_o=o]} -\left(\sum_{l=0}^\infty \E{[A_{i1}(O_l)-A_{i1}|O_0=o]}\right)y_k -\left(\sum_{l=0}^\infty \E{[A_{i2}(O_l)-A_{i2}|O_0=o]}\right)x_k\right\|\\
        &\leq \left\|\sum_{l=0}^\infty \E{[b_i(O_l)|O_o=o]} \right\|+ \left\|\left(\sum_{l=0}^\infty \E{[A_{i1}(O_k)-A_{i1}|O_0=o]}\right)y_k\right\| +\left\|\left(\sum_{l=0}^\infty \E{[A_{i2}(O_l)-A_{i2}|O_0=o]}\right) x_k\right\|\\
        &\leq \left\|\sum_{l=0}^\infty \E{[b_i(O_l)|O_o=o]} \right\|+ \left\|\sum_{l=0}^\infty \E{[A_{i1}(O_k)-A_{i1}|O_0=o]}\right\|\left\| y_k\right\|+\left\|\sum_{l=0}^\infty \E{[A_{i2}(O_l) -A_{i2}|O_0=o]}\right\|\left\| x_k\right\|\\
        &\leq  \frac{2}{1-\rho}\left[\max_{o\in\mathcal{S}}\|b_i(o)\|+ A_{max}\left\| y_k\right\| + A_{max}\left\| x_k\right\|\right]\tag{Lemma \ref{lem:mix_time_sum}}\\
        &\leq  \frac{2}{1-\rho}\left[b_{max} \sqrt{d}+ A_{max}\left\| y_k\right\| + A_{max}\left\| x_k\right\|\right]\\
        &\leq  \frac{2}{1-\rho}\left[b_{max} \sqrt{d}+ A_{max} [\left\| \xh_k\right\| + (1 + \|A_{22}^{-1}A_{21}\|) \|\yh_k\|] \right]\\
        &\leq  \frac{2}{1-\rho}\left[b_{max} \sqrt{d}+ \check{h}_1 [\left\| \xh_k\right\| +  \|\yh_k\|] \right].
    \end{align*}  
    \item \begin{align*}
        &\|\fh_i(o,x_{k+1},y_{k+1})-\fh_i(o,x_k,y_k)\| \\
        &= \left\|-\left(\sum_{l=0}^\infty \E{[A_{i1}(O_k)-A_{i1}|O_0=o]}\right)(y_{k+1} - y_k) -\left(\sum_{l=0}^\infty \E{[A_{i2}(O_k)-A_{i2}|O_0=o]}\right)(x_{k+1}-x_k)\right\|\\
        &\leq \left\|\left(\sum_{l=0}^\infty \E{[A_{i1}(O_k)-A_{i1}|O_0=o]}\right)(y_{k+1} - y_k)\right\| +\left\|\left(\sum_{l=0}^\infty \E{[A_{i2}(O_k)-A_{i2}|O_0=o]}\right)(x_{k+1}-x_k)\right\|\\
        &\leq \left\|\sum_{l=0}^\infty \E{[A_{i1}(O_k)-A_{i1}|O_0=o]}\right\|\left\|y_{k+1} - y_k\right\| +\left\|\sum_{l=0}^\infty \E{[A_{i2}(O_k) -A_{i2}|O_0=o]}\right\|\left\|x_{k+1}- x_k\right\|\\
        &\leq \frac{2}{1-\rho}A_{max}\left\|y_{k+1} - y_k\right\| + \frac{2}{1-\rho} A_{max} \left\| x_{k+1} - x_k\right\|\\
        &= \check{h}_2 (\left\|y_{k+1} - y_k\right\| +\left\|x_{k+1}- x_k\right\|).
    \end{align*}
    \item 
        \begin{enumerate}[label=(\roman*)]
            \item \begin{align*}
                \|\fh_i (o,x_k,y_k) \|_{Q_{22}} = &\sqrt{\langle \fh_i(o,x_k,y_k), Q_{22} \fh_i( o,x_k,y_k)}\\
                &\leq \sqrt{\gamma_{max}(Q_{22})}\|\fh_i(o,x_k,y_k)\|\\
                &\leq \sqrt{\gamma_{max}(Q_{22})}\frac{2}{1-\rho}\left[b_{max} \sqrt{d}+ \check{h}_1\left(\left\| \xh_k\right\| +  \|\yh_k\|\right) \right]\\
                &\leq \sqrt{\gamma_{max}(Q_{22})} \frac{2}{1-\rho}\left(b_{max} \sqrt{d}  +\frac{\check{h}_1}{\sqrt{\gamma_{min}(Q_{\Delta})}}\|\yh_k\|_{Q_{\Delta}} +\frac{\check{h}_1}{\sqrt{\gamma_{min}(Q_{22})}}\|\xh_k\|_{Q_{22}}\right)\\
                &\leq \frac{2}{1-\rho}\left(\sqrt{\gamma_{max}(Q_{22})}b_{max}\sqrt{d}  +\check{h}_3 \left(\|\yh_k\|_{Q_{\Delta}} +\|\xh_k\|_{Q_{22}}\right)\right)
                \end{align*}
            \item Similar to the previous part, we get
            \begin{align*}
                \|\fh_i(o,x_k,y_k)\|_{Q_{\Delta}} &\leq \frac{2}{1-\rho}\left(\sqrt{\gamma_{max}(Q_{\Delta})}b_{max}\sqrt{d}  +\check{h}_4\left(\|\yh_k\|_{Q_{\Delta}} +\|\xh_k\|_{Q_{22}}\right)\right)
                \end{align*}
        \end{enumerate}
    \item 
    \begin{align*}
        \|f_i(o,x_k,y_k)\| = &\|b_i(o)-(A_{i1}(o)-A_{i1})y-(A_{i2}(o)-A_{i2})x\| \\
        \leq & \|b_i(o)\|+\|A_{i1}(o)-A_{i1}\|\|y_k\|+\|A_{i2}(o)-A_{i2}\|\|x_k\| \\
        \leq & \max_{o'\in\mathcal{S}} \|b_i(o')\|+2A_{max}\|y_k\| +2A_{max}\|x_k\| \\
        \leq & b_{max}\sqrt{d} + 2A_{max}\|y_k\| +2A_{max}\|x_k\|
    \end{align*} 
    \item 
    \begin{enumerate}[label=(\roman*)]
        \item \begin{align*}
            \|f_i(o,x_k,y_k)\|_{Q_{22}} = &\sqrt{\langle f_i(o,x_k,y_k), Q_{22}f_i(o,x_k,y_k)}\\
            &\leq \sqrt{\gamma_{max}(Q_{22})}\|f_i(o,x_k,y_k)\|\\
            &\leq \sqrt{\gamma_{max}(Q_{22})}\left(b_{max}\sqrt{d}  +2A_{max}\|y_k\| +2A_{max}\|x_k\|\right)\\
            &= \sqrt{\gamma_{max}(Q_{22})}\left(b_{max}\sqrt{d}  +2A_{max}\|\yh_k\| +2A_{max}\|\xh_k-A_{22}^{-1}A_{21}y_k\|\right)\\
            &\leq \sqrt{\gamma_{max}(Q_{22})}\left(b_{max}\sqrt{d}  +2A_{max}(1+\|A_{22}^{-1}A_{21}\|)\|\yh_k\| +2A_{max}\|\xh_k\|\right)\\ 
            &\leq \sqrt{\gamma_{max}(Q_{22})}\left(b_{max}\sqrt{d}  +\frac{2A_{max}(1+\|A_{22}^{-1}A_{21}\|)}{\sqrt{\gamma_{min}(Q_{\Delta})}}\|\yh_k\|_{Q_{\Delta}} +\frac{2A_{max}}{\sqrt{\gamma_{min}(Q_{22})}}\|\xh_k\|_{Q_{22}}\right)\\
            &\leq \sqrt{\gamma_{max}(Q_{22})}b_{max}\sqrt{d}  +\check{h}_3\left(\|\yh_k\|_{Q_{\Delta}} +\|\xh_k\|_{Q_{22}}\right)
            \end{align*}
        \item Similar to the previous part, we get
            \begin{align*}
                \|f_i(o,x_k,y_k)\|_{Q_{\Delta}}
                &\leq \sqrt{\gamma_{max}(Q_{\Delta})}b_{max}\sqrt{d}  +\check{h}_4\left(\|\yh_k\|_{Q_{\Delta}} +\|\xh_k\|_{Q_{22}}\right)
            \end{align*}
    \end{enumerate}
    \item \begin{align*}
        \|A_{22}^{-1}A_{21}(-(\Delta\yh_k &+A_{12}\xh_k)+f_1(O_k,x_k,y_k))\|_{Q_{22}}\leq \|A_{22}^{-1}A_{21}\|_{Q_{22}}(\|\Delta\yh_k +A_{12}\xh_k\|_{Q_{22}}+\|f_1(O_k,x_k,y_k))\|_{Q_{22}})\\
        &\leq \|A_{22}^{-1}A_{21}\|_{Q_{22}}\left(\|\Delta \yh_k +A_{12}\xh_k\|_{Q_{22}}+ \sqrt{\gamma_{max}(Q_{22})} b_{max}\sqrt{d} +\check{h}_3 (\|\xh_k\|_{Q_{22}} +\|\yh_k\|_{Q_{\Delta}})\right)\\
        &\leq \|A_{22}^{-1} A_{21}\|_{Q_{22}}\Bigg(\|\Delta\|_{Q_{22}}\frac{\sqrt{\gamma_{max} (Q_{22})}}{\sqrt{\gamma_{min}(Q_{\Delta})}} \|\yh_k\|_{Q_{\Delta}} +\|A_{12}\|_{Q_{22}}\|\xh_k\|_{Q_{22}}\\
        &+\sqrt{\gamma_{max}(Q_{22})} b_{max}\sqrt{d} +\check{h}_3 (\|\xh_k\|_{Q_{22}} +\|\yh_k\|_{Q_{\Delta}})\Bigg)\\
        &\leq \check{h}_6\sqrt{d} +\check{h}_5 (\|\xh_k\|_{Q_{22}} +\|\yh_k\|_{Q_{\Delta}})
    \end{align*}
\end{enumerate}

\end{proof}

\begin{proof}[Proof of Lemma \ref{lem:crdue_upper_bnd_norm}]
\begin{enumerate}
    \item For $\xh_{k+1}$, from the proof of Lemma \ref{lem:boundedness} we have the following recursion
    \begin{align*}
        \xh_{k+1}=&(I-\alpha_kA_{22})\xh_k+\alpha_kf_{2}(O_k,x_k,y_k)\\
        &+\beta_kA_{22}^{-1}A_{21}(-(\Delta\yh_k+A_{12}\xh_k)+f_1(O_k,x_k,y_k)).
    \end{align*}
    Hence, 
    \begin{align}
        \|\xh_{k+1}\|_{Q_{22}} \leq&\underbrace{\|I- \alpha_k A_{22}\|_{Q_{22}}}_{T_1} \|\xh_k\|_{Q_{22}} +\alpha_k\underbrace{\|f_{2}(O_k, x_k,y_k) \|_{Q_{22}}}_{T_2}\nonumber\\
        &+ \beta_k\underbrace{\|A_{22}^{-1}A_{21}(-(\Delta\yh_k +A_{12}\xh_k)+f_1(O_k,x_k,y_k))\|_{Q_{22}}}_{T_3} \label{eq:x_rec_lem_crud}
    \end{align}
    For $T_1$ we have
    \begin{align*}
        \|I-\alpha_kA_{22}\|_{Q_{22}} \leq 1+\alpha_k\|A_{22}\|_{Q_{22}}.
    \end{align*}
    
    For $T_2$, using Lemma \ref{lem:f_hat_bound} we have $T_2\leq \sqrt{\gamma_{max}(Q_{22})} b_{max}\sqrt{d} +\check{h}_3 (\|\xh_k\|_{Q_{22}} +\|\yh_k\|_{Q_{\Delta}})$.
    
    For $T_3$, using Lemma \ref{lem:f_hat_bound} we have $T_3\leq \check{h}_6\sqrt{d} +\check{h}_5 (\|\xh_k\|_{Q_{22}} +\|\yh_k\|_{Q_{\Delta}})$.

    Hence, we have
    \begin{align}
        \|\xh_{k+1}\|_{Q_{22}} \leq & (1+\alpha_k \|A_{22}\|_{Q_{22}})\|\xh_k\|_{Q_{22}} + \alpha_k \left(\sqrt{\gamma_{max}(Q_{22})} b_{max}\sqrt{d} +\check{h}_3 (\|\xh_k\|_{Q_{22}} +\|\yh_k\|_{Q_{\Delta}})\right)\nonumber\\
        &+ \alpha_k\frac{\beta}{\alpha}\left( \check{h}_6\sqrt{d} +\check{h}_5 (\|\xh_k\|_{Q_{22}} +\|\yh_k\|_{Q_{\Delta}})\right)\nonumber\\
        \leq & \left(1 +\alpha_k\hh^{x}_{1}\right)  \|\xh_k\|_{Q_{22}} + \alpha_k\hh^{x}_{2}(\sqrt{d}+\|\yh_k\|_{Q_{\Delta}}) \label{eq:x_k_rec_no_norm}
    \end{align}


    where $\hh^{x}_{1}=\|A_{22}\|_{Q_{22}}+\check{h}_3+\frac{\beta}{\alpha}\check{h}_5$ and $\hh^{x}_{2}=\max\left\{\sqrt{\gamma_{max}(Q_{22})}b_{max}+\frac{\beta}{\alpha}\check{h}_6, \check{h}_3+\frac{\beta}{\alpha}\check{h}_5 \right\}$. Now squaring both sides of \eqref{eq:x_k_rec_no_norm}, we get:
    \begin{align}
        \|\xh_{k+1}\|^2_{Q_{22}}\leq & \left(1 +\alpha_k\hh^{x}_{1}\right)^2  \|\xh_k\|^2_{Q_{22}} + \alpha_k^2(\hh^{x}_{2})^2(\sqrt{d}+\|\yh_k\|_{Q_{\Delta}})^2 +2\alpha_k\hh^{x}_{2} \left(1 +\alpha_k \hh^{x}_{1}\right)  \|\xh_k\|_{Q_{22}}(\sqrt{d} + \|\yh_k\|_{Q_{\Delta}})\nonumber\\
        \leq &\left(1 +\alpha_k\hh^{x}_{1}\right)^2  \|\xh_k\|^2_{Q_{22}} + \alpha_k^2(\hh^{x}_{2})^2(\sqrt{d}+\|\yh_k\|_{Q_{\Delta}})^2+\hh^{x}_{2}\left(1 + \alpha_k\hh^{x}_{1}\right)  (\alpha_kd+2\alpha_k\|\xh_k\|_{Q_{22}}^2+\alpha_k\|\yh_k\|_{Q_{\Delta}}^2)\tag{By Cauchy-Schwartz}\\
        \leq &\left(1 +\alpha_k\left(\alpha(\hh^{x}_{1})^2+2\hh^x_1+2\hh^{x}_{2}\left(1 +\alpha\hh^{x}_{1}\right)\right)\right)  \|\xh_k\|^2_{Q_{22}} + 2\alpha_k\alpha(\hh^{x}_{2})^2(d+\|\yh_k\|_{Q_{\Delta}}^2)\nonumber\\
        &+\alpha_k\hh^{x}_{2}\left(1 +\alpha\hh^{x}_{1}\right)(d+\|\yh_k\|_{Q_{\Delta}}^2)\nonumber\\
        = &\left(1 +\alpha_k\hh_1^{xx}\right)  \|\xh_k\|^2_{Q_{22}} + \alpha_k\hh_2^{xx}(d+\|\yh_k\|_{Q_{\Delta}}^2)\label{eq:x_k_rec_no_norm_2}
    \end{align}
    where $\hh_1^{xx}=\alpha(\hh^{x}_{1})^2+2\hh^x_1+2\hh^{x}_{2}\left(1 +\alpha\hh^{x}_{1}\right)$ and $\hh_2^{xx}=2\alpha(\hh^{x}_{2})^2+\hh^{x}_{2}\left(1 +\alpha\hh^{x}_{1}\right)$.
    
    

    
    
    \item For $\yh_{k+1}$, we have the following recursion
    \begin{align*}
        \yh_{k+1}&=(I-\beta_k\Delta)\yh_k+\beta_kf_{1}(O_k,x_k,y_k)-\beta_kA_{12}\xh_k.
    \end{align*}
    Taking norm on both sides, we have
    \begin{align}
        \|\yh_{k+1}\|_{Q_{\Delta}} & \leq \|I-\beta_k\Delta\|_{Q_{\Delta}}\|\yh_k\|_{Q_{\Delta}}+\beta_k\|f_{1}(O_k,x_k,y_k)\|_{Q_{\Delta}} +\beta_k\|A_{12}\xh_k\|_{Q_{\Delta}}.\label{eq:y_rec_lem_crud}
    \end{align}
    We have
    \begin{align*}
        \|(I-\beta_k\Delta)\yh_k\|_{Q_{\Delta}}\leq \|I-\beta_k\Delta\|_{Q_{\Delta}}\|\yh_k\|_{Q_{\Delta}} \leq (1+\beta_k\|\Delta\|_{Q_{\Delta}})\|\yh_k\|_{Q_{\Delta}}.
    \end{align*}
    
    Furthermore, using Lemma \ref{lem:f_hat_bound}, we have $\|f_{1}(O_k,x_k,y_k)\|_{Q_{\Delta}}\leq \sqrt{\gamma_{max}(Q_{\Delta})} b_{max}\sqrt{d} +\check{h}_4 (\|\xh_k\|_{Q_{22}} +\|\yh_k\|_{Q_{\Delta}})$.
    
    Finally, we have:
    \begin{align*}
        \|A_{12}\xh_k\|_{Q_{\Delta}} \leq \|A_{12}\|_{Q_{\Delta}} \|\xh_k\|_{Q_{\Delta}}\leq \|A_{12}\|_{Q_{\Delta}} \frac{\sqrt{\gamma_{max}(Q_{\Delta})}}{\sqrt{\gamma_{min}(Q_{22})}}\|\xh_k\|_{Q_{22}}.
    \end{align*}
    Hence, we have
    \begin{align}
        \|\yh_{k+1}\|_{Q_{\Delta}}  \leq & (1+\beta_k\|\Delta\|_{Q_{\Delta}})\|\yh_k\|_{Q_{\Delta}} +\beta_k\left(\sqrt{\gamma_{max}(Q_{\Delta})} b_{max}\sqrt{d} +\check{h}_4 (\|\xh_k\|_{Q_{22}} +\|\yh_k\|_{Q_{\Delta}})\right)\nonumber\\
        &+\beta_k\|A_{12}\|_{Q_{\Delta}} \frac{\sqrt{\gamma_{max}(Q_{\Delta})}}{\sqrt{\gamma_{min}(Q_{22})}}\|\xh_k\|_{Q_{22}} \nonumber\\
        \leq& (1+\beta_k \hat{h}_1^y)\|\yh_k\|_{Q_{\Delta}} +\beta_k\hat{h}_2^y(\sqrt{d}+\|\xh_k\|_{Q_{22}}).\label{eq:y_k_rec_no_norm}
    \end{align}
    where $\hh^{y}_{1}=\|\Delta\|_{Q_{\Delta}} + \check{h}_4$ and $\hh^{y}_{2}=\max\left\{\sqrt{\gamma_{max}(Q_{\Delta})}b_{max} , \check{h}_4+\|A_{12}\|_{Q_{\Delta}} \frac{\sqrt{\gamma_{max}(Q_{\Delta})}}{\sqrt{\gamma_{min}(Q_{22})}} \right\}$. Squaring both sides of \eqref{eq:y_k_rec_no_norm}, we have
    \begin{align}
        \|\yh_{k+1}\|_{Q_{\Delta}}^2 &\leq (1 +\beta_k \hat{h}_1^y)^2\|\yh_k\|_{Q_{\Delta}}^2 +\beta_k^2(\hat{h}_2^y)^2(\sqrt{d}+\|\xh_k\|_{Q_{22}})^2 + 2\beta_k\hat{h}_2^y (1+\beta_k \hat{h}_1^y) \|\yh_k\|_{Q_{\Delta}} (\sqrt{d}+\|\xh_k\|_{Q_{22}})\nonumber\\
        & \leq (1+ \beta_k \hat{h}_1^y)^2\|\yh_k \|_{Q_{\Delta}}^2 +\beta_k^2 (\hat{h}_2^y)^2(\sqrt{d}+\|\xh_k\|_{Q_{22}})^2 + \hat{h}_2^y (1+ \beta \hat{h}_1^y) (\beta_kd+2\beta_k\|\yh_k\|_{Q_{\Delta}}^2 +  \beta_k\|\xh_k\|_{Q_{22}}^2)\tag{by Cauchy-Schwartz}\\
        & \leq (1+ \beta_k (2\hat{h}_1^y + \beta (\hat{h}_1^y)^2+2\hat{h}_2^y (1+ \beta \hat{h}_1^y)))\|\yh_k \|_{Q_{\Delta}}^2 +2\beta_k\beta (\hat{h}_2^y)^2(d+\|\xh_k\|_{Q_{22}}^2) \nonumber\\
        &+ \beta_k\hat{h}_2^y (1+ \beta \hat{h}_1^y) (d +  \|\xh_k\|_{Q_{22}}^2)\nonumber\\
        &= (1+ \beta_k\hh_1^{yy} )\|\yh_k \|_{Q_{\Delta}}^2 +\hh_2^{yy}\beta_k(d +  \|\xh_k\|_{Q_{22}}^2). \label{eq:y_k_rec_no_norm_2}
    \end{align}
    where $\hh_1^{yy}=2\hat{h}_2^y + \beta (\hat{h}_1^y)^2+2\hat{h}_2^y (1+ \beta \hat{h}_1^y)$ and $\hh_2^{yy}=2\beta (\hat{h}_2^y)^2 + \hat{h}_2^y (1+ \beta \hat{h}_1^y)$.

    \item Summing \eqref{eq:x_k_rec_no_norm_2} and \eqref{eq:y_k_rec_no_norm_2}, we get
    \begin{align*}
        U_{k+1} \leq (1+\alpha_k\hat{h}_1) U_k + \alpha_k \hat{h}_2(d+U_k)\\
        =(1+\alpha_k(\hat{h}_1+\hh_2)) U_k + \alpha_k \hat{h}_2d
    \end{align*}
    where $\hat{h}_1 = \max\left\{ \hh_1^{xx}, \frac{\beta}{\alpha}\hh_1^{yy}\right\}$ and $\hh_2=\max\left\{\hh_2^{xx},  \frac{\beta}{\alpha}\hh_2^{yy}\right\}$.

    \item 
    \begin{align}
        \|x_{k+1}-x_k\|&=\alpha_k\|A_{22}\xh_k-f_{2}(O_k,x_k, y_k)\|\nonumber\\
        &\leq \alpha_k\left(\|A_{22}\xh_k\|+ \|f_{2}(O_k,x_k,y_k)\|\right) \nonumber\\
        &\leq \alpha_k\left(\|A_{22}\| . \|\xh_k\|+  b_{max}\sqrt{d} + \check{h}_3 (\|\xh_k\|_{Q_{22}} +\|\yh_k\|_{Q_{\Delta}}) \right)\tag{Lemma \ref{lem:f_hat_bound}}\nonumber\\
        &\leq \alpha_k(\sqrt{d}+\|\xh_k\|_{Q_{22}} +\|\yh_k\|_{Q_{\Delta}})\nonumber
    \end{align}
    where $\hh_3=\max\left\{\|A_{22}\| +\check{h}_3 ,  b_{max} \right\}$.

    \item 
        \begin{align*}
            \|y_{k+1}-y_k\|&=\beta_k\|\Delta \yh_k+A_{12}\xh_k-f_{1}(O_k,x_k,y_k)\|\\
            &\leq \beta_k(\|\Delta\|\|\yh_k\|+\|A_{12}\|\|\xh_k\|+\|f_{1}(O_k,x_k,y_k)\|)\\
            &\leq \beta_k\left(\|\Delta\|\|\yh_k\|+\|A_{12}\|\|\xh_k\|+b_{max}\sqrt{d} + \check{h}_4 (\|\xh_k\|_{Q_{22}} +\|\yh_k\|_{Q_{\Delta}})\right)\tag{Lemma \ref{lem:f_hat_bound}}\\
            &\leq \beta_k\hh_4(\sqrt{d}+\|\xh_k\|_{Q_{22}} +\|\yh_k\|_{Q_{\Delta}})
        \end{align*}
        where $\hh_4=\max\left\{\|\Delta\|+\check{h}_4, \|A_{12}\| +\check{h}_4 ,  b_{max} \right\}$.

    
    

\end{enumerate}
    

\end{proof}    

\begin{proof}[Proof of Lemma \ref{lem:crude_Uk_bnd}]
    From Lemma \ref{lem:crdue_upper_bnd_norm}, for all $k\geq 0$, we have
    \begin{align*}
        U_{k}&\leq (1+\alpha_{k-1}(\hat{h}_1+\hh_2)) U_{k-1} + \alpha_{k-1} \hat{h}_2d\\
        &\leq \Pi_{i=0}^{k-1} (1+\alpha_i(\hat{h}_1+\hh_2)) U_0 + \hh_2d\sum_{i=0}^{k-1}\alpha_i\Pi_{j=i+1}^{k-1} (1+\alpha_{j}(\hat{h}_1+\hh_2))\\
        &\leq U_0 \exp\left((\hat{h}_1+\hh_2)\sum_{i=0}^{k-1}\alpha_i\right)  + \hh_2d\sum_{i=0}^{k-1}\alpha_i \exp\left((\hat{h}_1+\hh_2)\sum_{j=i+1}^{k-1}\alpha_{j}\right).
    \end{align*}
    For the first term, we have
    \begin{align*}
        \exp\left((\hat{h}_1+\hh_2)\sum_{i=0}^{k-1}\alpha_i\right)=&\exp\left((\hat{h}_1+\hh_2)\sum_{i=0}^{k-1}\frac{\alpha}{(i+K_0)^{\xi}}\right) \\
        \leq &\exp \left((\hat{h}_1 +\hh_2)\left(\frac{\alpha}{K_0^\xi}+\int_{x=0}^{k}\frac{\alpha}{(x+K_0)^{\xi}}dx\right)\right)\\
        = &\exp \left(\frac{(\hat{h}_1 +\hh_2)\alpha}{K_0^\xi}+(\hat{h}_1 +\hh_2)\left[\frac{\alpha}{(1-\xi)(x+K_0)^{\xi-1}}\right]_{x=0}^{k}\right)\\
        = &\exp \left(\frac{(\hat{h}_1 +\hh_2)\alpha}{K_0^\xi}+\frac{(\hat{h}_1 +\hh_2)\alpha}{(1-\xi)}\left[(k+K_0)^{1-\xi}-K_0^{1-\xi}\right]\right).
    \end{align*}
    Similarly, for the second term,
    \begin{align*}
        \sum_{i=0}^{k-1}\alpha_i \exp\left((\hat{h}_1+ \hh_2)\sum_{j=i+1}^{k-1}\alpha_{j}\right)\leq & \sum_{i=0}^{k-1}\frac{\alpha}{(i+K_0)^\xi} \exp\left(\frac{(\hat{h}_1 +\hh_2)\alpha}{(1-\xi)} \left[(k+K_0)^{1-\xi}-(i+K_0)^{1-\xi}\right]\right)\\
        = & \alpha\exp\left(\frac{(\hat{h}_1 +\hh_2)\alpha}{(1-\xi)}(k+K_0)^{1-\xi}\right)\sum_{i=0}^{k-1} 
        \frac{1}{(i+K_0)^\xi}\exp\left(-\frac{(\hat{h}_1 +\hh_2)\alpha}{(1-\xi)}(i+K_0)^{1-\xi} \right)\\
        \leq & \alpha\exp\left(\frac{(\hat{h}_1 +\hh_2)\alpha}{(1-\xi)}(k+K_0)^{1-\xi}\right)\sum_{i=0}^{k-1} 
        \frac{1}{(i+K_0)^\xi}\exp\left(-\frac{(\hat{h}_1 +\hh_2)\alpha}{(1-\xi)}(i+K_0)^{1-\xi} \right)\\
        = & \alpha\exp\left(\frac{(\hat{h}_1 +\hh_2)\alpha}{(1-\xi)}(k+K_0)^{1-\xi}\right) \bigg[ 
        \frac{1}{K_0^\xi}\exp\left(-\frac{(\hat{h}_1 +\hh_2)\alpha}{(1-\xi)}K_0^{1-\xi} \right) 
        \\
        &\quad\quad+\sum_{i=1}^{k-1}\frac{1}{(i +K_0)^\xi} \exp\left(-\frac{(\hat{h}_1 +\hh_2)\alpha}{(1-\xi)}(i+K_0)^{1-\xi} \right) \bigg]\\
        \leq & \alpha\exp\left(\frac{(\hat{h}_1 +\hh_2)\alpha}{(1-\xi)}(k+K_0)^{1-\xi}\right) \bigg[ 
        \frac{1}{K_0^\xi}\exp\left(-\frac{(\hat{h}_1 +\hh_2)\alpha}{(1-\xi)}K_0^{1-\xi} \right) 
        \\
        &\quad\quad+\int_{x=0}^{k}\frac{1}{(x +K_0)^\xi} \exp\left(-\frac{(\hat{h}_1 +\hh_2)\alpha}{(1-\xi)}(x+K_0)^{1-\xi} \right) \bigg]\\
        = & \alpha\exp\left(\frac{(\hat{h}_1 +\hh_2)\alpha}{(1-\xi)}(k+K_0)^{1-\xi}\right) \bigg[ 
        \frac{1}{K_0^\xi}\exp\left(-\frac{(\hat{h}_1 +\hh_2)\alpha}{(1-\xi)}K_0^{1-\xi} \right) 
        \\
        &\quad\quad-\frac{1}{(\hat{h}_1 +\hh_2)\alpha} \exp\left(-\frac{(\hat{h}_1 +\hh_2)\alpha}{(1-\xi)}(x+K_0)^{1-\xi} \right)\bigg]_{x=0}^{k} \bigg]\\
        \leq & \alpha\exp\left(\frac{(\hat{h}_1 +\hh_2)\alpha}{(1-\xi)}(k+K_0)^{1-\xi}\right) \bigg[ 
        \frac{1}{K_0^\xi}\exp\left(-\frac{(\hat{h}_1 +\hh_2)\alpha}{(1-\xi)}K_0^{1-\xi} \right) 
        \\
        &+\frac{1}{(\hat{h}_1 +\hh_2)\alpha} \exp\left(-\frac{(\hat{h}_1 +\hh_2)\alpha}{(1-\xi)}K_0^{1-\xi} \right)  \bigg]\\
        \leq & \alpha\left(\frac{1}{K_0^\xi}+\frac{1}{(\hat{h}_1 +\hh_2)\alpha}\right)\exp\left(\frac{(\hat{h}_1 +\hh_2)\alpha}{(1-\xi)}\left((k+K_0)^{1-\xi}-K_0^{1-\xi}\right)\right).
    \end{align*}
    Putting things together, we have
    \begin{align*}
        U_{k}\leq& U_0\exp \left(\frac{(\hat{h}_1 +\hh_2)\alpha}{K_0^\xi}+\frac{(\hat{h}_1 +\hh_2)\alpha}{(1-\xi)}\left[(k+K_0)^{1-\xi}-K_0^{1-\xi}\right]\right) \\
        &+\alpha\hat{h}_2 d\left(\frac{1}{K_0^\xi}+\frac{1}{(\hat{h}_1 +\hh_2)\alpha}\right)\exp\left(\frac{(\hat{h}_1 +\hh_2)\alpha}{(1-\xi)}\left((k+K_0)^{1-\xi}-K_0^{1-\xi}\right)\right)
    \end{align*}
\end{proof}

\begin{proof}[Proof of Lemma \ref{lem:boundedness} ]
    
    Recall that $Q_{22}$ and $Q_{\Delta}$ were defined such that 
    \begin{align*}
        A_{22}^\top Q_{22}+Q_{22}A_{22}=I\\
        \Delta^\top Q_{\Delta}+Q_{\Delta}\Delta=I.
    \end{align*}
    Note that by Assumption \ref{ass:hurwitz_main}, we can always find positive-definite matrices $Q_{22}$ and $Q_{\Delta}$ which satisfy the above equations. Furthermore, for all $k>k_C$, by Lemma \ref{lem:contraction_prop} we have $\|(I-\alpha_kA_{22})\|^2_{Q_{22}}\leq (1-a_{22}\alpha_k)$ and $\|(I-\beta_k\Delta)\|^2_{Q_{\Delta}}\leq (1-\delta\beta_k)$  for positive constants $a_{22}=\frac{1}{2\|Q_{22}\|}$ and $\delta=\frac{1}{2\|Q_{\Delta}\|}$. Throughout the proof, we consider $k>k_C$. 

    Recall $V_k=\E[\|\xh_k\|_{Q_{22}}^2]$ and $W_k=\E[\|\yh_k\|_{Q_{\Delta}}^2]$.
    
     First, we handle the $V_k$ term. 
    \begin{align*}
        x_{k+1}=&x_k-\alpha_k(A_{21}y_k+A_{22}x_k)+\alpha_kf_{2}(O_k,x_k,y_k)\\
        x_{k+1}+A_{22}^{-1}A_{21}y_{k+1}=&x_k+A_{22}^{-1}A_{21}y_{k}-\alpha_kA_{22}(x_k+A_{22}^{-1}A_{21}y_k)+\alpha_kf_{2}(O_k,x_k,y_k)+A_{22}^{-1}A_{21}(y_{k+1}-y_k)\\
        \xh_{k+1}=&(I-\alpha_kA_{22})\xh_k+\alpha_kf_{2}(O_k,x_k,y_k)+\beta_kA_{22}^{-1}A_{21}(-(A_{11}y_k+A_{12}x_k)+f_1(O_k,x_k,y_k))\\
        \xh_{k+1}=&(I-\alpha_kA_{22})\xh_k+\alpha_kf_{2}(O_k,x_k,y_k)\\
        &+\beta_kA_{22}^{-1}A_{21}(-(\underbrace{(A_{11}-A_{12}A_{22}^{-1}A_{21})}_{\Delta}\yh_k+A_{12}\xh_k)+f_1(O_k,x_k,y_k))
        \end{align*}
        Taking norm square and expectation thereafter, we get:
        \begin{align*}
            \E[\|\xh_{k+1}\|^2_{Q_{22}}]=&\E[\|(I-\alpha_kA_{22})\xh_k\|^2_{Q_{22}}]+\underbrace{\alpha_k^2\E[\|f_{2}(O_k,x_k,y_k)\|^2_{Q_{22}}]}_{T_1}\\
            &+\underbrace{\beta_k^2\E[\|A_{22}^{-1}A_{21}(-(\Delta\yh_k+A_{12}\xh_k)+f_1(O_k,x_k,y_k))\|^2_{Q_{22}}]}_{T_2}\\
            &+\underbrace{2\beta_k\E[\langle(I-\alpha_kA_{22})\xh_k, A_{22}^{-1}A_{21}(-(\Delta\yh_k+A_{12}\xh_k)+f_1(O_k,x_k,y_k)) \rangle_{Q_{22}}]}_{T_3}\\
            &+\underbrace{2\alpha_k\beta_k\E[\langle f_2(O_k,x_k,y_k), A_{22}^{-1}A_{21}(-(\Delta\yh_k+A_{12}\xh_k)+f_1(O_k,x_k,y_k)) \rangle_{Q_{22}}]}_{T_4}\\
            &+\underbrace{2\alpha_k\E[\langle(I-\alpha_kA_{22})\xh_k, f_2(O_k, x_k, y_k) \rangle_{Q_{22}}]}_{T_5}
    \end{align*}
    \begin{itemize}
        \item For $T_1$, by Lemma \ref{lem:f_hat_bound} we have
        \begin{align*}
            \|f_2(O_k, x_k, y_k)\|_{Q_{22}}^2\leq 3(\gamma_{max}(Q_{22}) b^2_{max}d +\check{h}_3^2 (\|\xh_k\|^2_{Q_{22}} +\|\yh_k\|^2_{Q_{\Delta}})),
        \end{align*}
        to get:
        \begin{align*}
            T_1\leq \alpha_k^2\brc_1(d+V_k+W_k).
        \end{align*}
        where $\brc_1=3\max\{\gamma_{max}(Q_{22}) b^2_{max}, \check{h}_3^2\}$.
        \item For $T_2$,  again we use Lemma \ref{lem:f_hat_bound} 
        to get:
        \begin{align*}
            T_2\leq & 3\beta_k^2\left(\check{h}_6^2d +\check{h}_5^2 (\|\xh_k\|_{Q_{22}}^2 +\|\yh_k\|_{Q_{\Delta}}^2)\right) \\
            \leq & \brc_2\beta_k^2(d+V_k+W_k)
        \end{align*}
        
        where $\brc_2=3\max\{\check{h}_6^2, \check{h}_5^2\}$.
        \item For $T_3$, we apply Cauchy-Schwarz inequality to get:
        \begin{align*}
           T_3\leq 2\beta_k\E[\|\xh_k\|_{Q_{22}}\|A_{22}^{-1}A_{21}\left((\Delta\yh_k+A_{12}\xh_k)-f_1(O_k,x_k,y_k))\right)\|_{Q_{22}}]
        \end{align*}
        Using AM-GM inequality $2ab\leq \frac{a^2}{\eta}+b^2\eta$ with $\eta=\frac{2\beta_k}{a_{22}\alpha_k}$, we get:
        \begin{align}
            T_3&\leq \frac{a_{22}\alpha_k}{2}\E[\|\xh_k\|_{Q_{22}}^2]+\frac{4\beta_k^2}{a_{22}\alpha_k}\E[\|A_{22}^{-1}A_{21}\left((\Delta\yh_k+A_{12}\xh_k)-f_1(O_k,x_k,y_k)\right)\|_{Q_{22}}^2]\nonumber\\
            &\leq \frac{a_{22}\alpha_k}{2}V_k+\frac{\brc_3\beta_k^2}{\alpha_k}(d+V_k+W_k)\label{c_3def}
        \end{align}
        where $\brc_3=\frac{4\brc_2}{a_{22}}$.
        \item For $T_4$, again applying Cauchy-Schwarz inequality, we get:
        \begin{align*}
            T_4\leq 2\alpha_k\beta_k\E[\|f_2(O_k,x_k,y_k)\|_{Q_{22}}\|A_{22}^{-1}A_{21}(-(\Delta\yh_k+A_{12}\xh_k)+f_1(O_k,x_k,y_k))\|_{Q_{22}}]
        \end{align*}
        Using AM-GM inequality and after some simple calculation, we get:
        \begin{align*}
            T_4\leq \brc_4\alpha_k\beta_k(d+V_k+W_k)
        \end{align*}
        where $\brc_4=\brc_1+\brc_2$.

        \item For $T_5$, we break it down into two terms:
        \begin{align*}
            T_5&=2\alpha_k\E[\langle(I-\alpha_kA_{22})\xh_k, f_2(O_k, x_k, y_k) \rangle_{Q_{22}}]\\
            &=\underbrace{2\alpha_k\E[\langle \xh_k, f_2(O_k, x_k, y_k) \rangle_{Q_{22}}]}_{T_{51}}\underbrace{-2\alpha_k^2\E[\langle A_{22}\xh_k, f_2(O_k, x_k, y_k) \rangle_{Q_{22}}]}_{T_{52}}
        \end{align*}
        By Remark \ref{rem:pois_eq_geo_mix}, we have a unique function $\hat{f}_2(O, x_k, y_k)$ such that,
        \begin{align*}
            \fh_2(O,x_k,y_k) =f_2(O,x_k,y_k)+ \sum_{o'\in S}P(o'|O) \fh_2(o',x_k,y_k),
        \end{align*}
    where $P(O'|O)$ is the transition probability corresponding to the Markov chain $\{O_k\}_{k\geq 0}$. 
    Therefore,
    \begin{align*}
        T_{51}=&2\alpha_k\E\left[\langle\xh_k, \fh_2(O_k,x_k,y_k)-\sum_{o'\in S}P(o'|O_k)\fh_2(o',x_k,y_k)\rangle_{Q_{22}}\right]\nonumber\\
        =&2\alpha_k\E\left[\langle\xh_k,\fh_2(O_k,x_k,y_k) -\E_{O_k}\fh_2 (\cdot,x_k,y_k) \rangle_{Q_{22}}\right]\nonumber\\
        =&2\alpha_k\E\left[\langle\xh_k,\fh_2(O_k,x_k,y_k)-\E_{O_{k-1}}\fh_2(\cdot,x_k,y_k)+\E_{O_{k-1}}\fh_2(\cdot,x_k,y_k)-\E_{O_k}\fh_2(\cdot,x_k,y_k)\rangle_{Q_{22}}\right]\nonumber\\
        =&2\alpha_k\E\left[\langle\xh_k, \E_{O_{k-1}}\fh_2(\cdot,x_k,y_k)-\E_{O_{k}}\fh_2(\cdot,x_k,y_k)\rangle\right]\tag{By tower property}\nonumber\\
        =&2\alpha_k\underbrace{\E[\langle\xh_k,\E_{O_{k-1}} \fh_2 (\cdot,x_k, y_k)\rangle_{Q_{22}}]}_{\bar{d}_k^x}- 2\alpha_k \underbrace{\E[\langle \xh_{k+1}, \E_{O_{k}} \fh_2 (\cdot,x_{k+1}, y_{k+1}) \rangle_{Q_{22}}]}_{\bar{d}_{k+1}^x}\\
        &+ \underbrace{ 2\alpha_k\E[ \langle\xh_{k+1},\E_{O_{k}}\fh_2 (\cdot,x_{k+1} ,y_{k+1})-\E_{O_{k}} \fh_2(\cdot,x_k,y_k)\rangle_{Q_{22}}]}_{T_{511}}+ \underbrace{2\alpha_k\E[\langle (\xh_{k+1}^\top-\xh_k^\top), \E_{O_{k}} \fh_2(\cdot,x_k,y_k)\rangle_{Q_{22}}]}_{T_{512}}
    \end{align*}
    For $T_{511}$, we use Cauchy-Schwarz inequality and the fact that $\hat{f}_2$ is Lipschitz, to get:
    \begin{align*}
        T_{511}\leq& 2\alpha_k \hh_2\sqrt{\gamma_{max}(Q_{22})}\E[\|\xh_{k+1}\|_{Q_{22}}(\|x_{k+1}-x_{k}\|+\|y_{k+1}-y_{k}\|)]\tag{Lemma \ref{lem:f_hat_bound}}\\
        \leq & \alpha_k\brc_{5} \E[\|\xh_{k+1}\|_{Q_{22}}(\hh_3\alpha_k + \hh_4\beta_k)(\sqrt{d}+\|\xh_k\|_{Q_{22}} +\|\yh_k\|_{Q_{\Delta}})]\tag{Lemma \ref{lem:crdue_upper_bnd_norm}}
    \end{align*}
    where $\brc_5=2\sqrt{\gamma_{max}(Q_{22})}\check{h}_2$. Applying AM-GM to the previous inequality, we get
    \begin{align*}
        T_{511}\leq & 0.5\alpha_k^2\left(\hh_3 + \hh_4\frac{\beta}{\alpha}\right)\brc_{5} \E[\|\xh_{k+1}\|_{Q_{22}}^2+(\sqrt{d}+\|\xh_k\|_{Q_{22}} +\|\yh_k\|_{Q_{\Delta}})^2]\\
        \leq & 0.5\alpha_k^2\left(\hh_3 + \hh_4\frac{\beta}{\alpha}\right)\brc_{5} \E\left[(1 +\alpha_k \hh_1^{xx}) V_k +\hh_2^{xx} \alpha_k(d+W_k)+3(d+\|\xh_k\|_{Q_{22}}^2 +\|\yh_k\|_{Q_{\Delta}}^2)\right]\tag{Lemma \ref{lem:crdue_upper_bnd_norm}}\\
        =& \alpha_k^2 \brc_{6}(d+V_k+W_k),
    \end{align*}
    where $\brc_{6} = 0.5\left(\hh_3 + \hh_4\frac{\beta}{\alpha}\right)\brc_{5}\max\{4 +\alpha\hh_1^{xx}, 3 +\hh_2^{xx}\} $.

    Similarly, for $T_{512}$, we use the Cauchy-Schwarz inequality to get:
    \begin{align*}
        T_{512}\leq& 2\alpha_k^2\E[\|-A_{22}\xh_k+f_{2}(O_k,x_k,y_k)\\
        &+\frac{\beta_k}{\alpha_k}A_{22}^{-1}A_{21}(-(A_{11}y_k+A_{12}x_k)+f_1(O_k,x_k,y_k))\|_{Q_{22}}\|\E_{O_{k}}\fh_2(\cdot,x_k,y_k)\|_{Q_{22}}]
    \end{align*}
    
    Applying AM-GM inequality $2ab\leq \frac{a^2}{\eta}+b^2\eta$ with $\eta=\frac{1-\rho}{2}$, we get:
    \begin{align*}
        T_{512}\leq& \alpha_k^2\E\bigg[\frac{2}{1-\rho}\|-A_{22}\xh_k+f_{2}(O_k,x_k,y_k)+\frac{\beta_k}{\alpha_k}A_{22}^{-1}A_{21}(-(A_{11}y_k+A_{12}x_k) + f_1(O_k, x_k,y_k))\|_{Q_{22}}^2\\
        &+ \frac{1-\rho}{2}\|\E_{O_{k}} \fh_2(\cdot, x_k,y_k)\|_{Q_{22}}^2\bigg]\\
        \leq& \alpha_k^2\brc_7(d+V_k+W_k) 
    \end{align*}
    where $\brc_7= \frac{2}{1-\rho}\left(\max\{4\|A_{22}\|_{Q_{22}}^2 +12\check{h}_3^2 , 12\gamma_{max}(Q_{22}) b_{max}^2 \} + \frac{2\beta^2\brc_2}{\alpha^2}\right) + \frac{6\gamma_{\max}(Q_{22})}{(1-\rho)} \max\left\{b_{max}^2\gamma_{max} , \frac{\check{h}_3^2}{4}\right\} $. Here, we used
    \begin{align*}
        \|-A_{22}\xh_k+f_{2}(O_k,x_k,y_k)+\frac{\beta_k}{\alpha_k}A_{22}^{-1}A_{21}(-(A_{11}y_k+A_{12}x_k) + f_1(O_k, x_k,y_k))\|_{Q_{22}}^2 \\
        \leq 2\|-A_{22}\xh_k+f_{2}(O_k,x_k,y_k)\|_{Q_{22}}^2+\frac{2\beta^2_k}{\alpha^2_k}\|A_{22}^{-1}A_{21}(-(A_{11}y_k+A_{12}x_k) + f_1(O_k, x_k,y_k))\|_{Q_{22}}^2\\
        \leq \left(\max\{4\|A_{22}\|_{Q_{22}}^2 +12\check{h}_3^2 , 12\gamma_{max}(Q_{22}) b_{max}^2 \} + \frac{2\beta^2\brc_2}{\alpha^2}\right)(d+V_k+W_k)\tag{Lemma \ref{lem:f_hat_bound}}
    \end{align*}

    Furthermore, by Lemma \ref{lem:f_hat_bound} we have
    \begin{align*}
        \E\left[\|\E_{O_{k}} \fh_2(\cdot, x_k,y_k)\|_{Q_{22}}^2\right] \leq & \frac{6\gamma_{\max}(Q_{22})}{(1-\rho)^2} \E\left[  b_{max}^2\gamma_{max}(Q_{22}) d+ \frac{\check{h}_3^2}{4}\left(\left\| \xh_k \right\|_{Q_{22}}^2 +  \|\yh_k\|_{Q_{\Delta}}^2\right) \right]\\ \leq & \frac{6\gamma_{\max}(Q_{22})}{(1-\rho)^2} \max\left\{b_{max}^2\gamma_{max} , \frac{\check{h}_3^2}{4}\right\} (d+V_k+W_k)
    \end{align*}

    Finally, for $T_{52}$, using Cauchy-Schwarz inequality and then AM-GM inequality, we have:
    \begin{align*}
        T_{52}&\leq \alpha_k^2(\E[\|A_{22}\xh_k\|_{Q_{22}}^2]+\E[\|f_2(O_k,x_k,y_k)\|_{Q_{22}}^2])\\
        &\leq \alpha_k^2(\E[\|A_{22}\|^2_{Q_{22}}\|\xh_k\|_{Q_{22}}^2] + 3\gamma_{max}(Q_{22}) b_{max}^2d +3\check{h}_3^2 (\|\xh_k\|_{Q_{22}}^2 +\|\yh_k\|_{Q_{\Delta}}^2) ) \tag{Lemma   \ref{lem:f_hat_bound}}\\
        &\leq \alpha_k^2\max\left\{\|A_{22}\|^2_{Q_{22}} + 3\check{h}_3^2 , 3\gamma_{max}(Q_{22}) b_{max}^2  \right\} (d+V_k+W_k)\\
        &\leq \brc_8\alpha_k^2(d+V_k+W_k),
    \end{align*}
    where $\brc_8= \max\left\{\|A_{22}\|^2_{Q_{22}} + 3\check{h}_3^2 , 3\gamma_{max}(Q_{22}) b_{max}^2  \right\}$.
    \end{itemize}
    Finally, by Lemma \ref{lem:contraction_prop}  we have that:
    \begin{align*}
        \E[\|(I-\alpha_kA_{22}) \xh_k\|^2_{Q_{22}}]\leq (1-a_{22}\alpha_k)V_k.
    \end{align*}

    Combining everything, we have:
    \begin{align}
        V_{k+1} \leq& (1-a_{22}\alpha_k)V_k + \alpha_k^2\brc_1(d +V_k+W_k) + \brc_2\beta_k^2(d+V_k+W_k) + \frac{a_{22}\alpha_k}{2}V_k+\frac{\brc_3\beta_k^2}{\alpha_k}(d+V_k+W_k) \nonumber\\
        &+ \brc_4 \alpha_k\beta_k (d+V_k +W_k) + \alpha_k^2 \brc_{6}(d+V_k+ W_k) + \alpha_k^2\brc_7 (d+ V_k+W_k) + \brc_8\alpha_k^2 (d+V_k+W_k)\nonumber\\
        &+2\alpha_k(\bar{d}_k^x-\bar{d}_{k+1}^x)\nonumber\\
        \leq & \left(1-\frac{a_{22}\alpha_k}{2}\right)V_k+ \alpha_k^2 \brc_{9} (d+V_k+W_k) +\frac{\brc_3\beta_k^2}{\alpha_k}(d+V_k+W_k) + 2\alpha_{k-1}\bar{d}_k^x-2\alpha_k \bar{d}_{k+1}^x+2(\alpha_k-\alpha_{k-1}) \bar{d}_{k}^x,\label{eq:V_k_recursion}
    \end{align}
    where $\brc_9 = \brc_1+\frac{\beta^2}{\alpha^2}\brc_2 + \frac{\beta}{\alpha}\brc_4 + \brc_6 + \brc_7 + \brc_8$.
    
    We bound the last term as follows:
    
    \begin{align}
        |(\alpha_k-\alpha_{k-1})\bar{d}_k^x|&\leq \frac{\xi}{\alpha}\alpha_k^2|\bar{d}_k^x|\tag{Lemma \ref{lem:step_size_gap}}\nonumber\\
        &\leq \frac{\xi}{\alpha}\alpha_k^2 \E[\|\xh_k\|_{Q_{22}}\|\E_{O_{k-1}} \fh_2 (\cdot,x_k, y_k)\|_{Q_{22}}]\nonumber\\
        &\leq \frac{\xi}{\alpha}\alpha_k^2 \E\left[\|\xh_k\|_{Q_{22}}\left(\frac{2}{1-\rho} \left[b_{max} \sqrt{\gamma_{max}(Q_{22})} \sqrt{d}+ \frac{\check{h}_3}{2}\left(\left\| \xh_k\right\|_{Q_{22}} +  \|\yh_k\|_{Q_{\Delta}}\right) \right]\right)\right] \tag{Lemma \ref{lem:f_hat_bound}}\nonumber\\
        &\leq \frac{2\xi}{(1-\rho)\alpha}\max\left\{b_{max}\sqrt{\gamma_{max}(Q_{22})} , \frac{\check{h}_3}{2}\right\} \alpha_k^2 \E\left[\|\xh_k\|_{Q_{22}}(  \sqrt{d} + \|\xh_k\|_{Q_{22}} +\|\yh_k\|_{Q_{\Delta}})\right]\nonumber\\
        &\leq \frac{\xi}{(1-\rho)\alpha}\max\left\{b_{max}\sqrt{\gamma_{max}(Q_{22})} , \frac{\check{h}_3}{2}\right\} \alpha_k^2 \E[\|\xh_k\|_{Q_{22}}^2+  3(d + \|\xh_k\|_{Q_{22}}^2 + \|\yh_k\|_{Q_{\Delta}}^2 )] \nonumber\\
        &= \frac{\brc_{10}}{2}  \alpha_k^2 (d+V_k+W_k) \label{eq:d^x_k_upper_bnd}
    \end{align}
    where $\brc_{10}=\frac{8\xi}{(1-\rho)\alpha}\max\left\{b_{max}\sqrt{\gamma_{max}(Q_{22})} , \frac{\check{h}_3}{2}\right\} $. Thus we get:
    \begin{align}\label{eq:V_k_rec}
        V_{k+1}\leq &\left(1-\frac{a_{22}\alpha_k}{2}\right)V_k+ \alpha_k^2 \brc_{9} (d+V_k+W_k) +\frac{\brc_3\beta_k^2}{\alpha_k}(d+V_k+W_k) +2\alpha_{k-1}\bar{d}_k^x-2\alpha_k \bar{d}_{k+1}^x\nonumber\\
        &+\check{c}_{10} \alpha_k^2 (d+V_k+W_k)\nonumber\\
        \leq & \left(1-\frac{a_{22}\alpha_k}{2}\right)V_k+ \alpha_k^2 \brc_{11} (d+V_k+W_k) +\frac{\brc_3\beta_k^2}{\alpha_k}(d+V_k+W_k) +2\alpha_{k-1}\bar{d}_k^x-2\alpha_k \bar{d}_{k+1}^x,
    \end{align}
    where $\brc_{11} = \brc_{9} + \brc_{10} $.

    Next, we handle $W_k$. We have
    \begin{align*}
        y_{k+1}&=y_k-\beta_k(A_{11}y_k+A_{12}x_k)+\beta_kf_{1}(O_k,x_k,y_k)\\
        \yh_{k+1}&=\yh_k-\beta_k((A_{11}-A_{12}A_{22}^{-1}A_{21})\yh_k+A_{12}\xh_k)+\beta_kf_{1}(O_k,x_k,y_k)\\
        \yh_{k+1}&=(I-\beta_k\Delta)\yh_k+\beta_kf_{1}(O_k,x_k,y_k)-\beta_kA_{12}\xh_k\\
\end{align*}
Taking norm square and expectation thereafter, we get:
\begin{align*}
        \E[\|\yh_{k+1}\|^2_{Q_{\Delta}}]&= \E[\|(I-\beta_k\Delta)\yh_k\|^2_{Q_{\Delta}}]+\underbrace{\beta_k^2\E[\|f_{1}(O_k,x_k,y_k)]\|^2_{Q_{\Delta}}}_{T_6}+\underbrace{\beta_k^2\E[\|A_{12}\xh_k\|^2_{Q_{\Delta}}]}_{T_7}\\
        &\underbrace{-2\beta_k\E[\langle(I-\beta_k\Delta)\yh_k, A_{12}\xh_k \rangle_{Q_{\Delta}}]}_{T_8}\underbrace{-2\beta_k^2\E\langle f_1(O_k,x_k,y_k), A_{12}\xh_k\rangle_{Q_{\Delta}}]}_{T_9}\\
        &+\underbrace{2\beta_k\E[\langle(I-\beta_k\Delta)\yh_k, f_1(O_k, x_k, y_k) \rangle_{Q_{\Delta}}]}_{T_{10}}.
    \end{align*}
    
    \begin{itemize}
        \item For $T_6$, using Lemma \ref{lem:f_hat_bound} we have 
        \begin{align*}
            T_6 &\leq 3\beta_k^2(\gamma_{max}(Q_{\Delta}) b_{max}^2d +\check{h}_4^2 (V_k +W_k))\\
            &= \brc_{12} \beta_k^2(d+ V_k +W_k)\nonumber,
        \end{align*}
        where $\brc_{12} = 3\max\{\gamma_{max}(Q_{\Delta}) b_{max}^2, \check{h}_4^2\}$.
        \item For $T_7$, we have
        \begin{align*}
            T_7&\leq \|A_{12}\|^2_{Q_{\Delta}}\beta_k^2 \frac{\gamma_{max}(Q_{\Delta})}{\gamma_{min}(Q_{22})}V_k
        \end{align*}
        \item For $T_8$, using Cauchy-Schwarz inequality, we have:
        \begin{align*}
            T_8 \leq& 2\beta_k \|A_{12}\|_{Q_{\Delta}} \|I -\beta_k \Delta\|_{Q_{\Delta}} \E[\| \yh_k\|_{Q_{\Delta}} \|\xh_k \|_{Q_{\Delta}}]\\
            \leq& 2 \beta_k \|A_{12}\|_{Q_{\Delta}} \E[\|\yh_k\|_{Q_{\Delta}} \|\xh_k\|_{Q_{\Delta}}]\tag{Assumption on $k$}\\
            \leq & \frac{\beta_k\delta}{2}\E[\|\yh_k\|_{Q_{\Delta}}^2] + \beta_k \frac{2\|A_{12}\|^2_{Q_{\Delta}}}{\delta}\E[\|\xh_k\|_{Q_{\Delta}}^2]\\
            \leq & \frac{\beta_k\delta}{2}W_k + \beta_k \frac{2\|A_{12}\|^2_{Q_{\Delta}}\gamma_{max}(Q_{\Delta})}{\gamma_{min}(Q_{22})\delta}V_k
        \end{align*}
        where for the one to last inequality we used AM-GM inequality $2ab\leq \frac{a^2}{\eta}+\eta b^2$ with $\eta=\frac{\delta}{2}$.
        \item For $T_9$, we have the following.
        \begin{align*}
            T_9 & \leq 2\beta_k^2 \E[\|f_1(O_k,x_k,y_k)\|_{Q_{\Delta}} . \|A_{12}\xh_k\|_{Q_{\Delta}}]\\
            &\leq \beta_k^2 \E[\|f_1(O_k,x_k,y_k)\|_{Q_{\Delta}}^2 + \|A_{12}\xh_k\|_{Q_{\Delta}}^2]\\
            &\leq \beta_k^2 [\check{c}_{12}(d+ V_k +W_k) + \|A_{12}\|_{Q_{\Delta}}^2\frac{\gamma_{max}(Q_{\Delta})}{\gamma_{min}(Q_{22})}V_k]
        \end{align*}

        \item  For $T_{10}$, we have
        \begin{align*}
            T_{10}&= \underbrace{2\beta_k \E[\langle \yh_k, f_1(O_k, x_k, y_k) \rangle_{Q_{\Delta}}]}_{T_{101}}\underbrace{-2\beta_k^2\E[\langle \Delta\yh_k, f_1(O_k, x_k, y_k) \rangle_{Q_{\Delta}}]}_{T_{102}}
        \end{align*}
        Similar to analysis of $T_5$, we have
        \begin{align*}
            T_{101} =&2\beta_k \underbrace{\E[\langle\yh_k, \E_{O_{k-1}} \fh_1 (\cdot, x_k, y_k)\rangle_{Q_{\Delta}}]}_{\bar{d}_k^y}- 2\beta_k \underbrace{\E[\langle \yh_{k+1}, \E_{O_{k}} \fh_1 (\cdot, x_{k+1}, y_{k+1}) \rangle_{Q_{\Delta}}]}_{\bar{d}_{k+1}^y}\\
            &+ \underbrace{ 2\beta_k\E[\langle \yh_{k+1},\E_{O_{k}}\fh_1 (\cdot,x_{k+1} ,y_{k+1})-\E_{O_{k}}\fh_1 (\cdot,x_k,y_k)\rangle_{Q_{\Delta}}]}_{T_{1011}}+ \underbrace{2\beta_k\E[\langle (\yh_{k+1}^\top-\yh_k^\top), \E_{O_{k}}\fh_1 (\cdot,x_k,y_k) \rangle_{Q_{\Delta}}]}_{T_{1012}}.
        \end{align*}
        For $T_{1011}$ we have
        \begin{align*}
            T_{1011}\leq & 2\beta_k\sqrt{\gamma_{max}(Q_{\Delta})}  \E\left[\|\yh_{k+1}\|_{Q_{\Delta}}\left\|\E_{O_{k}}\fh_1 (\cdot,x_{k+1} ,y_{k+1})-\E_{O_{k}}\fh_1 (\cdot,x_k,y_k)\right\|\right]\\
            \leq& 2\beta_k \check{h}_2\sqrt{\gamma_{max}(Q_{\Delta})} \E[\| \yh_{k+1}\|_{Q_{\Delta}}( \|x_{k+1}-x_k\| + \|y_{k+1}-y_k\|)]\tag{Lemma \ref{lem:f_hat_bound}}\\
            \leq& 2\beta_k \check{h}_2\sqrt{\gamma_{max}(Q_{\Delta})}(\alpha_k\hh_3+\beta_k\hh_4) \E[\|\yh_{k+1}\|_{Q_{\Delta}}(\sqrt{d} +\|\xh_k\|_{Q_{22}} +\|\yh_k\|_{Q_{\Delta}})]\tag{Lemma \ref{lem:crdue_upper_bnd_norm}}
        \end{align*}
        Applying AM-GM to the previous inequality, we get
        \begin{align*}
            T_{1011} \leq & \alpha_k\beta_k\left(\hh_3 + \hh_4\frac{\beta}{\alpha}\right)\check{h}_2\sqrt{\gamma_{max}(Q_{\Delta})} \E\left[\|\yh_{k+1}\|_{Q_{\Delta}}^2+(\sqrt{d}+ \|\xh_k\|_{Q_{22}} + \|\yh_k\|_{Q_{\Delta}})^2\right]\\
            \leq & \alpha_k\beta_k\left(\hh_3 + \hh_4\frac{\beta}{\alpha}\right)\check{h}_2\sqrt{\gamma_{max}(Q_{\Delta})} \E\Bigg[(1+ \beta_k \hh_1^{yy})\|\yh_k \|_{Q_{\Delta}}^2 +\hh_2^{yy}\beta_k(d +  \|\xh_k\|_{Q_{22}}^2)\\
            &+3(d+\|\xh_k\|_{Q_{22}}^2 +\|\yh_k\|_{Q_{\Delta}}^2)\Bigg]\tag{Lemma \ref{lem:crdue_upper_bnd_norm}}\\
            =&\alpha_k\beta_k \brc_{13}(d+V_k+W_k)
        \end{align*}
    
    where $\brc_{13}=\left(\hh_3 + \hh_4\frac{\beta}{\alpha}\right)\check{h}_2\sqrt{\gamma_{max}(Q_{\Delta})}\max\{4 +\alpha\hh_1^{yy}, 3 +\hh_2^{yy}\}$. For $T_{1012}$ we have:
    \begin{align*}
        T_{1012}&\leq 2\beta_k\E[\|\yh_{k+1}^\top-\yh_k^\top\|_{Q_{\Delta}}\E_{O_{k}}\|\fh_1 (\cdot,x_k,y_k)\|_{Q_{\Delta}}]\tag{by Cauchy-Schwartz}\\
        &=2\beta_k^2\E[\|-\Delta\yh_k+f_{1}(O_k,x_k,y_k)-A_{12}\xh_k\|_{Q_{\Delta}}\E_{O_{k}}[\|\fh_1 (\cdot,x_k,y_k)\|_{Q_{\Delta}}]]
    \end{align*}
    
    Applying AM-GM inequality $2ab\leq \frac{a^2}{\eta}+b^2\eta$ with $\eta=\frac{1-\rho}{2}$, we get:
    \begin{align*}
        T_{1012}\leq& \beta_k^2\E\bigg[\frac{2}{1-\rho}\|-\Delta\yh_k+f_{1}(O_k,x_k,y_k)-A_{12}\xh_k\|_{Q_{\Delta}}^2+\frac{1-\rho}{2}\E_{O_{k}} \|\fh_1 (\cdot,x_k,y_k)\|_{Q_{\Delta}}^2\bigg]\\
        \leq & \frac{2}{1-\rho} \beta_k^2 \left(\|\Delta\|_{Q_{22}}^2 +\brc_{12} +\|A_{12}\|_{Q_{\Delta}}^2\frac{\gamma_{max}(Q_{\Delta})}{\gamma_{min}(Q_{22})}\right)(d+\E[\|\xh_k\|_{Q_{22}}^2]+\E[\|\yh_k\|_{Q_{\Delta}}^2])\\
        &+\frac{(1-\rho)\beta_k^2}{2}\E\big[\|\E_{O_{k}} \fh_1(\cdot, x_k,y_k)\|_{Q_{\Delta}}^2\big]\\
        \leq& \beta_k^2\brc_{14} (d+V_k+W_k)
    \end{align*}
    where $\brc_{14}= \frac{2}{1-\rho}\left(\|\Delta\|_{Q_{22}}^2+\brc_{12}+\|A_{12}\|_{Q_{\Delta}}^2\frac{\gamma_{max}(Q_{\Delta})}{\gamma_{min}(Q_{22})}\right) + \frac{6}{1-\rho} \left( b_{max}^2\gamma_{max}(Q_{\Delta}) + \frac{\check{h}_4^2}{4}\right)$. Here for bounding\\ $\E\left[\|\E_{O_{k}} \fh_1(\cdot, x_k,y_k)\|_{Q_{\Delta}}^2\right]$, we use Lemma \ref{lem:f_hat_bound}.
    
    For $T_{102}$ we have:
    \begin{align*}
        T_{102}&\leq 2\beta_k^2\E[\|\Delta\yh_k\|_{Q_{\Delta}}\|f_1(O_k, x_k, y_k)\|_{Q_{\Delta}}]\\
        &\leq \beta_k^2\E[\|\Delta\yh_k\|^2_{Q_{\Delta}}+\|f_1(O_k, x_k, y_k)\|^2_{Q_{\Delta}}]\\
        &\leq \beta_k^2\brc_{15}(d+W_k+V_k)
    \end{align*}
    where $\brc_{15}=\|\Delta\|_{Q_\Delta}^2+\brc_{12}$.
    \end{itemize}
    Now, by definition of $Q_{\Delta}$, we have that:
    \begin{align*}
        \E[\|(I-\beta_k\Delta)\yh_k\|^2_{Q_{\Delta}}]\leq (1-\delta\beta_k)W_k.
    \end{align*}
    
    Combining everything, we have:
    \begin{align}
        W_{k+1} \leq& (1- \delta \beta_k)W_k + \beta_k^2\brc_{12}(d +V_k +W_k ) + \beta_k^2\|A_{12}\|^2_{Q_{\Delta}} \frac{\gamma_{max}(Q_{\Delta})}{\gamma_{min}(Q_{22})}V_k + \frac{\beta_k\delta}{2}W_k + \beta_k \frac{2\|A_{12}\|^2_{Q_{\Delta}}\gamma_{max}(Q_{\Delta})}{\gamma_{min}(Q_{22})\delta}V_k \nonumber\\
        &+ \beta_k^2\check{c}_{12}(d + V_k +W_k) + \beta_k^2\|A_{12}\|_{Q_{\Delta}}^2\frac{\gamma_{max}(Q_{\Delta})}{\gamma_{min}(Q_{22})}V_k + \beta_k \alpha_k \brc_{13}(d +V_k+W_k) \nonumber\\ &+ \beta_k^2 \brc_{14} (d+V_k +W_k)+\beta_k^2 \brc_{15}(d+V_k+W_k)+2\beta_k(\bar{d}_k^y-\bar{d}_{k+1}^y)\nonumber\\
        \leq & (1-\frac{\delta\beta_k}{2})W_k+ \alpha_k\beta_k\brc_{16} (d+ V_k+W_k) +\beta_k \frac{2\|A_{12}\|^2_{Q_{\Delta}}\gamma_{max}(Q_{\Delta})}{\gamma_{min}(Q_{22})\delta}V_k+ 2\beta_{k-1}\bar{d}_k^y -2\beta_k \bar{d}_{k+1}^y+2(\beta_k-\beta_{k-1}) \bar{d}_{k}^y,\label{eq:W_k_recursion}
    \end{align}
    where $\brc_{16}=\frac{\beta}{\alpha}\left(\brc_{12}+2\|A_{12}\|^2_{Q_{\Delta}} \frac{\gamma_{max}(Q_{\Delta})}{\gamma_{min}(Q_{22})}+\brc_{14}+ \brc_{15}\right)+ \brc_{13}$.

    We bound the last term as follows:
    \begin{align}
        |(\beta_k-\beta_{k-1})\bar{d}_k^y|&\leq \frac{1}{\beta}\beta_k^2|\bar{d}_k^y|\tag{Lemma \ref{lem:step_size_gap}}\nonumber\\
        &\leq \frac{1}{\beta}\beta_k^2 \E[\|\yh_k\|_{Q_{\Delta}}\|\E_{O_{k-1}} \fh_1 (\cdot,x_k, y_k)\|_{Q_{\Delta}}]\nonumber\\
        &\leq \frac{1}{\beta}\beta_k^2 \E\left[\|\yh_k\|_{Q_{\Delta}}\left(\frac{2}{1 - \rho}\left[b_{max}\sqrt{\gamma_{max}(Q_{\Delta})} \sqrt{d}+ \frac{\check{h}_4}{2}\left(\left\| \xh_k\right\|_{Q_{22}} +  \|\yh_k\|_{Q_{\Delta}}\right) \right] \right)\right] \tag{Lemma \ref{lem:f_hat_bound}}\nonumber\\
        &\leq \frac{2}{(1-\rho)\beta} \max\left\{b_{max}\sqrt{\gamma_{max}(Q_{22})} , \frac{\check{h}_4}{2}\right\} \beta_k^2 \E[\|\yh_k\|_{Q_{\Delta}}(  \sqrt{d} + \|\xh_k\|_{Q_{22}} +\|\yh_k\|_{Q_{\Delta}})]\nonumber\\
        &\leq \frac{1}{(1-\rho)\beta} \max\left\{b_{max}\sqrt{\gamma_{max}(Q_{22})} , \frac{\check{h}_4}{2}\right\} \beta_k^2 \E\left[\|\yh_k\|_{Q_{\Delta}}^2+  3(d + \|\xh_k\|_{Q_{22}}^2 + \|\yh_k\|_{Q_{\Delta}}^2 )\right] \nonumber\\
        &= \frac{\brc_{17}}{2} \beta_k^2 (d+ V_k +W_k), \label{eq:d^y_k_upper_bnd}
    \end{align}
    where $\brc_{17}=\frac{8}{(1-\rho)\beta} \max\left\{b_{max}\sqrt{\gamma_{max}(Q_{22})} , \frac{\check{h}_4}{2}\right\}$. Thus we get:
    \begin{align}
        W_{k+1}\leq & (1-\frac{\delta\beta_k}{2})W_k+ \alpha_k\beta_k\brc_{18} (1+V_k+W_k) +\beta_k \frac{2\|A_{12}\|^2_{Q_{\Delta}}\gamma_{max}(Q_{\Delta})}{\gamma_{min}(Q_{22})\delta}V_k+ 2\beta_{k-1}\bar{d}_k^y-2\beta_k \bar{d}_{k+1}^y, \label{eq:W_k_rec}
    \end{align}
    where $\brc_{18}=\brc_{16}+\frac{\beta}{\alpha}\brc_{17}$.

    Then, by adding \eqref{eq:V_k_rec} and \eqref{eq:W_k_rec} we get,
    \begin{align*}
        U_{k+1}\leq & (1-\frac{a_{22}\alpha_k}{2})V_k+ \alpha_k^2 \brc_{11} (d+V_k+W_k) +\frac{\brc_3\beta_k^2}{\alpha_k}(d+V_k+W_k) +2\alpha_{k-1}\bar{d}_k^x-2\alpha_k \bar{d}_{k+1}^x\\
        &+ (1-\frac{\delta\beta_k}{2})W_k+ \alpha_k\beta_k\brc_{18} (d+V_k+W_k) +\beta_k \frac{2\|A_{12}\|^2_{Q_{\Delta}}\gamma_{max}(Q_{\Delta})}{\gamma_{min}(Q_{22})\delta}V_k+ 2\beta_{k-1}\bar{d}_k^y-2\beta_k \bar{d}_{k+1}^y
    \end{align*}

    But we had $k>k_C$, and hence $\frac{2\|A_{12}\|^2_{Q_{\Delta}}\gamma_{max}(Q_{\Delta})}{\gamma_{min}(Q_{22})\delta}\beta_k + \frac{\brc_3\beta_k^2}{\alpha_k}\leq \frac{a_{22}\alpha_k}{4}$, and $\frac{\brc_3\beta_k^2}{\alpha_k}\leq \frac{\delta\beta_k}{4}$. Hence,
    \begin{align*}
        U_{k+1}\leq & (1-\frac{a_{22}\alpha_k}{4})V_k+ \alpha_k^2 \brc_{11} (d+V_k+W_k) +\frac{\brc_3\beta_k^2}{\alpha_k} d +2\alpha_{k-1}\bar{d}_k^x-2\alpha_k \bar{d}_{k+1}^x\\
        &+ (1-\frac{\delta\beta_k}{4})W_k+ \alpha_k\beta_k\brc_{18} (d+V_k+W_k) + 2\beta_{k-1}\bar{d}_k^y-2\beta_k \bar{d}_{k+1}^y.\\
        \leq & (1-\frac{a_{22}\alpha_k}{4})V_k+ \alpha_k^2( \brc_{11}+ \frac{\beta}{\alpha}\brc_{18}) (d +V_k+W_k) +\frac{\brc_3\beta_k^2}{\alpha_k}d +2\alpha_{k-1}\bar{d}_k^x-2\alpha_k \bar{d}_{k+1}^x\\
        &+ (1-\frac{\delta\beta_k}{4})W_k+ 2\beta_{k-1}\bar{d}_k^y-2\beta_k \bar{d}_{k+1}^y.
    \end{align*}
     
     It is sufficient to have that $(\brc_{11} + \frac{\beta}{\alpha}\brc_{18})\alpha_k^2\leq \frac{a_{22}\alpha_k}{4}$ and $(\brc_{11}+\frac{\beta}{\alpha}\brc_{18})\alpha_k^2\leq \frac{\delta\beta_k}{4}$. Therefore, it is further sufficient to have $(\brc_{11}+\frac{\beta}{\alpha}\brc_{18})\alpha^2\frac{1}{{(k+1)}^{2\xi}}\leq \min\{\frac{a_{22}\alpha}{4},\frac{\delta \beta}{4}\}\frac{1}{k+1}$, which happens for
    $k\geq \left((\brc_{11} +\frac{\beta}{\alpha}\brc_{18})\alpha^2/\min\{\frac{a_{22}\alpha}{4},\frac{\delta \beta}{4}\}\right)^{\frac{1}{2\xi-1}}$.

    We define $\brc_{19} = \brc_{11}+\frac{\beta}{\alpha}\brc_{18}$. Then, for all $k\geq \max\left\{k_C,\left((\brc_{11} +\frac{\beta}{\alpha}\brc_{18})\alpha^2/\min\{\frac{a_{22}\alpha}{4},\frac{\delta \beta}{4}\}\right)^{\frac{1}{2\xi-1}}\right\}:= k_2$, we have,
    \begin{align}
        U_{k+1} \leq&  V_k + W_k+ 2\alpha_{k-1} \bar{d}_k^x+ 2\beta_{k-1} \bar{d}_k^y- 2\alpha_k \bar{d}_{k+1}^x-2\beta_k \bar{d}_{k+1}^y+ \brc_{19} \alpha_k^2 d + \frac{\brc_3 \beta_k^2}{\alpha_k}d \nonumber\\
        =&  U_k+2\alpha_{k-1}\bar{d}_k^x+2\beta_{k-1}\bar{d}_k^y-2\alpha_k\bar{d}_{k+1}^x-2\beta_k\bar{d}_{k+1}^y+\brc_{19} \alpha_k^2 d+ \frac{\brc_3\beta_k^2}{\alpha_k}d. \label{eq:U_k_rec}
    \end{align}
    Summing from $k_2$ to $K$, we have
    \begin{align*}
        U_{K+1}\leq U_{k_2}+ 2\alpha_{k_2-1}\bar{d}_{k_2}^x-2\alpha_K \bar{d}_{K+1}^x + 2\beta_{k_2-1}\bar{d}_{k_2}^y - 2\beta_{K}\bar{d}_{K+1}^y + \brc_{19}d \sum_{k=k_2}^K \alpha_k^2+ \brc_3d\sum_{k=k_2}^K \frac{\beta_k^2}{\alpha_k}.
    \end{align*}

   From \eqref{eq:d^x_k_upper_bnd}, we have $|\bar{d}_k^x|\leq \frac{\brc_{10}\alpha}{2\xi}(d+U_k)$ and from \eqref{eq:d^y_k_upper_bnd} we have $\bar{d}_k^y\leq \frac{\brc_{17}\beta}{2}(d+U_k)$. By the choice of $k_C$, we have $\alpha_{k} \frac{\brc_{10}\alpha}{\xi}\leq 0.3$ and $\beta_{k}\brc_{17}\beta\leq 0.3$. Since we assume $k\geq k_2 - 1$, and we have $k_2>k_C$, we get for all $K\geq k_2$
    \begin{align}
        U_{K+1}\leq& U_{k_2}+ 2\beta_{k_2 -1} \bar{d}_{k_2}^y+ 2\alpha_{k_2-1}\bar{d}_{k_2}^x+0.6(d+U_{K+1})   + \brc_{19}d\sum_{k=k_2}^K \alpha_k^2 +\brc_3d\sum_{k=k_2}^K \frac{\beta_k^2}{\alpha_k} \nonumber\\
        \implies 0.4U_{K+1}\leq &U_{k_2}+ 2\beta_{k_2 -1} \bar{d}_{k_2}^y+ 2\alpha_{k_2-1}\bar{d}_{k_2}^x+0.6 d  +  \frac{\brc_{19}\alpha^2}{2\xi-1} d + \frac{\brc_3\beta^2}{\alpha (1-\xi)} d\nonumber\\
        \implies 0.4U_{K+1}\leq& U_{k_2}+ 0.6(d +U_{k_2})+0.6d +  \frac{\brc_{19}\alpha^2}{2\xi-1} d + \frac{\brc_3\beta^2}{\alpha (1-\xi)} d\nonumber\\
        = &1.2d + 1.6U_{k_{2}}+\frac{\brc_{19}\alpha^2}{2\xi-1} d + \frac{\brc_3\beta^2}{\alpha (1-\xi)} d\nonumber\\
        \implies \E[\|x_k\|^2]+\E[\|y_k\|^2]\leq& 2(1+\|A_{22}^{-1}A_{21}\|^2)\max\{\gamma_{max}(Q_{22}),\gamma_{max}(Q_{\Delta}) \}(\E[\|\xh_k\|^2_{Q_{22}}]+\E[\|\yh_k\|^2_{Q_{\Delta}}])\nonumber\\
        \leq 2(1+\|A_{22}^{-1}A_{21}\|^2) &\max\{\gamma_{max}(Q_{22}),\gamma_{max}(Q_{\Delta}) \}\left(3d + 4U_{k_{2}} + \frac{2.5\brc_{19}\alpha^2}{2\xi-1} d + \frac{2.5\brc_3\beta^2}{\alpha (1-\xi)} d\right)\nonumber\\
        =\brc_{20}U_{k_2} + \brc_{21}d,&\label{U_k_bnd}
    \end{align}
    for obvious choice of $\brc_{20}$ and $\brc_{21}$. We use Lemma \ref{lem:crude_Uk_bnd} to upper bound $U_{k_2}$ as
    \begin{align*}
        U_{k_2}\leq& U_0\exp \left(\frac{(\hat{h}_1 +\hh_2)\alpha}{K_0^\xi}+\frac{(\hat{h}_1 +\hh_2)\alpha}{(1-\xi)}\left[(k_2+K_0)^{1-\xi}-K_0^{1-\xi}\right]\right) \\
        &+\hh_2d\alpha\left(\frac{1}{K_0^\xi}+\frac{1}{(\hat{h}_1 +\hh_2)\alpha}\right)\exp\left(\frac{(\hat{h}_1 +\hh_2)\alpha}{(1-\xi)}\left((k_2+K_0)^{1-\xi}-K_0^{1-\xi}\right)\right).
    \end{align*}
    Note that $U_0=\mathcal{O}(d)$. Hence,
    \begin{align*}
        U_{K}\leq d\left(\brc_{20}U_{k_2}/d + \brc_{21}\right) .
    \end{align*}

    For the second part of the Lemma, recall that $\tx_k=\xh_k+L_ky_k$ and $\ty_k=y_k$. Thus, we have
    \begin{align*}
        \E[\|\tx_k\|^2]&\leq 2(\E[\|\xh_k\|^2]+\|L_k\|^2\E[\|y_k\|^2])\\
        &\leq 2(\E[\|\xh_k\|^2]+\kappa_{Q_{22}}^2\E[\|y_k\|^2])\tag{Lemma \ref{lem:L_k_bound}}
    \end{align*}
    Thus, we have
    \begin{align*}
        \E[\|\tx_k\|^2]+\E[\|\ty_k\|^2]&\leq 2\max\{\gamma_{max}(Q_{22}),\gamma_{max}(Q_{\Delta}) \}(\E[\|\xh_k\|^2_{Q_{22}}]+(1+\kappa_{Q_{22}}^2)E[\|\yh_k\|^2_{Q_{\Delta}}]) \\
        &\leq 2(1+\kappa_{Q_{22}}^2)\max\{\gamma_{max}(Q_{22}),\gamma_{max}(Q_{\Delta}) \}U_k\\
        &\leq \frac{(1+\kappa_{Q_{22}}^2)d}{(1+\|A_{22}^{-1}A_{21}\|^2)}\left(\brc_{20}U_{k_2}/d+\brc_{21}\right). 
    \end{align*}
    Define $\brc=\max\left(1, \frac{(1+\kappa_{Q_{22}}^2)}{(1+\|A_{22}^{-1}A_{21}\|^2)}\right)\left(\brc_{20}U_{k_2}/d+\brc_{21}\right)$. Thus we have,
    \begin{align*}
        \E[\|x_k\|^2]+\E[\|y_k\|^2]&\leq  \brc d\\
        \E[\|\tx_k\|^2]+\E[\|\ty_k\|^2]&\leq  \brc d.
    \end{align*}
\end{proof}

\begin{proof}[Proof of Lemma \ref{lem:tel_term_bound}]
The proof for part (2) and (4) follow in the exact manner as part (1) and (3), respectively. Thus, to avoid repetition, we will only present proof for part (1) and (3).
    \begin{enumerate}
        \item[1.] Using Cauchy-Schwarz inequality, we have
        \begin{align*}
            \left\|\E\left[\left(\E_{O_{k-1}}\fh_i(\cdot,x_k,y_k)\right)\tx_k^\top\right]\right\|&\leq \sqrt{\E\left[\left\|\left(\E_{O_{k-1}}\fh_i(\cdot,x_k,y_k)\right)\right\|^2\right]}\sqrt{\E[\|\tx_k\|^2]}\\
            &\leq \frac{2}{1-\rho}\sqrt{\E\left[b_{max} \sqrt{d}+ A_{max}\left\| y_k\right\| + A_{max}\left\| x_k\right\|\right]^2}\sqrt{\E[\|\tx_k\|^2]}\\
            &\leq \frac{2\sqrt{3}}{1-\rho}\sqrt{b^2_{max}d+ A^2_{max}\left(\E[\left\| x_k\right\|^2+\E[\left\| y_k\right\|^2]\right)}\sqrt{\E[\|\tx_k\|^2]}\\
            &\leq \frac{2\sqrt{3d}}{1-\rho}\brc_{f}\sqrt{\E[\|\tx_k\|^2]}.\tag{Lemma \ref{lem:boundedness}}
        \end{align*}
        \item[3.] Recall that $d_k^x= d_k^{xw}+\frac{\beta_k}{\alpha_k}(L_{k+1}+A_{22}^{-1}A_{21})d_k^{xv}$. Using Part 1 of this lemma, we have
        \begin{align*}
            \|d_k^x\|&\leq  \|d_k^{xw}\|+\frac{\beta_k}{\alpha_k}\left\|(L_{k+1}+A_{22}^{-1}A_{21})\right\|\|d_k^{xv}\|\\
            &\leq \frac{2\sqrt{3d}}{1-\rho}\brc_{f}\left(1+\frac{\beta}{\alpha}\varrho_x\right)\sqrt{\E[\|\tx_k\|^2]}.
        \end{align*}
    \end{enumerate}
\end{proof}

\begin{proof}[Proof of Lemma \ref{lem:noise_crude_bound}]
    \begin{enumerate}
    \item Recall that $v_k=b_1(O_{k})-(A_{11}(O_{k})-A_{11})y_k-(A_{12}(O_{k})-A_{12})x_k$. Thus, we have
    \begin{align*}
        \E[\|v_k\|^2]&\leq 3\left(\|b_1(O_{k})\|_2^2+\E[\|A_{11}(O_{k})-A_{11}\|^2\|y_k\|^2]+\E[\|A_{12}(O_{k})-A_{12}\|^2\|x_k\|^2]\right)\\
        &\leq 3\left(b_{\max}^2d+4A_{max}^2\left(\E[\|y_k\|^2]+\E[\|x_k\|^2]\right)\right)\\
        &\leq 3d\left(b_{\max}^2+4A_{max}^2\brc\right).\tag{Lemma \ref{lem:boundedness}}
    \end{align*}
    \item This part follows in the exact manner as the previous one.
    \item Since $u_k=w_k+\frac{\beta_k}{\alpha_k}(L_{k+1}+A_{22}^{-1}A_{21})v_k$, we have
    \begin{align*}
        \E[\|u_k\|^2]&\leq 2\E[\|w_k\|^2]+\frac{2\beta^2}{\alpha^2}\|(L_{k+1}+A_{22}^{-1}A_{21})\|^2\E[\|v_k\|^2]\\
        &\leq 2\E[\|w_k\|^2]+\frac{2\beta^2}{\alpha^2}\varrho_x^2\E[\|v_k\|^2]\tag{Lemma \ref{lem:L_k_bound}}\\
        &\leq 3d\left(2+\frac{2\beta^2}{\alpha^2}\varrho_x^2\right)\left(b_{\max}^2+4A_{max}^2\brc\right).\tag{Part 1 and 2 of this Lemma}
    \end{align*}
\end{enumerate}
\end{proof}

\subsubsection{Induction dependent lemmas}\label{sec:Aux_lem_dep}
\begin{lemma}\label{lem:go_from_xp_to_x} 
Assume at time $k$, Eqs. \ref{eq:lem2_10}, \ref{eq:lem2_20} and \ref{eq:lem2_30} are satisfied with $\max\{\|\tilde{C}'^x_k\|_{Q_{22}},\|\tilde{C}'^{xy}_k\|_{Q_{22}},\|\tilde{C}'^y_k\|_{Q_{\Delta,\beta}}, 1\}= \hbar<\infty$. Then we have the following.
\begin{enumerate}
    \item $ \|\tX_{k}\|\leq \alpha_k\underbar{c}_1d +  \hbar\kappa_{Q_{22}}\zeta_k^x$.
    \item $\|\tY_{k}\|\leq \beta_k \underbar{c}_2d+\hbar \kappa_{Q_{\Delta, \beta}}\zeta_k^y$.
    \item $\E[\|x_k\|^2]\leq \alpha_k\underbar{c}_3d^2+\hbar d \underbar{c}_4\zeta_k^x$.
    \item $\E[\|y_k\|^2] \leq \beta_k \underbar{c}_2d^2+\hbar d\kappa_{Q_{\Delta, \beta}}\zeta_k^y$.
    \item $\E[\|\tx_{k+1}\|^2] \leq \underbar{c}_5d^2 \alpha_k + \underbar{c}_6d\hbar \zeta_k^x$.
    \item $\E[\|\ty_{k+1}\|^2] \leq \underbar{c}_7d^2 \beta_k + \underbar{c}_8d\hbar \zeta_k^y$.
\end{enumerate}
For an exact expression of the constants, refer to the proof of the lemma.
    
\end{lemma}

\begin{lemma}\label{lem:F_conv}
    Consider $x_k,y_k$ as iterations generated by \eqref{eq:two_time_scale}, $O_k$ as Markovian noise in these iterations, and $\tilde{O}_k$ as independent Markovian noise generated according to the stationary distribution of the Markov chain $\{O_i\}_{i\geq0}$. Also, suppose that Eq. \ref{eq:lem2_10}, \ref{eq:lem2_20} and \ref{eq:lem2_30} are satisfied at time $k$ with $\max\{\|\tilde{C}'^x_k\|_{Q_{22}}, \|\tilde{C}'^{xy}_k\|_{Q_{22}},\|\tilde{C}'^y_k\|_{Q_{22}}, 1\}\leq \hbar<\infty$. Then, we have 
    \begin{enumerate}
        \item $\|\E[F^{(i,j)}(O_{k+1},O_k,x_k,y_k)-F^{(i,j)}(\tilde{O}_{k+1},\tilde{O}_k,x_k,y_k)]\|\leq \hat{g}_1d^2\sqrt{\alpha_k}+\hat{g}_2d\hbar\sqrt{\zeta_k^x}$.
        \item $\E[F^{(i,j)}(\tilde{O}_{k+1},\tilde{O}_k,x_k,y_k)]=\sum_{l=1}^\infty \E{[b_i(\tilde{O}_l)b_j(\tilde{O}_0)^\top]}+\hat{R}^{(i,j)}_k$, where $\|\hat{R}^{(i,j)}_k\|\leq \hat{g}_3 d^2\sqrt{\alpha_k}+\hat{g}_2 d\hbar\sqrt{\zeta_k^x}$.
    \end{enumerate}
    For an exact expression for the constants, please refer to the proof of this lemma.
\end{lemma}

\begin{lemma}\label{lem:noise_bound}
     Assume at time $k>k_0$, where $k_0$ is specified in the proof of Lemma \ref{lem:x_xy_y_prime}, Eqs. \ref{eq:lem2_10}, \ref{eq:lem2_20} and \ref{eq:lem2_30} are satisfied with $\max\{\|\tilde{C}'^x_k\|_{Q_{22}},\|\tilde{C}'^{xy}_k\|_{Q_{22}},\|\tilde{C}'^y_k\|_{Q_{\Delta,\beta}}, 1\}= \hbar<\infty$. Then we have the following.
    \begin{enumerate}
        \item \label{item:1} 
        For $i,j\in\{1,2\}$, we have $\E[f_i(O_k,x_k,y_k) f_j(O_k,x_k,y_k)^\top]=\Gamma_{ij}+\cR^{(i,j)}_k$, 
        \text{where}~~ $\|\cR^{(i,j)}_k\| \leq \cc_1d^2 \sqrt{\alpha_k} + \cc_2d\hbar \sqrt{\zeta_k^x}$
        \item \label{item:4} $\E[u_k u_k^\top]=\Gamma_{22}+\cR^u_k,\\~~\text{where}~~ \|\cR^u_k\|\leq \left(1+\frac{\beta}{\alpha}\varrho_x\right)^2\left(\cc_1d^2 \sqrt{\alpha_k} + \cc_2d\hbar \sqrt{\zeta_k^x}\right)+\frac{\beta_k}{\alpha_k}\varrho_x
        \left(\|\Gamma_{21}\|+ \frac{\beta}{\alpha}\varrho_x\|\Gamma_{11}\|\right)$.
    \end{enumerate}
    For exact characterization of the constants please refer to the proof.
\end{lemma}

\begin{lemma}\label{lem:noise_bound2}
    Assume at time $k>k_0$, where $k_0$ is specified in the proof of Lemma \ref{lem:x_xy_y_prime}, Eqs. \ref{eq:lem2_10}, \ref{eq:lem2_20} and \ref{eq:lem2_30} are satisfied with $\max\{\|\tilde{C}'^x_k\|_{Q_{22}},\|\tilde{C}'^{xy}_k\|_{Q_{22}},\|\tilde{C}'^y_k\|_{Q_{\Delta,\beta}}, 1\}=\hbar<\infty$. Then,  we have
    \begin{enumerate}
        \item \label{item:21} $\E[f_1(O_k,x_k,y_k)\ty_k^\top]=\beta_k\sum_{j=1}^\infty \E{[b_1(\tilde{O}_j)b_1(\tilde{O}_0)^\top]} + d_{k}^{yv}-d_{k+1}^{yv}+G^{(1,1)}_k;~~\text{where}~~ \|G^{(1,1)}_k\|\leq g_1d^2\alpha_k\sqrt{\beta_k}+g_2d\hbar\alpha_k\sqrt{\zeta_k^y}$
        \item \label{item:11} $\E[f_1(O_k,x_k,y_k)\tx_k^\top]=\alpha_k\sum_{j=1}^\infty \E{[b_1(\tilde{O}_j)b_2(\tilde{O}_0)^\top]}+d_{k}^{xv} - d_{k+1}^{xv}+G^{(1,2)}_k;~~\text{where}~~ \|G^{(1,2)}_k\|\leq g_3d^2(\alpha_k^{1.5}+\beta_k)+g_4d\hbar\alpha_k\sqrt{\zeta_k^x}$
        \item \label{item:41} $\E[f_2(O_k,x_k,y_k)\ty_k^\top]=\beta_k\sum_{j=1}^\infty \E{[b_2(\tilde{O}_j)b_1(\tilde{O}_0)^\top]}+d_{k}^{yw}-d_{k+1}^{yw}+G^{(2,1)}_k;~~\text{where}~~ \|G^{(2,1)}_k\|\leq g_1d^2\alpha_k\sqrt{\beta_k}+g_2d\hbar\alpha_k\sqrt{\zeta_k^y}$
        \item \label{item:31} $\E[f_2(O_k,x_k,y_k)\tx_k^\top]=\alpha_k\sum_{j=1}^\infty \E{[b_2(\tilde{O}_j)b_2(\tilde{O}_0)^\top]}+d_{k}^{xw}-d_{k+1}^{xw}+G^{(2,2)}_k;~~\text{where}~~ \|G^{(2,2)}_k\|\leq g_3d^2(\alpha_k^{1.5}+\beta_k)+g_4d\hbar\alpha_k\sqrt{\zeta_k^x}$.
    \end{enumerate}
    For exact characterization of the constants please refer to the proof.
\end{lemma}

\subsubsection{Proof of induction dependent lemmas}

\begin{proof}[Proof of Lemma \ref{lem:go_from_xp_to_x}] 
\begin{enumerate}
    \item Since have $\tX_k' = \tX_k+\alpha_k (d_k^x+{d_k^x}^\top)$, we have 
    \begin{align*}
        \tX_k=\tX_k'-\alpha_k (d_k^x+{d_k^x}^\top) = \alpha_k \Sigma^x +R_k,
    \end{align*}
    where $R_k= \tilde{C}'^x_k\zeta_k^x-\alpha_k (d_k^x+{d_k^x}^\top)$. 

    Using Lemma \ref{lem:tel_term_bound}, we get
    \begin{align*}
        \|d_k^x\|\leq & \frac{2\sqrt{3d}}{1-\rho}\brc_{f}\left(1+\frac{\beta}{\alpha}\varrho_x\right)\sqrt{\E[\|\tx_k\|^2]}\\
        \leq & \frac{2\sqrt{3d}}{1-\rho}\brc_{f}\left(1+\frac{\beta}{\alpha}\varrho_x\right)\sqrt{\brc d}\tag{Lemma \ref{lem:boundedness}}
    \end{align*}

    Hence,
    \begin{align*}
        \|\tX_k\|&\leq \alpha_k\|\Sigma^x\| + \|\tilde{C}'^x_k\|\zeta_k^x + 2\alpha_k d \left(\frac{2\sqrt{3\brc}}{1-\rho}\brc_{f}\right)\\
        &\leq \alpha_k \underbar{c}_1 d +  \hbar\kappa_{Q_{22}}\zeta_k^x,
    \end{align*}
    where $\underbar{c}_1 = \sigma^x \tau_{mix} + 2 \left(\frac{2\sqrt{3\brc}}{1-\rho}\left(1+\frac{\beta}{\alpha}\varrho_x\right)\brc_{f}\right)$.

    \item Since $\tY_k=\tY_k' - \beta_k(d_k^{yv}+{d_k^{yv}}^\top)=\beta_k\Sigma^y+\tilde{C}'^y_k\zeta_k^y-\beta_k(d_k^{yv}+{d_k^{yv}}^\top)$, we have
    \begin{align*}
        \|\tY_k\|\leq& \beta_k d\tau_{mix}\sigma^y+\hbar \kappa_{Q_{\Delta, \beta}}\zeta_k^y + 2\beta_k \|d_k^{yv}\|\\ 
        \leq& \beta_k d\tau_{mix}\sigma^y+\hbar \kappa_{Q_{\Delta, \beta}}\zeta_k^y + \beta_k  \frac{4\sqrt{3d}}{1-\rho}\brc_{f}\sqrt{\E[\|\ty_k\|^2]}\tag{Lemma \ref{lem:tel_term_bound}}\\
        \leq& \beta_k d\tau_{mix}\sigma^y+\hbar \kappa_{Q_{\Delta, \beta}}\zeta_k^y + \beta_k  \frac{4\sqrt{3d}}{1-\rho}\brc_{f}\sqrt{\brc d}\tag{Lemma \ref{lem:boundedness}}\\
        =& \beta_k \underbar{c}_2d+ \hbar \kappa_{Q_{\Delta, \beta}}\zeta_k^y
    \end{align*}
    where $\underbar{c}_2=\sigma^y\tau_{mix} + \frac{4\sqrt{3\brc} }{1-\rho} \brc_{f}$. 
    
    \item 
    \begin{align*}
        \E[\|x_k\|^2] =& \E[\|\tx_k-(L_k+A_{22}^{-1}A_{21}) \ty_k\|^2]\\
        \leq &2\E[\|\tx_k\|^2 +\|L_k+A_{22}^{-1}A_{21}\|^2 \| \ty_k\|^2]\\
        \leq &2d\|\tX_k\| +\|L_k+A_{22}^{-1}A_{21}\|^2 d \| \tY_k\|\\
        \leq &2[d\|\tX_k\| +\|L_k+A_{22}^{-1}A_{21}\|^2 d \| \tY_k\|]\\
        \leq & 2d (\alpha_k \underbar{c}_1 d +  \hbar\kappa_{Q_{22}}\zeta_k^x) + 2d \varrho_x (\beta_k \underbar{c}_2d+ \hbar \kappa_{Q_{\Delta, \beta}}\zeta_k^y)\tag{Lemma \ref{lem:L_k_bound}}\\
        =&\alpha_k d^2\underbar{c}_3+ d\underbar{c}_4 \hbar\zeta_k^x, \tag{Using $\zeta_k^y\leq \zeta_k^x$}
    \end{align*}
    where $\underbar{c}_3=2\underbar{c}_1 + \frac{2\beta}{\alpha}\underbar{c}_2 \varrho_x $ and $\underbar{c}_4=2\kappa_{Q_{22}} + 2 \varrho_x \kappa_{Q_{\Delta, \beta}}$.

    \item \begin{align*}
            \E[\|y_k\|^2] \leq d\|Y_k\| = d\|\tY_k\| = \beta_k \underbar{c}_2d^2+ \hbar d \kappa_{Q_{\Delta, \beta}}\zeta_k^y
        \end{align*}

    \item We have
    \begin{align*}
    \E[\|\tx_{k+1}\|^2]&=\E [\|(I-\alpha_k B_{22}^k)\tx_k+\alpha_k u_k\|^2] \\
    &\leq 2\E [\|(I-\alpha_k B_{22}^k)\|^2\|\tx_k\|^2+\alpha_k^2 \|u_k\|^2]
    \end{align*}
    For the first term, recall that $B^k_{22}=\frac{\beta_k}{\alpha_k}(L_{k+1}+A_{22}^{-1}A_{21})A_{12}+A_{22}$. Thus, $\|B^k_{22}\|\leq \frac{\beta}{\alpha}\varrho_x\|A_{12}\|+\|A_{22}\|$. In addition, from Lemma \ref{lem:go_from_xp_to_x}, we have $\E[\|x_k\|^2]\leq\alpha_k\underbar{c}_3d^2+\hbar d \underbar{c}_4\zeta_k^x$. Furthermore, by lemma \ref{lem:noise_crude_bound} we have $\E[\|u_k\|^2] \leq 6d\left(1+\frac{\beta}{\alpha}\varrho_x^2\right)\left(b_{\max}^2+4A_{max}^2\brc\right)$. Combining together the previous bounds, we have
    \begin{align*}
        \E[\|\tx_{k+1}\|^2] \leq \underbar{c}_5 d^2 \alpha_k + \underbar{c}_6 d\hbar \zeta_k^x,
    \end{align*}
    where $\underbar{c}_5 = 2\underbar{c}_3\left(1+\alpha \left(\frac{\beta}{\alpha}\varrho_x\|A_{12}\|+\|A_{22}\|\right)^2\right)+12\left(1+\frac{\beta}{\alpha}\varrho_x^2\right)\left(b_{\max}^2+4A_{max}^2\brc\right)$ and \\$\underbar{c}_6=2\underbar{c}_4\left(1+\alpha \left(\frac{\beta}{\alpha}\varrho_x\|A_{12}\|+\|A_{22}\|\right)^2\right)$.

    \item From Eq. \eqref{eq:y_t_update_2}, we have
        \begin{align*}
            \E[\|\ty_{k+1}\|^2]&=\E[\|(I-\beta_kB^k_{11}\ty_k)+\beta_kA_{12}\tx_k+\beta_kv_k\|^2]\\
            &\leq 3\E[\|I-\beta_kB^k_{11}\|^2\|\ty_k\|^2+\beta_k^2\|A_{12}\|^2\|\tx_k\|^2+\beta_k^2\|v_k\|^2]
        \end{align*}
        Recall that $B^k_{11}=\Delta-A_{12}L_k$. Thus, $\|B^k_{11}\|\leq\|\Delta\|+\|A_{12}\|\kappa_{Q_{22}}=\varrho_y$. Thus, we have
        \begin{align*}
            \E[\|\ty_{k+1} \|^2] \leq & 3\E[2(1+ \beta^2\varrho_y^2)\|\ty_k\|^2+\beta_k^2\|A_{12}\|^2\|\tx_k\|^2+\beta_k^2\|v_k\|^2]\\
            \leq & 3\E[2(1+ \beta^2\varrho_y^2)\|\ty_k\|^2+\beta_k^2\|A_{12}\|^2\|\tx_k\|^2+3\beta_k^2d\left(b_{\max}^2+4A_{max}^2\brc\right)]\\
            \leq & 3\E[2(1+ \beta^2\varrho_y^2)(\beta_k \underbar{c}_2d^2+ \hbar d \kappa_{Q_{\Delta, \beta}}\zeta_k^y) +\beta_k^2 \|A_{12}\|^2(\alpha_k\underbar{c}_3d^2+ \hbar d \underbar{c}_4\zeta_k^x)\\
            &+ 3\beta_k^2d\left(b_{\max}^2+ 4A_{max}^2\brc\right)]\\
            = & \underbar{c}_7 d^2\beta_k + \hbar d\underbar{c}_8 \zeta_k^y
        \end{align*}
        where $\underbar{c}_7 = 6(1+ \beta^2\varrho_y^2 ) \underbar{c}_2 + \beta \|A_{12}\|^2\alpha \underbar{c}_3 + 3\beta \left(b_{\max}^2+ 4A_{max}^2\brc\right)$ and $\underbar{c}_8 = 6(1+ \beta^2\varrho_y^2 ) \kappa_{Q_{\Delta, \beta}} + \beta^2 \|A_{12}\|^2 \underbar{c}_4$.
\end{enumerate}

\end{proof}

\begin{proof}[Proof of Lemma \ref{lem:F_conv}]
\begin{enumerate}
    \item Recall that $F^{(i,j)}(O_{k+1},O_k,x_k,y_k)=E\left[ \fh_i(O_{k+1},x_k,y_k)f_j(O_k,x_k,y_k)^\top\right]$. Thus, we have
    \begin{align*}
        \|\E[F^{(i,j)}&(O_{k+1},O_k,x_k,y_k)-F^{(i,j)}(\tilde{O}_{k+1},\tilde{O}_k,x_k,y_k)]\|\\ =& \bigg\|\E\left[\left( \fh_i(O_{k+1},x_k,y_k)\right)(f_j(O_k,x_k,y_k))^\top - \left( \fh_i(\tilde{O}_{k+1},x_k,y_k)\right)(f_j(\tilde{O}_k,x_k,y_k))^\top\right]\bigg\|\\
        =&\bigg\|\E\bigg[\left(C_i(O_{k+1})-C_{i1}(O_{k+1})y_k - C_{i2}(O_{k+1})x_k\right) \left(b_j(O_k)-(A_{j1}(O_k)-A_{j1})y_k-(A_{j2}(O_k)-A_{j2})x_k\right)^\top\\
        &-(C_i(\tO_{k+1})-C_{i1}(\tO_{k+1})y_k - C_{i2}(\tO_{k+1})x_k) \left(b_j(\tO_k)-(A_{j1}(\tO_k)-A_{j1})y_k-(A_{j2}(\tO_k)-A_{j2})x_k\right)^\top\bigg]\bigg\|\\
        \leq&\|\E[C_i(O_{k+1})b_j(O_k)^\top-C_i(\tO_{k+1})b_j(\tO_k)^\top]\|+\|R_k\|,
    \end{align*}
    where $R_k$ includes all the remaining terms. Denote $\Lambda_k=(O_k,O_{k+1})$ and $\tilde{\Lambda}_k=(\tO_k,\tO_{k+1})$. Clearly, $\Lambda_k$ is a Markov chain, and $\tilde{\Lambda}_k$ is another independent Markov chain following the stationary distribution of $\Lambda_k$. By definition of the function $C_i$ and the mixing property of the Markov chain, we have 
    \begin{align*}    \max_{o,o'}\|C_i(o')b_j(o)^\top\|\leq &\max_{o} \|C_i(o)\|\max_{o} \|b_j(o)\| \\
    \leq & \frac{2b_{max}}{1-\rho}\sqrt{d} . b_{max} \sqrt{d}\tag{Lemma \ref{lem:mix_time_sum}}\\
    = & \frac{2b_{max}^2}{(1-\rho)} d.
    \end{align*}
    Hence, by geometric mixing of the Markov chain, $\|\E[C_i(O_{k+1})b_j(O_k)^\top-C_i(\tO_{k+1})b_j(\tO_k)^\top]\|$ goes to zero geometrically fast. Hence,
    \begin{align*}
        \|\E[C_i(O_{k+1})b_j(O_k)^\top-C_i(\tO_{k+1})b_j(\tO_k)^\top]\|& \leq \frac{4b_{max}^2}{(1-\rho)} d\rho^k\\
        &\leq \frac{4b_{max}^2}{(1-\rho)} \left(\frac{\xi/2}{e\ln(1/\rho)}+K_0\right)^{\xi/2} d\sqrt{\alpha_k}. \tag{Lemma \ref{lem:lambert}}
    \end{align*}

    For $R_k$, we have
    \begin{align*}
        \|R_k\|\leq& \frac{8b_{max}A_{max}}{1-\rho}\sqrt{d}\E[\|x_k\|+\|y_k\|] + \frac{4A_{max}^2}{1-\rho} \E[2\|x_k\|^2+2\|y_k\|^2+4\|x_k\|\|y_k\|]\tag{Cauchy-Schwarz inequality}\\
        \leq& \frac{8b_{max}A_{max}}{1-\rho}\sqrt{d}\E[\|x_k\|+\|y_k\|] + \frac{16A_{max}^2}{1-\rho} \E[\|x_k\|^2+\|y_k\|^2]\tag{AM-GM inequality}\\
        \leq& \frac{8b_{max}A_{max}}{1-\rho}\sqrt{d}\left(\sqrt{\E[\|x_k\|^2]} + \sqrt{\E[\|y_k\|^2]}\right) + \frac{16A_{max}^2}{1-\rho} \E\left[\|x_k\|^2+\|y_k\|^2\right]\tag{Jensen's inequality}\\
        \leq& \frac{8b_{max}A_{max}}{1-\rho}\sqrt{d}\left(\sqrt{\alpha_k \underbar{c}_3 d^2 + \hbar d \underbar{c}_4\zeta_k^x} + \sqrt{\beta_k \underbar{c}_2 d^2 + \hbar d \kappa_{Q_{\Delta,\beta}}\zeta_k^y}\right) \\
        &+ \frac{16A_{max}^2}{1-\rho} \left(\alpha_k \underbar{c}_3 d^2 + \hbar d \underbar{c}_4\zeta_k^x + \beta_k \underbar{c}_2 d^2 + \hbar d \kappa_{Q_{\Delta,\beta}}\zeta_k^y\right)\tag{Lemma \ref{lem:go_from_xp_to_x}}
    \end{align*}

    Combining both the bounds together, we have
    \begin{align*}
        \|\E[F^{(i,j)}&(O_{k+1},O_k,x_k,y_k)-F^{(i,j)}(\tilde{O}_{k+1},\tilde{O}_k,x_k,y_k)]\| \leq \hat{g}_1 d^2\sqrt{\alpha_k}+\hat{g}_2 d\hbar\sqrt{\zeta_k^x},\tag{$\zeta_k^x\leq \zeta_k^y$}
    \end{align*}
    where $\hat{g}_1 = \frac{8b_{max}A_{max}}{1-\rho} (\sqrt{\underbar{c}_3} + \sqrt{\beta\underbar{c}_2/\alpha}) + \frac{16A_{max}^2}{1-\rho} (\underbar{c}_3\sqrt{\alpha} + \beta\underbar{c}_2/\sqrt{\alpha}) + \frac{4b_{max}^2}{(1-\rho)} \left(\frac{\xi/2}{e\ln(1/\rho)}+K_0\right)^{\xi/2}$ and $\hat{g}_2=\frac{8b_{max}A_{max}}{1-\rho}\left(\sqrt{\underbar{c}_4}+\sqrt{\kappa_{Q_{\Delta,\beta}}}\right)+\frac{16A_{max}^2}{1-\rho}\left(\underbar{c}_4+\kappa_{Q_{\Delta,\beta}}\right)$.

    \item 
    \begin{align*}
        \E[F^{(i,j)}(\tilde{O}_{k+1},\tilde{O}_k,x_k,y_k)]&=\E{\left[\left(\sum_{l=k+1}^\infty \E{[b_i(\tilde{O}_l)|\tilde{O}_{k+1}]}-C_{i2}(\tO_{k+1})x_k-C_{i1}(\tO_{k+1})y_k\right)f_j(\tilde{O}_k,x_k,y_k)^\top\right]}\tag{Lemma \ref{lem:possion_sol}}\\
        &=\E\Bigg[\left(\sum_{l=k+1}^\infty \E{[b_i(\tilde{O}_l)|\tilde{O}_{k+1}]}-C_{i2}(\tO_{k+1})x_k-C_{i1}(\tO_{k+1})y_k\right)\\
        &~~~~~~~~~~~~~~~~~~~~~~~~~~~~~~~~~~~~~~~~~ \left(b_j(\tilde{O}_k)-(A_{j2}(\tilde{O}_k)-A_{j2})x_k-(A_{j1}(\tilde{O}_k)-A_{j1})y_k\right)^\top\Bigg]\\   
        &=\E{\left[\sum_{l=k+1}^\infty \E{[b_i(\tilde{O}_l)b_j(\tilde{O}_k)^\top|\tilde{O}_{k+1}]}\right]} + \hat{R}^{(i,j)}_k\\
        &=\sum_{l=k+1}^\infty \E[\E{[b_i(\tilde{O}_l)b_j(\tilde{O}_k)^\top|\tO_{k+1}]}]+\hat{R}^{(i,j)}_k\tag{Fubini-Tonelli theorem}\\
        &=\sum_{l=k+1}^\infty \E{[b_i(\tilde{O}_l)b_j(\tilde{O}_k)^\top]}+\hat{R}^{(i,j)}_k\tag{Tower Property}\\
        &=\sum_{l=1}^\infty \E{[b_i(\tilde{O}_l)b_j(\tilde{O}_0)^\top]}+\hat{R}^{(i,j)}_k,\tag{Stationarity of $\tilde{O}_k$}
    \end{align*}
\end{enumerate}
where $\hat{R}^{(i,j)}_k$ represents the remainder terms. Using the exact arguments as in the previous part, we have
\begin{align*}
    \|\hat{R}^{(i,j)}_k\|\leq& \frac{8b_{max}A_{max}}{1-\rho}\sqrt{d}\E[\|x_k\|+\|y_k\|] + \frac{4A_{max}^2}{1-\rho} \E[2\|x_k\|^2+2\|y_k\|^2+4\|x_k\|\|y_k\|]\tag{Cauchy-Schwarz inequality}\\
        \leq& \frac{8b_{max}A_{max}}{1-\rho}\sqrt{d}\E[\|x_k\|+\|y_k\|] + \frac{16A_{max}^2}{1-\rho} \E[\|x_k\|^2+\|y_k\|^2]\tag{AM-GM inequality}\\
        \leq& \frac{8b_{max}A_{max}}{1-\rho}\sqrt{d}\left(\sqrt{\E[\|x_k\|^2]} + \sqrt{\E[\|y_k\|^2]}\right) + \frac{16A_{max}^2}{1-\rho} \E\left[\|x_k\|^2+\|y_k\|^2\right]\tag{Jensen's inequality}\\
        \leq& \frac{8b_{max}A_{max}}{1-\rho}\sqrt{d}\left(\sqrt{\alpha_k \underbar{c}_3 d^2 + \hbar d \underbar{c}_4\zeta_k^x} + \sqrt{\beta_k \underbar{c}_2 d^2 + \hbar d \kappa_{Q_{\Delta,\beta}}\zeta_k^y}\right) \\
        &+ \frac{16A_{max}^2}{1-\rho} \left(\alpha_k \underbar{c}_3 d^2 + \hbar d \underbar{c}_4\zeta_k^x + \beta_k \underbar{c}_2 d^2 + \hbar d \kappa_{Q_{\Delta,\beta}}\zeta_k^y\right)\tag{Lemma \ref{lem:go_from_xp_to_x}}\\
        &\leq \hat{g}_3 d^2\sqrt{\alpha_k}+\hat{g}_2 d\hbar\sqrt{\zeta_k^x},\tag{$\zeta_k^x\leq \zeta_k^y$}
\end{align*}
where $\hat{g}_3=\frac{8b_{max}A_{max}}{1-\rho} (\sqrt{\underbar{c}_3} + \sqrt{\beta\underbar{c}_2/\alpha}) + \frac{16A_{max}^2}{1-\rho} (\underbar{c}_3\sqrt{\alpha} + \beta\underbar{c}_2/\sqrt{\alpha})$.
    
\end{proof}

\begin{proof}[Proof of Lemma \ref{lem:noise_bound}]
Assume that $\psi_k^i=b_i(O_k)-(A_{i1}(O_k)-A_{i1})y_k-(A_{i2}(O_k)-A_{i2})x_k$ for $i\in\{1,2\}$. Note that $\psi_k^{(1)}=v_k$ and $\psi_k^{(2)}=w_k$. For arbitrary $i,j\in\{1,2\}$ We have:
\begin{align*}
    \psi_k^{(i)}{\psi_k^{(j)}}^\top=& b_i(O_k) b_j(O_k)^\top-(A_{i1}(O_k)-A_{i1}) y_kb_j(O_k)^\top-(A_{i2}(O_k)-A_{i2}) x_kb_j(O_k)^\top\\&-b_i(O_k) y_k^\top(A_{j1}(O_k)-A_{j1})^\top+(A_{i1}(O_k)-A_{i1})y_ky_k^\top(A_{j1}(O_k)-A_{j1})^\top\\
    &+(A_{i2}(O_k)-A_{i2})x_ky_k^\top(A_{j1}(O_k)-A_{j1})^\top-b_i(O_k)x_k^\top(A_{j2}(O_k)-A_{j2})^\top\\
    &+(A_{i1}(O_k)-A_{i1})y_kx_k^\top(A_{j2}(O_k)-A_{j2})^\top+(A_{i2}(O_k)-A_{i2})x_kx_k^\top(A_{j2}(O_k)-A_{j2})^\top.
\end{align*}
We will analyze each term separately and use \cite[Lemma 23]{kaledin2020finite} extensively without stating to decompose the expectation of the outer product of two random vectors. 
\begin{itemize}
    \item Let $\Tilde{O}_k$ be a Markov chain with starting distribution as stationary distribution. Then:
    \begin{align*}
        \|\E[b_i(O_k)b_j(O_k)^\top]\|&= \E[b_i(O_k)b_j(O_k)^\top]-\E[b_i(\Tilde{O}_k)b_j(\Tilde{O}_k)^\top]+\E[b_i(\Tilde{O}_k)b_j(\Tilde{O}_k)^\top]\\
        &= \Gamma_{ij}+\E[b_i(O_k)b_j(O_k)^\top]-\E[b_i(\Tilde{O}_k)b_j(\Tilde{O}_k)^\top].
    \end{align*}
    We have
    \begin{align*}
        \|\E[b_i(O_k) b_j(O_k)^\top] - \E[b_i(\Tilde{O}_k) b_j(\Tilde{O}_k)^\top]\| &\leq \max_o\|b_i(o) b_j(o)^\top\| \max_{o'} d_{TV}(P^{k}(\cdot|o') || \mu(\cdot))\\
        &\leq \max_o\|b_i(o) b_j(o)^\top\| \rho^{k}\\
        &\leq b_{max}^2 d\rho^k,
    \end{align*}
    where the second inequality is due to the geometric mixing of the Markov chain stated in Remark \ref{rem:pois_eq_geo_mix}.
    
    Using Lemma \ref{lem:lambert}, $\rho^{k}\leq \frac{1}{\sqrt{\alpha}}\left(\frac{\xi}{2e \log(1/\rho)}+K_0\right)^{\xi/2}\sqrt{\alpha_k}$. Hence, we have $\|\E[b_i(O_k)b_j(O_k)^\top]-\E[b_i(\Tilde{O}_k)b_j(\Tilde{O}_k)^\top]\|\leq \frac{b_{max}^2d}{\sqrt{\alpha}}\left(\frac{\xi}{2e \log(1/\rho)}+K_0\right)^{\xi/2}\sqrt{\alpha_k}$ for all $k>0$.

    \item For the $5th$ term, we have the following:
    \begin{align*}
        \|\E[(A_{i1}(O_k)-A_{i1})y_ky_k^\top(A_{j1}(O_k)-A_{j1})^\top]\|&\leq 4A^2_{max}\E[\|y_ky_k^\top\|]\\
        &= 4A^2_{max}\E[\|y_k\|^2]\\
        &\leq 4A^2_{max}\left(\beta_k\underbar{c}_2d^2+\hbar d\kappa_{Q_{\Delta,\beta}}\zeta_k^y\right)\tag{Lemma \ref{lem:go_from_xp_to_x}}
    \end{align*}

    \item For the $9th$ term, we shall do the following:
    \begin{align*}
        \|\E[(A_{i2}(O_k)-A_{i2})x_kx_k^\top(A_{j2}(O_k)-A_{j2})^\top]\|&\leq 4A^2_{max}\E[\|x_kx_k^\top\|]\\
        &= 4A^2_{max}\E[\|x_k\|^2]\\
        &\leq  4A^2_{max}\left(\alpha_k\underbar{c}_3d^2 +\hbar d\underbar{c}_4\zeta_k^x\right).\tag{Lemma \ref{lem:go_from_xp_to_x}}
    \end{align*}
    
    \item For the $2nd$ and $4th$ terms:
    \begin{align*}
        \|\E[(A_{i1}(O_k)-A_{i1})y_kb_j(O_k)^\top]\|&\leq \sqrt{\E[\|b_j(O_k)\|^2]}\sqrt{\E[\|(A_{i1}(O_k)-A_{i1})y_k\|^2]}\\
        &\leq 2A_{max}b_{max}\sqrt{\E[\|y_k\|^2]}\\
        &\leq 2A_{max}b_{max}\left(\sqrt{\beta_k\underbar{c}_2d^2+\hbar d\kappa_{Q_{\Delta,\beta}}\zeta_k^y} \right)\tag{Lemma \ref{lem:go_from_xp_to_x}}\\
        &\leq 2A_{max}b_{max}\left(d\sqrt{\beta_k\underbar{c}_2} + \hbar\sqrt{ d\kappa_{Q_{\Delta,\beta}}\zeta_k^y} \right)
    \end{align*}
    where the last inequality is by $\hbar \geq 1 $.
    
     Similarly for the 4$th$ term.
    \item For the $3rd$ and $7th$ terms:
    \begin{align*}
        \|\E[b_i(O_k)&x_k^\top(A_{j2}(O_k)-A_{j2})^\top]\|\\
        &\leq \sqrt{\E[\|b_i(O_k)\|^2]}\sqrt{\E[\|(A_{i2}(O_k)-A_{i2})x_k\|^2]}\\
        &\leq 2A_{max}b_{max}\sqrt{\E[\|x_k\|^2]}\\
        &\leq 2A_{max} b_{max}\left( d\sqrt{\alpha_k\underbar{c}_3} + \hbar \sqrt{ d\underbar{c}_4\zeta_k^x}\right)\tag{Lemma \ref{lem:go_from_xp_to_x}}
    \end{align*}
    
    Similarly for the $7th$ term.
    \item For the $6th$ and $8th$ terms:
    \begin{align*}
        &\|\E[(A_{i1}(O_k)-A_{i1})y_kx_k^\top(A_{j2}(O_k)-A_{i2})^\top]\|
        \\ &\leq \E[\|(A_{i1}(O_k)-A_{i1})y_k\|^2]+\E[\|(A_{j2}(O_k)-A_{j2})x_k\|^2]\tag{Young's Inequality}\\
        &\leq 4A^2_{max}\left(\beta_k\underbar{c}_2d^2+\hbar d\kappa_{Q_{\Delta,\beta}}\zeta_k^y+\alpha_k\underbar{c}_3d^2 +\hbar d\underbar{c}_4\zeta_k^x\right)\tag{Lemma \ref{lem:go_from_xp_to_x}}\\
        &\leq 4A^2_{max}d^2\alpha_k\left(\underbar{c}_2\frac{\beta}{\alpha}+\underbar{c}_3\right) +4\hbar dA_{max}^2\left(\kappa_{Q_{\Delta,\beta}}+ \underbar{c}_4\right)\zeta_k^x \tag{$\zeta_k^y\leq \zeta_k^x$}
    \end{align*}
    \end{itemize}
    
    Hence, we have
    \begin{align*}
        \E \left[\psi_k^{(i)}{\psi_k^{(j)}}^\top\right] = \Gamma_{ij}+\cR^{(i,j)}_k
    \end{align*}
    where $\|\cR^{(i,j)}_k\| \leq \cc_1d^2 \sqrt{\alpha_k} + \cc_2d\hbar \sqrt{\zeta_k^x}$. Here
    \begin{align*}
        \cc_1&=\frac{b_{max}^2}{\sqrt{\alpha}}\left(\frac{\xi}{2e \log(1/\rho)}+K_0\right)^{\xi/2}+12A_{max}^2\sqrt{\alpha}\left(\underbar{c}_2\frac{\beta}{\alpha}+\underbar{c}_3\right)+4A_{max}b_{max}\left(\sqrt{\frac{\beta}{\alpha}\underbar{c}_2}+\sqrt{\underbar{c}_3}\right)\\
        \cc_2&=12A_{max}^2\left(\kappa_{Q_{\Delta, \beta}}+\underbar{c}_4\right)+4A_{max}b_{max}\left(\sqrt{\kappa_{Q_{\Delta, \beta}}}+\sqrt{\underbar{c}_4}\right).
    \end{align*}
    
    This proves the part \ref{item:1} of the Lemma. 

   For the last part, $\E[u_ku_k^\top]$,  we have:
    Given that $u_k=w_k+\frac{\beta_k}{\alpha_k}(L_{k+1}+A_{22}^{-1}A_{21})v_k$:
\begin{align*}
    u_ku_k^\top&=w_kw_k^\top+\frac{\beta_k}{\alpha_k}w_kv_k^\top(L_{k+1}+A_{22}^{-1}A_{21})^\top+\frac{\beta_k}{\alpha_k}(L_{k+1}+A_{22}^{-1}A_{21})v_kw_k^\top\\
    &+\left(\frac{\beta_k}{\alpha_k}\right)^2(L_{k+1}+A_{22}^{-1}A_{21})v_kv_k^\top(L_{k+1}+A_{22}^{-1}A_{21})^\top
\end{align*}

We will again analyse each term separately. 
    \begin{itemize}
        \item $\E[w_kw_k^\top]=\Gamma_{22}+\cR^{(2,2)}_k;~~\text{where}~~ \|\cR^{(2,2)}_k\|\leq \cc_1d^2 \sqrt{\alpha_k} + \cc_2d\hbar \sqrt{\zeta_k^x}$.
        \item $\frac{\beta_k}{\alpha_k}\|\E[w_kv_k^\top]\|\|(L_{k+1}+A_{22}^{-1}A_{21})^\top\|\leq \frac{\beta_k}{\alpha_k} \varrho_x\left(\|\Gamma_{21}\|+\cc_1d^2 \sqrt{\alpha_k} + \cc_2d\hbar \sqrt{\zeta_k^x}\right)$
        \item $\left(\frac{\beta_k}{\alpha_k}\right)^2\|(L_{k+1}+A_{22}^{-1}A_{21})\|\|\E[v_kv_k^\top]\|(L_{k+1}+A_{22}^{-1}A_{21})^\top\|\leq \left(\frac{\beta_k}{\alpha_k}\right)^2\varrho_x^2\Big(\|\Gamma_{11}\|+\cc_1d^2 \sqrt{\alpha_k} + \cc_2d\hbar \sqrt{\zeta_k^x}\Big)$  
    \end{itemize}
Hence, 
$$\E[u_ku_k^\top]=\Gamma_{22}+\cR_k^{u},$$ 

where $\|\cR_k^{u}\|\leq \left(1+\frac{\beta}{\alpha}\varrho_x\right)^2\left(\cc_1d^2 \sqrt{\alpha_k} + \cc_2d\hbar \sqrt{\zeta_k^x}\right)+\frac{\beta_k}{\alpha_k}\varrho_x
\left(\|\Gamma_{21}\|+\frac{\beta}{\alpha}\|\Gamma_{11}\|\varrho_x\right)$.
\end{proof}

\begin{proof}[Proof of Lemma \ref{lem:noise_bound2}]
The results in part (3) and (4) of this Lemma follow in exactly same manner as part (1) and (2), respectively. Hence, we only present proof for the first two parts to avoid repetition.

\begin{enumerate}
    \item By definition, we had $v_k=f_1(O_k,x_k,y_k)$. By Remark \ref{rem:pois_eq_geo_mix}, we have a unique function $\fh_1(o,x_k,y_k)$ such that
\begin{align*}
    \fh_1(o,x_k,y_k)=f_1(o,x_k,y_k)+\sum_{o'\in S}P(o'|o)\fh_1(o',x_k,y_k)
\end{align*}
where $P(o'|o)$ is the transition probability corresponding to the Markov chain $\{O_k\}_{k\geq 0}$. 
    Hence,
\begin{align}
    \E[v_k\ty_k^\top]=&\E[f_1(O_k,x_k,y_k)\ty_k^\top]\label{eq:vy}\\
    =&\E\left[\left(\fh_1(O_k,x_k,y_k)-\sum_{o'\in S}P(o'|O_k)\fh_1(o',x_k,y_k)\right)\ty_k^\top\right]\nonumber\\
    =&\E\left[\left(\fh_1(O_k,x_k,y_k)-\E_{O_k}\fh_1(\cdot,x_k,y_k)\right)\ty_k^\top\right]\nonumber\\
    =&\E\left[\left(\fh_1(O_k,x_k,y_k)-\E_{O_{k-1}}\fh_1(\cdot,x_k,y_k)+\E_{O_{k-1}}\fh_1(\cdot,x_k,y_k)-\E_{O_k}\fh_1(\cdot,x_k,y_k)\right)\ty_k^\top\right]\nonumber\\
    =&\E\left[\left(\E_{O_{k-1}}\fh_1(\cdot,x_k,y_k)-\E_{O_{k}}\fh_1(\cdot,x_k,y_k)\right)\ty_k^\top\right]\tag{Tower property}\nonumber\\
    =&\E\bigg[\left(\E_{O_{k-1}}\fh_1(\cdot,x_k,y_k)\right)\ty_k^\top-\left(\E_{O_{k}}\fh_1(\cdot,x_{k+1},y_{k+1})\right)\ty_{k+1}^\top\nonumber\\
    &+\left(\E_{O_{k}}\fh_1(\cdot,x_{k+1},y_{k+1})-\E_{O_{k}}\fh_1(\cdot,x_k,y_k)\right)\ty_{k+1}^\top+\left(\E_{O_{k}}\fh_1(\cdot,x_k,y_k)\right)(\ty_{k+1}^\top-\ty_k^\top)\bigg]\nonumber\\
    =&d_{k}^{yv}-d_{k+1}^{yv}+\E\bigg[\underbrace{\left(\E_{O_{k}}\fh_1(\cdot,x_{k+1},y_{k+1})-\E_{O_{k}}\fh_1(\cdot,x_k,y_k)\right)\ty_{k+1}^\top}_{T_{1}}+\underbrace{\left(\E_{O_{k}}\fh_1(\cdot,x_k,y_k)\right)(\ty_{k+1}^\top-\ty_k^\top)}_{T_{2}}\bigg]\nonumber
\end{align}

For $T_{1}$, we have
\begin{align*}
    \E[\|T_{1}\|]\leq & \check{h}_2 \E[(\|x_{k+1}-x_k\|+\|y_{k+1}-y_k\|).\|\ty_{k+1}\|]\tag{Lemma \ref{lem:f_hat_bound}}\\
    \leq & \check{h}_2\left(1+\frac{\beta}{\alpha}\right)  \alpha_k \E[(A_{max}\|x_{k}\|+A_{max}\|y_k\|+b_{max}\sqrt{d}).\|\ty_{k+1}\|]\tag{Eq. \eqref{eq:two_time_scale}}\\
    \leq & \check{h}_2\left(1+\frac{\beta}{\alpha}\right)  \alpha_k\sqrt{\E[( A_{max}\|x_{k}\| +A_{max}\|y_k\|+ b_{max}\sqrt{d})^2]} \sqrt{\E[\|\ty_{k+1}\|^2]}\tag{Cauchy-Schwarz}\\
    \leq & \check{h}_2\left(1+\frac{\beta}{\alpha}\right)\sqrt{3}\alpha_k\sqrt{\E[A_{max}^2(\|x_{k}\|^2+\|y_k\|^2)+b_{max}^2d]}\left(\sqrt{\underbar{c}_7}d \sqrt{\beta_k} + \sqrt{\underbar{c}_8d\hbar \zeta_k^y}\right)\tag{Lemma \ref{lem:go_from_xp_to_x}}\\
    \leq & \check{h}_2\left(1+\frac{\beta}{\alpha}\right)\sqrt{3}\alpha_k\sqrt{b_{max}^2+A_{max}^2\brc}\left(\sqrt{\underbar{c}_7}d^{1.5} \sqrt{\beta_k} + d\hbar\sqrt{\underbar{c}_8 \zeta_k^y}\right),\tag{Lemma \ref{lem:boundedness}}
\end{align*}

In addition, using the Eq. \eqref{eq:x_t_update_2}, we have
\begin{align*}
\E[T_{2}]=&\E\left[\left(\E_{O_{k}}\fh_1(\cdot,x_k,y_k)\right)(\ty_{k+1}^\top-\ty_k^\top)\right]\\
=&\E\left[\left(\E_{O_{k}}\fh_1(\cdot,x_k,y_k)\right)\left(-\beta_kB_{11}^k\ty_k-\beta_kA_{12}\tx_k+\beta_kv_k\right)^\top\right]\\
=&\beta_k\underbrace{\E\left[\left(\E_{O_{k}}\fh_1(\cdot,x_k,y_k)\right)v_k^\top\right]}_{T_{21}} \\
&-\beta_k\underbrace{ \E\left[\left(\E_{O_{k}}\fh_1(\cdot,x_k,y_k)\right)\left(B_{11}^k\ty_k\right)^\top\right]}_{T_{22}} -\beta_k\underbrace{\E\left[\left(\E_{O_{k}}\fh_1(\cdot,x_k,y_k)\right)\left(A_{12}\tx_k\right)^\top\right]}_{T_{23}}.
\end{align*}

\begin{itemize}
    \item For $T_{21}$, denote $\Tilde{O}$ as the random variable with distribution coming from the stationary distribution of the Markov chain $\{O_k\}_{k\geq 0}$. We have 
    \begin{align*}
        \E&\left[\left(\E_{O_{k}} \fh_1(\cdot,x_k,y_k)\right)(f_1(O_k,x_k,y_k))^\top\right]
        =\E\left[\left( \fh_1(O_{k+1},x_k,y_k)\right)(f_1(O_k,x_k,y_k))^\top\right]\tag{Tower property}\\
        =& \E[\left( \fh_1(\tilde{O}_{k+1},x_k,y_k)\right)(f_1(\tilde{O}_k,x_k,y_k))^\top]\\
        &+\E\left[\left( \fh_1(O_{k+1},x_k,y_k)\right)(f_1(O_k,x_k,y_k))^\top\right]-\E\left[\left( \fh_1(\tilde{O}_{k+1},x_k,y_k)\right)(f_1(\tilde{O}_k,x_k,y_k))^\top\right]\\
        =&\E[F^{(1,1)}(\tilde{O}_{k+1},\tilde{O}_k,x_k,y_k)]+\E[F^{(1,1)}(O_{k+1},O_k,x_k,y_k)]-\E[F^{(1,1)}(\tilde{O}_{k+1},\tilde{O}_k,x_k,y_k)]
    \end{align*}
     Using Part (2) on the first term and Part (1) on the second term of Lemma \ref{lem:F_conv}, we have
     \begin{align*}
         \E\left[\left(\E_{O_{k}} \fh_1(\cdot,x_k,y_k)\right)(f_1(O_k,x_k,y_k))^\top\right]=&\sum_{l=1}^\infty \E{[b_1(\tilde{O}_l)b_1(\tilde{O}_0)^\top]}+\hat{R}_k^{(1,1)}\tag{Lemma \ref{lem:F_conv}}
     \end{align*}
     where $\left\|\hat{R}_k^{(1,1)}\right\|\leq (\hat{g}_1+\hat{g}_3)d^2\sqrt{\alpha_k}+2\hat{g}_2d\hbar\sqrt{\zeta_k^x}$.

    \item For $T_{22}$, we have 
    \begin{align*}
        \|T_{22}\|&\leq \E[\|\E_{O_{k}}\fh_1(\cdot,x_k,y_k)\|\|B_{11}^k\|\|\ty_k\|]\\
        &\leq \frac{2\varrho_y}{1-\rho}\E\left[\left(b_{max}\sqrt{d}+A_{max}\left(\|x_k\|+\|y_k\|\right)\right)\|\ty_k\|\right]\tag{Lemma \ref{lem:f_hat_bound}}\\
        &\leq \frac{2\sqrt{3}\varrho_y}{1-\rho}\sqrt{\E\left[b_{max}^2d+A_{max}^2\left(\|x_k\|^2+\|y_k\|^2\right)\right]}\sqrt{\E[\|\ty_k\|^2]}\tag{Cauchy-Schwarz inequality}\\
        &\leq \frac{2\sqrt{3d}\varrho_y}{1-\rho}\sqrt{b_{max}^2+A_{max}^2\brc}\sqrt{\E[\|\ty_k\|^2]}\tag{Lemma \ref{lem:boundedness}}\\
        &\leq  \frac{2\sqrt{3d}\varrho_y}{1-\rho}\sqrt{b_{max}^2+A_{max}^2\brc}\sqrt{\beta_k \underbar{c}_2d^2+\hbar d\kappa_{Q_{\Delta, \beta}}\zeta_k^y}\tag{Lemma \ref{lem:go_from_xp_to_x}}\\ 
        &\leq \frac{2\sqrt{3d}\varrho_y}{1-\rho}\sqrt{b_{max}^2+A_{max}^2\brc}\left( \sqrt{\underbar{c}_2}d\sqrt{\beta_k}+\hbar \sqrt{\kappa_{Q_{\Delta, \beta}}d}\sqrt{\zeta_k^y}\right).
    \end{align*}

    \item For $T_{23}$, we have
    \begin{align*}
        T_{23}&\leq \E\left[\left\|\E_{O_{k}}\fh_1(\cdot,x_k,y_k)\right\|\left\|A_{12}\|\|\tx_k\right\|\right]\\
        &\leq \frac{2\|A_{12}\|}{1-\rho}\E\left[\left(b_{max}\sqrt{d}+A_{max}\left(\|x_k\|+\|y_k\|\right)\right)\|\tx_k\|\right]\tag{Lemma \ref{lem:f_hat_bound}}\\
        &\leq \frac{2\sqrt{3}\|A_{12}\|}{1-\rho}\sqrt{\E\left[b_{max}^2d+A_{max}^2\left(\|x_k\|^2+\|y_k\|^2\right)\right]}\sqrt{\E[\|\tx_k\|^2]}\tag{Cauchy-Schwarz inequality}\\
        &\leq \frac{2\sqrt{3d}\|A_{12}\|}{1-\rho}\sqrt{b_{max}^2+A_{max}^2\brc}\sqrt{\E[\|\tx_k\|^2]}\tag{Lemma \ref{lem:boundedness}}\\
       &\leq  \frac{2\sqrt{3d}\|A_{12}\|}{1-\rho}\sqrt{b_{max}^2+A_{max}^2\brc}\sqrt{\alpha_k \underbar{c}_3d^2+\hbar d\underbar{c}_4\zeta_k^x}\tag{Lemma \ref{lem:go_from_xp_to_x}}\\ 
        &\leq \frac{2\sqrt{3d}\|A_{12}\|}{1-\rho}\sqrt{b_{max}^2+A_{max}^2\brc}\left( \sqrt{\underbar{c}_3}d\sqrt{\alpha_k}+\hbar \sqrt{\underbar{c}_4d}\sqrt{\zeta_k^x}\right).
    \end{align*}
\end{itemize}
Note that $\alpha_k\sqrt{\beta_k}\geq \sqrt{\frac{\beta}{\alpha}}\beta_k\sqrt{\alpha_k}$ and $\alpha_k\sqrt{\zeta_k^y}\geq \frac{\beta}{\alpha}\beta_k\sqrt{\zeta_k^x}$. Hence, 
\begin{align*}
    T_{2}=\beta_k\sum_{j=1}^\infty \E{[b_1(\tilde{O}_j)b_1(\tilde{O}_0)^\top]} + R_k^1,
\end{align*}
where $\|R_k^2\|\leq d^2\alpha_k\sqrt{\beta_k}\left(\sqrt{\frac{\beta}{\alpha}}\left(\hat{g}_1+\hat{g}_3\right)+\frac{2\sqrt{3}}{1-\rho}\sqrt{b_{max}^2+A_{max}^2\brc}\left(\frac{\beta}{\alpha}\varrho_y\sqrt{\underbar{c}_2}+\|A_{12}\|\sqrt{\frac{\beta}{\alpha}}\sqrt{\underbar{c}_3}\right)\right)
+d\hbar\alpha_k\sqrt{\zeta_k^y}\frac{\beta}{\alpha}\\
\times\left(2\hat{g}_2+\frac{2\sqrt{3d}}{1-\rho}\sqrt{b_{max}^2+A_{max}^2\brc}\left(\varrho_y\sqrt{\kappa_{Q_{\Delta, \beta}}}+\|A_{12}\|\sqrt{\underbar{c}_4}\right)\right)$.

Combining the bounds for $T_1$ and $T_2$, we get
\begin{align*}
    \E[v_k\ty_k^\top]&=d_{k}^{yv}-d_{k+1}^{yv}+\beta_k\sum_{l=1}^\infty \E{[b_i(\tilde{O}_l)b_j(\tilde{O}_0)^\top]}+G_{k}^{(1,1)}
\end{align*}
where $\left\|G_k^{(1,1)}\right\|\leq g_1d^2\alpha_k\sqrt{\beta_k}+g_2d\hbar\alpha_k\sqrt{\zeta_k^y}$. Here 
\begin{align*}
    g_1&=d^2\left(\sqrt{b_{max}^2+A_{max}^2\brc}\left(\check{h}_2\left(1+\frac{\beta}{\alpha}\right)\sqrt{3\underbar{c}_7}+\frac{2\sqrt{3}}{1-\rho}\left(\frac{\beta}{\alpha}\varrho_y\sqrt{\underbar{c}_2}+\|A_{12}\|\sqrt{\frac{\beta}{\alpha}}\sqrt{\underbar{c}_3}\right)\right)+\sqrt{\frac{\beta}{\alpha}}(\hat{g}_1+\hat{g}_3)\right),\\
    g_2&=d\left(\sqrt{b_{max}^2+A_{max}^2\brc}\left(\check{h}_2\left(1+\frac{\beta}{\alpha}\right)\sqrt{3\underbar{c}_8}+\frac{\beta}{\alpha}\frac{2\sqrt{3}}{1-\rho}\left(\varrho_y\sqrt{\kappa_{Q_{\Delta, \beta}}}+\|A_{12}\|\sqrt{\underbar{c}_4}\right)\right)+\frac{2\beta}{\alpha}\hat{g}_2\right).
\end{align*}

\item \label{part:b}
     By definition, we had $v_k=f_1(O_k,x_k,y_k)$. By Remark \ref{rem:pois_eq_geo_mix}, we have a unique function $\fh_1(o,x_k,y_k)$ such that
\begin{align*}
    \fh_1(o,x_k,y_k)=f_1(o,x_k,y_k)+\sum_{o'\in S}P(o'|o)\fh_1(o',x_k,y_k)
\end{align*}
where $P(o'|o)$ is the transition probability corresponding to the Markov chain $\{O_k\}_{k\geq 0}$. 
    Hence,
\begin{align}
    \E[v_k\tx_k^\top]=&\E[f_1(O_k,x_k,y_k)\tx_k^\top]\label{eq:vx}\\
    =&\E\left[\left(\fh_1(O_k,x_k,y_k)-\sum_{o'\in S}P(o'|O_k)\fh_1(o',x_k,y_k)\right)\tx_k^\top\right]\nonumber\\
    =&\E\left[\left(\fh_1(O_k,x_k,y_k)-\E_{O_k}\fh_1(\cdot,x_k,y_k)\right)\tx_k^\top\right]\nonumber\\
    =&\E\left[\left(\fh_1(O_k,x_k,y_k)-\E_{O_{k-1}}\fh_1(\cdot,x_k,y_k)+\E_{O_{k-1}}\fh_1(\cdot,x_k,y_k)-\E_{O_k}\fh_1(\cdot,x_k,y_k)\right)\tx_k^\top\right]\nonumber\\
    =&\E\left[\left(\E_{O_{k-1}}\fh_1(\cdot,x_k,y_k)-\E_{O_{k}}\fh_1(\cdot,x_k,y_k)\right)\tx_k^\top\right]\tag{Tower property}\nonumber\\
    =&\E\bigg[\left(\E_{O_{k-1}}\fh_1(\cdot,x_k,y_k)\right)\tx_k^\top-\left(\E_{O_{k}}\fh_1(\cdot,x_{k+1},y_{k+1})\right)\tx_{k+1}^\top\nonumber\\
    &+\left(\E_{O_{k}}\fh_1(\cdot,x_{k+1},y_{k+1})-\E_{O_{k}}\fh_1(\cdot,x_k,y_k)\right)\tx_{k+1}^\top+\left(\E_{O_{k}}\fh_1(\cdot,x_k,y_k)\right)(\tx_{k+1}^\top-\tx_k^\top)\bigg]\nonumber\\
    =&d_{k}^{xv}-d_{k+1}^{xv}\\
    &+\E\bigg[\underbrace{\left(\E_{O_{k}}\fh_1(\cdot,x_{k+1},y_{k+1})-\E_{O_{k}}\fh_1(\cdot,x_k,y_k)\right)\tx_{k+1}^\top}_{T_{3}}+\underbrace{\left(\E_{O_{k}}\fh_1(\cdot,x_k,y_k)\right)(\tx_{k+1}^\top-\tx_k^\top)}_{T_{4}}\bigg]\nonumber
\end{align}

For $T_{3}$, we have
\begin{align*}
    \E[\|T_{3}\|]\leq & \check{h}_2 \E[(\|x_{k+1}-x_k\|+\|y_{k+1}-y_k\|).\|\tx_{k+1}\|]\tag{Lemma \ref{lem:f_hat_bound}}\\
    \leq & \check{h}_2\left(1+\frac{\beta}{\alpha}\right)  \alpha_k \E[(A_{max}\|x_{k}\|+A_{max}\|y_k\|+b_{max}\sqrt{d}).\|\tx_{k+1}\|]\tag{Eq. \eqref{eq:two_time_scale}}\\
    \leq & \check{h}_2\left(1+\frac{\beta}{\alpha}\right)  \alpha_k\sqrt{\E[( A_{max}\|x_{k}\| +A_{max}\|y_k\|+ b_{max}\sqrt{d})^2]} \sqrt{\E[\|\tx_{k+1}\|^2]}\tag{Cauchy-Schwarz}\\
    \leq & \check{h}_2\left(1+\frac{\beta}{\alpha}\right)\sqrt{3}\alpha_k\sqrt{\E[A_{max}^2(\|x_{k}\|^2+\|y_k\|^2)+b_{max}^2d]}\left(\sqrt{\underbar{c}_5}d \sqrt{\beta_k} + \sqrt{\underbar{c}_6d\hbar \zeta_k^y}\right)\tag{Lemma \ref{lem:go_from_xp_to_x}}\\
    \leq & \check{h}_2\left(1+\frac{\beta}{\alpha}\right)\sqrt{3}\alpha_k\sqrt{b_{max}^2+A_{max}^2\brc}\left(\sqrt{\underbar{c}_5}d^{1.5} \sqrt{\beta_k} + d\hbar\sqrt{\underbar{c}_6 \zeta_k^y}\right),\tag{Lemma \ref{lem:boundedness}}
\end{align*}
In addition, using the Eq. \eqref{eq:x_t_update_2}, we have
\begin{align*}
\E[T_{4}]=&\E\left[\left(\E_{O_{k}}\fh_1(\cdot,x_k,y_k)\right)(\tx_{k+1}^\top-\tx_k^\top)\right]\\
=&\E\left[\left(\E_{O_{k}}\fh_1(\cdot,x_k,y_k)\right)\left(-\alpha_k(B_{22}^k\tx_k)+\alpha_k w_k+\beta_k(L_{k+1}+A_{22}^{-1}A_{21})v_k\right)^\top\right]\\
=&\alpha_k\underbrace{\E\left[\left(\E_{O_{k}}\fh_1(\cdot,x_k,y_k)\right)w_k^\top\right]}_{T_{41}} \\
&-\alpha_k\underbrace{ \E\left[\left(\E_{O_{k}}\fh_1(\cdot,x_k,y_k)\right)\left(B_{22}^k\tx_k\right)^\top\right]}_{T_{42}} +\beta_k\underbrace{\E\left[\left(\E_{O_{k}}\fh_1(\cdot,x_k,y_k)\right)\left((L_{k+1}+A_{22}^{-1}A_{21})v_k\right)^\top\right]}_{T_{43}}
\end{align*}
\begin{itemize}
    \item For $T_{41}$, denote $\Tilde{O}$ as the random variable with distribution coming from the stationary distribution of the Markov chain $\{O_k\}_{k\geq 0}$. We have 
    \begin{align*}
        \E&\left[\left(\E_{O_{k}} \fh_1(\cdot,x_k,y_k)\right)(f_2(O_k,x_k,y_k))^\top\right]
        =\E\left[\left( \fh_1(O_{k+1},x_k,y_k)\right)(f_2(O_k,x_k,y_k))^\top\right]\tag{Tower property}\\
        =& \E[\left( \fh_1(\tilde{O}_{k+1},x_k,y_k)\right)(f_2(\tilde{O}_k,x_k,y_k))^\top]\\
        &+\E\left[\left( \fh_1(O_{k+1},x_k,y_k)\right)(f_2(O_k,x_k,y_k))^\top\right]-\E\left[\left( \fh_1(\tilde{O}_{k+1},x_k,y_k)\right)(f_2(\tilde{O}_k,x_k,y_k))^\top\right]\\
        =&\E[F^{(1,2)}(\tilde{O}_{k+1},\tilde{O}_k,x_k,y_k)]+\E[F^{(1,2)}(O_{k+1},O_k,x_k,y_k)]-\E[F^{(1,2)}(\tilde{O}_{k+1},\tilde{O}_k,x_k,y_k)]
    \end{align*}
    
    Using Part (2) on the first term and Part (1) on the second term of Lemma \ref{lem:F_conv}, we have
     \begin{align*}
         \E\left[\left(\E_{O_{k}} \fh_1(\cdot,x_k,y_k)\right)(f_2(O_k,x_k,y_k))^\top\right]=&\sum_{l=1}^\infty \E{[b_1(\tilde{O}_l)b_2(\tilde{O}_0)^\top]}+\hat{R}_k^{(1,2)}\tag{Lemma \ref{lem:F_conv}}
     \end{align*}
     where $\left\|\hat{R}_k^{(1,2)}\right\|\leq (\hat{g}_1+\hat{g}_3)d^2\sqrt{\alpha_k}+2\hat{g}_2d\hbar\sqrt{\zeta_k^x}$. 

    \item For $T_{42}$, we have 
    \begin{align*}
        \|T_{42}\|&\leq \E[\|\E_{O_{k}}\fh_1(\cdot,x_k,y_k)\|\|B_{22}^k\|\|\tx_k\|]\\
        &\leq \frac{2\left(\frac{\beta}{\alpha}\varrho_x+\|A_{22}\|\right)}{1-\rho}\E\left[\left(b_{max}\sqrt{d}+A_{max}\left(\|x_k\|+\|y_k\|\right)\right)\|\tx_k\|\right]\tag{Lemma \ref{lem:f_hat_bound}}\\
        &\leq \frac{2\sqrt{3}\left(\frac{\beta}{\alpha}\varrho_x+\|A_{22}\|\right)}{1-\rho}\sqrt{\E\left[b_{max}^2d+A_{max}^2\left(\|x_k\|^2+\|y_k\|^2\right)\right]}\sqrt{\E[\|\tx_k\|^2]}\tag{Cauchy-Schwarz inequality}\\
        &\leq \frac{2\sqrt{3d}\left(\frac{\beta}{\alpha}\varrho_x+\|A_{22}\|\right)}{1-\rho}\sqrt{b_{max}^2+A_{max}^2\brc}\sqrt{\E[\|\tx_k\|^2]}\tag{Lemma \ref{lem:boundedness}}\\
        &\leq  \frac{2\sqrt{3d}\left(\frac{\beta}{\alpha}\varrho_x+\|A_{22}\|\right)}{1-\rho}\sqrt{b_{max}^2+A_{max}^2\brc}\sqrt{\alpha_k \underbar{c}_3d^2+\hbar d\underbar{c}_4\zeta_k^x}\tag{Lemma \ref{lem:go_from_xp_to_x}}\\ 
        &\leq \frac{2\sqrt{3}\left(\frac{\beta}{\alpha}\varrho_x+\|A_{22}\|\right)}{1-\rho}\sqrt{b_{max}^2+A_{max}^2\brc}\left( \sqrt{\underbar{c}_3}d^{1.5}\sqrt{\alpha_k}+d\hbar \sqrt{\underbar{c}_4}\sqrt{\zeta_k^x}\right).
    \end{align*}
    \item For $T_{43}$, we have
    \begin{align*}
        T_{43}&\leq \E\left[\left\|\E_{O_{k}}\fh_1(\cdot,x_k,y_k)\right\|\left\|(L_{k+1}+A_{22}^{-1}A_{21})\|\|v_k\right\|\right]\\
        &\leq \frac{2\varrho_x}{1-\rho}\E\left[(b_{max}\sqrt{d}+A_{max}(\|x_k\|+\|y_k\|))\left\|v_k\right\|\right]\tag{Lemma \ref{lem:possion_sol} and Lemma \ref{lem:L_k_bound}}\\
        &\leq \frac{\varrho_x}{1-\rho}\E\left[(b_{max}\sqrt{d}+A_{max}(\|x_k\|+\|y_k\|))^2+\left\|v_k\right\|^2\right]\tag{AM-GM inequality}\\
        &\leq \frac{3d\varrho_x}{1-\rho}\left(2b_{max}^2+5A_{max}^2\brc\right)\tag{Lemma \ref{lem:boundedness} and Lemma \ref{lem:noise_crude_bound}}
    \end{align*}
\end{itemize}

Hence, 
\begin{align*}
    T_4=\alpha_k\sum_{j=1}^\infty \E{[b_1(\tilde{O}_j)b_2(\tilde{O}_0)^\top]} + R_k^2,
\end{align*}
where $\|R_k^2\|\leq \left(\hat{g}_1+\hat{g}_3+\frac{2\sqrt{3}\left(\frac{\beta}{\alpha}\varrho_x+\|A_{22}\|\right)}{1-\rho}\sqrt{b_{max}^2+A_{max}^2\brc}\sqrt{\underbar{c}_3}\right)d^2\alpha_k^{1.5}+\beta_k\frac{3d\varrho_x}{1-\rho}\left(2b_{max}^2+5A_{max}^2\brc\right)+\left(2\hat{g}_2+\frac{2\sqrt{3}\left(\frac{\beta}{\alpha}\varrho_x+\|A_{22}\|\right)}{1-\rho}\sqrt{b_{max}^2+A_{max}^2\brc}\sqrt{\underbar{c}_4}\right)d\hbar\alpha_k\sqrt{\zeta_k^x}$.

Combining the bounds for $T_3$ and $T_4$ and using $\zeta_k^y\leq \zeta_k^x$, we get,
\begin{align*}
    \E\left[\left(\E_{O_{k}} \fh_1(\cdot,x_k,y_k)\right)(f_2(O_k,x_k,y_k))^\top\right]=d_{k}^{xv}-d_{k+1}^{xv}+\alpha_k\sum_{j=1}^\infty \E{[b_1(\tilde{O}_j)b_2(\tilde{O}_0)^\top]}+G^{(1,2)}_k
\end{align*}
where $\left\|G^{(1,2)}_k\right\|\leq g_3d^2(\alpha_k^{1.5}+\beta_k)+g_4d\hbar\alpha_k\sqrt{\zeta_k^x}$. Here 
\begin{align*}
    g_3&=\max\Bigg\{\sqrt{b_{max}^2+A_{max}^2\brc}\left(\check{h}_2\left(1+\frac{\beta}{\alpha}\right)\sqrt{\frac{3\alpha\underbar{c}_5}{\beta}}+\frac{2\sqrt{3}\left(\frac{\beta}{\alpha}\varrho_x+\|A_{22}\|\right)}{1-\rho}\sqrt{\underbar{c}_3}\right)+\hat{g}_1+\hat{g}_3,\\
    &~~~~~~~~\quad\quad\frac{3\varrho_x}{1-\rho}\left(2b_{max}^2+5A_{max}^2\brc\right) \Bigg\},\\
    g_4&=\sqrt{b_{max}^2+A_{max}^2\brc}\left(\check{h}_2\left(1+\frac{\beta}{\alpha}\right)\sqrt{3\underbar{c}_5}+\frac{2\sqrt{3}\left(\frac{\beta}{\alpha}\varrho_x+\|A_{22}\|\right)}{1-\rho}\sqrt{\underbar{c}_4}\right)+2\hat{g}_2.
\end{align*}

\end{enumerate}
\end{proof}

\subsection{Additional Lemmas}\label{sec:add_lem}

\begin{lemma}\cite[Proposition 21.2.3]{douc2018markov}\label{lem:possion_sol}
    Consider a finite state space Markov chain with the set of state space as $S$ and let $\mu(\cdot)$ denote the stationary distribution. For any $o\in S$ and arbitrary $x$ and $y$ define $f(o, x, y)=b(o)-(A_1(o))x-(A_2(o))y$ such that $\sum_{o\in S}\mu(o)f(o)=0$. Then one of the solutions for Poisson equation is given by:
    \begin{align*}
        \hat{f}(o, x, y)&=\sum_{k=0}^\infty \E{\left[f(O_k, x, y)|O_0=o\right]}\\
        &=\sum_{k=0}^\infty \E{[b(O_k)|O_o=o]}-\left(\sum_{k=0}^\infty \E{[A_1(O_k))|O_0=o]}\right)x-\left(\sum_{k=0}^\infty \E{[A_2(O_k))|O_0=o]}\right)y,
    \end{align*}
    where each infinite summation is finite for all $o\in S$.
\end{lemma}

\begin{lemma}\label{lem:mix_time_sum}
Consider an Ergodic Markov chain $\{O_k\}_{k\geq 0}$ with the transition probability $P(\cdot|\cdot)$ and the stationary distribution $\mu$ and let $\rho$ be the mixing rate of this Markov chain. Consider the functions $h_1, h_2, h_3: \mathcal{S}\rightarrow \mathbb{R}^{d_1\times d_2}$ for arbitrary integers $d_1$ and $d_2$. For all $o\in \mathcal{S}$, we have
\begin{align*}
    \left\|\sum_{k=0}^\infty \E\left[h_1(O_k)-h_1(\tO_k)\Big|O_0=o\right]\right\|\leq \frac{2}{1-\rho}\max_{o\in \mathcal{S}}\left\|h_1(o)\right\|,
\end{align*}
where $\{\tO_k\}_{k\geq 0}$ is an independent stationary Markov chain. 

Furthermore, if $\E[h_2(\tO_k)]=0,~\forall~k\geq 0$, we have
\begin{align*}
    \left\|\sum_{k=0}^\infty \E\left[h_2(\tO_k) h_3(\tO_0)^\top\right]\right\| \leq \frac{1}{1-\rho}\max_{o\in \mathcal{S}} \left\| h_2(o) \right\| \max_{o\in \mathcal{S}} \left\| h_3(o) \right\|. 
\end{align*}
    
\end{lemma}
\begin{proof}[Proof of Lemma \ref{lem:mix_time_sum}]
    An Ergodic Markov chain enjoys an exponential mixing property \cite{levin2017markov}, that is, for all $o\in\mathcal{S} $, we have $d_{TV}(P^k(\cdot|o)||\mu(\cdot))\leq\rho^k$ for some $\rho\in[0,1)$.
    \begin{align*}
    \left\|\sum_{k=0}^\infty \E\left[h_1(O_k)-h_1(\tO_k)\Big|O_0=o\right]\right\|\leq & \sum_{k=0}^\infty \left\|\E\left[[h_1(O_k)-h_1(\tO_k)]\Big|O_0=o\right]\right\|\\
    &= \sum_{k=0}^\infty \left\|\sum_{o'\in\mathcal{S}} (P^k(o'|o)-\mu(o'))h_1(o')\right\|\\
    &\leq  \sum_{k=0}^\infty \sum_{o'\in\mathcal{S}} \left|P^k(o'|o) -\mu(o') \right|\left\|h_1(o')\right\|\\
    &\leq  \sum_{k=0}^\infty \max_{o''\in \mathcal{S}}\left\|h_1(o'')\right\| \sum_{o'\in \mathcal{S}} \left|P^k(o'|o) - \mu(o') \right|\\
    &\leq 2\max_{o''\in \mathcal{S}}\left\|h_1(o'')\right\|  \sum_{k=0}^\infty  d_{TV}(P^k(\cdot|o)||\mu(\cdot)).
    \end{align*}

In addition, we have
    
    \begin{align*}
    \sum_{k=0}^\infty d_{TV}(P^k(\cdot|o)| \mu(\cdot)) \leq & \sum_{k=0}^\infty\rho^k \\ 
    =&\frac{1}{1-\rho}
\end{align*}
The first result follows by combining the inequalities.

For the second part, we have
\begin{align*}
    \left\|\sum_{k=0}^\infty \E \left[h_2(\tO_k) h_3(\tO_0)^\top \right] \right\|\leq & \sum_{k=0}^\infty \left\|\E \left[h_2(\tO_k) h_3(\tO_0)^\top \right]\right\|\\
    \leq &\sum_{k=0}^\infty \max_{o} \left\|\E \left[h_2(\tO_k)|\tO_0=o\right]\right\|\left\| h_3(o) \right\|\\
    \leq & \sum_{k=0}^\infty \max_{o} \left\|\sum_{o'\in \mathcal{S}}P^k(o'|o) h_2(o')\right\|\left\| h_3(o) \right\|\\
    =&\sum_{k=0}^\infty \max_{o} \left\|\sum_{o'\in \mathcal{S}}(P^k(o'|o)-\mu(o')) h_2(o')\right\|\left\| h_3(o) \right\|\\
    \leq&\sum_{k=0}^\infty \max_{o} \sum_{o'\in \mathcal{S}}|P^k(o'|o)-\mu(o')| \left\|h_2(o')\right\|\left\| h_3(o) \right\|\\
    \leq&\sum_{k=0}^\infty \max_{o} \left\| h_2(o) \right\|\max_{o} \left\| h_3(o) \right\| \max_{o'}d_{TV}(P^k(\cdot|o')||\mu(\cdot)) \\
    \leq & \frac{1}{1-\rho}\max_{o}\left\| h_2(o) \right\|\max_{o} \left\| h_3(o) \right\| 
    \end{align*}
\end{proof}

\begin{lemma}\label{lem:Lyap_bound}
    Consider a Hurwitz matrix $A$, a symmetric $\Sigma$, and the solution $P$ to the Lyapunov equation $AP+PA^\top=\Sigma$. We have
    \begin{align*}
        \|P\|\leq \|\Sigma\| \|U\|\|U^{-1}\| \sum_{n,n'=0}^m  {n+n'\choose n} \frac{1}{(-2r)^{n+n'+1}}, 
    \end{align*}
    where $U$ is the generalized eigen vector of $A$, and $m$ is the largest algebraic multiplicity of the matrix $A$ and $r=\max_i\mathfrak{Re}[\lambda_i]$, where $\lambda_i$ is the $i$-th eigen value.
\end{lemma}

\begin{proof}[Proof of Lemma \ref{lem:Lyap_bound}]
    We know that the solution of the Lyapunov function $AP+PA^\top=\Sigma$ can be written as $P=\int_0^\infty e^{A\tau}\Sigma e^{A^\top\tau}d\tau$. Hence, 
    \begin{align*}
        \|P\|= &\left\|\int_0^\infty e^{A\tau}\Sigma e^{A^\top\tau}d\tau \right\|\\
        \leq & \|\Sigma\|\int_0^\infty \|e^{A\tau}\|^2 d\tau.
    \end{align*}
    Consider the Jordan canonical form of $A$ as $A=UJU^{-1}$. Then we have $e^{A\tau}=Ue^{J\tau}U^{-1}$, and hence $\|e^{A\tau}\|
    \leq\|U\|\|U^{-1}\|\|e^{J\tau}\|$. But we know that $\|e^{J\tau}\| \leq \max_i e^{r_i\tau}\sum_{n=0}^{m_i}\tau^n/n! \leq \max_i e^{r_i\tau} \max_i \sum_{n=0}^{m_i} \tau^n/n! = e^{r\tau} \sum_{n=0}^{m}\tau^n/n!$. Here $r_i = \mathfrak{Re}[\lambda_i]$, where $\lambda_i$ is the $i$-th eigen value  and $m_i$ is its algebraic multiplicity. In addition, $r=\max_i r_i<0$ and $m=\max_i m_i$. Hence, we have
    \begin{align*}
        \int_0^\infty \|e^{A\tau}\|^2 d\tau \leq &\int_0^\infty e^{2r\tau} \left[\sum_
        {n=0}^{m}\tau^n/n!\right]^2d\tau\\
        \leq & \sum_{n,n'=0}^m \int_0^\infty e^{2r\tau} \tau^{n+n'}/(n!n'!)d\tau\\
        = & \sum_{n,n'=0}^m \frac{1}{-2r}\int_0^\infty e^{-z} (-z/2r)^{n+n'}/(n!n'!) dz\\
        = & \sum_{n,n'=0}^m \frac{1}{(-2r)^{n+n'+1}\times n!n'!} \int_0^\infty e^{-z} z^{n+n'+1-1} dz\\
        = & \sum_{n,n'=0}^m \frac{(n+n')!}{(-2r)^{n+n'+1}\times n!n'!} =  \sum_{n,n'=0}^m {n+n'\choose n}\frac{1}{(-2r)^{n+n'+1}}
    \end{align*}

    \end{proof}

\begin{lemma}\label{lem:L_k_abs_bound} Consider the recursion
\begin{align*}
    L'(I-b B_{11})&=(I-a A_{22})L+b A_{22}^{-1}A_{21}B_{11}.
\end{align*}
where $a$ and $b$ are some arbitrary constants, $B_{11}=\Delta-A_{12}L$, and $A_{22}$ is a Hurwitz matrix. Assume that the constants $a$ and $b$ satisfy $$\frac{b}{a}\leq \frac{a_{22}/2}{(\|A_{22}^{-1}A_{21}\|_{Q_{22}}+1)(\|\Delta\|_{Q_{22}}+\|A_{12}\|_{Q_{22}})}$$
    $$b\leq \frac{1}{2(\|\Delta\|_{Q_{22}}+\|A_{12}\|_{Q_{22}})\kappa_{Q_{22}}};~~a\leq \frac{1}{2\|Q_{22}\|\|A_{22}\|^2_{Q_{22}}}.$$ 
    If $\|L\|_{Q_{22}}\leq 1$, then
    \begin{align*}
        L'&=(I-a A_{22})L+b D(L).
    \end{align*}
    where $D(L)=(A_{22}^{-1}A_{21}+(I-a A_{22})L)B_{11}(I-b B_{11})^{-1}$. Furthermore, $\|L'\|_{Q_{22}}\leq 1$ and $\|D(L)\|_{Q_{22}}\leq c_D=2(\|A_{22}^{-1}A_{21}\|_{Q_{22}}+1)(\|\Delta\|_{Q_{22}}+\|A_{12}\|_{Q_{22}})$.
\end{lemma}
\begin{proof}[Proof of Lemma \ref{lem:L_k_abs_bound}]
By definition, we have $\|B_{11}\|_{Q_{22}}=\|\Delta-A_{12}L\|_{Q_{22}}\leq \|\Delta\|_{Q_{22}}+\|A_{12}\|_{Q_{22}}$. Thus, by the assumption $b$, we have $\sqrt{\frac{\gamma_{\max}(Q_{22})}{\gamma_{\min}(Q_{22})}}b\|B_{11}\|_{Q_{22}}\leq \frac{1}{2}$ which implies $b\|B_{11}\|_2\leq \frac{1}{2}$. Thus, $I-b B_{11}$ is invertible and we have,
\begin{align*}
    L'&=((I-a A_{22})L+b A_{22}^{-1}A_{21}B_{11})(I-b B_{11})^{-1}\\
    &=(I-a A_{22})L+b D(L)
\end{align*}
where $D(L)=(A_{22}^{-1}A_{21}+(I-a A_{22})L)B_{11}(I-b B_{11})^{-1}$. Recall due to the assumption on $b$, $\|I-b B_{11}\|_{Q_{22}}\geq 1/2$, which implies that $\|D(L)\|_{Q_{22}}\leq 2(\|A_{22}^{-1}A_{21}\|_{Q_{22}}+1)\|B_{11}\|_{Q_{22}}$. Thus, we have,
\begin{align*}
    \|L'\|_{Q_{22}}&\leq (1-a_{22}a)\|L\|_{Q_{22}}+b\|D(L)\|_{Q_{22}}\tag{Lemma \ref{lem:contraction_prop}}\\
    &\leq (1-a_{22}a)\|L\|_{Q_{22}}+a_{22}a\left(\frac{b}{a_{22}a}\|D(L)\|_{Q_{22}}\right)\\
    &\leq (1-a_{22}a)+a_{22}a\left(\frac{2b}{a_{22}a}(\|A_{22}^{-1}A_{21}\|_{Q_{22}}+1)(\|\Delta\|_{Q_{22}}+\|A_{12}\|_{Q_{22}})\right)\\
    &\leq 1.
\end{align*}
\end{proof}

\begin{lemma}\label{lem:step_size_gap} 
For any $\xi \geq 0$, and for all $n\geq 1$, we have \begin{align*}
    \frac{1}{n^{\xi}}-\frac{1}{(n+1)^{\xi}}\leq \frac{\xi}{n^{\xi+1}}.
    \end{align*}
\end{lemma}

\begin{proof}[Proof of Lemma \ref{lem:step_size_gap}]
Define the function $f(x)=\frac{1}{(x+n)^\xi}$. By Taylor's theorem, for $x\in [0,1]$, and for some $z\in [0,x]$, we have
\begin{align*}
    f(x)=f(0)+f'(z)x = \frac{1}{n^{\xi}}-\frac{x\xi}{(n+z)^{\xi+1}}.
\end{align*}
Hence, by choosing $x=1$,
\begin{align*}
    \frac{1}{n^\xi}-\frac{1}{(n+1)^\xi} = \frac{\xi}{(n+z)^{\xi+1}}\leq \frac{\xi}{n^{\xi+1}}
\end{align*}
    
\end{proof}

\begin{lemma}\label{lem:lambert}
    For any $\xi\in (0,1)$, $\rho<1$, and $n\geq 1$, we have
    \begin{align*}
        \rho^x(x+n)^{\xi}\leq \left(\frac{\xi}{e\ln(1/\rho)}+n\right)^\xi ~\forall x\geq 0.
    \end{align*}
\end{lemma}

\begin{proof}[Proof of Lemma \ref{lem:lambert}]
    \begin{align*}
        \rho^x(x+n)^\xi&=(\rho^{\frac{x}{\xi}}x+\rho^{\frac{x}{\xi}}n)^\xi.
    \end{align*}
    Since $x\geq 0$ and $\rho<1$, we can bound the second term by $n$. For the first term, we have
    \begin{align*}
        \rho^{\frac{x}{\xi}}x=e^{\frac{x}{\xi}\ln(\rho)}x.
    \end{align*}
    The maximum value of this function is $\frac{\xi}{e\ln(1/\rho)}$ which is achieved at $x=\frac{\xi}{\ln(1/\rho)}$. Combining the above two bounds, we get
    \begin{align*}
        \rho^x(x+n)^\xi\leq \left(\frac{\xi}{e\ln(1/\rho)}+n\right)^\xi ~\forall x\geq 0.
    \end{align*}
\end{proof}

\begin{lemma}\label{lem:norm_bound}
For any symmetric matrix $X\in \mathbb{R}^{d\times d}$, we have
\begin{align*}
    trace(X)\leq d\|X\|.
\end{align*}
\end{lemma}

\begin{proof}[Proof of Lemma \ref{lem:norm_bound}]
    By eigenvalue decomposion of $X$, we have $X=\Lambda \Sigma \Lambda^\top$. Taking the trace of $X$, we have $trace(X) = trace(\Lambda \Sigma \Lambda^\top) = trace(\Sigma \Lambda\Lambda^\top) = trace(\Sigma) = \sum_i \sigma_i\leq d \sigma_{\max} = d\|X\|$. 
\end{proof}

\begin{lemma}\label{lem:contraction_prop}
Suppose $-A$ is a Hurwitz matrix. Define $Q$ to be the solution to Lyapunov equation,
$$A^\top Q+QA=I$$
Then for all $\epsilon\in [0, \frac{1}{2\|Q\|\|A\|_Q^2}]$ 
$$\|I-\epsilon A\|_{Q}^2\leq (1-a\epsilon),~~~\text{where } a=\frac{1}{2\|Q\|}.$$

\end{lemma}

\begin{proof}[Proof of Lemma \ref{lem:contraction_prop}]
    Using the definition of matrix norm we have:
    \begin{align*}
        \|I-\epsilon A\|_Q^2&=\max_{\|x\|_{Q}= 1}x^\top(I-\epsilon A)^\top Q(I-\epsilon A)x\\
        &=\max_{\|x\|_{Q}= 1} \left(x^\top Qx-\epsilon x^\top(A^\top Q+QA)x+\epsilon^2x^\top A^\top QAx\right)\\
        &\leq 1-\epsilon\min_{\|x\|_Q=1}\|x\|^2+\epsilon^2 \max_{\|x\|_Q=1}\|Ax\|_Q^2\\
        &\leq 1-\epsilon\frac{1}{\|Q\|}+\epsilon^2\|A\|_Q^2.
    \end{align*}
    For any $\epsilon\in\left[0, \frac{1}{2\|Q\|\|A\|_Q^2}\right]$, we have:
    \begin{align*}
        \|I-\epsilon A\|_Q^2\leq 1-\frac{\epsilon}{2\|Q\|}.
    \end{align*}
\end{proof}


\section{Dimension dependence of the convergence result of \cite{kaledin2020finite}}\label{sec:kaledin_d_dependency}
In this section, we will list the dimensional scaling of various constants in \cite{kaledin2020finite} in a sequential manner which will enable us to find the dimensional dependence of their final result. Note that we compare their dependence under the same set of assumptions as ours. Specifically, we assume that the $\ell_2$-norm of the vectors in $\mathbb{R}^d$ have $\mathcal{O}(\sqrt{d})$ dependence while the matrix $\ell_2$-norms do not scale with $d$.

All the references in the following are for \cite{kaledin2020finite}.
\begin{enumerate}
    \item Assumption B3: The constant $\bar{b}=\mathcal{O}(\sqrt{d})$. 
    \item Page 24: Due to the $d$-dependency of $\bar{b}$, both $m_V$ and $m_W$ are $\mathcal{O}(\sqrt{d})$.
    \item Page 24: Using Eq. (36), $\tilde{m}_V$ and $\tilde{m}_W$ are $\mathcal{O}(d)$ and $\tilde{m}_{VW}=\mathcal{O}(d^2)$.
    \item Eq. (62), Page 24: $\tilde{C}_0$ is $\mathcal{O}(d^2)$. 
    \item Eq. (64), Page 25: $\tilde{E}_0^{WV}=\mathcal{O}(\sqrt{d})$.
    \item Page 25: $\tilde{m}_{\Delta\tilde{\theta}}$ and $\tilde{m}_{\Delta\tilde{w}}$ are both $\mathcal{O}(d)$.
    \item Eq. 67, Page 26: We know that $\|\tilde{w}_0\|$ and $\|\tilde{\theta}_0\|$ are both $\mathcal{O}(\sqrt{d})$. Hence, $\tilde{C}_i=\mathcal{O}(d^2)$ for $i=1,2,3,4$.
    \item Eq. 67, Page 28: $\tilde{C}_i^{\tilde{w}}=\mathcal{O}(d^2)$ for $i=0,1.2,3$.
    \item Page 28: $\tilde{C}_i^{\tilde{w}'}=\mathcal{O}(d^4)$ for $i=1,2$.
    \item Page 29: $\tilde{C}_i^{\tilde{w}''}=\mathcal{O}(d^4)$ for $i=0,1,2,3$.
    \item Eq. (73), Page 30: $\tilde{C}_0^{\tilde{\theta},\tilde{w}}=\mathcal{O}(d^2)$ and $\tilde{C}_i^{\tilde{\theta},\tilde{w}}=\mathcal{O}(d^4)$ for $i=1,2$.
    \item Page 32: $\tilde{C}_0^{(0)}=\mathcal{O}(d^2)$ and $\tilde{C}_i^{(0)}=\mathcal{O}(d^4)$ for $i=1,2$.
    \item Page 33: $\tilde{C}_i^{(1,0)}=\mathcal{O}(d^4)$ for $i=0,1,2$. In addition, we have $E_0^V=\mathcal{O}(\sqrt{d})$.
    \item Page 35: $\hat{C}_i^{(1,1)}=\mathcal{O}(d)$ for $i=0,1,3$ and $\hat{C}_2^{(1,1)}=\mathcal{O}(d^2)$.
    \item Eq. (77), Page 36: $\tilde{C}_i^{(1,1)}=\mathcal{O}(d^3)$ for $i=0,3$ and $\tilde{C}_i^{(1,1)}=\mathcal{O}(d^4)$ for $i=1,2$.
    \item Eq. (78), Page 36: $\tilde{C}_i^{\tilde{\theta}}=\mathcal{O}(d^4)$ for $i=0,1,2$.
    \item Page 37: $C_1^{\tilde{\theta},mark}=\mathcal{O}(d^4)$. In addition, assuming $\tilde{C}_0^{\tilde{\theta}}=C_0^{\tilde{\theta},mark}(1+V_0)$, and noticing that $V_0=\mathcal{O}(d)$ (since it the sum of squared norm of vectors), we have $C_0^{\tilde{\theta},mark} = \mathcal{O}(d^3)$.
    \item Eq. (80), Page 37: $\tilde{C}_0^{\hat{w}}=\mathcal{O}(d^6)$. In addition, $C_1^{\hat{w},mark}=\mathcal{O}(d^6)$.
    \item Page 37: Assuming $\tilde{C}_0^{\hat{w}}=C_0^{\hat{w},mark}(1+V_0)$, and noticing that $V_0=\mathcal{O}(d)$, we have $C_0^{\hat{w},mark}(1+V_0) = \mathcal{O}(d^5)$.
\end{enumerate}
Finally, combining these bounds, we get $\E[\|\theta_k-\theta^*\|^2]=\mathcal{O}(d^5)$ and $\E[\|w_k-A_{22}^{-1}(b_2-A_{21}\theta_k)\|^2]=\mathcal{O}(d^7)$ in Eq. (14) and Eq. (15), respectively.

\section{Details for the simulation}\label{appendix:sim}
\subsection{Simulation details for Fig. \ref{fig:xi}}
For simulation,consider a $1$-d linear SA with $|S|=2$ for Markovian noise. The transition probability is given by:
\begin{align*}
    P=\begin{bmatrix}
5/8 & 3/8\\
3/4 & 1/4
\end{bmatrix},~\mu=[2/3, 1/3]
\end{align*}
The update matrices (in 1-d case scalars) were chosen as the following:
\begin{align*}
    A_{11}(1)&=-0.5;~~A_{11}(2)=-2;~~A_{11}=-1\\
    A_{12}(1)&=-1;~~A_{12}(2)=-1;~~A_{12}=-1\\
    A_{21}(1)&=2.5;~~A_{21}(2)=1;~~A_{21}=2\\
    A_{22}(1)&=0;~~A_{22}(2)=3;~~A_{22}=1\\
    b_1(1)&=-3/2;~~b_1(2)=3;~~b_1=0\\
    b_2(1)&=3;~~b_2(2)=-6;~~b_2=0
\end{align*}
For the step size, $\alpha=1$ and $\beta=1$. Observe that $\Delta=A_{11}-A_{12}A_{22}^{-1}A_{21}=1$ and therefore $-(\Delta-\beta^{-1}/2)$ is Hurwitz. We sample $x_0$ and $y_0$ uniformly from $[-5,5]$. The bold lines are the mean across five sample paths, whereas the shaded region is the standard deviation from the mean path. The plots start from $0.1$ instead of $0$. This is done intentionally so that the initial randomness dies down.

\subsection{Simulation details for Fig. \ref{fig:beta}}
Again, we consider a $1$-d linear SA with $|S|=2$ for the Markovian noise. The transition probability is same as before, i.e.:
\begin{align*}
    P=\begin{bmatrix}
5/8 & 3/8\\
3/4 & 1/4
\end{bmatrix},~\mu=[2/3, 1/3].
\end{align*}
The update matrix (scalar in 1-d case) is as follows:
\begin{align*}
    A_{11}(1)&=1;~~A_{11}(2)=1;~~A_{11}=1\\
    A_{12}(1)&=-1;~~A_{12}(2)=-1;~~A_{12}=-1\\
    A_{21}(1)&=0;~~A_{21}(2)=0;~~A_{12}=0\\
    A_{22}(1)&=0;~~A_{22}(2)=3;~~A_{22}=1\\
    b_1(1)&=0;~~b_1(2)=0;~~b_1=0\\
    b_2(1)&=3;~~b_2(2)=-6;~~b_2=0
\end{align*}
For the step size, $\alpha=1$ and $\xi=0.75$. Observe that $\Delta=1$ and therefore $-(\Delta-\beta^{-1}/2)$ is Hurwitz. We sample $x_0$ and $y_0$ uniformly from $[-5,5]$. The bold lines are the mean across five sample paths, whereas the shaded region is the standard deviation from the mean path. The plots start from $0.1$ instead of $0$. This is done intentionally so that the initial randomness dies down.

\subsection{Simulation details for Fig. \ref{fig:ssa_div}}
Again, we consider a $1$-d linear SA with $|S|=2$ for the Markovian noise. The transition probability is given by:
\begin{align*}
    P=\begin{bmatrix}
        1/4 & 3/4\\
        3/4 & 1/4
    \end{bmatrix},~\mu=[1/2, 1/2].
\end{align*}
The update matrix (scalar in 1-d case) is as follows:
\begin{align*}
    A_{11}(1)&=-3;~~A_{11}(2)=-5;~~A_{11}=-4\\
    A_{12}(1)&=2;~~A_{12}(2)=10;~~A_{12}=6\\
    A_{21}(1)&=3;~~A_{21}(2)=-5;~~A_{21}=-1\\
    A_{22}(1)&=1;~~A_{22}(2)=1;~~A_{22}=1\\
    b_1(1)&=-3000;~~b_1(2)=3000;~~b_1=0\\
    b_2(1)&=9000;~~b_2(2)=-9000;~~b_2=0
\end{align*}
For the step size, $\alpha=\beta=1$ and $\xi=1$. The block matrix $A$ is given by:
\begin{align*}
    A=\begin{bmatrix}
        -4 & 6\\
        -1 & 1
    \end{bmatrix}
\end{align*}
Observe that the matrix $-A$ has eigenvalues $1,2$ and therefore, it is not Hurwitz. The mean squarer errors shown in the plot are averages over five sample paths.

\section{Discussion on the best choice of step size}\label{sec:best_step_size}
Consider the linear SA \eqref{eq:linearSA}. In order to get a faster convergence suppose that we run the second time-scale $y_{k+1}=(1-\beta_k)y_k+\beta_k x_k$ where $\beta_k=\frac{\beta}{k}$. Notice that with the choice of $\beta=1$, we again derive the Polyak-Ruppert averaging iterate \eqref{eq:Polyak}. An interesting question to investigate is why the optimal choice of $\beta$ is equal to $1$. 

According to Theorem \ref{thm:Markovian_main}, the leading term in the convergence of $\E[y_ky_k^\top]$ is $\beta_k\Sigma^y$. Furthermore, by \eqref{eq:sigma_y_def_main} we have $\Sigma^y=(\Gamma^y+\Sigma^xA_{22}^{-T}+A_{22}^{-1}\Sigma^x)/(2-\beta^{-1})$. Hence, to find $\beta$ that minimizes the norm of $\Sigma^y$, we need to choose $\beta$ which minimizes $h(\beta)=\beta^2/(2\beta-1)$. The plot of the function $h(\beta)$ is shown in Figure \ref{fig:h(beta)}. Clearly, this function is minimized at $\beta=1$, and hence the Polyak-Ruppert averaging is optimal. 
\begin{figure}[!ht]
    \centering
    \includegraphics{Arxiv/images/beta_func.tex}
    \caption{The function $h(\beta)=\frac{\beta^2}{2\beta-1}$}
    \label{fig:h(beta)}
\end{figure}

\end{document}